\documentclass[12pt]{article} 
\usepackage{smile}
\usepackage{graphicx}
\usepackage{booktabs} 
\usepackage{fullpage}
\usepackage{natbib}
\usepackage{tikz}
\usepackage{listings}
\usepackage{thm-restate}
\usepackage{wrapfig}
\usepackage{lineno}
\usepackage{caption}
\usepackage{acronym}
\acrodef{bma}[BMA]{Bayesian Model Averaging}
\acrodef{icl}[ICL]{In-Context Learning}
\acrodef{cot}[CoT]{Chain-of-Thought}
\acrodef{io}[I/O]{Input/Output}
\acrodef{gpt}[GPT]{Generative Pre-trained Transformer}
\acrodef{tot}[ToT]{Tree-of-thoughts}
\acrodef{dfs}[DFS]{Depth-first-search}
\acrodef{bfs}[BFS]{Breadth-first-search}
\acrodef{sc}[SC]{Self-Consistency}
\acrodef{si}[SI]{Selection-Inference}
\acrodef{mle}[MLE]{Maximum Likelihood Estimate}
\acrodef{ff}[FF]{Feed-Forward}
\acrodef{mha}[MHA]{Multi-Head Attention}
\acrodef{mae}[MAE]{Masked AutoEncoders}
\acrodef{nlp}[NLP]{Natural Language Processing}
\acrodef{cv}[CV]{Computer Vision}
\acrodef{llm}[LLM]{Large Language Models}
\acrodef{hmm}[HMM]{Hidden Markov Model}
\acrodef{rkhs}[RKHS]{Reproducing Kernel Hilbert Space}
\newcommand*{\E}{\mathbb{E}}

\def\[#1\]{\begin{bmatrix}#1\end{bmatrix}}

\newcommand{\xinyang}[1]{{
\color{orange} [\text{Xiyang:} #1] 
}}

\newcommand{\revise}[2]{
{ #2}
}

\def\cme{{\mathtt{CME}}}

\def\att{\mathtt{attn}}
\def\softmax{\mathtt{softmax}}

\def\mha{{\mathtt{mha}}}

\def\noend{\notag \\}

\def\fk{\mathfrak{K}}
\def\fl{\mathfrak{L}}

\newcommand{\Att}{\mathtt{attn}}
\newcommand{\KL}{\mathtt{KL}}
\newcommand{\pt}{\mathtt{prompt}}

\newcommand{\bbE}{\mathbb{E}}

\newcommand{\bbN}{\mathbb{N}}

\newcommand{\bbR}{\mathbb{R}}

\newcommand{\rmd}{\mathrm{d}}

\newcommand{\calB}{\mathcal{B}}
\newcommand{\calC}{\mathcal{C}}
\newcommand{\calD}{\mathcal{D}}

\newcommand{\calF}{\mathcal{F}}

\newcommand{\calJ}{\mathcal{J}}
\newcommand{\calK}{\mathcal{K}}
\newcommand{\calL}{\mathcal{L}}

\newcommand{\calN}{\mathcal{N}}
\newcommand{\calO}{\mathcal{O}}
\newcommand{\calP}{\mathcal{P}}

\newcommand{\calS}{\mathcal{S}}

\newcommand{\calV}{\mathcal{V}}

\newcommand{\calX}{\mathcal{X}}

\newcommand{\hatC}{\hat{C}}

\newcommand{\tilz}{\tilde{z}}

\newcommand{\barB}{\bar{B}}

\newcommand{\barD}{\bar{D}}

\newcommand{\hrho}{\hat{\rho}}

\title{\huge Unveiling the Statistical Foundations of Chain-of-Thought Prompting Methods}
\author{Xinyang Hu\thanks{Yale University. Email:\texttt{\{xinyang.hu, siyu.chen.sc3226, zhuoran.yang\}@yale.edu}.} 
\qquad Fengzhuo Zhang\thanks{National University of Singapore. Email:\texttt{fzzhang@u.nus.edu}.} 
 \qquad  Siyu Chen$^\ast$  \qquad  Zhuoran Yang$^\ast$}
 \date{}

\begin{document}
\maketitle

\begin{abstract}

Chain-of-Thought (CoT) prompting and its variants have gained popularity as effective methods for solving multi-step reasoning problems using pretrained large language models (LLMs). In this work, we analyze CoT prompting from a statistical estimation perspective, providing a comprehensive characterization of its sample complexity. 
To this end, we introduce a multi-step latent variable model that encapsulates the reasoning process, where the latent variable encodes the task information. Under this framework, we demonstrate that when the pretraining dataset is sufficiently large, the estimator formed by CoT prompting is equivalent to a Bayesian estimator. This estimator effectively solves the multi-step reasoning problem by aggregating a posterior distribution inferred from the demonstration examples in the prompt.

Moreover, we prove that the statistical error of the CoT estimator can be decomposed into two main components: (i) a prompting error, which arises from inferring the true task using CoT prompts, and (ii) the statistical error of the pretrained LLM. We establish that, under appropriate assumptions, the prompting error decays exponentially to zero as the number of demonstrations increases. 
Additionally, we explicitly characterize the approximation and generalization errors of the pretrained LLM. 
Notably, we construct a transformer model that approximates the target distribution of the multi-step reasoning problem with an error that decreases exponentially in the number of transformer blocks.
Our analysis extends to other variants of CoT, including Self-Consistent CoT, Tree-of-Thought, and Selection-Inference, offering a broad perspective on the efficacy of these methods. We also provide numerical experiments to validate the theoretical findings.

\end{abstract}
\tableofcontents

\section{Introduction} \label{sec: intro}

Autoregressive Large Language Models (LLMs) have tremendously revolutionized the field of Natural Language Processing (NLP) and related areas due to their striking ability to understand languages and follow instructions. These models, based on the transformer architecture \citep{vaswani2017attention}, are probabilistic models that predict the next token based on preceding tokens, also known as a \emph{prompt}. The training of LLMs typically involves two phases: \emph{pretraining} and \emph{post-training}. During the \emph{pretraining} phase, the LLMs are trained on vast text corpora via unsupervised learning \citep{ahmad2021unified, zoph2020rethinking, erhan2010does, hendrycks2019using}. This process allows them to acquire a broad understanding of language and general knowledge. Subsequently, additional post-training approaches, including supervised fine-tuning \citep{wei2021finetuned} and reinforcement learning with human feedback (RLHF) \citep{ouyang2022training}, are adopted to enhance the chat capabilities of LLMs. Finally, the trained LLMs are deployed to interact with human users, with their neural network parameters remaining fixed.

Human users interact with LLMs through \emph{prompting}, which refers to text generation conditioned on the prompts provided by the users. Designing effective prompts to induce specific desired behaviors in LLMs is known as \emph{prompt engineering} \citep{sahoo2024systematic}, which is largely a heuristic enterprise. Prompt engineering represents a \emph{paradigm shift} from standard statistical learning. Specifically, when using LLMs to solve a task via prompting, the LLMs essentially ``learn'' from the prompts by passing them through the neural network with fixed parameters, which have been trained without data from the desired task.

One of the most widely used prompting heuristics is \ac{icl}~\citep{brown2020language, dong2022survey}, a technique that enables LLMs to comprehend concepts by including a few examples in the prompt. This involves feeding the LLM with a few input-output examples and then asking for the output corresponding to a new input. On many tasks, the LLM can successfully extract the relationship between inputs and outputs and generalize it to the new input to get the desired output. This simple and intuitive prompting method has recently drawn considerable research interest and has become the foundation for many sophisticated prompting methods designed for more complicated tasks \citep{wei2022chain, zhou2022least, kim2022self, zhang2022automatic, rubin2021learning, sorensen2022information, creswell2022selection, yao2023tree, wang2022self}.

A prominent example of \ac{icl} is \ac{cot} prompting \citep{wei2022chain}, which generalizes ICL for multi-step reasoning tasks. Specifically, the vanilla version of few-shot \ac{cot} proposes including intermediate reasoning steps in addition to input and output in the demonstration examples, helping LLMs understand the reasoning path from input to output. Building upon vanilla \ac{cot}, other sophisticated variants of \ac{cot} have been proposed to efficiently select reasoning paths via majority votes or tree search \citep{creswell2022selection, yao2023tree, wang2022self}.

While \ac{cot} prompting methods have found great empirical success in multi-step reasoning problems such as arithmetic, commonsense, and symbolic reasoning, there is still a lack of theoretical understanding of why CoT works and how it compares with vanilla ICL. In this work, we aim to rigorously understand why the practice of “{\bf pretrained LLM + CoT prompting}” is capable of solving multi-step reasoning problems. Additionally, we aim to demystify how the \emph{transformer architecture} of the LLMs and the \emph{intermediate reasoning steps} in the prompts contribute to this success. Specifically, we aim to answer the following four questions:
\begin{description}
\item \textbf{(a)} What are the statistical estimators constructed by \ac{cot} and its variants?
\item \textbf{(b)} What are the statistical properties of these estimators?
\item \textbf{(c)} How does the transformer architecture enable the LLMs to learn these estimators?
\item \textbf{(d)} Does \ac{cot} prompting always outperform vanilla ICL?
\end{description}

To answer Question (a), we introduce a multi-step latent variable model that captures the data-generating process involved in multi-step reasoning. 
Under this model, a sequence of $H+1$ variables $\{Z_0,\ldots, Z_{H}\}$ is generated according to a distribution conditioned on a latent variable $\theta^*$, where $H$ stands for the number of reasoning steps. 
Here, $Z_0$ and $Z_H$ are the input and output, respectively, and $\{Z_{h}\}_{h=1}^{H-1} $ 
are the intermediate steps. The parameter $\theta^*$ captures the underlying statistical task and is a random variable taking values in a set $\Theta$. Under this model, in \ac{cot} prompting, the LLM is given $n$ examples of such sequences sampled conditioned on $\theta^*$ and asked to generate the output $Z_{H}$ corresponding to a queried input $Z_0 = z_0^{\texttt{test}}$. 
We further assume the LLM is pretrained by predicting the next reasoning step, with the training data generated from this model with $\theta^*$ sampled from the prior distribution. 
Under this setup, we answer Question (a) by proving that \ac{cot} prompting methods based on the pretrained LLM produce \ac{bma} estimators \citep{hoeting1999bayesian}. Specifically, based on the examples in the prompt, the \ac{llm} implicitly learns a posterior distribution over the latent variable $\theta^*$ and then generates the output $Z_H$ by aggregating over this posterior distribution.

Furthermore, the main effort of this paper is devoted to answering Question (b). To analyze the statistical error of the \ac{cot} estimator, we first decompose the statistical error $\mathtt{err}_\mathrm{CoT}$ into the sum of (i) pretraining error and (ii) prompting error. In particular, the pretraining error characterizes the statistical error arising from training the LLM to predict the next reasoning step on finite data. 
This error is further upper-bounded by the sum of \emph{approximation error} and \emph{generalization error}. To control the approximation error, we construct a class of transformers that directly approximate the population distribution while capturing the multi-step problem structure. This result is combined with a \emph{Pac-Bayes} generalization error bound \citep{mcallester1998some, alquier2021user} to establish the pretraining error bound. Moreover, the prompting error reflects the statistical error of the \ac{bma} estimator based on finite examples. We upper-bound this statistical error for variants of \ac{cot} methods in terms of both the last-iterate and average-iterate errors. By combining the analyses of pretraining and prompting errors, we provide a complete answer to Question (b).

Furthermore, to answer Question (c), we show that the attention mechanism enables the LLM to approximately encode \ac{bma} within the transformer architecture. In other words, prompting a pretrained \ac{llm} induces an output distribution that closely approximates the \ac{bma} estimator. Additionally, the transformer architecture also plays a role in the analysis of the approximation error, which is a byproduct of the analysis of pretraining error.

Finally, to answer Question (d), we specialize the theoretical result to the case where $H = 1$, which reduces to vanilla ICL. Our theory shows that \ac{cot} is always no worse than vanilla ICL in expectation among all tasks, but this dominance does not hold point-wisely for arbitrary tasks or prompts. Intuitively, with only a finite number of examples, it is possible that \ac{cot} jeopardizes performance when the intermediate reasoning steps included in the prompt are not sufficiently informative. We empirically validate this argument on a synthetic task, showing that whether \ac{cot} prompting is beneficial is nuanced.

In conclusion, as an initial attempt to bridge the gap between theory and practice, this paper establishes a comprehensive understanding of \ac{cot} and its variant prompting methods with theoretical guarantees. We hope this work will pave the way for further exploration into the theoretical underpinnings of prompt engineering.

\vspace{2mm}

{\noindent \bf Notation.}
 Let $[T]$ denote $\{1,\cdots,T\}$. We use $V_i$ to denote the $i$-th coordinate of the vector $V$. We adopt $\{\PP_{\theta}\given \theta\in \Theta\}$ to denote a parametric family of distributions parameterized by $\theta \in \Theta$. The Kullback-Leibler (KL) divergence between two distributions $\PP$ and $ \QQ$ is denoted by  $$\KL(\PP , \QQ)=\E_{\PP(x)}\bigg[\log\frac{p(x)}{q(x)}\bigg],\,$$ where $p$ and $q$ are the densities of $\PP$ and $\QQ$ with respect to a reference distribution. Furthermore, we define the conditional KL divergence as 
 \begin{align*}
     \KL\big(\PP(Y=\cdot \given X=x) , \QQ(Y=\cdot \given X=x)\big)=\E_{y\sim \PP(\cdot \given x)}\bigg[\log\cfrac{p(y\given x)}{q(y\given x)}\bigg],
 \end{align*}
 which is a function of $x$.

 Let $\PP$ and $\QQ$ denote two probability measures defined over a measurable space $(\Omega, \calF)$. The total variation (TV) distance between $\PP$ and $\QQ$ is 
 $$
 \TV(\PP, \QQ) = \sup_{A\in \calF}|\PP(A)-\QQ(A)|,
 $$
 which is also half of the $\ell_1$ distance between density function $p(x)$ and $q(x)$.

 The Hellinger distance between two distributions $\PP$ and $ \QQ$ is
$$\text{H}(\PP, \QQ) = \frac{1}{\sqrt{2}} \bigg(\int \big(\sqrt{p(x)} - \sqrt{q(x)}\big)^2 \ud x \bigg)^{1/2} = \bigg(1 - \int \sqrt{p(x)q(x)} \ud x\bigg)^{1/2}.
$$
In addition, we use $\softmax(\cdot)$ to denote the softmax function, which maps a vector to a probability distribution. 
In particular, for any vector $x \in \RR^n$ and any $i \in [n]$, the $i$-th entry of 
$\softmax (x) $ is given by 
 $[\softmax(x)]_{i}=\exp(x_{i})/\sum_{j=1}^{n}\exp(x_{j})$. 
 For a matrix $X\in \mathbb{R}^{d_1\times d_2}$, we denote the $i$-th row and column using $X_{i:}$ and $X_{:,i}$, respectively. The $\ell_{p,q}$ norm of $X$ is defined as $\|X\|_{p,q} = (\sum_{i=1}^{d_2} \|X_{:,i}\|_{p}^q)^{1/q}$. Furthermore, we use $\|X\|_\mathrm{F} = \|X\|_{2,2}$ to denote Frobenius norm. For a set $\cL$, let $\cP(\cL)$ denote the family of probability distributions over $\cL$. We use $z_{1:h}$ to denote the vector $[z_1, \cdots, z_h]$. We adopt $\cL^{*}$ to denote the set of all the sequences where each component is in the set $\cL$.

\section{Related Works}

Our work adds to the literature on theoretically understanding prompting methods. 
In particular, our work is closely related to \ac{cot} prompting and its variants. 
In addition, our work is related to the body of works that aim to understand the ability of \ac{icl} and \ac{cot} from both empirical and theoretical perspectives.

\vspace{2mm}

\noindent \textbf{CoT Prompting and its Variants.}  
The vanilla \ac{cot} prompting method is proposed in \cite{wei2022chain} for solving multi-step reasoning problems using LLMs. 
Based on this work,
many variants of \ac{cot} have been proposed to enhance the efficiency and reliability of LLMs in solving multi-step reasoning problems. See, e.g., \cite{yao2023tree, wang2022self, creswell2022selection,zhou2022least, chen2022program,zhang2023multimodal, besta2024graph} and also see \cite{chu2023survey,zhang2023igniting} for recent surveys of \ac{cot} methods. 
In particular, our work offers a theoretical understanding for vanilla \ac{cot} and variants including \ac{sc} \ac{cot} \citep{wang2022self}, \ac{si} \cite{creswell2022selection}, and \ac{tot} \citep{yao2023tree}.

\vspace{2mm}

\noindent \textbf{Existing Research on Understanding \ac{icl}.}
Our work is closely related to the body of works that aim to understand the ability of \ac{icl} from both empirical and theoretical perspectives. From an empirical point of view, \cite{garg2022can, min2022rethinking,krishnamurthy2024can, zhang2022automatic, dziri2024faith, olsson2022context} explore the understanding of the behavior and capability of \ac{icl}.
In particular, \cite{garg2022can} show that transformers can learn unseen linear functions via \ac{icl}. \cite{min2022rethinking} demonstrate that shuffled input-output pairs in few-shot \ac{icl} induce little degradation in the performance on a range of classification and multi-choice tasks. \cite{dziri2024faith} study how transformer-based \ac{llm}s solve compositional tasks and their limitations in reasoning. 

From a theoretical perspective, 
\cite{akyurek2022learning, von2023transformers, bai2023transformers, dai2023can, wang2023label} establish theoretical understandings of \ac{icl}. The theories proposed in these works mainly offer two explanations of ICL:  (i) LLMs perform ICL by running iteration optimization algorithms such as gradient descent, and (ii) LLMs perform ICL by implementing Bayesian inference through the architecture. The works \cite{akyurek2022learning, von2023transformers, bai2023transformers, dai2023can} indicate that \ac{icl} implicitly implements the gradient descent or least-square algorithms from the function approximation perspective. \cite{hou2023towards} hypothesize that \ac{llm}s implicitly perform multi-step reasoning within their architecture by going through a reasoning tree.
\cite{li2023transformers} derive the generalization bound for \ac{icl} from the view of multi-task learning. \cite{hahn2023theory} adopt a linguistic point of view and bounds the \ac{icl} error using description length. The works \cite{ahn2023transformers,huang2023context, fu2023transformers, mahankali2023one, wu2023many} consider linear attention models to study the performance of \ac{icl}, which restricts the function class that can be represented by transformers to linear functions. 

Another line of work lies in the Bayesian interpretation of the \ac{icl} paradigm \citep{jiang2023latent, wang2023large, xie2021explanation,wies2023learnability,zhang2023and, he2024words}. Under the Bayesian framework, \cite{xie2021explanation} use \ac{hmm}\citep{rabiner1986introduction} to model the token generation process and assume access to the true language distribution. However, the \ac{hmm} assumption is restrictive, and the perfect pretraining assumption does not incorporate the pretraining phase into the story. To this end, \cite{wies2023learnability} relax these two assumptions by adopting a general i.i.d.\ data model and analyzing a pretrained model that well approximates the true distribution given any token sequence, which is also unrealistic. These works do not mention the relationship between transformer architecture, pretraining process, and the Bayesian interpretation of \ac{icl}. 

Among these works, our work is most related to \cite{zhang2023and} and \cite{he2024words}.
In particular, \cite{zhang2023and} adopt a latent variable model that generalizes the HMM model in \cite{xie2021explanation}, and show that ICL can be explained as a BMA estimator under this model. 
They also establish the statistical error of the BMA estimator and connect it to the attention mechanism. 
\cite{he2024words} further extend this BMA framework for studying LLM-based decision-making problems, where an LLM is used as a policy. 
They bring about the equivalence between the LLM-based policy trained by predicting the next action given the history and a  Bayesian version of imitation learning. This ability of \ac{llm} allows the decision maker to take optimal actions in each timestep when the pretraining data contains the optimal actions provided by the oracle.
Our work builds on the ideas of  \cite{zhang2023and} and extends the Bayesian framework to \ac{cot} and its variants method. 
Compared to these works, the output of CoT is obtained by multi-step generating using the LLM, but ICL and imitation learning only involve one-step generations. 
To this end, we propose a multi-step latent variable model, and establish new analyses for the errors in both pretraining and prompting stages.  
For example, to bound the pretraining error, we construct a family of transformer models that explicitly take the multi-step structure into account. 

\vspace{2mm}

\noindent \textbf{Existing Research on Understanding \ac{cot}.} Our work aims to understand the capability and behavior of \ac{cot}. The following works provide an interpretation of \ac{cot} from both experimental and theoretical perspectives. \cite{saparov2022language, shi2022language, paul2023refiner, wang2212towards, tang2023large, madaan2022text} offer practical insight by exploring the performance and capability of \ac{cot} reasoning empirically. On the theoretical side, \cite{merrill2023expresssive, feng2023towards, li2023dissecting, prystawski2024think} explore the reason behind the improvement in reasoning induced by \ac{cot}. \cite{wu2023analyzing, tutunov2023can,hou2023towards, wang2023label} investigate the ability demonstrated by \ac{cot} through examining the internal mechanism of the transformer architecture. 

Currently, the understanding of \ac{cot} is still limited and requires further investigation. In this work, we adopt a statistical point of view to establish a refined characterization of the statistical properties of \ac{cot} in both the pretraining and prompting stages.

\section{Background} \label{sec:background}

In this section, we introduce the background knowledge about transformer-based large language models and \ac{cot} prompting.

\vspace{2mm}

{\noindent \bf Autoregressive \ac{llm}s.} Most commercial \ac{llm}s such as GPT-4 \citep{openai2023gpt}, Claude \citep{claude3}, Llama \citep{touvron2023llama}, and Gemini \citep{team2023gemini}, are autoregressive in the sense that they generate in a token-by-token fashion. 
An autoregressive \ac{llm}, denoted by $\PP_\mathrm{LLM}$, is a conditional probability model that continuously predicts future tokens based on a sequence of past tokens, known as the prompt. Here we denote the space of all the tokens as $\calX$. Given an input prompt $S_t =(x_1,...,x_t)\in \calX^{t}$, to generate the response to it, the \ac{llm} first generate the next token as $x_{t+1}\in\calX\sim \PP_\mathrm{LLM}(\cdot \given S_t)$. Then it appends the generated token $x_{t+1}$ to the end of $S_{t}$ to form $S_{t+1} = (S_{t},x_{t+1})$. The \ac{llm} will generate $x_{t+2}$ based on $S_{t+1}$, and it repeats this generation process till the generation of the end of the sentense.



\vspace{2mm}

{\noindent \bf Transformers and Attention Mechanism.}
The transformer model is based on the \ac{mha} mechanism \citep{bahdanau2014neural, phuong2022formal}, together with other modules such as the tokenizer and the positional embeddings \citep{wang2020position,su2023roformer}, residual connections, feed-forward networks, and layer normalization \citep{ba2016layer}. 
The tokenizer maps the input sequence to a sequence of vectors in Euclidean space, and the positional embeddings add the position information of tokens to these vectors. 

The attention mechanism captures the relationship between different tokens, which is the backbone of transformer-based \ac{llm} \citep{devlin2018bert}. The attention mechanism takes in queries, keys, and values as inputs, and outputs the response of each query as a weighted sum of values,  where the weights are the similarity scores between the query and the keys.
Specifically, let $K \in \mathbb{R}^{L\times d_k}$ and $ V \in \mathbb{R}^{L\times d_v}$ denote the $L$ key and value vectors, respectively. The attention output of a single query $q \in \mathbb{R}^{d_k}$ is computed as:
\begin{align}
    \Att(q, K, V) = V^T \softmax(Kq), \label{eq: softmax attn}
\end{align}
where $\softmax(Kq)$ is a probability distribution over $[L]$.
Here $\softmax(Kq)$ quantifies the similarity between the query $q$ and each row of  $K$, which is used to aggregate the value vectors of $V$. The attention $\Att(Q, K, V)$ that takes in multiple queries outputs the responses as $V^T \softmax(KQ^T)$, where $Q\in\mathbb{R}^{d_k}$ contains $L$ query vectors.
The predefined attention mechanism captures the relationship between the keys and queries via a single softmax module, and thus is called single-head attention. \ac{mha} refers to passing the inputs through multiple attention functions in parallel, and outputs the aggregation of these sub-modules. Taking $X \in \RR^{L\times r}$ as the input, a \ac{mha} layer with $\eta$ heads outputs
\begin{align}\label{eq: mha}
    \mha(X, W_\mathrm{mha}) = \sum_{i=1}^\eta \Att(XW_i^{Q}, XW_i^{K}, XW_i^{V}),
    \end{align}
The parameter set $W_\mathrm{mha} = \{W_i^{Q}, W_i^{K}, W_i^{V}\}_{i=1}^\eta$ are the weight matrices for queries, keys, and values, where $W_i^{Q}\in \RR^{r \times d_k}$, $W_i^{K}\in \RR^{r \times d_k}$, and $W_i^{V}\in \RR^{r \times d_v}$.
Intuitively, different heads can attend to different parts of the data, and thus MHA offers a more expressive model class. Compared to the \ac{mha} defined in \cite{vaswani2017attention}, we absorb the matrix $W_{i}^{O}$ into $W_{i}^{V}$ for each head.

Each \ac{mha} layer is followed by a \ac{ff} layer. Given an input \(X \in \RR^{L \times r}\), a \ac{ff} layer with \(d_F\) neurons maps the input $X$ to
\begin{align}\label{eq: ffn}
    \mathtt{ff}(X, W_{\mathrm{ff}}) = \mathtt{ReLU}(XW_{\mathrm{ff},1})W_{\mathrm{ff},2}, \qquad \text{where} \quad W_\mathrm{ff}=\{W_{\mathrm{ff},1}\in\mathbb{R}^{r \times d_F},W_{\mathrm{ff},2}\in \mathbb{R}^{d_F \times r}\}
\end{align}
 are weight matrices. 
There are also normalization layers between the MHA and FF layers. We defer their details to  Appendix \ref{app: pre-training process} for brevity. 

\vspace{2mm}

{\noindent \bf \ac{llm} Training.} 
The training of an \ac{llm} involves two stages: (i) pretraining \citep{zoph2020rethinking} and (ii) post-training \citep{ouyang2022training,wei2021finetuned}.
In the pre-training stage, the \ac{llm} is trained to predict the next token based
on a large corpus of text data by maximizing likelihood. 
The log likelihood function of a token sequence $S_T =(x_1,\cdots,x_T)$ for \ac{llm}s is given by $\sum_{t=1}^T\log \PP_{\mathrm{LLM}} (x_t | S_{t-1})$, where $\PP_{\mathrm{LLM}}$ denotes the conditional distribution induced by the \ac{llm}. 
The pretraining dataset consists of a large number of token sequences from diverse datasets. For the popular \ac{llm}s such as GPT-4 \citep{openai2023gpt}, pretraining datasets are internet-scale and contain billions or trillions of tokens from a variety of sources, such as Wikipedia, news articles, and books \citep{openai2023gpt}. 
The goal of pretraining is to learn a general-purpose \ac{llm} that can generate coherent text and capture the statistical structure of natural language. 
Then in the second stage, the pretrained \ac{llm} is finetuned on a much smaller labeled dataset consisting of question-answer pairs or human feedbacks \citep{wei2021finetuned, ouyang2022training}. Fine-tuning can be either based on supervised learning, reinforcement learning, or both. The goal is to adapt the \ac{llm} to a chatbot-style model that can interact with humans and generate conversations that align with human values.

\vspace{2mm}

{\noindent \bf Prompting a Pretrained \ac{llm}.}
Users interact with \ac{llm}s by providing a piece of text, known as the ``prompt'', and let the \ac{llm}s generate a token sequence based on the given prompt. Here the network parameters of the \ac{llm} is fixed and the \ac{llm} is not trained on the prompt. Due to the autoregressive nature, in the sequel, we slightly abuse the notation by regarding $\PP_\mathrm{LLM}$ as a mapping from a prompt to a probability distribution over the output token sequence. Then prompting an \ac{llm} is equivalent to sampling a the output token sequence, $\mathtt{output}  \sim \PP_\mathrm{LLM}(\cdot \given \pt)$, where $\pt$ is the input token sequence.

\vspace{2mm}

{\noindent \bf In-Context Learning.}  \ac{icl} refers to the learning process of the  \ac{llm}s where they learn from prompts without tuning the parameters \citep{dong2022survey}. In the vanilla version of \ac{icl}, we prompt an \ac{llm} with a collection of input-output pairs, known as ``examples'' or ``demonstrations'', and a new input query. 
We expect the \ac{llm}s to learn the underlying pattern of the input-output pairs and generate a desired output associated with the input query following the same pattern.  More concretely, let $\{(x_i, y_i)\}_{i=1}^T$ be a collection of $T$ examples satisfying $y_i = f_*(x_i)$, where $f_*$ is the underlying input-output relationship, and $x_{i},y_{i}\in \calX^{*}$. Let $x_q$ denote a new input query, and we concatenate the examples and the query to form the prompt $\pt =  (x_1, y_1,\cdots,x_{T},y_{T}, x_q) $. The \ac{llm} is able to learn in an in-context fashion if $y_q \sim \PP_\mathrm{LLM}(\cdot \given \pt)$ satisfies $y_q = f_*(x_q)$. For example, we can prompt the \ac{llm}s with ``{\fontfamily{qcr}\selectfont grass is green, apple is red, sky is}'' to let the \ac{llm}s output the color of the sky. Here (``{\fontfamily{qcr}\selectfont grass is}'',  ``{\fontfamily{qcr}\selectfont green}'') and (``{\fontfamily{qcr}\selectfont apple is}'', ``{\fontfamily{qcr}\selectfont  red}'') are examples, ``{\fontfamily{qcr}\selectfont sky is}'' is the query, and the desired output $y_{q}$ is {\fontfamily{qcr}\selectfont blue}. Generating $y_q$ based on $\pt$ is called ``in-context learning'' because the \ac{llm} learns the desired relationship $f_*$ purely from the prompt without updating the network parameters of the \ac{llm}.  

\vspace{2mm}

{\noindent \bf Chain of Thought and its Variants.} 
When the input-output relationship is complex, it is challenging for \ac{llm}s to learn this relationship directly from input-output pairs in the prompt. The complex relationship usually appears in the multi-step reasoning problem. 
For example, calculating a long math equation $5\times 20+(5+3)\div 2 = 100 + 8\div 2 = 100 + 4 = 104$ involves a series of operations. 
It is difficult to learn a function that can directly output the result.
\ac{cot} is a prompting technique that aims to solve multi-step reasoning tasks by providing multiple input-output examples together with intermediate reasoning steps in the prompts~\citep{wei2022chain}. 
By guiding \ac{llm}s through a sequence of intermediate reasoning steps before arriving at a final answer, we expect to decompose a complicated reasoning problem into a sequence of simple subtasks that can be learned via vanilla \ac{icl}. 
We formulate a \ac{cot} prompt with $H$ steps as $  \mathtt{prompt}_{\mathrm{CoT}}(n) = (z_{0:H}^1,\cdots,z_{0:H}^n, z_0^{\mathrm{test}})$, and we will denote this as $(\{z_{0:H}^i\}_{i=1}^n, z_0^{\mathrm{test}})$ in the following for ease of notation. Here for each $i\leq n$, $(z_0^i,z_H^i)$ corresponds to the input-output pair of the example $(x_i,y_i)$ in vanilla \ac{icl}, and $(z_1^i, \cdots, z_{H-1}^i)$ denotes the intermediate reasoning steps of the example. For example, in the math equation calculation problem, $z_{0} =$``$5\times 20+(5+3)\div 2$'', $z_{1}=$``$100 + 8\div 2$'', and $z_{3} =$``$104$''.
Here $H$ is fixed throughout this paper.
We recover the vanilla \ac{icl} prompts by omitting the intermediate steps, i.e., setting $H=1$. We will show more concrete examples in Section~\ref{subsec: bma interpretation of cot}. For simplicity, we assume each reasoning step \( z \) takes value in a finite set \( \calL\subseteq\calX^{*} \), with each element uniquely identified with an embedding vector in the Euclidean space $\RR^{d_{\cZ}}$ for some integer $d_{\cZ}$. We let $\calL^*$ denote the set of sequences consisting of reasoning steps, e.g., $\mathtt{prompt}_{\mathrm{CoT}}(n) \in \cL^*$.


Furthermore, as conditional probability models, \ac{llm}s are intrinsically \textit{stochastic}. For problems such as solving mathematical questions, however, there is often a unique answer. 
To further boost the probability of finding the correct answer, variants of \ac{cot} leverage multi-step reasoning with various selection techniques to solve more complicated reasoning and decision-making problems. For instance, \ac{sc}-\ac{cot} \citep{wang2022self} uses majority vote, \ac{tot} \citep{yao2023tree} adopts tree search methods, and \ac{si} \citep{creswell2022selection} further introduces a selection module in each reasoning step.

In the next section, we will introduce a multi-step latent variable model to interpret the \ac{cot} prompting method as a Bayesian model averaging estimator.



\section{A Latent Variable View of Multi-Step Reasoning} \label{sec: methodology}

In this section, we show that \ac{cot} prompting can be understood as a Bayesian estimator on a multi-step latent variable dynamical model. 
In particular, we propose a {\bf multi-step latent variable model} in Section~\ref{subsec: bma interpretation of cot} to capture the multi-step reasoning process, which is further generalized in Appendix~\ref{app: generalized model} to the non-i.i.d. setting. 
Then in Section~\ref{subsec: pre-training and cot prompting under the statistical model}, we study the practice of  \ac{cot} prompting of pretrained LLMs from a statistical perspective. 
In Section~\ref{subsec: attention parameterizes bma}, we show that such a practice is equivalent to a \ac{bma} estimator for the multi-step latent variable model, which answers Question (a) raised in Section~\ref{sec: intro}.
Moreover, we show that the softmax attention mechanism in the transformer architecture parameterizes the \ac{bma} algorithm, which partially answers Question (c).

\subsection{A Multi-Step Latent Variable Model} \label{subsec: bma interpretation of cot}

We introduce a multi-step latent variable model to capture the multi-step reasoning process of \ac{cot}, which serves as the data-generating model for studying \ac{cot}.

\vspace{2mm}

\noindent \textbf{\ac{cot} Prompting Paradigm.} 
Recall that we define the \ac{cot} prompt $\mathtt{prompt}_{\mathrm{CoT}}(n)=(\{z_{0:H}^i\}_{i=1}^n, z_0^{\mathrm{test}})=(\Upsilon_n,z_0^\mathrm{test})$ in Section \ref{sec:background}, which contains $n$ demonstration examples $\Upsilon_n=\{s_i\}_{i=1}^n=\{z_{0:H}^i \}_{i=1}^n$ and a testing query $z_0^\mathrm{test}$. 
To generate such a prompt, we first specify a latent concept vector, which is denoted as $\theta^* \in \Theta $. Here $\Theta$ denotes the set of all the latent concepts. Semantically, $\theta^*$ determines the task we would like to achieve via \ac{cot}, e.g., the color description of objects, the calculation of math equations. Thus, we will use the terms task and latent concept interchangeably in the following. Statistically, the latent concept $\theta^{*}$ specifies the task-specific joint distribution $\PP(\cdot \given \theta^*)$ of demonstration examples and testing query in the prompt, which will be specified later in \eqref{eq:latent_var_model}. 
Given the generated prompt $\pt_{\mathrm{CoT}}(n)$, we feed it to  the \ac{llm}, and the \ac{llm} recursively generates the intermediate steps $(z_1^\mathrm{test},\ldots, z_{H-1} ^\mathrm{test} )$ and the final answer $z_H^\mathrm{test}$ via 
\begin{align}\label{eq:cot_generation}
z_{h+1}^\mathrm{test}\sim \PP_{\mathrm{LLM}} (\cdot\given \pt_\mathrm{CoT}(n),z_1^\mathrm{test},\ldots, z_{h}^\mathrm{test}),\qquad \forall  h\in [H-1].
\end{align}

To evaluate the performance of \ac{cot}, we compare the distribution of $z_H^\mathrm{test}$ in \eqref{eq:cot_generation} with the ground truth distribution $\PP(z_H^\mathrm{test}\given \pt_\mathrm{CoT}(n),\theta^*)$, which is the target task-specific distribution of the final answer given the prompt. 
We illustrate the \ac{cot} paradigm with a concrete example as follows.

As a concrete example, consider the task $\theta^* = $ ``{\fontfamily{qcr}\selectfont calculate twice the area code of the given country.}'' The prompt in Figure~\ref{fig:cot_icl_example} is a \ac{cot} prompt with $n=2$ and $H=2$, where the input of the first example is $z_0^1=$``{\fontfamily{qcr}\selectfont The US = ?}'', the first step of the solution is $z_1^1=$``{\fontfamily{qcr}\selectfont The US has area code 1}", and the second step of the solution is $z_2^1=$``{\fontfamily{qcr}\selectfont so the answer is 2}".
The query is $z_0^\mathrm{test}=$``{\fontfamily{qcr}\selectfont Japan = ?}'', and the desired task-specific answer is ``{\fontfamily{qcr}\selectfont 126}''.
When tested on ChatGPT~\citep{achiam2023gpt},
it indeed outputs the correct answer with an intermediate reasoning step: ``{\fontfamily{qcr}\selectfont Japan has area code 81, so the answer is 162.}''
In comparison, the vanilla \ac{icl} prompt has $x^1=$``{\fontfamily{qcr}\selectfont The US = ?}'', $y^1=$``{\fontfamily{qcr}\selectfont The answer is 2}'', $x^2=$``{\fontfamily{qcr}\selectfont France = ?}'', $y^2=$``{\fontfamily{qcr}\selectfont The answer is 66}'', and $x^\mathrm{test}=$``{\fontfamily{qcr}\selectfont Japan = ?}''. 
In this case, however, ChatGPT is unable to provide the correct answer because it fails to find the relationship between the area code and the country.\footnote{Both the CoT and vanilla ICL prompts are tested on ChatGPT (GPT-3.5-turbo-16k) with the temperature set to zero. See Section~\ref{subsec: area code} for the details.}
See Figure \ref{fig:cot_icl_example}
for a visual illustration of \ac{cot} and vanilla \ac{icl} prompts. 
Thus, seen from this example, by providing additional reasoning steps, \ac{cot} prompts can significantly boost the accuracy of the \ac{llm} compared with vanilla \ac{icl} prompts. 


\begin{center}
\includegraphics[width=0.94\textwidth]{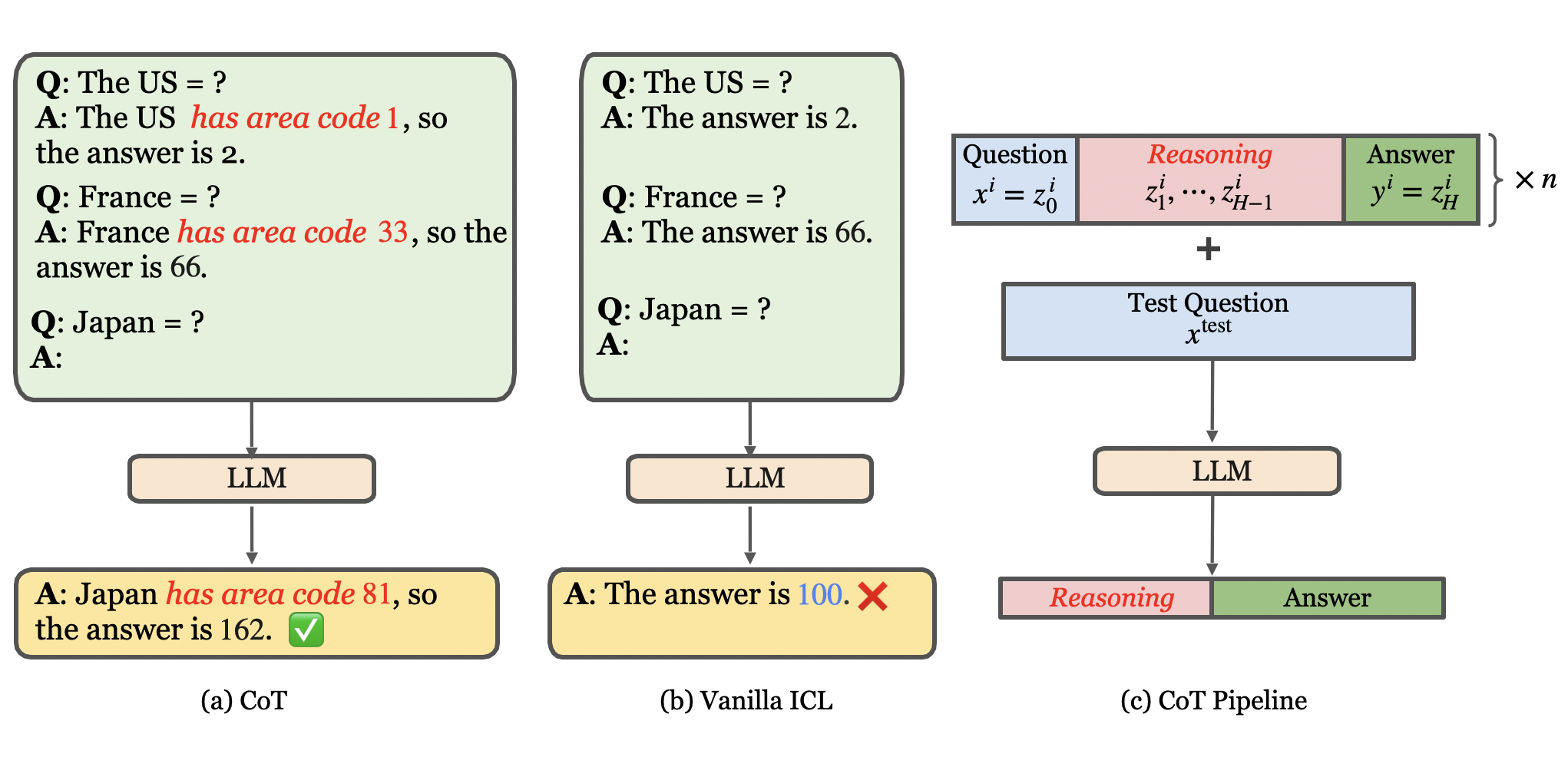}
\captionof{figure}{An illustration of \ac{cot} and vanilla \ac{icl}. 
Figure (a) shows the \ac{cot} prompt and the corresponding output of ChatGPT (GPT-3.5-turbo-16k). 
The intermediate reasoning is shown in red. 
The output of ChatGPT follows the pattern in the prompt, which consists of a reasoning step, followed by the desired answer. 
Figure (b) shows the result of using the corresponding vanilla \ac{icl} prompt, which includes of two input-output pairs. 
In this case, ChatGPT fails to provide the correct answer.
Figure (c) illustrates a general pipeline of \ac{cot} prompting with $n$ demonstration examples. Each example includes an input question, $H-1$ intermediate reasoning steps, and the final answer. 
}
\label{fig:cot_icl_example}
\end{center}

\vspace{2mm}

\noindent \textbf{The Multi-Step Latent Variable Model.}
To analyze \ac{cot} from a statistical perspective, we need to specify the pre-mentioned task-specific distribution $\PP(\cdot \given \theta^*)$, which serves as the data-generating distribution for the \ac{cot} prompt. 
We assume that the concept $\theta^*$ is a random variable sampled from a prior $\pi\in\cP(\Theta)$ and the examples $\{s^i\}_{i=1}^n$ are i.i.d. sequences conditioning on $\theta^* \in \Theta$.
For any $\theta \in \Theta$, when $\theta ^* = \theta$, 
within the reasoning chain $i\in[n]$, we 
sample $z_{0:H}^i$ according to the following stochastic dynamical system with joint distribution $\PP(s^i\given \theta^* = \theta )$ given by 
\begin{align}
    \PP(s^i\given \theta^* = \theta ): \qquad z_0^i = f_{\theta}\big(\zeta^i\big),  \quad 
z_h^i    = F_{\theta}\big(z_0^{i},\cdots, z_{h-1}^{i}, \epsilon_h^i\big),  \quad \forall  1\leq h \leq H.  
\label{eq:latent_var_model}
\end{align}
Here $\{ \zeta^i, \{ \epsilon_h^i \}_{h\in [H]}  \}_{i \in [n]}$  are i.i.d.\  noise variables, and $f_{\theta}$  and  $F_{\theta}$ are two functions parameterized by $\theta\in \Theta$.
The same is true for the test sample $z_{0:H}^{\mathrm{test}}$ and this distribution will serve as the target distribution for \ac{llm} to learn in context during the prompting stage.
Specifically, $f_{\theta}$ generates the first query $z_0^i$ based on the task $\theta^* = \theta $, and $F_{\theta}$ models the evolution of the ``reasoning process'' $\{ z_{h}^i \}_{h\in[H]}$.
Specifically, each $z_h^i$ depends on all of the previous reasoning steps as well as the latent variable $\theta^*$. 
The rationale behind this model is that the generation of these reasoning steps is autoregressive and the distribution of the whole sequence is 
specific to the task $\theta^*$.  
The random variables $\{ \epsilon_h^i\} _{ h\in [H]}$  allow the reasoning process to be stochastic. 
See Figure \ref{fig:latent_var_model} for an illustration of this model.  

Intuitively, $\theta^*$ represents the latent concept that specifies the task, e.g., ``calculate twice the area code of the given country, including the identification of the area code and the multiplication calculation'', ``solving an arithmetic problem with each intermediate step'' or ``writing a science fiction novel, detailing the thought process at each step''. 
As a concrete example, consider $\theta^*$ as the task of ``solving an arithmetic problem with each intermediate step''. 
The input $z_0^i$ is an arithmetic problem 
described using natural language. 
To get the final answer, a few intermediate arithmetic operations need to be performed.
The intermediate reasoning steps just corresponds to these operations described in natural language, and thus the transition depends on the task $\theta^*$.
See Figure \ref {fig:interpret_latent_variable_model}
for an illustration.

\begin{figure}[t]
    \centering
    \includegraphics[width=0.8\textwidth]{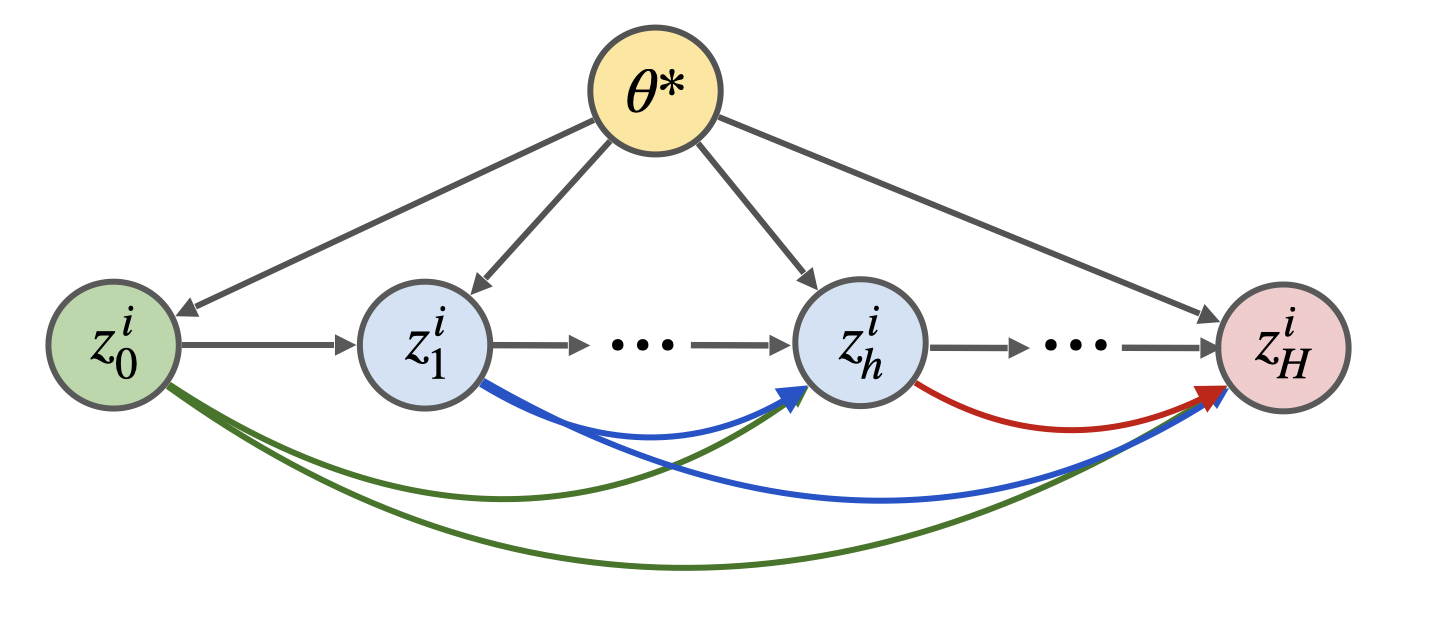}
    \caption{An illustration of the multi-step latent-variable model defined in \eqref{eq:latent_var_model}.
    According to this graphical model, for any $h \geq 1 $,  each step $z_h^i$ of $i$-th example depends on the previous steps $\{ z_{\ell}^i \}_{ \ell < h} $ and the hidden concept $\theta^*$. }
    \label{fig:latent_var_model}
\end{figure}

\revise{Furthermore, we assume that testing sequence has the ``ground truth'' distribution also given by 
\begin{align*}
    z_0^\mathrm{test} = f_{\theta^*}\big(\zeta^\mathrm{test}\big),  \qquad 
    z_h^\mathrm{test}    = F_{\theta^*}\big(z_0^{\mathrm{test}},\cdots, z_{h-1}^{\mathrm{test}}, \epsilon_h^\mathrm{test}\big),  \qquad \forall  1\leq h \leq H. 
\end{align*}
Thus, the model in \eqref{eq:latent_var_model} specifies the   conditional distributions $\PP(z_h^i \given z_0^i,z_1^i,\cdots,z_{h-1}^i, \theta^*)$ for all $i$ and $h$, and also for the test example $z_{0:H}^\mathrm{test}$.}{}


\begin{figure}[t]
    \centering
    \includegraphics[width = 0.94\textwidth]{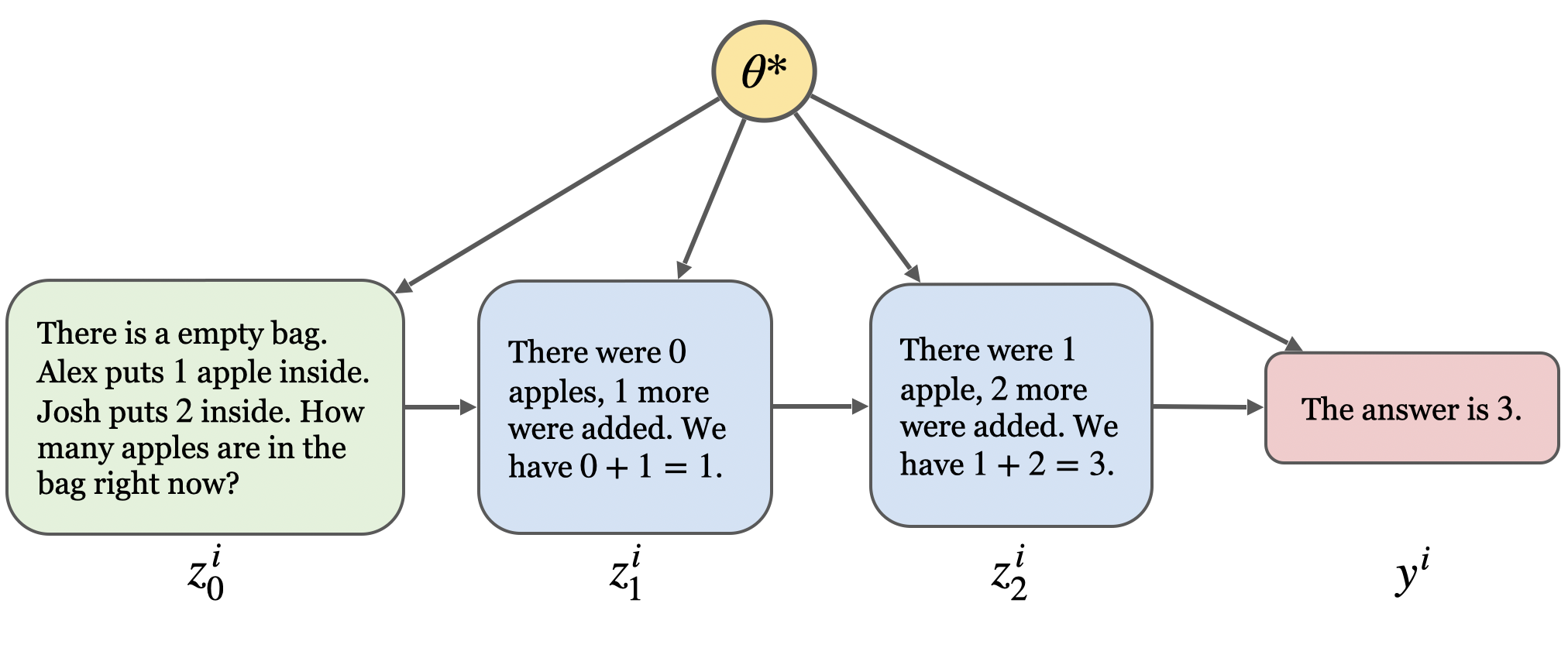}
    \caption{An instantiation of the model in~\eqref{eq:latent_var_model} in the context of arithmetic problems.
    Here $\theta^*$ stands for ``solving an arithmetic problem with intermediate steps'', and 
    $z_0^i$ describes a new arithmetic problem generated independently from any other demonstrations.
    Then each subsequent step, $z_1^i$, $z_{2}^i$, and $y^i$, depends on both the previous steps and the hidden task $\theta^*$.}
    \label{fig:interpret_latent_variable_model}
\end{figure}


%

Finally, note that setting $H=1$, we obtain a latent variable model for vanilla \ac{icl}, which is studied in \cite{wang2023large}.
Furthermore, our general model can be made more concrete by defining $\theta^{*}$ as a sequence of latent variables $\{\theta_h^*\}_{h\in [H]}$ characterizing the distribution of $z_{0:H}^i$, 
where the latent variables also have an autoregressive structure. 
Such a model is studied in \cite{jiang2023latent}. 
An limitation of our model in \eqref{eq:latent_var_model}  is that the $n$ demonstration examples  $\{s_i\}_{i=1}^n$ in $\pt_{\mathrm{CoT}}(n)$ are assumed to be i.i.d. 
In practice, the demonstration examples might be composed in a dependent manner which is beyond the i.i.d. assumption.
We will introduce a more general model Appendix~\ref{app: generalized model} which (i) includes latent variables for each reasoning process that are governed by a latent dynamical system, and (ii) allows the generation of the demonstration examples to be dependent. 

\subsection{Pretrained LLM + CoT Prompting} \label{subsec: pre-training and cot prompting under the statistical model}


The previous section proposes a latent variable model that captures the multi-step reasoning process of \ac{cot}.
Based on this model, we will formulate the estimator constructed by \ac{cot} prompting on a pretrained autoregressive \ac{llm} from a statistical perspective.

\vspace{2mm}

\noindent \textbf{Pretraining LLM.} We assume that the \ac{llm} is pretrained with data generated according to the model in \eqref{eq:latent_var_model}. Specifically, the \ac{llm} is pretrained on $N$ documents, where each document is generated by the model in \eqref{eq:latent_var_model} with a task-specific concept $\theta^{ *}_\ell  \overset{\mathrm{i.i.d}}{\sim} \pi$ for all $\ell \in [N]$.
Within each document $\ell \in [N]$, there are $T$ examples $\{s^{t,\ell}\}_{t=1}^T$ independently from the model in \eqref{eq:latent_var_model} with the same task $\theta^{ *}_\ell$, where $s^{t,\ell} = (z_0^{t,\ell}, \cdots, z_H^{t,\ell})$.
In other words, the training dataset contains $NT$ examples with diverse tasks. 

We let $\{\PP_{\rho} \given \rho\in \cP_{\mathrm{LLM}} \}$ denote the conditional distributions induced by the \ac{llm} with  parameters $\rho \in \cP_{\mathrm{LLM}}$, where $\cP_{\mathrm{LLM}}$ is the parameter space of \ac{llm} and will be specified in Section~\ref{subsec: imperfect}. Then pretraining the autoregressive \ac{llm} corresponds to finding the maximum likelihood estimator $\hat \rho$, i.e., 
\begin{align}\label{eq:pretraining_loss}
    \hat \rho = \argmin_{\rho \in \cP_{\mathrm{LLM}}} -\frac{1}{NT (H+1)}\sum_{\ell \in [N], t\in [T]} \sum_{h=0}^{H}\log \PP_\rho \big (z_{h}^{t,\ell} \given \Upsilon_{t-1,\ell},\{z_j^{t,\ell}\}_{j=0}^{h-1} \big ),
\end{align}
where $\Upsilon_{t,\ell} = \{s^{k, \ell}\}_{k \in [t] }$ is the first $t$ examples in the $\ell-$th document.
Given $\hat \rho$, we denote $\PP_\mathrm{LLM} = \PP_{\hat \rho}$ as the distribution induced by the pretrained \ac{llm} and write them interchangeably in the sequel.

\vspace{2mm}

{\noindent \bf CoT Prompting as an Estimator.}
After pretraining, we fix the parameter of the \ac{llm} as $\hat \rho$ and prompt the \ac{llm} with a \ac{cot} prompt $\pt_\mathrm{CoT} (n) = (\Upsilon_n,z_0^\mathrm{test})$. 
To connect the pretraining and prompting stages, we note that prompting a pretrained \ac{llm} with $\pt$ induces a conditional distribution $\PP_\mathrm{LLM}\big(\cdot \given \pt\big)$. 
When using a \ac{cot} prompt, we aim to induce the \ac{llm} to eventually \emph{generate a desired final answer} defined by \eqref{eq:latent_var_model}. 
The distribution of the final answer $y^\mathrm{test} = z_{H}^\mathrm{test}$ induced by the \ac{llm} via \ac{cot} reasoning is $\PP_\mathrm{LLM} (y^\mathrm{test} = \cdot \given \pt_\mathrm{CoT}(n) )$,
which is given by marginalizing out the intermediate steps $(z_1^{\mathrm{test}},\cdots, z_{H-1}^{\mathrm{test}})$ of \ac{cot}. 
To evaluate the statistical error of such an estimator, we consider  the Kullback-Leibler (KL) divergence 
\begin{align}
    \mathtt{err}_\mathrm{CoT}= \KL\Big(\PP\big(y^\mathrm{test}=\cdot \given \pt_\mathrm{CoT}(n),\theta^* \big) , \PP_\mathrm{LLM}\big(y^\mathrm{test}=\cdot \given \pt_\mathrm{CoT}(n)\big)\Big). \label{eq: err cot}
\end{align}
The error metric in \eqref{eq: err cot} is of particular interest to us, as our primary concern is the accuracy of the final result.
In the sequel, we consider $\theta^*\in\Theta$ to be fixed but unknown for our result.
Note that $\mathtt{err}_\mathrm{CoT}$ is a random variable where the randomness stems from $\pt_\mathrm{CoT}(n)$ during the prompting stage and the learned model parameters $\hat \rho$ during the pretraining.

\vspace{2mm}

\noindent \textbf{Error Decomposition.}
In the following, we briefly outline the error decomposition of the statistical error $\mathtt{err}_\mathrm{CoT}$ in~\eqref{eq: err cot}.
Intuitively, the statistical error has two sources:  a {\bf pretraining error} and a  {\bf   prompting error}. 
The pretraining error arises due to the finiteness of training data points and it decays to zero as $N$ increases. 
This pretraining error essentially is the statistical error of pretraining problem in \eqref{eq:pretraining_loss} and is irrelevant to the prompting stage. 
The prompting error reflects the error incurred by using $n$ examples to elicit the desired answer from the \ac{llm}.
Such an error appears even when the \ac{llm} is perfectly pretrained.
Intuitively, with more examples, the \ac{llm} has more information to infer the task $\theta^*$ and learn to generate the desired reasoning steps. 
Thus, the prompting error should decrease as $n$ increases. 
Moreover, the success of \ac{cot} prompting also depends on how well the examples in the prompt align with the testing query $z_0^\mathrm{test}$.
If the examples in the prompt are not informative enough for answering the testing query, the prompting error will be large.
Such an intuition is formalized by Lemma \ref{lem:error_decomp} in Section~\ref{subsec: perfect}, which shows that the prompting error can be further decomposed into two parts: a {\bf query error} and an {\bf in-context error}. 
The query error quantifies the distributional shift between the testing query and the examples in the prompt, and the in-context error quantifies the error due to the \ac{llm} not knowing the true task $\theta^*$ and having to make an inference based on the $n$ examples.

\begin{table}[t]
    \centering
    \begin{tabular}{|l|l|}
    \hline
    \textbf{Error Sources} & \textbf{Description}  \\
    \hline
    Pretraining error & Statistical error of the pretrained LLM  \\
    \hline
    Prompting Error & Combination of query error and in-context error  \\
    \hline
    Query error & Distributional shift between testing query $z_0^\textrm{test}$ and $n$  prompt examples  \\
    \hline
    In-context error & Statistical error of inferring  $\theta^*$  based on the $n$ prompt examples\\
    \hline
    \end{tabular}
    \caption{Summary of the three sources of errors in \ac{cot} prompting.  }
    \label{table:error_sources}
    \end{table}

\vspace{2mm}

\subsection{BMA Interpretation of CoT} \label{subsec: attention parameterizes bma}

In the following, 
we show that the \ac{cot} estimator $\PP_\mathrm{LLM}\big(\cdot \given \pt_\mathrm{CoT}(n) \big)$
can be understood as a Bayesian model averaging (BMA) estimator for the latent variable model in \eqref{eq:latent_var_model}.



\vspace{2mm}

\noindent \textbf{Pretrained \ac{llm} $+$  \ac{cot} $\approx$ \ac{bma}.} 
Recall that the pretraining process of \ac{llm} is given in \eqref{eq:pretraining_loss}, where the data is generated from the latent variable model in \eqref{eq:latent_var_model}. 
When $N$ and $T$ are sufficiently large, we expect the pretrained \ac{llm} to approximate the true distribution of the pretraining dataset well. 
Note that the tasks in the pretraining dataset are sampled from the prior $\pi$. 
When we replace $\PP_{\rho}$ in \eqref{eq:pretraining_loss} by the true data distribution, 
for any random document with $T$ examples, 
by Bayes' rule, we have that
\begin{align}\label{eq:bayes_factorization}
\PP\bigl (z_h^{t} = \cdot  \given \Upsilon_{t-1}, \{ z_{j}^{t}\}_{j=0}^{h-1}\big)  = \int_{\Theta} \PP\big(z_h^{t} = \cdot  \given \{ z_{j}^{t }\}_{j=0}^{h-1}, \theta\big) \cdot \pi\big(\theta \given \Upsilon_{t-1}, \{ z_{j}^{t}\}_{j=0}^{h-1}\big) \mathrm{d}\theta,
\end{align}
where $\Upsilon_{t-1} = \{s^{k}\}_{k \in [t-1] }$ contains the first $t-1$ examples in the document, and $\pi(\cdot \given \Upsilon_{t-1}, \{ z_{j}^{t}\}_{j=0}^{h-1})$ is the posterior distribution of the task. 
Here we use the fact that the examples are i.i.d. conditioning on the task.

Note that we expect that the pretrained \ac{llm} is approximately the same as the left-hand side of \eqref{eq:bayes_factorization}.
Based on \eqref{eq:bayes_factorization}, we can further marginalize the intermediate steps and obtain a similar factorization for $\PP_{\textrm{LLM}}(z_H^{t} \given \Upsilon_{t-1}, z_0^{t, \ell} )$.
Since the examples in the \ac{cot} prompt are generated from the same distribution as the pretraining data, we can set $t = n+1$ and get the following lemma.

\begin{lemma}\label{lem: bma.cot}
Let the pretraining data be generated according to the latent variable model specified in \eqref{eq:latent_var_model}.
Consider  the population counterpart of the MLE in \eqref{eq:pretraining_loss}, i.e., we let 
the number of documents $N$ goes to infinity.
Suppose that the \ac{llm}s have enough capacity, i.e., $\PP\in \{\PP_{\rho} \given \rho\in \cP_{\mathrm{LLM}}\}$, and the CoT prompt $\mathtt{prompt}_{\mathrm{CoT}}(n)$ has nonzero density under the pretraining distribution,
we have  
\begin{align*}
\PP_{\mathrm{LLM}}\big(y^\mathrm{test} = \cdot  \given \mathtt{prompt}_{\mathrm{CoT}}(n)\big)
    &= \int_{\Theta}\PP(y^\mathrm{test} = \cdot    \given  z_0^\mathrm{test},\theta)\pi\big(\theta \given \mathtt{prompt}_{\mathrm{CoT}}(n)\big) \mathrm{d}\theta. 
 \end{align*}
 \end{lemma}
 A detailed proof of this lemma is deferred to  Appendix~\ref{proof: bma}. This lemma implies that  \ac{cot} prompting based on a perfectly pretrained \ac{llm} performs \ac{bma}.
That is, the \ac{cot} estimator is constructed in three steps: (i) the \ac{llm} first constructs a posterior of the task $\theta$, then (ii)for each task $\theta$, the \ac{llm} predicts the final answer $y^\mathrm{test}$ based on the prompt, and (iii) finally, the \ac{llm} aggregates the predictions over the posterior of the task $\theta$.
Such a \ac{bma} interpretation is also established for vanilla \ac{icl} in \cite{zhang2023and}, which is recovered by our result when setting $H = 1$. 
We provide a detailed proof of this result and extend it to the more complicated model in Appendix~\ref{proof: bma}.

\subsection{Attention Approximately Parameterizes BMA} \label{subsec:attention_parameterize_bma}

We now show that the attention mechanism in the transformer architecture is able to encode the \ac{bma} algorithm for a special case of the latent variable model in \eqref{eq:latent_var_model}.

\vspace{2mm}

{\noindent \bf A Simplified Model.} In this special case, we Let $f_{\theta^{*}} $ in \eqref{eq:latent_var_model} be a function independent of $\theta^*$, i.e., the inputs do not depend on $\theta^*$. 
Moreover, we assume that $F_{\theta^*}$ in \eqref{eq:latent_var_model} encodes a linear model in the latent space. 
Specifically, 
for any $ h \in [H]$, let $d_k$ and $d_{v}$ be two  integers and let  
$k \colon \cL  \rightarrow \RR^{d_k} $ and  be $v \colon \cL  \rightarrow \RR^{d_v} $ be two feature mappings that maps each reasoning step $z_h^i$ to vectors. 
Moreover, assume $v$ is invertible. 
Then, we assume each reasoning step $z_{h}^i$ is generated from a Gaussian linear model with another feature mapping $ \phi$:
\begin{align}
    z_h^i & = v^{-1}\big(\theta^* \phi (k_h^i)+\epsilon_h^i \big)= F_{\theta^*} \big(z_0^i, z_1^i,\cdots,z_{h-1}^i,   \epsilon_h^i \big), \qquad \forall h \in [H]. \label{eq: linear model}
\end{align}
Here we define  
$k_h^i = \big[k(z_0^i), k(z_1^i),\cdots, k(z_{h-1}^i),0,\cdots,0 \big]   \in \RR^{H\cdot d_k}$ as the features of the first $h-1$ steps of the $i$-th example, and we pad $(H-h)$ zero vectors to ensure that $\{ k_{h}^i \}_{h\in[H]}$ live in the same Euclidean space. 
Moreover, $\phi \colon \RR^{H\cdot d_k } \rightarrow \RR^{d_{\phi}} $  is another feature mapping
that maps $k_h^i$ to some Euclidean space, where $\theta^* \in \RR^{d_v\times d_\phi}$ is a linear operator. 
The simplified model in \eqref{eq: linear model} thus postulates that $k_h^i$ and $v(z_h^i)$ satisfy a kernelized linear model. 
Moreover, we assume the  noise $\epsilon_h^i \overset{\iid}{\sim} \cN(0, \sigma^2)$ are i.i.d. and independent of everything else.

We note that our theoretical result in this section only relies on the invertibility of $v$ and that there exist feature maps $v$, $k$, and $\phi$ such that the model in \eqref{eq:latent_var_model} admits a simpler form as in \eqref{eq: linear model}.
This model specifies a linear dynamical system in the feature space.  
Thanks to the flexibility of these feature maps, this model captures a rich class of distributions.

\vspace{2mm}
{\bf \noindent The \ac{bma} Estimator.} 
To study the \ac{bma} estimator under this model, we further impose a Gaussian prior over $\theta^*$. 
Specifically, we assume that the entries of $\theta^*$ are i.i.d. with prior distribution $\cN(0, \lambda)$ for some fixed  $\lambda > 0$.
Based on the $n$ examples in the \ac{cot} prompt  $\pt_{\mathrm{CoT}}(n)$, 
we define $V^n = (v(z_h^i))_{h=1, i=1}^{H,n} \in \RR^{d_v \times Hn}$ and $K^n = (k_h^i)_{h=1, i=1}^{H,n} \in\RR^{d_k \times Hn}$ and let $\phi (K^n)$ denote the $\RR^{d_{\phi} \times H  n}$ feature matrix induced by $K^n$. 
Under the simplified model, the inputs $\{z_0^{i}\}_{i\in [n]}$ and $z_0^{\textrm{test}}$ do not contain information about $\theta^*$. 
Thus, conditioning on   $\pt_{\mathrm{CoT}}(n)$, the posterior distribution of $\theta^*$ is a Gaussian distribution, centered at the ridge estimator 
\begin{align*}
    \bar \theta^n = V^n \phi(K^n)^\top \big(\phi(K^n)\phi(K^n) ^\top  +\sigma^2/\lambda \cdot I    \big)^{-1},
\end{align*}
where $I $ is the identity matrix of size $\RR^{d_{\phi} \times d_{\phi} } $.

Given any $\theta \in \Theta$ as an estimate of $\theta^{*}$  and $z_0^\textrm{test}$, to predict $y^\textrm{test} = z_{H}^\textrm{test}$ according to the linear model in  \eqref{eq: linear model},
it suffices to generate $\{ v_h^\mathrm{test} = v(z_h^\mathrm{test}) \}_{h\in[H]}$ autoregressively. 
Specifically, for any $h  \geq 0$, conditioning on $\{ z_{0}^\mathrm{test}, \ldots, z_{h-1}^\mathrm{test}\}$, the   distribution of $v_{h}^\mathrm{test} $ is  $\cN( \theta \phi (k_h^\mathrm{test}), \sigma^2 I) $ where  we define $k_{h}^\mathrm{test}$ as 
\begin{align}\label{eq:define_query}
k_{h}^\mathrm{test}=\big(k(z_0^\mathrm{test}), k(z_1^\mathrm{test})\cdots, k(z_{h-1}^\mathrm{test}),0\cdots,0 \big)   . 
\end{align}
Therefore, to get the \ac{bma} estimator, we aggregate the distribution of $v_h^\mathrm{test}$ according to the posterior distribution of $\theta$, and return the mean value as the predictor, which is given by 
\begin{align}\label{eq:bma_estimator}
\bar v_h^\mathrm{test} = V^n \phi(K^n)^\top \big(\phi(K^n) \phi(K^n)^\top+  \sigma^2/\lambda \cdot I  \big)^{-1}  \phi(k_{h}^\mathrm{test}). 
\end{align}
The final \ac{bma} estimator is given by $\{ v^{-1} ( \bar v_h^\mathrm{test})\}_{h\in [ H]} $.

\vspace{2mm}
{\noindent \bf Estimator Produced by Transformer.} In the following, we introduce another autoregressive estimator based on a transformer with softmax attention. 
Transformer is a mapping that maps a sequence of vectors to another sequence of vectors and the mapping involves three components.
In particular, we pack the $n$ examples in the prompt as a sequence of $Hn$ vectors, followed by the test instance.

Our transformer is a composition of a copy head,  a softmax attention layer, and a position-wise transformation. 
The copy head takes the original sequence as input, and copies the previous reasoning steps within the same example at each position. 
Specifically, for any 
example $i$ and any $h$, 
the output of the copy head is $(z_0^i, \ldots, z_{h}^i, 0, \ldots, 0)$, where $H-h+1$ zeros so that the output vectors have the same dimension. The same operation is done for the test instance. 
Such a copy head can be explicitly constructed in theory using standard transformer architectures \citep{feng2023towards} and is also shown to emerge in various empirical works \citep{olsson2022context, von2023transformers}. 

The output of the copy head is then passed to a standard softmax attention layer, which involves the construction of keys, queries, values, and their calculations. 
Moreover, these three quantities are defined for each position. 
Specifically, for each $(i, h)$, we define both the key and query as $k_h^i$ in \eqref{eq: linear model}, and the value as $v_h^i = v(z_h^i)$, where $v$ is the feature map appearing in \eqref{eq: linear model}.  
Moreover, for the test example, 
for each $h \geq 0$, we define the query as $q_h^\mathrm{test} = \phi(k_h^\mathrm{test})$, which is used to attend to the keys of the $n$ examples, aggregate the corresponding values, and get the output. 
More concretely, we define the attention output given $\pt_{\mathrm{CoT}}(n)$ and $h-1$ intermediate outputs $z_1^{\mathrm{test}}, \ldots, z_{h-1} ^{\mathrm{test}}$, the output of the softmax attention is given by 
\begin{align} \label{eq:attention}
\mathtt{attn} (q_h^\mathrm{test}, \mathtt{keys} , \mathtt{values}) = \sum_{i \in [n], \ell   \in [H]} \frac{   \exp( \langle q_h^\mathrm{test} , k_\ell ^i\rangle ) \cdot v_{\ell}^i  } {    \sum_{i '\in [n], \ell'   \in [H]}   \exp( \langle q_h^\mathrm{test} , k_{\ell'} ^{i'}\rangle )}.
\end{align}
Finally, the output is passed through a transformation function $v^{-1} $, which yields 
$$
z_{h}^\mathrm{test} = v^{-1} ( \mathtt{attn} (q_h^\mathrm{test}, \mathtt{keys} , \mathtt{values})  ) , \qquad \forall h \in [H].
$$
This newly generated $z_{h}^\mathrm{test}$ is then used to compute the query $q_{h+1}^{\mathrm{test}}$, which is then used for generating $z_{h+1}^{\textrm{test}}$, and so on. 
See Figure  \ref{fig:interpret_attn} for an illustration of this transformer.

\begin{figure}[h]
    \centering
    \includegraphics[width = 0.94\textwidth]{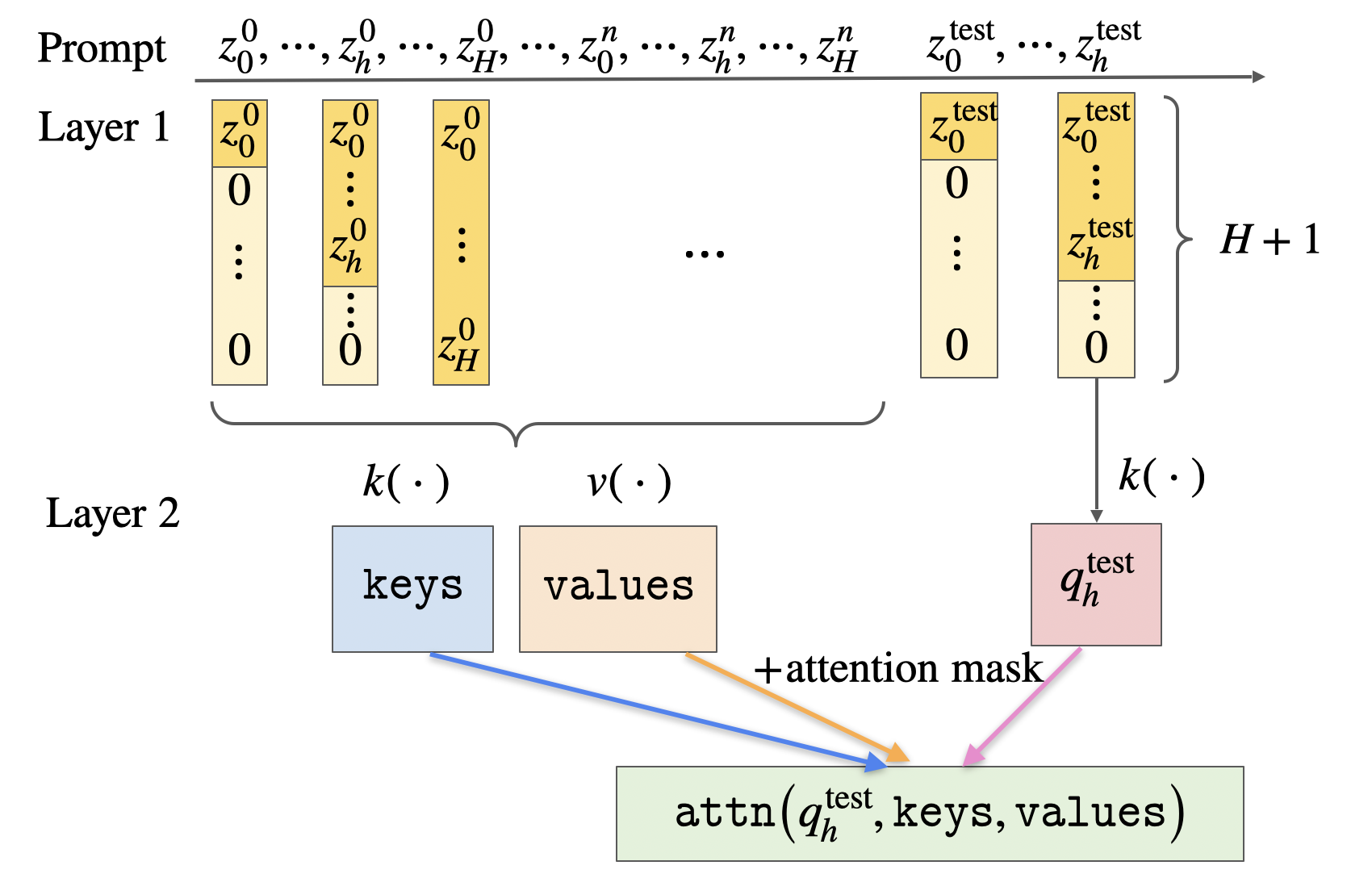}
    \caption{An illustration of how \ac{cot} is represented by a two-layer transformer. Layer 1 serves as a copy head, which copies the previous steps $\{z_j^i\}_{j=1}^{h-1}$ to the current position $z_h^i$. Next, the feature mappings $v$ and $k$ map the outputs of Layer 1  to values and keys, respectively. During the generation of $z_{h+1}^\mathrm{test}$, the attention mechanism takes in key and value matrices $\mathtt{keys}$ and $\mathtt{values}$ from the demonstrations to predict the result for query $q_h^\mathrm{test}$, where $\mathtt{keys}$ and $\mathtt{values}$ are formed by stacking $\{k_h^i\}_{i=1,h=1}^{n, H}$ and  $\{v_h^i\}_{i=1,h=1}^{n, H}$, respectively. Note that $\mathtt{keys}$ and $\mathtt{values}$ do not contain keys and values computed from the generated reasoning steps $\{z_j^\mathrm{test}\}_{j=1}^h$. We can achieve this by masking out the corresponding positions. }
    \label{fig:interpret_attn}
\end{figure}

In the following proposition, we prove that under certain conditions, the \ac{bma} estimator coincides with the transformer output up to a scaling factor when $n$ goes to infinity.

\begin{proposition}\label{prop: attn}
We assume the feature mappings $k$ and $v$ take bounded values  and  $\|v(z)\|_2 = 1$ for all input $z \in \cL$. 
Besides, let 
 $\phi$ in \eqref{eq: linear model} be a feature map with finite dimension. 
Then, there exists an absolute constant $C$, and parameter $\lambda=n^{-2/3}$ such that for any fixed $k_{h}^{\text{test}}$, the \ac{bma} estimator in \eqref{eq:bma_estimator} and the attention output in \eqref{eq:attention} coincide as $n$ goes to infinity up to a scaling factor. 
That is, we have 
$$
\lim_{n \to \infty} \max_{h\in[H]} \| \bar v_h^\mathrm{test}  - C\cdot\Att(q_h^{\mathrm{test}}, \mathtt{keys}, \mathtt{values}) \|_2  = 0.
$$
\end{proposition}
This proposition shows that there exists a special model satisfying \eqref{eq:latent_var_model} (the model in \eqref{eq: linear model}) such that the \ac{bma} estimator of this model can be approximately implemented by a transformer. 
Furthermore, 
to see why such a proposition can be expected, 
we note that the attention output in \eqref{eq:attention} corresponds to the Nadaraya–Watson Kernel regressor \citep{hastie2009elements} with an exponential kernel, where we regress the value on the key, and predict on the query.
Moreover, the \ac{bma} estimator in \eqref{eq:bma_estimator} corresponds to ridge regression. 
These two estimators are both consistent and thus converge to the same thing as $n$ goes to infinity. 
Furthermore, we will provide a detailed proof in Appendix~\ref{proof: prop attn}.  This proof generalizes a similar result in \cite{zhang2023and} for vanilla \ac{icl} by handling the multi-step autoregressive structure of \ac{cot} prompting. Note that we assume that the feature mappings \( k \) and \( v \) take reasoning steps as input, making the \(\mathtt{keys}\) and \(\mathtt{values}\) specific to reasoning steps rather than tokens. However, this can be easily generalized to tokens.

\section{Statistical Errors of CoT Prompting} \label{subsec: perfect}


In this section, we study the error incurred during the prompting stage. 
We first state an error decomposition result and then study the vanilla \ac{cot} prompting in Section~\ref{subsec: Statistical Rates of Vanilla cot}. 
Then we extend the theory to three variants of \ac{cot} in Section~\ref{subsubsec: variants} and compare \ac{cot} with vanilla \ac{icl} in Section~\ref{subsubsec: comparison}. 
Regarding the four questions raised in the introduction, this section answers Question (b) partly and Question (d).

\subsection{Statistical Errors of Vanilla CoT } \label{subsec: Statistical Rates of Vanilla cot}

Recall that we define the statistical error induced by the \ac{cot} prompting $\mathtt{err}_\mathrm{CoT}$ in equation~\eqref{eq: err cot} and the error comes from both pertaining and prompting stages, as listed in table~\ref{table:error_sources}.
We explicitly decompose these two error sources as follows. To this end, we first state a regularity condition for the pretrained $\ac{llm}$. 
Before we proceed, let us define the following partial prompt $\pt^h_{\mathrm{CoT}}(i) $.
For any integers $i \in [0, n-1]$ and $h \in [H]$, we let 
$\pt^h_{\mathrm{CoT}}(i) = \{ s^j \}_{j\leq i} \cup \{ z_{0}^{i+1}, \ldots, z_{h-1}^{i+1}\}$. 
That is,  $\pt^h_{\mathrm{CoT}}(i)$ contains the first $i$ demonstration examples and the first $h$ steps of the $(i+1)$-th example. Let \(\theta^*\) denote the target task. In Section~\ref{subsec: perfect}, we assume that the prompt is generated from the ground truth distribution, meaning that \(\mathtt{prompt}_\mathrm{CoT}(n) \sim \PP(\cdot \mid \theta^*)\).


\begin{assumption} \label{assump: density bd}

We assume there exists  a positive number  $b^*$  such that for any  $0\leq h\leq H$, $0\leq i\leq n-1$, and $\pt^h_{\mathrm{CoT}}(i)\in\cL^{*}$, we have for the data distribution $\PP$ and the pretrained model $\PP_{\mathrm{LLM}}$ that
\begin{align*}
  \sup_{z \in \cL }   \Bigl|\log \PP(z^{i+1}_h = z  \given \pt^h_{\mathrm{CoT}}(i))  -\log \PP_\mathrm{LLM}(z^{i+1}_h = z \given \pt^h_{\mathrm{CoT}}(i)) \Big| \leq b^*. 
\end{align*}
\end{assumption}

This assumption 
postulates that the true distribution $\PP$  of the model in \eqref{eq:latent_var_model} and that learned by the \ac{llm} are close. 
The proximity is measured in terms of the log likelihood. 
We will justify the existence of $b^*$  in Section~\ref{subsec: imperfect} under explicit assumptions on pretraining.

\begin{lemma}[\ac{cot} Error Decomposition]\label{lem:error_decomp}

    Under Assumption \ref{assump: density bd}, the statistical error $\mathtt{err}_\mathrm{CoT}$ in \eqref{eq: err cot} can be upper bounded by the sum of a pretraining error and a  prompting error, i.e., 
\begin{align*}
\mathtt{err}_\mathrm{CoT} \leq \mathtt{err}_\mathrm{pre}(\PP,\PP_{\hat \rho};\pt_\mathrm{CoT}(n)) + \mathtt{err}_\mathrm{prompt}(\PP,\theta^*,\pt_\mathrm{CoT}(n)).
\end{align*}
where we define the prompting error as
\begin{align}
&\mathtt{err}_\mathrm{prompt}(\PP,\theta^*,\pt_\mathrm{CoT}(n)) \nonumber \\
&\quad=  \KL\bigl(\PP(y^{\mathrm{test}}= \cdot \given z_0^{\mathrm{test}} ,\theta^*) , \PP(y^{\mathrm{test}}= \cdot \given \pt_{\mathrm{CoT}}(n)) \bigr) \nonumber\\
&\quad\qquad + 2\sqrt{2}Hb^* \cdot   \KL^{1/2}  \bigl(\PP(y^{\mathrm{test}}= \cdot \given z_0^{\mathrm{test}} ,\theta^*), \PP(y^{\mathrm{test}}= \cdot \given \pt_{\mathrm{CoT}}(n))\big) , \label{eq: err pt}
    \end{align}
and the pretraining error as
\begin{align}
\mathtt{err}_\mathrm{pre}(\PP,\PP_{\hat \rho};\pt_\mathrm{CoT}(n))&=\KL\bigl(\PP(y^{\mathrm{test}}= \cdot \given \pt_{\mathrm{CoT}}(n)) , \PP_{\hat \rho}(y^{\mathrm{test}}= \cdot \given \pt_{\mathrm{CoT}}(n)) \bigr). \label{eq: err pre}
    \end{align}
\end{lemma}

We provide a detailed proof of this lemma in Appendix~\ref{proof: err cot}. 
Note that the prompting error is defined on the distribution $\PP$, which corresponds to the perfectly pretrained \ac{llm}, and thus this error is independent of the pretraining of \ac{llm}s. 
In the following, we focus solely on the prompting error. 
Moreover, we assume that $z_0^{\mathrm{test}}$ has the same distribution as $\{ z_0^i \}_{i\in [n]}$ for simplicity, i.e., the query error is zero. We will allow a distributional shift in the next section. 

With no distributional shift in $z_0^\mathrm{test}$, we can essentially regard the test instance as the $(n+1)$-th example, since all the examples are conditionally i.i.d. when they are conditioned on $\theta^*$. 
Thus, in this section, we will only study how the $n$ prompt examples help a perfectly pretrained \ac{llm} infer $\theta^*$, namely, the in-context error.

\vspace{2mm}

{\noindent \bf Equivalence Class Induced by Multi-Step Reasoning.} 
 In the following, to simplify the notation, we use $X = Z_0, Z_1, \ldots, Z_{H} = Y$ to denote a random trajectory sampled from the model in \eqref{eq:latent_var_model}. Note that the prompting error in \eqref{eq: err pt} only concerns the distribution of the output $Y$ and neglects the intermediate reasoning steps $Z_{1}, \ldots, Z_{H-1}$. 
As a result, it is possible that there exists another $\theta \in \Theta$ with the same distribution of $Y$. Such a relationship induces a set of equivalence classes over $\Theta$.

\begin{definition}[Equivalence Classes over $\Theta$] \label{def: eqv class}
Let $\PP(Z_0, Z_1,\cdots, Z_{H-1},Y\given \theta)$ denote the joint distribution of $Z_{0:H}$ conditioning on the latent variable $\theta^* = \theta$. 
We define an equivalence relation $\sim$  based on conditional density of $Y$ given $ Z_0$ as follows.
\begin{align*}
    \theta \sim \theta' \text{ if and only if}~ \PP(Y=y \given Z_0=z_0,\theta)=\PP(Y=y\given Z_0=z_0,\theta'), \quad \forall 
(z_0,y).
\end{align*}
This relation $\sim$ induces a set of  
equivalence classes over $\Theta$. 
In particular, 
for any $\theta$, define 
$\Theta_{\mathrm{eq}}(\theta)=\{\theta' \in \Theta:   \space \PP(y \given z_0,\theta)=\PP(y\given z_0,\theta'), \forall 
 (z_0,y)\}$
 as the set of parameters equivalent to $\theta$, i.e., the equivalence class represented by $\theta$. 
 Let $\tilde \Theta $  denote the complete set of representatives of all \emph{disjoint} equivalent classes. 
 Then  $\Theta_{\mathrm{eq}}(\theta) \cap \Theta_{\mathrm{eq}}(\theta') = \emptyset$ for all $\theta, \theta'\in \tilde \Theta$ and  
we can further write $\Theta $ as $\cup_{\theta \in \tilde \Theta} \Theta_{\mathrm{eq}}(\theta)$.

\end{definition}


The intuition of the equivalence relation $\sim$ is that there might be multiple reasoning paths that all lead to the correct answer. 
For example, Newtonian, Lagrangian, and Hamiltonian mechanics are three different approaches to classical mechanics. 
Their intermediate steps are different but will lead to the same answer. 
Based on this intuition, any parameter in $\Theta_{\mathrm{eq}}(\theta^*)$ is equally good for predicting $Y$, and we only need to infer $\Theta_{\mathrm{eq}}(\theta^*)$ from \ac{cot} prompts. 

We state a regularity condition for $\ac{cot}$ prompting in terms of $\Theta_{\mathrm{eq}}(\theta^*)$.


\begin{assumption}\label{assumption: seperation}
Given a  task $\theta^*$ during \ac{cot} prompting, let $\Theta^{\complement}=\Theta \backslash \Theta_{\mathrm{eq}}(\theta^*)$ denote the complement of the equivalence class of $\theta^*$. 
We assume that there exists a strict separation between the ground truth task $\theta^*$ and any other tasks $\theta \in \Theta^\complement$. 
Specifically,  there exists $\lambda > 0$ that lower bounds the  Hellinger distance:
\begin{align*}
        \inf_{\theta \in  \Theta^{\complement}}\text{H}^2\big(\PP(Z_{0:H}  = \cdot  
        \given \theta^*) , \PP(Z_{0:H}  &= \cdot  \given \theta)\big)\geq \lambda,
    \end{align*}
     where $\text{H}^2(\cdot , \cdot)$ denotes the squared Hellinger distance.
Moreover, we assume tasks in $\Theta_\mathrm{eq}(\theta^*)$ are well covered  by the pretraining distribution in the sense that  
$\pi\big(\Theta_\mathrm{eq}(\theta^*)\big)>0$.   
\end{assumption}

This assumption requires the true task 
$\theta^*$ is $\lambda$-separated from any other $\theta$ outside of the equivalence class $ \Theta_{\mathrm{eq}}(\theta^*)$. 
Parameter $\lambda$ serves as a margin of separation. 
Moreover, we assume that the prior $\pi$ put  considerable density on $\Theta_{\mathrm{eq}}(\theta^*)$. 
This means that the task tested during prompting stage has been covered in the pretraining dataset. 
Based on this assumption, we establish the statistical error of the \ac{cot} estimator as follows.

\begin{theorem}\label{thm: rate_cot}

Let $\mathtt{err}_\mathrm{pre}$ denote the pretraining error defined in \eqref{eq: err pre} and assume $\Theta$ to be a finite and discrete set. 
Under 
Assumptions~\ref{assump: density bd} and \ref{assumption: seperation}, 
with probability $1-\delta$, the statistical error $\mathtt{err}_\mathrm{CoT}$ defined in~\eqref{eq: err cot} satisfy
$$\mathtt{err}_\mathrm{CoT}\leq\mathcal{O}\big(Hb^*\delta^{-1} 
\cdot \pi(\theta^*)^{-1/2}\cdot \big|{\Theta^\complement}\big|\cdot  e^{-\lambda n}\big)+\mathtt{err}_\mathrm{pre}.
        $$ 
Here $\cO(\cdot)$ hides absolute constants and the probability is with respect to the randomness of the \ac{cot} prompt.

\end{theorem}

This theorem shows that when the tasks are well separated,
the prompting error converges to zero 
\textbf{exponentially  fast} when $n$ increases. 
Note that the convergence rate depends on the separation $\lambda$. 
A larger $\lambda $ means that $\Theta_\mathrm{eq}(\theta^*)$ is more distinguishable from the rest of the tasks, leading to a faster convergence rate. 
Besides, the error also depends on $H$ and the size of ${\Theta^\complement}$. Intuitively, these two terms characterize how the error increases as the problem size grows. 
Moreover, the dependence on $b^* $ comes from Lemma \ref{lem:error_decomp}, which is due to replacing the pretrained \ac{llm} by the population distribution. 
The statistical error also depends on $\pi(\theta^*)^{-1/2}$, which means the prompting error is smaller for tasks better covered by the pretraining distribution. 

Here we assume that $\Theta$ is finite. In Appendix \ref{proof: rate_cot} we will further extend Theorem \ref{thm: rate_cot} to the more challenging case where $\Theta$ can be a continuous set and present a detailed proof.




Ideally, we would like to have an upper bound that scales with $|\tilde \Theta|$ because it is the actual number of all possible hypotheses when it comes to predicting $Y$ based on $X$.
Whereas $|{\Theta^\complement}|$ in Theorem \ref{thm: rate_cot} can be much larger than $|\tilde \Theta| $ because it is comparable to $|\Theta|$.  
To have a better bound, we impose an additional assumption postulating that the distributions within each equivalence class are close.


\begin{assumption}\label{assumption: close}
Let $\tilde \Theta$ be a representative set of the equivalence classes introduced in Definition \ref{def: eqv class}. 
We assume that there exist   positive numbers $\alpha$ and $\alpha_0$  such that for all $\theta\in \Tilde \Theta$ and $ \theta' \in \Theta_\mathrm{eq}(\theta)$, we have 
\begin{align*}
    \sup_{z_{0:H} }\bigg|\log \frac{\PP(Z_{0:H} = z_{0:H}  \given \theta)}{\PP(Z_{0:H} = z_{0:H}  \given \theta')}\bigg|\leq \alpha, \qquad \sup_{z_{0} }\bigg|\log \frac{\PP(Z_{0} = z_{0}  \given \theta)}{\PP(Z_{0 } = z_{0}  \given \theta')}\bigg|\leq \alpha_0.
\end{align*}
Moreover, we assume that $\alpha \in (0, \lambda)$, where 
 $\lambda$ appears in Assumption~\ref{assumption: seperation}. 
\end{assumption}

 Assumptions \ref{assumption: seperation} and \ref{assumption: close}  imply that distributions are similar within each equivalence class but disparate between equivalence classes. 
 We establish a new upper bound as follows.


\begin{theorem} \label{thm: rate_cot2}
 
Under  Assumptions \ref{assump: density bd}, \ref{assumption: seperation}, and \ref{assumption: close}, with probability $1-\delta$ over the randomness of the \ac{cot} prompt, 
    we have $$\mathtt{err}_\mathrm{CoT}\leq\mathcal{O}\big(Hb^*\pi(\theta^*)^{-1/2}\delta^{-1} \big|\Tilde{\Theta}\big| e^{-(\lambda-\alpha) n+\alpha_0}\big)+\mathtt{err}_\mathrm{pre}.$$


\end{theorem}

Compared with the previous Theorem~\ref{thm: rate_cot} that solely requires Assumption~\ref{assumption: seperation}, we have a better dependency on the size of parameter space, from $|\Theta^\complement|$ to $|\tilde \Theta|$, at a cost of a slower rate of exponential decay. 
The proof of this theorem is deferred to Appendix \ref{proof: rate_cot2}, where we also include an extension to the case where $\Theta$ is continuous.

In summary, we have shown that in the prompting stage, as the number of examples grows,  the statistical error of \ac{cot} prompting decays exponentially to an intrinsic error due to pretraining.
In the following, we will extend the above results to a few variants of \ac{cot}.

\subsection{Statistical Errors of Variants of CoT }\label{subsubsec: variants}
The predictions of \ac{llm}s are inherently stochastic, which is a main source of LLM hallucination 
\citep{huang2023survey, tonmoy2024comprehensive}. 
To increase the prediction accuracy, various selection techniques such as majority vote \citep{wang2022self} and tree search \citep{yao2023tree} are combined with \ac{cot}. 
In the following, we modify Theorem \ref{thm: rate_cot} for a few variants of \ac{cot}, including Self-Consistency \ac{cot} \citep{wang2022self}, Tree-of-Thought  \citep{yao2023tree}, and Selection-Inference \citep{creswell2022selection}. 
For simplicity, we also assume zero pretraining error and input query does not have a distributional shift, i.e., \(\PP_\mathrm{LLM} =  \PP\) and \(z_0^\mathrm{test} \sim \PP(\cdot \given \theta^*)\).


\subsection*{Self-Consistency \ac{cot} (SC-COT)} \label{subsec:sc}

Given the same prompt as in vanilla \ac{cot}, i.e., $\pt_{\mathrm{CoT}}(n)$, \ac{sc}-\ac{cot} first generate $K$ i.i.d. reasoning paths and then output the final answer by a \emph{majority vote}. 
That is,  we first 
  sample $K$ i.i.d. reasoning paths $\{z_{1:H}^{\mathrm{test},i} \}_{i=1}^K \sim \PP \big(\cdot \given \pt_{\mathrm{CoT}}(n)\big)$ and then report the mode of the empirical distribution of $\{y^{\mathrm{test},i}\}_{i=1}^K$, denoted by $y_K^*$. 
 The empirical distribution of $\{y^{\mathrm{test},i}\}_{i=1}^K $ is denoted by  $p_K(y) = K^{-1}  \sum_{i=1}^K \mathbf{1}\{y^{\mathrm{test},i} = y\}$, $\forall y \in \calL$. 
 The sample mode \(y_K^*\) is defined as $y_K^* = \argmax _{y\in \calL} p_K(y)$, where we pick any element if there are multiple maximizers. See Figure \ref{fig:sc-cot} for an illustration.

Recall that if the underlying task $\theta^*$ is already known, the answers should be generated according to $\PP(y^\mathrm{test} = \cdot  \given z_0^\mathrm{test}, \theta^*) $. We assume the desired answer is the mode of this distribution and it is unique.

\begin{assumption}\label{assump: sc}
We define $y^* $ as the mode of the distribution $\PP(y^\mathrm{test} = \cdot  \given z_0^\mathrm{test}, \theta^*) $, i.e., 
$y^* = \argmax_{y \in \calL} \PP(y^\mathrm{test} = y  \given z_0^\mathrm{test}, \theta^*) $
Moreover, we define the gap between the mode and the second-largest probability mass as   
\begin{align*}
\epsilon = \min_{y \in \calL, y \neq y^*}\{\PP \big(y^\mathrm{test} = y^* \given z_0^{\mathrm{test}}, \theta^*\big) - \PP \big(y^\mathrm{test} = y \given z_0^{\mathrm{test}}, \theta^*\big)\},
\end{align*}
which is assumed to be strictly positive. 
\end{assumption}

This assumption ensures that the population mode $y^*$ is uniquely defined with a margin $\epsilon$. 
This condition is satisfied by reasoning problems where the answer is unique, e.g., factual commonsense and mathematical reasoning. 
Intuitively, when the number of examples in \ac{cot} prompt is large, 
$\PP( y^{\mathrm{test} } =\cdot  \given  \pt_{\mathrm{CoT}}(n) ) $ is close to $\PP(y^\mathrm{test} = \cdot  \given z_0^\mathrm{test}, \theta^*)$, as guaranteed by Theorem \ref{thm: rate_cot}. 
Then, when $K$ is sufficiently large in  \ac{sc}-\ac{cot}, we expect that the sample mode $y_K^*$ coincides with the population mode $y^*$. This justifies the effectiveness of \ac{sc}-\ac{cot}.


\begin{figure}[t] 
    \centering
    \includegraphics[width = 0.94\textwidth]{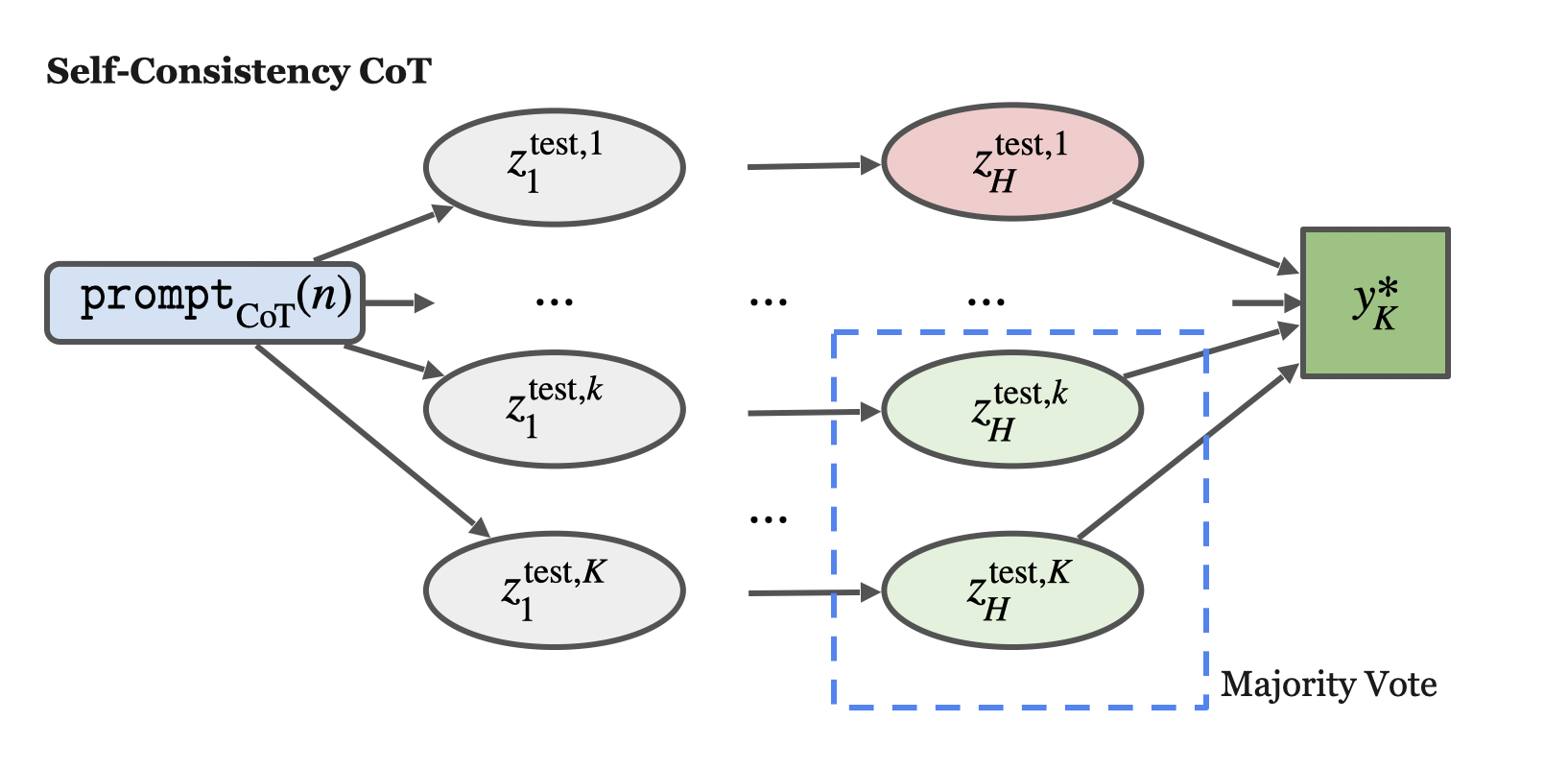}
    \caption{An illustration of the \ac{sc}-\ac{cot} prompting method.
    This method creates the final answer $y_K^*$ based on two steps. First, we sample $K$ i.i.d. 
    reasoning paths $\{ z_{0:H}^{\mathrm{test},i} \}_{i=1}^K$ given the \ac{cot} prompt, and then 
   report $y_K^*$ by a  majority vote  based on $\{ y^{\mathrm{test},i } = z_{H}^{\mathrm{test},i}  \}_{i=1}^K$. }
    \label{fig:sc-cot}
\end{figure}

\begin{corollary}[Statistical Error of \ac{sc}-\ac{cot}] \label{cor: self-consistency cot}

Consider \ac{sc}-\ac{cot} prompting with $K$ reasoning paths and $n$ \ac{cot} examples.
Under  Assumptions~\ref{assump: density bd}, \ref{assumption: seperation}, and~\ref{assump: sc}, 
when $n$ is sufficiently large such that 
$$
n = \Omega\Big(    \Big( \log \bigl( |\Theta^\complement|  / \pi(\theta^*)\big) + \log (1/ \epsilon  )  \Big)\Big/ \lambda \Big), 
$$
 with probability at least $1-e^{-\lambda n/2}$, the probability that the \ac{sc}-\ac{cot} produces the wrong output decreases exponentially in $K$, i.e., 
$$\PP\big(y_K^*\neq y^*\given \pt_{\mathrm{CoT}}(n)\big)\leq 2 |\calL| \cdot \exp \bigg(-\frac{3K \epsilon^2  }{24+ 8 \epsilon}\bigg).$$
\end{corollary}

This corollary shows that sampling $K$ independent reasoning paths boosts the output accuracy. 
In particular, when $\epsilon \in (0,1)$, for any $\delta \in (0,1)$, as long as $K = \Omega( \log (|\calL| / \delta) / \epsilon )$, 
$y_K^* = y^*$ holds  with 
 probability at least $1- e^{-\lambda n /2} - \delta$. The proof of this corollary can be found in Appendix~\ref{proof: self-consistency cot}.


\subsection*{Tree-of-Thought (ToT)} \label{subsec:tot}

Recall that \ac{sc}-\ac{cot} samples multiple parallel reasoning paths and performs a selection in the last step. 
Tree-of-Thought  \citep{yao2023tree}
instead proposes to include selection in each step. 
In this setup, the goal is to generate a reasoning path $z_{1:H}^\mathrm{test}$ that maximizes a task-specific value function $V_{\theta}^*$. We define this population problem as follows.

{\noindent \bf Population Problem.} 
The goal of ToT is to select the optimal reasoning path that solves a desired task. 
Mathematically, for each step $h$, let $t_h = (z_0, \ldots, z_{h})$ denote the partial history up to step $h$. 
Let $V_{\theta^*}$ be a function that maps each partial history to a value in $[0,1]$. 
Intuitively, $V_{\theta^*}$ can be viewed as the success probability of the partial history for solving task $\theta^*$. 
Starting from $z_0^\mathrm{test}$, the optimal reasoning path is obtained by solving 
\begin{align}\label{eq:optimal_path_tot}
t_h^{\mathrm{test},*} = (t_{h-1}^{\mathrm{test},*}, z_h^{\mathrm{test},*})  , \qquad \textrm{where}~~ z_h^{\mathrm{test},*} = {\textstyle \argmax_{z_h^\mathrm{test}} } V_{\theta^*} (t_{h-1}^{\mathrm{test}, *} , z_h^\mathrm{test}), ~~t_0^{\mathrm{test}, *} = z_0^\mathrm{test}.   
\end{align}
Moreover, let 
$\PP(z_{0:H}^\mathrm{test} = \cdot \given \theta^*)$ be the task-specific distribution of the multi-step latent variable model defined in \eqref{eq:latent_var_model}. 
At the population level, the goal is to draw samples from such a distribution, and select the optimal reasoning path according to the value function $V_{\theta^*}$. 
In the following, we condition on $\pt_{\mathrm{CoT}}(n)$, and thus the optimal reasoning path $z_{0:H}^{\mathrm{test}, *}$ can be regarded fixed. 
\vspace{2mm}

{\noindent \bf Tree-of-Thought with Breadth-First-Search.} As we do not have access to the distribution $\PP(z_{0:H}^\mathrm{test} = \cdot \given \theta^*)$, ToT proposes to sample from the LLM and then approximately solve \eqref{eq:optimal_path_tot} via selection. 
To simplify the notation, for each 
 $h\in [H]$, we denote \( t^\mathrm{test}_{h-1} = (z_0^{\mathrm{test}}, \ldots, z_{h-1}^{\mathrm{test}}) \), which is the partial history of the test example up to step $h-1$.
In step $h$, instead  
of passing the complete prompt $\pt_\mathrm{CoT}(n)$, we truncate each demonstration in $\pt_\mathrm{CoT}(n)$ up to step $h$ and denote the truncated prompt by $\pt_h(n) = \{z^i_{0:h} \given z^i_{0:H} \in \pt_\mathrm{CoT}(n)\}$.
Then the \ac{llm} samples $z_h^\mathrm{test}   \sim \PP(\cdot \mid \pt_h(n), t^\mathrm{test}_{h-1}) $ and obtain $t_h^\mathrm{test}$, and so on. 

In the sequel, we only discuss a version of ToT that maintains a candidate set of partial histories $\cT_h $ for each step $h$, constructed using Breadth-First-Search (BFS). 
Specifically, the algorithm involves two integer parameters, $K$ and $B$, which specify the number of samples drawn in each step and the size of each $\cT_h$, respectively. 
Let $\cT_{0} = \{ z_0^{\mathrm{test}}\}$. 
Suppose $\cT_{h}$ is already constructed and $|\cT_h| = B$. Let its elements be denoted by $\{ t_h^1, \ldots, t_h^{B}\}$. 
For any $b\in [B]$, the algorithm will include both $t_h^b$ and $\pt_{h+1}(n)$ as the prompt sequence and do $K$ i.i.d. one-step reasoning with the perfectly trained LLM, i.e., 
$\{z_{h+1}^{b,i}\}_{i=1}^K \overset{i.i.d.}{\sim} \PP(  \cdot \mid \pt_{h+1}(n), t^b_{h})$.
Thus, we obtain $KB$ partial histories for the $(h+1)$-th step:
$\{ (t^b_{h}, z_{h+1}^{b,i}) \colon i \in [K], b\in [B]\}$. 
Then we sort these partial histories according to the value function $V_{\theta^*}$, and define $\cT_{h+1}$ as the top $B$ elements. 
That is, 
\begin{align}\label{eq:define_candiates_tree_search}
\cT_{h+1} = \{ t_{h+1}^b \}_{b\in[B]} = \{ \mathrm{top }~B~(t^b_{h}, z_{h+1}^{b,i})\textrm{'s} ~\mathrm{in~terms~of}~ V_{\theta^*}(t^b_{h}, z_{h+1}^{b,i}), i \in [K], b\in [B]\}.
\end{align}
Finally, when $\cT_H$ is constructed, 
we define $\hat t_{H} = \argmax _{t_H\in \cT_{H}} V_{\theta^*} (t_H)$
and use $\hat t_H$ as the final prediction. See Figure \ref{fig:tot} for an illustration.

\begin{figure}[t]
    \centering
    \includegraphics[width = 0.9\textwidth]{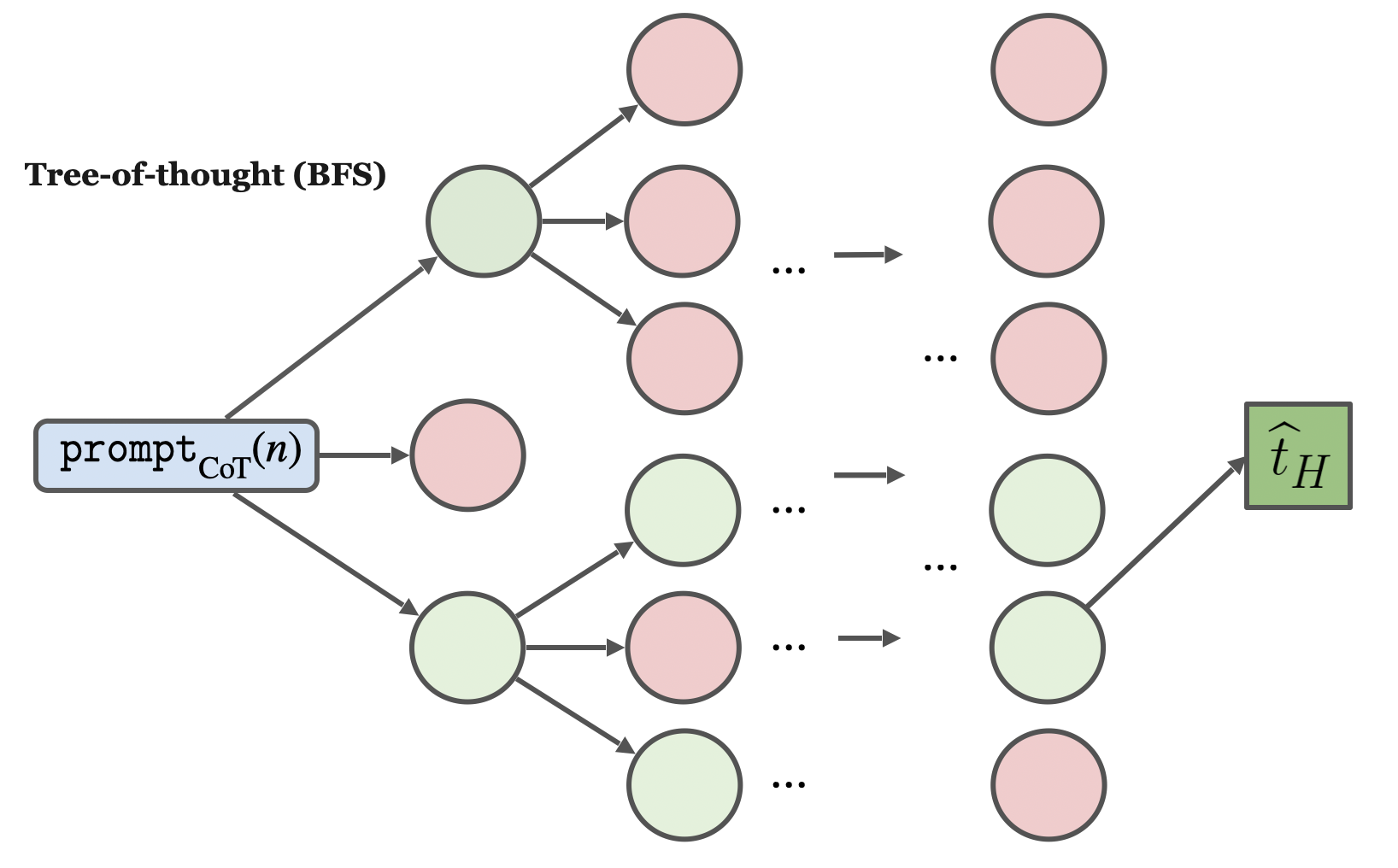}
    \caption{An illustration of \ac{tot} with BFS selection. 
    Here BFS samples $K= 3$ i.i.d. reasoning steps based on the prompt and selects $B = 2$ partial histories. 
    The red ovals represent the pruned nodes, and the green ovals represent the active nodes.
    At each step \(h \in [H]\), \(K\) candidates are generated per active node, and thus there are $KB$ candidates. 
    Then $B$ nodes survive the selection according to the value function \(V_{\theta^*}(\cdot)\). At the final step $h=H$, we select a single best candidate and output the corresponding chain as $\hat t_H$.}
    \label{fig:tot}
\end{figure}

Compared to \ac{sc}-\ac{cot}, ToT uses a more sophisticated selection method based on the value function. 
For this to be effective, $V_{\theta^*} $ need to sign well with the task-specific distribution $\PP(\cdot \given \theta^*)$. Recall that for each $h$ we do one-step reasoning for $K$ times and only keep a candidate set of the histories. To ensure that the desired reasoning path is contained in the candidate set, we require that the optimal one-step reasoning $z_h^{\mathrm{test}, *}$ can be sampled out with high probability for each $h\in[H]$, which in turn requires sufficient coverage of $t_h^{\mathrm{test}, *}$ under $\PP(\cdot \given \theta^*)$ which is approximated by the LLM.
For example, $t_h^{\mathrm{test}, *}$ should have sufficient probability under $\PP(\cdot \given \theta^*)$, and different tasks should have sufficient separation. 
We impose the following assumption for theoretical analysis. 


\begin{assumption}\label{assump:tot-1}
For the given task $\theta^*$, we assume the optimal reasoning path $t_H^*$ is uniquely defined by  \eqref{eq:optimal_path_tot}. Moreover, we assume that the task $\theta^*$ is well covered by the pretraining distribution, i.e., $\pi(\theta^*)>0$.
Furthermore, we assume that the tasks in $\Theta$ are well separated such that the following two conditions are satisfied:
\begin{itemize}
    \item [(i)] Task $\theta^*$ is uniquely identified by the optimal reasoning path, i.e., $\theta^* = \argmax_{\theta' \in \Theta} \PP(t_h^{\mathrm{test},*} \mid  \theta = \theta')$ for each $h\in [H]$;
    \item [(ii)] For any $h\in [H]$, there exists $\lambda _h > 0$ such that $
        \text{H}^2 \big(\PP(Z_{0:h}=\cdot \given \theta),\PP(Z_{0:h}=\cdot \given \theta^*)\big)\geq \lambda_h.
        $ for all $\theta \neq \theta^*$, where $H(\cdot, \cdot)$ denotes the Hellinger distance. 
\end{itemize}

We note that here we require a separation in Hellinger distance for every step $h$, which is slightly stronger than Assumption \ref{assumption: seperation}. Moreover, condition (i) above shows that the equivalence classes specified in Definition \ref{def: eqv class} are in fact singletons. 
\end{assumption}


\begin{proposition}
    [Statistical Error of ToT] \label{cor: tot-bfs}
Consider \ac{tot}  prompting based on $n$ \ac{cot} examples and BFS with $B = 1$. 
Let $\hat t_H$ be the final output and define $\lambda^* = \min_{h\in [H]} \lambda_h$. 
Let $\epsilon \in (0,1)$ be any sufficiently small number. 
Under Assumption~\ref{assump:tot-1}, when $n$ is sufficiently large such that
\begin{align*}
    n\geq \frac{1}{\lambda^*}\bigg(2\log \big(H|\Theta| \big) + \log \big((1-\pi(\theta^*))/\pi(\theta^*)\big) +\log(1/\epsilon)\bigg),
\end{align*}
then with probability at least $1-e^{-n\lambda^*/2}$, the probability of outputting a suboptimal reasoning path $\hat t_H$ decreases exponentially with $K$. That is, we have 
\begin{align*}
    \PP\big(\hat t_H \neq t_H^{\mathrm{test}, *} \given \pt_\mathrm{CoT}(n)\big)&\leq \sum_{h=1}^H\Big(1-p_h^* +\epsilon p_h^*\Big)^{K},
\end{align*}
where $p_h^*=\PP(z_h^{\mathrm{test}, *} \given t_{h-1}^{\mathrm{test}, *},\theta^*)$ for each $h\in [H]$.

\end{proposition}

This proposition shows that ToT significantly reduces the probability of introducing a suboptimal optimal reasoning path, which decreases exponentially in \( K \). 
 Without the BFS-based selection step, even when $n$ goes to infinity, the probability of generating $t_H^{\mathrm{test}, *} $ is only  $\prod_{h\in[H]} p_h^*$. 
 Here we only focus on the simplest case where $B = 1$, but our analysis can be generalized to $B > 1$ with some additional effort. 
 We defer the proof to Appendix~\ref{proof: tot-bfs}.

\vspace{2mm}
\subsection*{Selection-Inference (SI)} \label{subsec:si}


Selection-Inference (SI) \citep{creswell2022selection} is a structured LLM reasoning method 
that decomposes each step of reasoning into two components --- a selection module that retrieves relevant facts from the context and an inference module that predicts the next step solely based on the selected facts. 
To this end, \ac{si} uses an \ac{llm} as both a selection module and an inference module through prompting.  The selection module extracts information from the reasoning path and the inference module predicts the next reasoning step based on the information extracted from the selection module. 

\vspace{2mm}
{\noindent \bf A Hierarchical Latent Variable Model.} In the context of SI,
we assume a special case of the model in \eqref{eq:latent_var_model} with a hierarchical structure. 
Specifically, we assume the latent variable $\theta^* $ has two component $ \theta^* = (\theta_\mathrm{se}^*, \theta_\mathrm{in}^*)$ and 
the examples of reasoning paths are i.i.d. given $\theta^*$, which has a prior distribution $\pi$.
Let $ \{z_0, \ldots, z_{H}\}$ be a reasoning path. 
We let $t_h = \{ z_0,\ldots, z_h\}  $
be the partial history up to step $h$. 
We assume that $ z_{h+1} $ depends on $t_h$ only through a subset of $t_h$, denoted by $\tau_{h+1}$, and $\tau_{h+1}\subseteq t_h $  is selected from $t_h$. 
Specifically, the joint distribution of $\PP(z_{0:H} \given \theta^*)$ is given by 
\begin{align}\label{eq:si_model}
    z_0 \sim \PP(z_0 = \cdot \given \theta^*), \qquad \tau_{h+1} \sim \PP(  \tau_{h+1} = \cdot \given  t_h , \theta_\mathrm{se}^*), \qquad z_{h+1} \sim \PP( z_{h+1} = \cdot \given \tau_{h+1} , \theta_\mathrm{in}^* ),
\end{align}
where $ t_{0} = \{z_0\} $ and $t_{h} = t_{h-1} \cup \{ z_{h}\} $.
Intuitively, this model captures the fact that reasoning often involves summarizing existing information and making predictions. 
The selection module outputs a summary of the existing information that is sufficient for reasoning, and the inference module conducts reasoning based on summarized information. In the example shown in Figure \ref{figure:si_example}, $z_0$ contains the background information and a question, $\tau_1$ summarizes part of the information contained in  $z_0$ and generates the first intermediate reasoning step $z_1$. 
Then $\tau_2$ summarizes $\{z_0, z_1\}$ and $z_2$ is generated from $\tau_2$, which answers the question.

\begin{figure}[h]
    \centering
    \includegraphics[width = 0.97\textwidth]{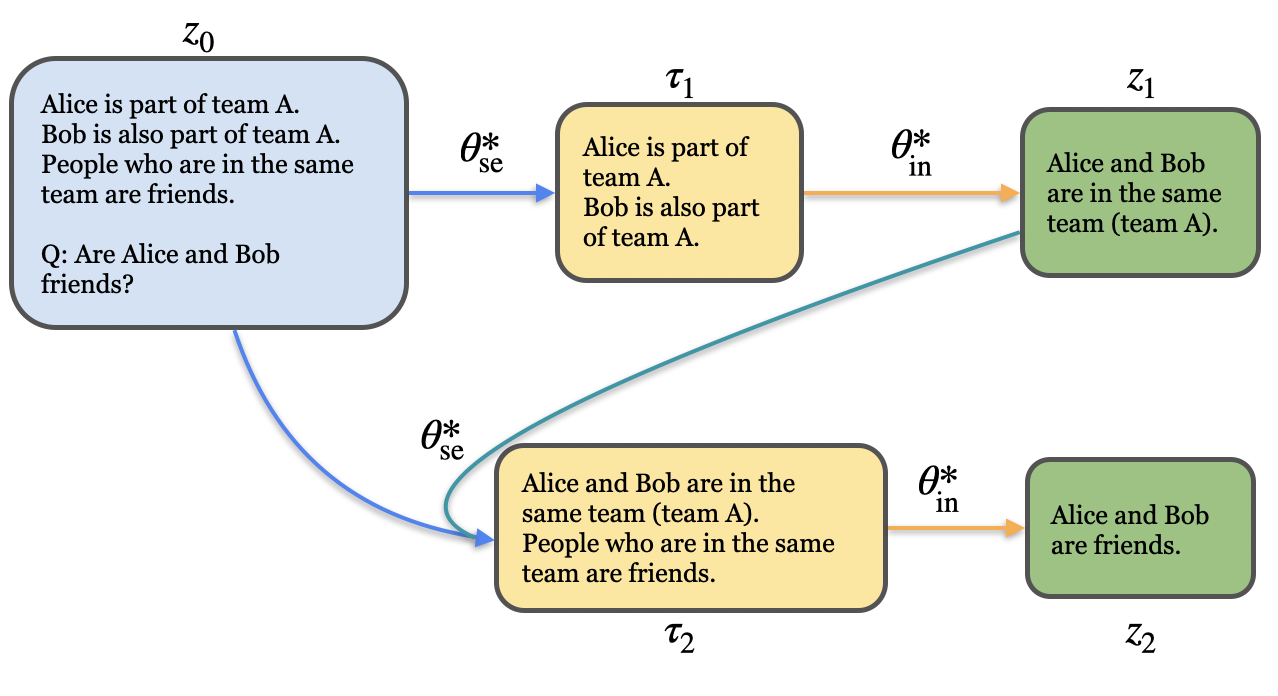}
    \caption{An example generated from the hierarchical model in \eqref{eq:si_model}. The query \(z_0\) (blue box) includes a question and all necessary information to answer it. The first selection step \(\tau_1\) (top yellow box) picks two sentences from \(z_0\) indicating which teams Alice and Bob are in. The first inference step \(z_1\) (top green box) uses the information from \(\tau_1\) to infer that Alice and Bob are on the same team. Together, \(\tau_1\) and \(z_1\) form a single reasoning step. In the second selection step, \(\tau_2\) (bottom yellow box) selects information from \(t_1 = \{z_0, z_1\}\), with the first sentence of $\tau_2$ from \(z_1\) and the second from \(z_0\). The second inference step \(z_2\) (bottom green box) answers the question using only the information provided in \(\tau_2\). }
    \label{figure:si_example}
\end{figure}

\vspace{2mm}
{\noindent \bf SI Prompting.} 
The SI prompting method solves a multi-step reasoning problem following the hierarchical structure specified in \eqref{eq:si_model}, with the unknown task $\theta^*$ inferred implicitly via in-context learning. 
Specifically, given a desired task $\theta^*$ and a query input $z_0^\mathrm{test}$, 
we sample $n$ i.i.d. samples from the distribution in \eqref{eq:si_model}, denoted by $\{z_{0:H}^i,  \tau_{1:h}^i,  \}_{h\in[H], i\in [n]}$. 
We define \( S_\mathrm{se}(n) \) and \( S_\mathrm{in}(n) \) as \begin{align} \label{eq:define_si_examples}
  S_\mathrm{se}(n)   = \{ t_{h-1}^i, \tau_h^i\}_{h\in[H], i\in[n]} , \qquad    S_\mathrm{in}(n)  =  \{   \tau_h^i, z_h^i \}_{h\in[H], i\in[n]}, 
\end{align} 
where $t_h^i$ is the partial history of the $i$-th example.
That is, $S_\mathrm{se}(n)$ and $S_\mathrm{in}(n)$ contain the demonstration examples for selection and inference, respectively. 

These demonstration examples are combined with the intermediate steps of the test example as the prompts, which are used to solve the test example. 
Specifically, starting from $z_0^{\mathrm{test}}$ and $t_0^\mathrm{test} = \{z_0^\mathrm{test}\}$, we generate a reasoning path via 
\begin{align}\label{eq:si_prompt_output}
\tau_{h} ^{\mathrm{test}} \sim \PP_{\mathrm{LLM}} ( \cdot \given   S_\mathrm{se}(n) , t_{h-1}^\mathrm{test}), \quad z_h ^{\mathrm{test} } \sim \PP_{\mathrm{LLM}} ( \cdot \given   S_\mathrm{in}(n) , \tau_{h}^\mathrm{test}), \quad t_{h}^\mathrm{test} = \{ t_h^\mathrm{test}, z_{h}^\mathrm{test} \},
\end{align}
for all $h\in [H]$. 
The final output is $y^\mathrm{test} = z^\mathrm{test}_H$. See Figure \ref{figure:si_prompt} for an illustration of the prompting process. 
Notice that when $z_0^\mathrm{test}$, $S_\mathrm{se}(n)$ and $S_\mathrm{in}(n)$ are fixed, \eqref{eq:si_prompt_output} specifies a Markov chain such that the marginal distribution of $y^\mathrm{test}$ is fully determined by the LLM. We let $\PP_{\mathrm{SI}}( y^\mathrm{test} = \cdot \given S_\mathrm{se}(n), S_\mathrm{in}(n), z_0^\mathrm{test}) $ denote such a distribution, which is essentially the estimator constructed by SI prompting.

\begin{figure}[t]
    \centering
    \includegraphics[width = 0.98\textwidth]{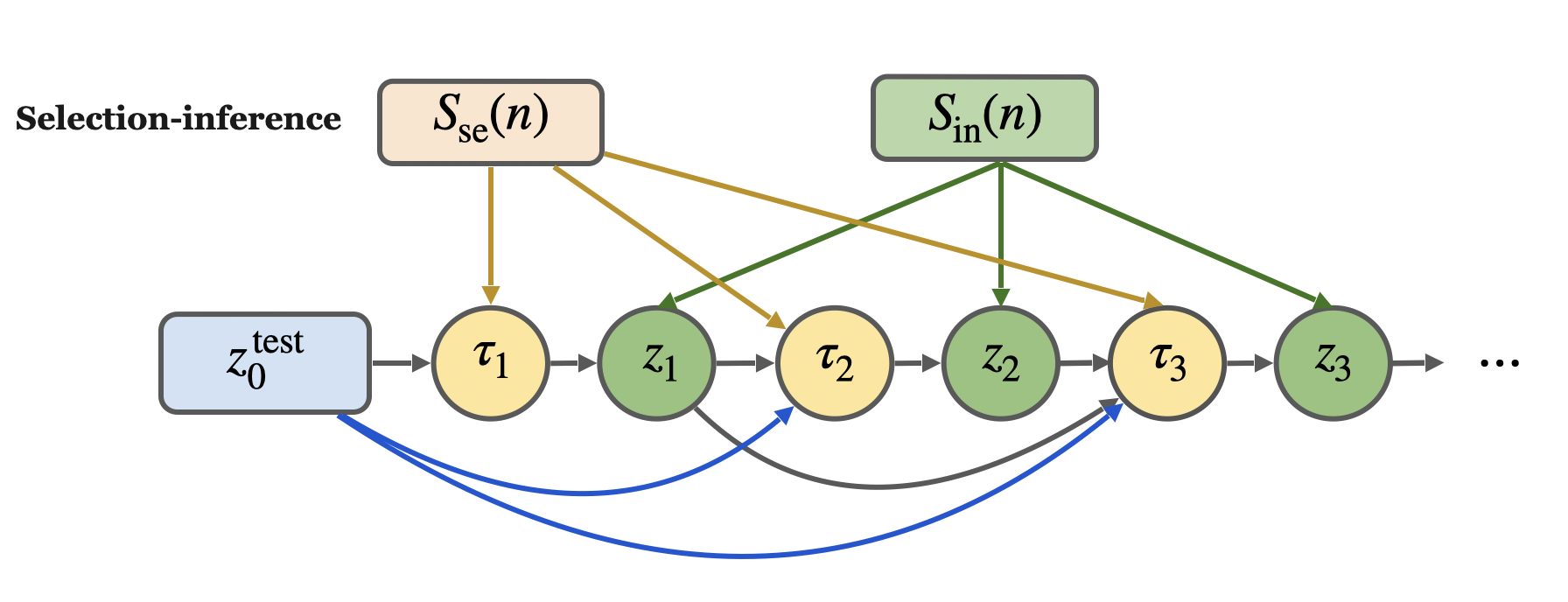}
    \caption{ An illustration of the graphical model of \ac{si} prompting. In \ac{si} prompting, each step in \ac{cot} is divided into two substeps: selection and inference. The selection module retrieves relevant facts from the context \( t_{h-1} \), and the inference module derives inference $z_h$ based on these selected facts $\tau_h$.}
    \label{figure:si_prompt}
\end{figure}

We note that SI can be viewed as a generalization of CoT in the sense that there are still $H$ reasoning steps. However, SI has an additional hierarchical structure that selects subsets of the partial histories. 
We can interpret SI as a version of \ac{cot} where each step of \ac{cot} is decomposed into two substeps, aiming to conduct Bayesian inference of 
 $\theta^*_\mathrm{se}$ and $\theta^*_\mathrm{in}$ separately.

In the following, we will establish the statistical error of the estimator constructed by SI prompting, under the assumption that the underlying data distribution is specified by the model in \eqref{eq:si_model} and the LLM is perfectly pretrained. 
We introduce an assumption in the same vein as Assumption \ref{assumption: seperation}.

\begin{assumption}\label{assumption: seperation2}
Given a  target task $\theta^*$, let $\Theta^{\complement}=\Theta \backslash \Theta_{\mathrm{eq}}(\theta^*)$ denote the complement of the equivalence class of $\theta^*$. 
We assume that there exists a strict separation between the ground truth task $\theta^*$ and any other task in $ \Theta^\complement$. 
Specifically, there exist positive numbers  $\lambda_{\mathrm{q}},\lambda_{\mathrm{I}}, \lambda_{\mathrm{S}}>0$ such that
\begin{gather*}
\inf_{\theta \in  \Theta^{\complement}} \text{H}^2 \big(\PP(z_0 = \cdot  \given \theta^*) , \PP(z_0 = \cdot \given \theta) \big) \geq\lambda_{\mathrm{q}},\\
\inf_{\theta \in  \Theta^{\complement}}\sum_{h=1}^{H} \E_{\theta^*}\text{H}^2 \big(\PP(\tau_{h} = \cdot \given \theta^*_\mathrm{se}, t_{h-1}) , \PP(\tau_{h} = \cdot \given \theta_{\mathrm{se}},t_{h-1}) \big) \geq \lambda_{\mathrm{S}},\\
    \inf_{\theta \in  \Theta^{\complement}}\sum_{h=1}^{H} \E_{\theta^*}\text{H}^2 \big(\PP(z_{h}=\cdot \given \theta^*_\mathrm{in},\tau_h) , \PP(z_{h} = \cdot \given \theta_{\mathrm{in}},\tau_h) \big) \geq \lambda_{\mathrm{I}}.
\end{gather*}
This assumption specifies the separation requirements for \(\theta_\mathrm{se}^*\) and \(\theta_\mathrm{in}^*\) individually. Based on this assumption, we establish the statistical error of  the \ac{si} estimator as follows.
    
\end{assumption}


\begin{corollary}[Sample Complexity of Selection-Inference]\label{cor: selection-inference cot}
  Consider \ac{si} prompting with $n$ examples whose distribution is given by \eqref{eq:si_model} with a given task $\theta^*(\theta_{\mathrm{se}}^*, \theta_{\mathrm{in}}^*)$. We assume that $\Theta$ is a finite set and the LLM is perfectly pretrained with data according to the model in \eqref{eq:si_model} with a prior distribution $\pi$.
 Under Assumption~\ref{assumption: seperation2}, we have 
  \begin{align*}
            \KL \Big(\PP(y^{\mathrm{test}} = \cdot \given  z_0^{\mathrm{test}},\theta^*)  ,   \PP_{\mathrm{SI} } \big(y^{\mathrm{test}} =   \cdot  \given S_\mathrm{se}(n),S_\mathrm{in}(n), z_0^{\mathrm{test}}\big) \Big) 
              \leq \mathcal{O}\big(\pi(\theta^*)^{-1}\delta^{-2} \big|{\Theta^\complement}\big|^2 e^{-2\lambda_{\mathrm{SI}} n}\big),
        \end{align*} 
        with probability $1-\delta$, where $\lambda_{\mathrm{SI}}=\lambda_{\mathrm{S}}+\lambda_{\mathrm{I}}+\lambda_{\mathrm{q}}$. 
        Here $\calO(\cdot)$ hides absolute constants and $\PP_{\mathrm{SI} }(y^{\mathrm{test}} \given S_\mathrm{se}(n),S_\mathrm{in}(n), z_0^{\mathrm{test}})  $
        is the marginal distribution of $y^{\mathrm{test}} $ according to \eqref{eq:si_prompt_output}.  

\end{corollary}

This corollary shows that the prompting error of the  SI  decays to zero exponentially fast as $n$ goes to infinity. 
Moreover, here the exponential factor depends on $\lambda_{\mathrm{SI}}$, which contains the separation of both the selection and inference parts. We defer the proof to Appendix~\ref{proof: selection-inference cot}.

In summary, in this part, we extend the statistical analysis of the vanilla \ac{cot} estimator to three variants of \ac{cot} --- \ac{sc}-\ac{cot}, ToT, and SI. 
We interpret these prompting methods as statistical estimators, and establish their statistical errors of them under an ideal case where the pretraining error of the LLM is zero. 
The analysis can be easily extended to the realistic case with a nonzero pretraining error, which is  separately discussed in Section \ref{subsec: imperfect}.



\subsection{Vanilla CoT versus Vanilla ICL and Truncated CoT} \label{subsubsec: comparison}

Recall that vanilla \ac{icl} is a special case of \ac{cot} without intermediate reasoning steps, i.e.,  $H = 1$. 
In the following, we aim to address  Question (d) raised in the introduction by directly comparing vanilla \ac{icl} and \ac{cot} under the same model. 
We focus on the latent variable model in \eqref{eq:latent_var_model}.
Let $\theta^*$ denote the task during the prompting stage, and let  \(\pt_\mathrm{ICL}(n) = \{z_0^{i}, z_H^i\}_{i=1}^n \cup \{ 
 z_0^\mathrm{test} \} \) and \(\pt_\mathrm{CoT}(n) = \{z_{0:H}^{i}\}_{i=1}^n \cup \{  z_0^\mathrm{test}\} \) denote a \ac{icl} prompt and a \ac{cot} prompt respectively. 
 Thus, vanilla \ac{icl} and \ac{cot} yield estimators $\PP_{\mathrm{LLM}} ( y^\mathrm{test} = \cdot \given \pt_\mathrm{ICL}(n))$ and $\PP_{\mathrm{LLM}} ( y^\mathrm{test} = \cdot \given \pt_\mathrm{CoT}(n))$ respectively.

Recall that we show in Section \ref{subsec: attention parameterizes bma} that the \ac{cot} estimator based on a perfectly pretrained LLM corresponds to a Bayesian model averaging estimator. 
Such a claim also holds for vanilla \ac{icl}.
Therefore, we have 
\begin{align}\label{eq:icl_estimator}
\PP_{\mathrm{LLM}} ( y^\mathrm{test} = \cdot \given \pt_\mathrm{ICL}(n)) \approx \int _{\Theta} \PP \big(y^{\mathrm{test}} = \cdot  \given \theta, z_0^\mathrm{test}\big) \cdot \pi(\theta \given \pt_\mathrm{ICL}(n) )\rmd \theta,
\end{align}
where $\PP (y^{\mathrm{test}} = \cdot  \given \theta, z_0^\mathrm{test})$ is the marginal distribution of $y^{\mathrm{test}}$ given $z_0^\mathrm{test}$ under the model in \eqref{eq:latent_var_model} with parameter $\theta$, and $\pi(\theta \given \pt_\mathrm{ICL}(n) )$ is the posterior distribution. We will justify \eqref{eq:icl_estimator} in Appendix \ref{proof: dom2}. 
The following proposition shows that CoT always outperforms vanilla ICL in an average sense.


\begin{proposition}[CoT Outperforms Vanilla ICL]\label{prop: dom}
Let $\pi$ denote the prior distribution over $\Theta$.  Consider the ideal case where the LLM is perfectly pretrained, for any number of demonstration examples $n\geq 0$, we have 
\begin{align*}
    &\E_{\theta^* \sim \pi} \E_{\mathtt{prompt}_{\mathrm{CoT}}(n)\sim \PP(\cdot \given \theta^*)} \bigg[\KL \Big(\PP_{\mathrm{LLM}} (y^{\mathrm{test}}=\cdot \given z_0^{\mathrm{test}},\theta^*)  ,   \PP\big(y^{\mathrm{test}}=\cdot \given  \mathtt{prompt}_{\mathrm{CoT}}(n)\big) \Big) \bigg]\\
    & \quad \leq \E_{\theta^* \sim \pi}\E_{\mathtt{prompt}_{\mathrm{ICL}}(n)\sim \PP(\cdot \given \theta^*)}\bigg[\KL \Big(\PP(y^{\mathrm{test}} = \cdot \given z_0^{\mathrm{test}},\theta^*)  ,   \PP\big(y^{\mathrm{test}} = \cdot  \given \mathtt{prompt}_{\mathrm{ICL}}(n)\big) \Big)\bigg].
\end{align*}
    
\end{proposition}


This proposition shows that averaged over the randomness of the task $\theta\sim \pi$ and prompts, {\bf \ac{cot} is at least as good as vanilla \ac{icl}}. Intuitively, this makes sense because conditioning more information yields a better posterior estimator. Since these estimators can both be interpreted as BMA estimators, having a better posterior leads to a smaller statistical error. 

We can also extend this property to truncated \ac{cot} methods, which refers to prompting with demonstrations that omit some intermediate steps. 
More precisely, let $\calJ \subset [H-1]$ contain the indices of intermediate steps that are included in the reasoning path.
We define a truncated \ac{cot} prompt with $n$ examples as  $\pt_\calJ(n) = \{z_0^i,y^i\}_{i=1}^n \cup \{z_{j}^{i}\}_{j\in \calJ}\cup\{ z_0^\mathrm{test}\} $.  
Then vanilla \ac{icl} is a special case where all intermediate steps are omitted, i.e., $\calJ = \emptyset$, and \ac{cot} corresponds to the case where $\cJ = [H-1]$. We extend Proposition~\ref{prop: dom} to such a general case in Appendix~\ref{proof: dom2}, which shows that including more reasoning steps in the prompt is always beneficial in an average sense. 

However, we would like to emphasize that 
the dominance of CoT over vanilla ICL \textbf{does not hold pointwisely} for an arbitrary task $\theta^* \in \Theta$. 
In other words, it is possible that there exists a task $\theta^*$ and $n$ prompt examples such that \ac{cot} is worse than vanilla ICL. 
Intuitively, this happens when the intermediate reasoning paths are not sufficiently informative. 
This phenomenon is empirically observed in \cite{lanham2023measuring} on the HellaSwag benchmark \citep{zellers2019hellaswag}. In the following, we also provide numerical experiments based on a specially designed toy task to illustrate this fact.

\vspace{2mm}
{\noindent \bf Vanilla ICL vs. COT on the CityEquation Task.} 
We handcraft an arithmetic reasoning task named ``CityEquation", which involves solving arithmetic calculations based on city names. 
Each equation involves addition $(+)$ or the minus $(-)$ operations between city names,
where the output is obtained by evaluating the formula with city names substituted by their longitudes. 
For instance, ``$\mathtt{Paris} + \mathtt{Beijing} = 118$'' because the longitudes of Paris and Beijing are $2$ and $116$ respectively.


\vspace{2mm}
{\noindent \bf Data Construction.}  We choose $20$ major cities around the world, and generate random city equations by randomly selecting two cities and an operation in $\{+,-\}$. 
We construct the test data set using $200$ distinct equations and use another $10$ different equation as the examples in the prompting stage.  

\vspace{2mm}
{\noindent \bf Prompting Methods and Results.} 
We test five prompting methods: vanilla \ac{icl} and four \ac{cot} variants. 
We consider an \emph{informative} version of \ac{cot} that includes the full reasoning path and four \emph{partially informative} versions that either contain some irrelevant facts or omit some relevant intermediate steps. 
In particular, in \emph{partially informative \ac{cot}-(b)}, we include some demographic information of the cities in the equations, which, although truthful facts, are not related to longitudes, which is the key to getting the final answer. 
Then the last two versions additionally include some useful reasoning steps. See Table \ref{table:city_eqn_prompt} for an example in the prompts and Appendix~\ref{app: prompts} for more details. 
When evaluating these methods, we include 10 examples in the prompt, followed by a new testing instance.
The prompt is passed to GPT-4 \citep{achiam2023gpt} with the temperature set to zero, and the reported answer is compared with the desired answer to evaluate the accuracy. 
We report the average accuracy over 200 random testing instances in Table \ref{table:city_name_results}. 


\begin{table}[t]
    \centering
    \begin{tabular}{p{0.35\linewidth} | p{0.6\linewidth}}
    \hline
    \hline
      Type   & An Example \\ \hline
      Vanilla \ac{icl}   & \textbf{Q}: ``London - Lagos'' \textbf{A}: ``-3.''
 \\
 \hline
      \textcolor{orange}{Informative} \ac{cot}   & \textbf{Q}: ``London - Lagos'' \textbf{A}: `\textcolor{blue}{Using the longitudes of cities, the equation ``London - Lagos'' translates as ``London'' = 0, ``Lagos'' = 3. Here the longitudes of the western hemisphere are negative numbers. And we round the coordinates to the nearest integer. This gives the result.} The answer is -3.''\\
      \hline
      \textcolor{blue}{Partially informative} \ac{cot}-\textcolor{blue}{(a)}   & \textbf{Q}: ``London - Lagos''\textbf{A}: ``\textcolor{blue}{London has longitude: 0.} The answer is -3.''\\
      \hline
      \textcolor{red}{Partially informative} \ac{cot}-\textcolor{red}{(b)}  & \textbf{Q}: ``London - Lagos'' \textbf{A}: ``\textcolor{red}{London is home to approximately 9 million residents, with 59.8 percent being White, 18.5 percent Asian, 13.3 percent Black, 5 percent Mixed, and 3.4 percent identifying as Other.} The answer is -3.''\\
      \hline
      \textcolor{brown}{Partially informative} \ac{cot}-\textcolor{brown}{(c)} & \textbf{Q}: ``London - Lagos'' \textbf{A}: ``\textcolor{red}{London is home to approximately 9 million residents, with 59.8 percent being White, 18.5 percent Asian, 13.3 percent Black, 5 percent Mixed, and 3.4 percent identifying as Other.} \textcolor{blue}{London has longitude: 0.} The answer is -3.''\\
\hline
      \textcolor{purple}{Partially informative} \ac{cot}-\textcolor{purple}{(d)}  & \textbf{Q}: ``London - Lagos'' \textbf{A}: ``\textcolor{red}{London is home to approximately 9 million residents, with 59.8 percent being White, 18.5 percent Asian, 13.3 percent Black, 5 percent Mixed, and 3.4 percent identifying as Other.} \textcolor{blue}{London has longitude: 0. Lagos has longitude: 3.} The answer is -3.''\\
    \hline
    \hline
    \end{tabular}

    \caption{An example of the five prompting methods evaluated on  the CityEquation task. }
    \label{table:city_eqn_prompt}
\end{table}

 





 \begin{table}[t]
    \centering
    \begin{tabular}{ 
    p{0.1\linewidth} | p{0.13\linewidth} | p{0.15\linewidth} | p{0.15\linewidth}| 
    p{0.15\linewidth} |
    p{0.15\linewidth}}
    \hline \hline 
     Vanilla \ac{icl} & \textcolor{orange}{Informative} \ac{cot} & \textcolor{blue}{PI} \ac{cot}-\textcolor{blue}{(a)} & \textcolor{red}{PI} \ac{cot}-\textcolor{red}{(b)} & \textcolor{brown}{PI} \ac{cot}-\textcolor{brown}{(c)} & \textcolor{purple}{PI} \ac{cot}-\textcolor{purple}{(d)}\\
     \hline
     59.5\% & 81.5\% &70.5\%  & 2.5\% & 66\% & 80\%\\
     \hline \hline 
    \end{tabular}

    \caption{Average accuracy of the five prompting methods on the CityEquation task. The results are based on 200 random testing instances. ``PI CoT'' stands for the partially informative chain of thought method. }
    \label{table:city_name_results}
\end{table}

As shown in this table, informative \ac{cot} achieves the highest accuracy at $81.5\%$, the four versions of partially informative (PI) \ac{cot} have accuracy levels of $70.5\%$,  $2.5\%$, $66\%$, and $80\%$, respectively. 
The errors made by informative \ac{cot} are due to rounding errors.
Moreover, compared to vanilla \ac{icl}, version (a) of PI CoT includes a piece of relevant information (longitude of London), which significantly helps the reasoning. 
Comparing vanilla \ac{icl} to versions (b)--(d), we see that including more relevant information in the intermediate reasoning steps improves the accuracy. 
In particular, version (b)--(d) include the first one to three steps of the same reasoning steps, where the first reasoning step is a piece of irrelevant information. 
Version (b) of PI COT, only including such irrelevant information, performs drastically worse than vanilla ICL. 
This observation supports our remark that CoT does not always outperforms vanilla ICL. 
 Note that here the intermediate steps added in the version (b) of PI CoT are true facts. 

In conclusion, \ac{cot} prompting can perform worse than vanilla \ac{icl}. While \ac{cot} offers more information, its effectiveness depends on the relevance of the additional information provided in intermediate steps. Truncated informative \ac{cot} slightly enhances vanilla \ac{icl}'s accuracy by hinting at relevant details like longitude. Informative \ac{cot} performs even better by outlining the entire reasoning process. In contrast, uninformative \ac{cot}, despite providing more information, introduces irrelevant details that disrupt performance. Partially informative \ac{cot} combines useful and irrelevant information, resulting in accuracy between that of truncated informative \ac{cot} and uninformative \ac{cot}.

\section{Statistical Errors  of CoT with  Pretraining Errors}\label{subsec: imperfect}

Recall that in Lemma~\ref{lem:error_decomp}  we decompose the \ac{cot} error $\mathtt{err}_\mathrm{CoT}$ into  (i) a  pretraining error~\eqref{eq: err pre} and (ii) a prompting error~\eqref{eq: err pt}. 
We have analyzed the prompting error in Section~\ref{subsec: perfect}. In this section, we establish a statistical analysis of the pretraining error and then obtain a complete characterization of the error of the estimator constructed by ``pretrained LLM + \ac{cot} prompting". Thus, in this section, we answer Question (b) raised in the introduction. 



We rigorously describe the pretraining process in Section~\ref{subsubsec: pretraining process}. Next, in Section~\ref{subsubsec: Build Approximators using Transformers}, we construct a class of transformer networks that directly approximates the underlying distribution $\PP$. 
We characterize the pretraining error in 
Section~\ref{subsubsec: Pretraining Performance Analysis}  and establish the statistical errors of \ac{cot} under realistic assumptions in Section~\ref{subsubsec: cot Convergence Rate with Pretraining Error}. 


\subsection{Setup of LLM Pretraining} \label{subsubsec: pretraining process}

Recall that we introduce the pretraining of \ac{llm} in Section \ref{subsec: pre-training and cot prompting under the statistical model}. We consider the pretraining of an autoregressive \ac{llm} with data sampled from the model in \eqref{eq:latent_var_model}. 
The \ac{llm} is a transformer that maps a sequence of reasoning steps to a probability distribution.  

\vspace{2mm}
{\noindent \bf Transformer Architecture.} 
We let $\mathcal{TF} ( D, \eta, r, d_F, d_{k}, d_{v})$ denote the class of transformers that maps a sequence of reasoning steps to a probability distribution over $\cL$, where the input sequence is embedded in $\RR^r$, followed by $D$ sequentially stacked transformer blocks and a final softmax layer that outputs a distribution. 
Here the embedding contains both token and positional embedding. 
Moreover, each of the $D$ transformer blocks includes a multi-head attention (MHA) layer with $\eta$ parallel heads and a fully connected feedforward (FF) layer. 
The embedding dimensions of the queries, keys, and values in \ac{mha} in \eqref{eq: mha} are $d_{k}$, $d_{k}$, and $d_{v}$, respectively. Here queries and keys share the same dimension to calculate the inner product. We assume that $d_{v}=r$ to guarantee that the output dimension is the same as the input dimension, which avoids defining the dimensions for modules in all the layers. The results for the general case can be easily generalized. For the \ac{ff} layer in \eqref{eq: ffn}, the dimensions of the hidden feature and the output are $d_{F}$ and $r$, respectively. Both components have residual connections,  followed by layer normalization. 
See Figure~\ref{fig:pre-train-intext} for an illustration of the transformer architecture, where the transformer block consisting of a \ac{mha} and a \ac{ff} layer is illustrated on the right.


In terms of the network parameters, for any $d\in [D]$, we let $W_{\mathrm{mha}}^d = (W^{Q,d}_i, W^{K,d}_i, W^{V,d}_i)_{i=1}^\eta$ denote the weight matrices of the $\eta$ heads, and let 
$W^d_{\mathrm{ff},1}$ and  $W^d_{\mathrm{ff},2}$ denote the weight matrices of the FF layer. 
The mathematical expressions of \ac{mha} and \ac{ff} layers are given in \eqref{eq: mha} and \eqref{eq: ffn}. We adopt $\gamma_1^d, \gamma_2^d \in \RR^{r\times r} $ to denote the parameters of the residual links in the $d$-th module. 
Moreover, for the output softmax layer, we fix the temperature as $\tau$ and let $W_\mathrm{softmax} $ denote the weight matrix. 
For ease of presentation, we defer the mathematical details of the transformer to Appendix~\ref{app: pre-training process}. We let $\rho $ denote all the network parameters of the transformer. Furthermore, we consider a bounded transformer class with parameters bounded in
\begin{align}
    \cP_{\mathrm{LLM}} &= \Big\{\rho : \| \gamma_1^d\|_\infty, \|\gamma_2^d\|_\infty \leq 1, \|W^{Q, d}_i\|_F ,\|W^{K, d}_i\|_F ,\|W^{V, d}_i\|_F \leq B_M, \| 
    W_{\mathrm{ff},1}^{d}\|_F,\|W_{\mathrm{ff},2}^{d}\|_F \leq B_{F}, \notag \\
    & \qquad  
     \|W_{\mathrm{softmax}}\|_{1,2}\leq B_S, \forall d \in [D], i \in [\eta] \Big\} \subseteq \mathcal{TF}(D, \eta, r, d_F, d_k, d_v),\label{eq:P_LLM_class}
\end{align}
where $B_M, B_F, B_S$ are the upper bounds are upper bounds of the norm of weight matrices. We assume these parameters are fixed and larger than one.



\begin{figure}[h]
    \centering
    \includegraphics[width=0.65\textwidth]{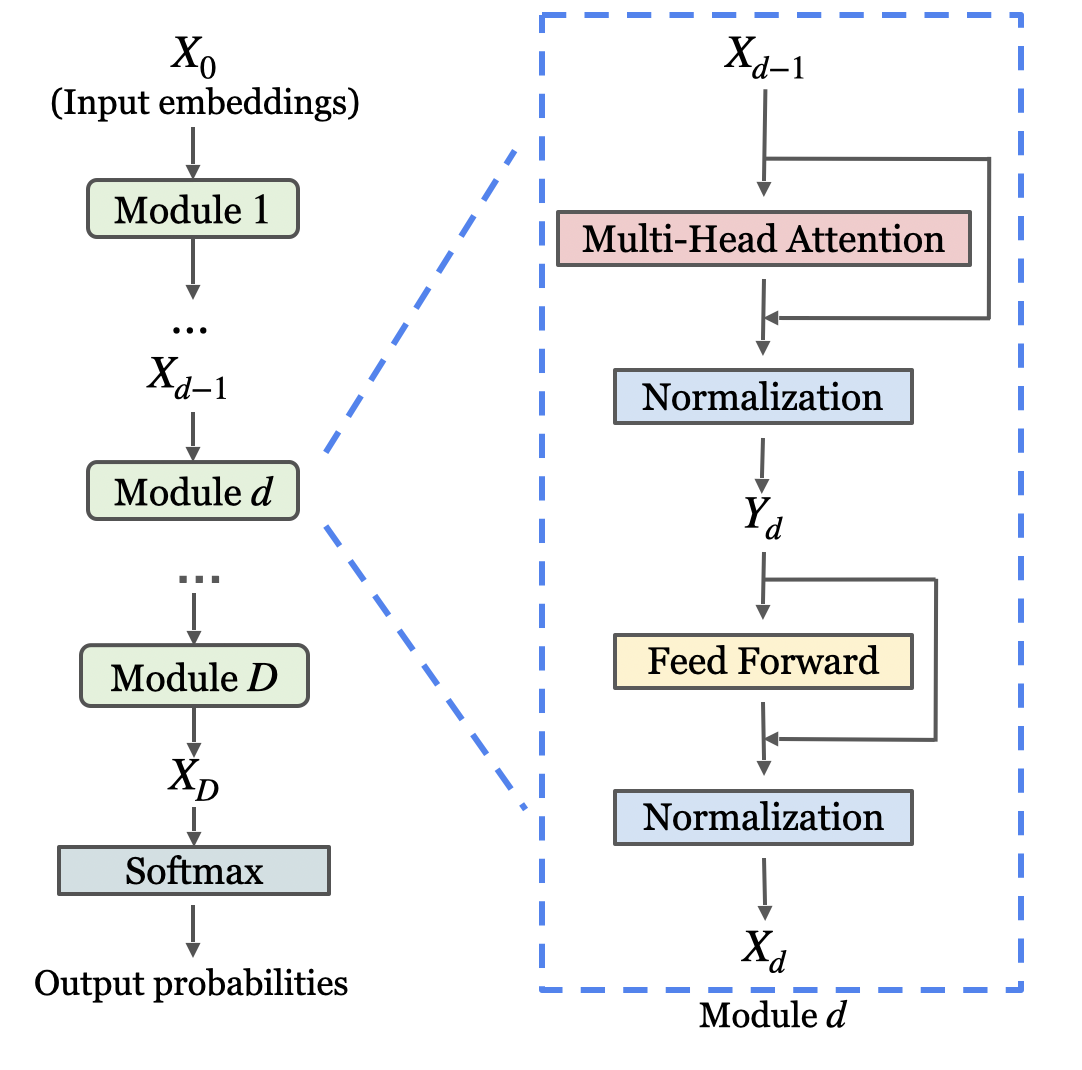}
    \caption{An illustration of a transformer with depth $D$.   The left-hand side shows the network with $D$ sequentially stacked transformer blocks followed by a final softmax layer.  The right-hand side zooms in on a single transformer block, which comprises a \ac{mha} layer and a \ac{ff} layer, connected by normalization layers. }
    \label{fig:pre-train-intext}
\end{figure}

\vspace{2mm}
{\noindent \bf Pretraining Data and MLE Estimation.} The dataset, denoted by $\calD_{N,T}$  contains $N$ independent trajectories, each with $T$ examples. 
For each trajectory $\ell\leq N$, we first sample an i.i.d. task $\theta_\ell^*$ from the prior $  \pi$. 
Conditioning on the task parameter $\theta_\ell^*$, we generate $T$ examples $\{s^{k,\ell}\}_{k=1}^T \sim \PP(\cdot \given \theta_{\ell} ^*)$ according to the model in \eqref{eq:latent_var_model} and concatenate them to form a trajectory. 
Here $s^{k,\ell} = (z_0^{k,\ell}, \cdots, z_H^{k,\ell})$ denotes the $k$-th example of the task $\theta_\ell^*$, and we view each $z_j^{k, \ell}$ as a reasoning step.
Thus a 
sequence contains $T(H+1)$ elements in total. 
For any $h \geq 0$, the reasoning steps before $z_{h}^{t, \ell}$ is denoted by 
 $S_h^{t,\ell} =(\Upsilon_{t-1,\ell},\{z_j^{t,\ell}\}_{j=0}^{h-1})$, where $\Upsilon_{t-1,\ell}$ contains the first $t-1$ examples of the $\ell$-th trajectory. 
 Then we can write the dataset
 $\mathcal{D}_{N,T}$ as $  \{(S_h^{t,\ell} ,z_{h}^{t,\ell})\}_{h=0,t=1, \ell=1}^{H, T, N}$. 
The pre-trained \ac{llm}, denoted by $\hat \rho$, as defined in \eqref{eq:pretraining_loss},  is obtained by solving the maximum likelihood estimation (MLE) based on the dataset $\cD_{N,T}$. 
We set $\PP_\mathrm{LLM} = \PP_{\hat \rho}$, where  $\PP_{\hat \rho} $ denotes the conditional distribution specified by the transformer with parameter $\hat \rho$. 
We neglect the optimization issue and assume that the MLE in \eqref{eq:pretraining_loss} can be obtained.  
We note that when the transformer class is sufficiently expressive, we expect that $\PP_\mathrm{LLM} $ learns the conditional distribution of \(z_h^{t,\ell}\) given   \(S_h^{t,\ell}\), which is given in \eqref{eq:bayes_factorization}.  


   We note that in the pretraining process described above, we train a transformer that takes all sequences of the reasoning steps \( S \in \calL^* \) as input and predicts the next reasoning step \( z \in \calL \). This setup can be easily generalized to the autoregressive prediction of the next token instead of the next reasoning step based on the prompt.

\subsection{Pretraining Error Analysis} \label{subsubsec: Pretraining Performance Analysis}

We will show that the pretraining error can be written as a sum of an \textbf{approximation error} and a \textbf{generalization error}. 
The analysis is based on the PAC-Bayes framework \citep{mcallester1998some, alquier2021user}. 
Before presenting this result, 
we introduce two regularity assumptions as follows.

\begin{assumption} \label{assumption: bd magnitude}
Note that we assume that each reasoning step in $\cL$ is identified with a unique Euclidean vector.
We assume that $\cL$ is a bounded set. That is, there exists $R >0$
    We assume that there exists $R>0$ such that $\|z \|_2 \leq R $ for all $z \in \cL$. 
\end{assumption}

This assumption ensures that the input space of the transformer network is bounded, which is commonly imposed by the literature on nonparametric statistics \citep{zhang2023and}.


\begin{assumption} \label{assumption: lower bd}
    For the model in \eqref{eq:latent_var_model}, we assume that for any $z \in \mathcal{L}, \theta \in \Theta$, and any sequence of reasoning steps $S \in \calL^*$, $\PP(z \given S, \theta)>c_0$ for some constant $c_0>0$.
\end{assumption}
This assumption requires the conditional probability of the next reasoning step $z$ to be lower-bounded at any element of $\cL$. 
This means that the generation of the reasoning path is stochastic. 
Similar assumptions have also been imposed in existing works~\citep{xie2021explanation,jiang2023latent}. 
As we will see in Appendix~\ref{proof: pre-train}, 
this assumption implies Assumption~\ref{assump: density bd}  with $b^* = \log (\max\{c_0^{-1}, 1+|\calL|\exp(B_S/\tau)\})$, where $B_{S}$ appears in \eqref{eq:P_LLM_class}.

Besides, to simplify the notation, we let $\E_{S\sim\mathcal{D}} $ denote the empirical distribution with respect to the pretraining data set $\cD_{N,T}$. 
Specifically, for any function 
  $f:\mathcal{L}^* \rightarrow \mathbb{R}$  we define 
  \begin{align}
      \E_{S\sim\mathcal{D}}[f(S)]= N^{-1} (H+1)^{-1} \cdot T^{-1} \sum_{\ell=1}^N \sum_{t=1}^{T}\sum_{h=0}^H \E [f(S_h^{t, \ell})], \label{eq:avg_ex}
  \end{align}
  where the expectation is taken with respect to the joint distribution of $\cD_{N,T}$. 
  We establish the pretraining error in the following proposition.

\begin{proposition}[Pretraining Error Bound] \label{prop: pre-train}
Under Assumptions~\ref{assumption: bd magnitude} and \ref{assumption: lower bd}, with probability at least $1-\delta$,  the pretrained \ac{llm} $\PP_{\hrho}$ in \eqref{eq:pretraining_loss} satisfies that 
     \begin{align*}
        &\E_{S \sim\mathcal{D}}\Big[\TV\big(\PP(\cdot\given  S ),\PP_{\hat \rho}(\cdot \given  S )\big)\Big]\\
        &\quad=O\bigg(\underbrace{\inf_{\rho^{*}\in \calP_{\mathrm{LLM}}}\sqrt{\E_{ S \sim\mathcal{D}}\KL\big(\PP(\cdot| S ),\PP_{\rho^{*}}(\cdot| S )\big)}+\frac{\sqrt{b^*}\log (TH/\delta)}{N^{1/4}}}_{\displaystyle \text{approximation error}}\nonumber\\
        &\quad\qquad\qquad\qquad\qquad\qquad\qquad+\underbrace{\frac{1}{\sqrt{N}}\Big(\bar D\log(1+NTH\barB)+\log\frac{TH}{\delta}}_{\displaystyle \text{generalization error}}\Big)\bigg),
    \end{align*}
     where $\bar B=\tau^{-1}R\eta B_{S}B_F^2B_M^3$ and $\bar D=D^{2} r (d_{F}+d_k+r)+r\cdot d_{y}$ are parameters determined by the transformer architecture in \eqref{eq:P_LLM_class}. Besides, we have $b^* = \log \big(\max\{c_0^{-1}, 1+|\calL|\exp( B_S/\tau)\}\big)$. We use $\Delta_\mathrm{pre}(N,T,\delta)$ to denote the right-hand side of this equation.
    
\end{proposition}
This Proposition is proved using the PAC-Bayes framework. The proof is adapted from    \cite{zhang2023and} and deferred to Appendix~\ref{proof: pre-train}.

Proposition \ref{prop: pre-train} shows that the pretraining error can be decomposed into an approximation error and a generalization error. 
The approximation error is a sum of a KL divergence term, and an additional $N^{-1/4} $ terms that arise from concentration. The approximation error is small if the transformer class is sufficiently expressive. 
Moreover, the generalization error decays to zero as $N$ increases, and $\bar D$ captures the complexity of the transformer model. 
This error increases with the sequence length $T\cdot H$ mildly through a logarithmic factor.


\subsection{Transformers as Conditional Distribution  Approximators} \label{subsubsec: Build Approximators using Transformers}
In the following, we present the approximation result. 
We will construct a transformer with parameters in $\cP_{\mathrm{LLM}}$ that captures the multi-step reasoning structure of \ac{cot}. 
More importantly, we will prove that the approximation error decays to zero exponentially as the network depth $D$ increases. We first present an informal version of the theory as follows.

\begin{proposition}[Approximation Error, Informal]\label{prop: kl bd}
Let $S_h^t= (\Upsilon_{t-1}, \{ z_j^t\}_{j=0}^{h-1} )$ be the sequence of reasoning steps that includes  $t-1$ examples of reasoning paths $\Upsilon_{t-1}$  and the first $h-1$ steps of the $t$-th example $ \{ z_j^{t-1}\}_{j=0}^h$. 
    Then if the target distribution $\PP$ of the model in \eqref{eq:latent_var_model} has a sufficiently smooth density and the transformer model $\calP_{\mathrm{LLM}}$ in \eqref{eq:P_LLM_class} is sufficiently expressive, then there exists a transformer with at most $\cO(D)$ number of blocks and parameter $\rho^* \in \calP_\mathrm{LLM}$ such that 
\begin{align*}
\max_{\substack{S_h^t\in \calL^*}}\KL\big(\PP(z_h^t = \cdot\,|\,S_h^t),\PP_{\rho^{*}}(z_h^t = \cdot\,|\, S_h^t)\big) =O\bigg( \exp\bigg(-\frac{ \big(D-C\log(2H))/H\big)^{1/4}}{5B}\bigg)\bigg),
\end{align*}
for any $t\in [T]$ and $0\leq h\leq H$ when $D$ goes to infinity. In particular, $C > 0 $ is an absolute constant,   and $B$ appears in Assumption~\ref{assumption: smooth}.
  
\end{proposition}

This proposition shows that the approximation error decays exponentially in $D$. 
This exponential accuracy is based on the construction of a neural network approximator in \cite{elbrachter2021deep} for smooth functions. 
Moreover, note that $S_h^{t}$ has $(t-1)\cdot (H+1) + h $ reasoning steps in total. 
An appealing feature of this proposition is that the approximation error is independent of $t$, thanks to leveraging the permutation invariance structure of the target distribution. 
Specifically, when viewing $\PP(z_h^t = \cdot \given S_h^t)$ as a function of $S_h^t$, it is invariant to the permutation of the $t-1$ examples. 
Our transformer approximator directly leverages such invariance in the attention mechanism, thus obtaining an approximation error independent of $t$. 
However,  $\PP(z_h^t = \cdot \given S_h^t)$ can be drastically different across $h\in[H]$. Concretely, each reasoning step represents a different procedure described by different distributions. To handle this fact, our transformer treats each step $h\in [H]$ differently and uses a separate transformer subnetwork to predict each $z_h$. 
These subnetworks, each containing multiple attention blocks, are stacked vertically. 
And we leverage the 
position embedding to let the transformer identify the step index $h$ of $S_h^t$, and then pass the input to the $h$-th subnetwork. See Figure \ref{fig:network_sketch} for an illustration of the construction. The formal statement of Proposition  \ref{prop: kl bd} and its detailed proof are deferred to  Appendix~\ref{prop:formal_construction}.


\begin{figure}[t]
    \centering
    \includegraphics[width = 0.98\textwidth]{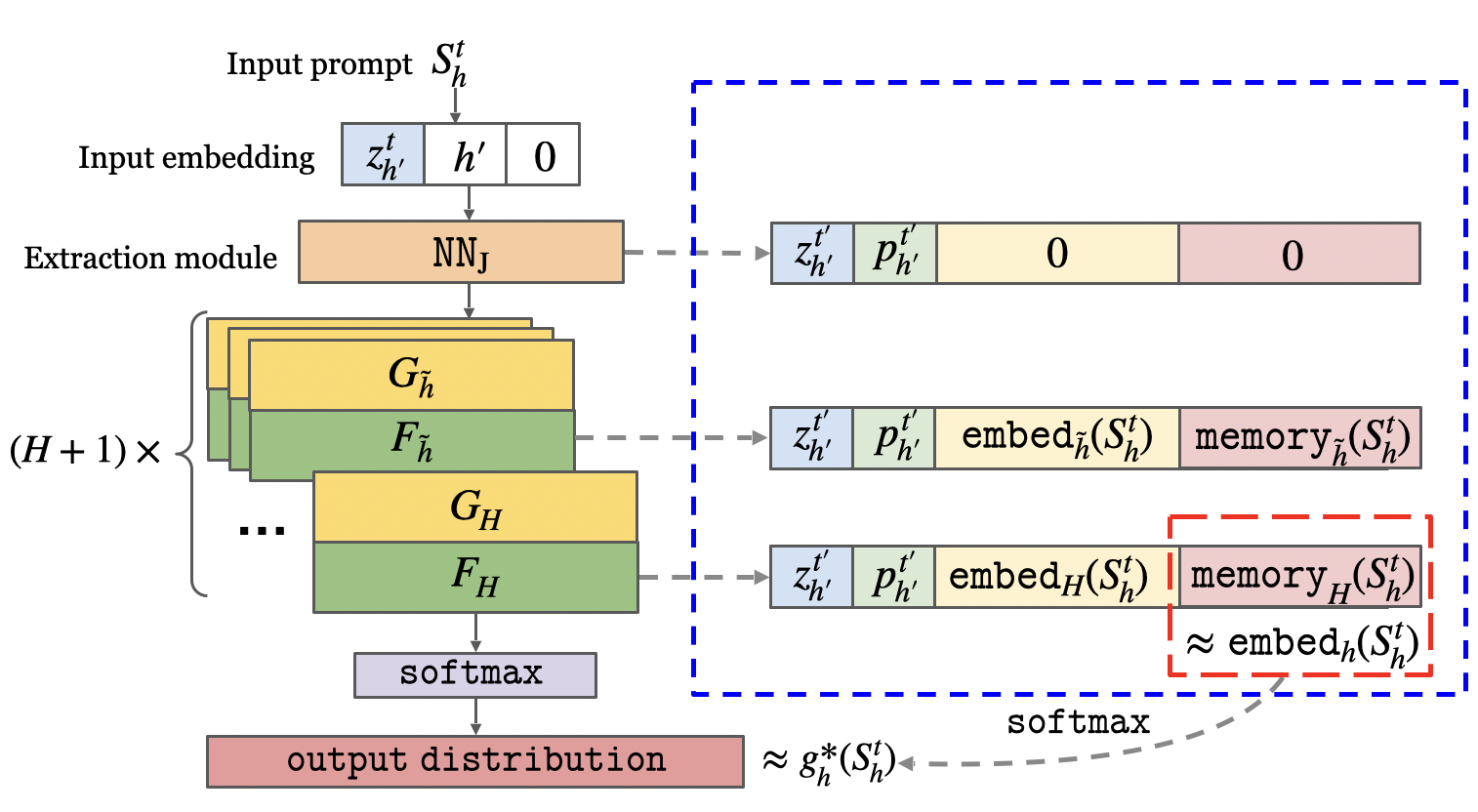}
    \caption{An illustration of the transformer constructed for proving Proposition \ref{prop: kl bd}.  After the input embedding, the first module $\mathtt{NN}_\mathrm{J}$ extracts the step-index \(h\) from \(S_h^t\) and copy it to all previous positions. 
    Then, the transformer go through $H+1$ submodule pairs
    $\{G_{\tilde h}, F_{\tilde h}\}_{\tilde h=0}^{H}$,
    where each 
    pair is designed to approximate the conditional distribution associated with a specific step \(h \in [H]\). 
   The output of $F_{H}$ goes through a softmax output layer to get the final output distribution, which is close to the target distribution $\PP(\cdot \given S_h^t)$.}
    \label{fig:network_sketch}
\end{figure}

\subsection{Statistical Error of CoT with Out-of-Distribution Queries} \label{subsubsec: cot Convergence Rate with Pretraining Error}

In this section, we combine the analysis of pretraining error and prompting error to derive a comprehensive characterization of the statistical error of \ac{cot}. 
Moreover, we will tackle the {\bf query error}, which arises due to the distributional shift of the test instance $z_0^\mathrm{test}$. 

Specifically, the query error arises if $z_0^\mathrm{test}$ is not sampled from the same task $\theta^*$ as the $n$ prompt examples. For instance, if the examples in the prompts are about the ``solving arithmetic problems'', but we query a new philosophical question. Then the knowledge incorporated in the examples is not useful for answering the query and thus we expect a large error.
More rigorously, let 
$\Upsilon_n=\{s_j\}_{j=1}^n$ denote the $n$ \ac{cot} examples sampled  the model in \eqref{eq:latent_var_model} with task $\theta^* $. 
Thus, $\pt_\mathrm{CoT}(n) = (z_0^{\mathrm{test}}, \Upsilon_n)$. 
Under this model, the query has a distribution $\PP(z_0^\mathrm{test} = \cdot \given \theta^*)$.
We let $\mu(z_0^\mathrm{test} = \cdot \given \Upsilon_n) $ denote the distribution of an out-of-distribution (OOD) query, whose distribution might depend on $\Upsilon_n$. The difference between these two distributions reflects the query error. 

Besides, we let 
 $\PP_\mathrm{CoT}$   denote the joint distribution of $\Upsilon_n$ and an OOD query, i.e.,   $ \allowbreak \PP_{\mathrm{CoT}} ( \pt_\mathrm{CoT}(n) \given \theta^*) = \PP(\Upsilon_n \given \theta^*) \cdot \mu (z_0^\mathrm{test} \given \Upsilon_n) $. We make the following assumption about the distributional shift due to the OOD query. 


\begin{assumption}\label{assump: cover}

 We assume the distributional shift is mild in the sense that 
 \(\PP_{\mathrm{CoT}}\) is covered by the pretraining distribution. 
 That is, for any fixed $\theta^*\in \Theta$,
there exists a constant \(\kappa>0\) such that \(\PP_{\mathrm{CoT}}(\pt_\mathrm{CoT}(n) \given \theta^*) \leq \kappa \PP(\pt_\mathrm{CoT}(n) \given \theta^*)\) for any \(\pt_\mathrm{CoT}(n)\in\calL^{*}\) with \(n\leq T\).
\end{assumption}
Here $\kappa$ captures the magnitude of the distributional shift. This assumption requires that the test query $z_0^\mathrm{test}$ cannot be too arbitrary -- its distribution should have sufficient density under the pretraining distribution. 
Intuitively, we cannot expect the \ac{llm}s to answer questions beyond the knowledge contained in the pretraining dataset. Note that when there is no distributional shift in the query, we have $\kappa=1$. Then the analysis of $\mathtt{err}_\mathrm{prompt}$ is reduced to Theorem~\ref{thm: rate_cot}.
Recall that the pretraining data distribution mixes the task distribution $\theta\sim \pi$. Under the model in \eqref{eq:latent_var_model}, this assumption is satisfied if we set 

$$
\kappa =   {\textstyle \sup_{z\in \cL, S \in \cL^*}}  \big | \mu(z_0^\mathrm{test} =z   \given \Upsilon_n = S  ) / \PP(z_0^\mathrm{test} = z \given \theta^*)  \big| ,
$$
which is no more than $1/c_0  $ under Assumption \ref{assumption: lower bd}.
Thus, the distributional shift is small if task $\theta^*$ is well covered in the pretraining distribution, and the distribution $\mu$ is similar to the true query distribution $\PP(z_0^\mathrm{test} = \cdot \given \theta^*)$. 

\vspace{2mm}
{\noindent \bf Combining Pretraining and Prompting Errors.}  With this assumption,   we   combine Theorem \ref{thm: rate_cot}, Proposition~\ref{prop: kl bd} and Proposition~\ref{prop: pre-train} to obtain a complete characterization of the statistical error of \ac{cot} –– the statistical estimator obtained by first pretraining an LLM using dataset $\cD_{NT}$ and then prompting the pretrained LLM using a CoT prompt with $n$ examples.  
The result is given in the following corollary.

\begin{corollary}[Complete Characterization of $\mathtt{err}_\mathrm{cot}$] \label{cor: rate with error}
Recall that $\mathtt{err}_\mathrm{CoT}$ is defined in \eqref{eq: err cot}.
     Under Assumptions~\ref{assumption: seperation}, \ref{assumption: bd magnitude}, \ref{assumption: lower bd}, and \ref{assump: cover}, with probability at least $1-\delta$,  we have  
\begin{align*}
        &\E_{\PP_{\mathrm{CoT}}} [\mathtt{err}_\mathrm{CoT}] \nonumber \\
        &\quad \leq \calO\bigg(Hb^* \bigg(\frac{\pi(\Theta^\complement)}{\pi(\theta^*)}C(\theta^*)\kappa\bigg)^{1/4} \cdot \big| \Theta^\complement \big|^{1/2} \cdot \exp(-n\lambda/2) +  \kappa TH \pi(\theta^*)^{-1} \cdot (1+b^*) \cdot \Delta_{\mathrm{pre}}(N,T,\delta) \bigg)
\end{align*}
when 
$\Theta$ is finite, where $C(\theta^*) $ is defined as 
$$C(\theta^*) = \sup_{\theta\in \Theta^\complement}\sqrt{\chi^2\big(\PP( z_0^\mathrm{test}=\cdot\given \theta), \PP( z_0^\mathrm{test}=\cdot\given \theta^*)\big)+1}.$$
\end{corollary}
In this corollary, we combine the previous result for the perfectly pretrained model in Theorem~\ref{thm: rate_cot} with the pretraining error analysis in Proposition~\ref{prop: pre-train}, and the query error that quantifies the distributional shift of the prompt. The proof can be found in Appendix~\ref{proof: rate with error}. For conciseness, we stated the result for the case where $\Theta$ is discrete and finite. This can be generalized the result to the continuous case by applying the second part of Theorem~\ref{thm: rate_cot_cts}. 
This corollary shows that, when the distributional shift is mild ($\kappa= \cO(1)$) and $|\Theta | $ is finite, to achieve any desired accuracy level $\varepsilon \in (0, 1)$, it suffices to let:
\begin{itemize}
\item $D=H\Big(5B\cdot \log\big(\kappa T H \pi(\theta^*)^{-1}\cdot 3(1+b^*)/\epsilon \big)\Big)^4+C\log(2H)$ for the transformer depth, where the absolute constant $C>0$ is from Proposition~\ref{prop: kl bd}.
 \item $ n \geq 2/\lambda\cdot \log\bigg(3Hb^* \Big(\pi(\Theta^\complement)\cdot (\pi(\theta^*))^{-1} \cdot C(\theta^*)\cdot\kappa\Big)^{1/4} \cdot \big| \Theta^\complement \big|^{1/2}/\epsilon \bigg)$ in $\pt_{\mathrm{CoT}} (n) $. 
    \item $T \geq n +1 $, and $N \geq \Big(3\sqrt{b^*}\cdot \log(TH/\delta)\cdot \kappa T H \pi(\theta^*)^{-1}\cdot (1+b^*)/\epsilon\Big)^4$ in the pretraining dataset $\cD_{N,T}$.

\end{itemize}
Then Corollary \ref{cor: rate with error}  shows that $\E_{\PP_{\mathrm{CoT}}} [\mathtt{err}_\mathrm{CoT}] 
 \leq \varepsilon $ with probability at least $1 - \delta$.

    

\section{Experiments} \label{sec:experiments}

In this section, we provide empirical evidence to support the theory.  
We validate the statistical errors of the \ac{cot} estimator and compare it with vanilla ICL. 
To this end, we train transformer models from scratch, where the training data is sampled from random regression tasks satisfying \eqref{eq:latent_var_model}. The pretrained transformer is then tested on a new task via prompting.  We present the experiment results as follows.

\subsection{Experiment Settings} 

\vspace{3mm}

\noindent \textbf{Regression Tasks.}
We consider the in-context regression problem \citep{garg2022can} where the goal is to learn a class of functions $\cF = \{ f(\cdot; \theta) , \theta \in \Theta) \} $ via prompting a pretrained transformer. 
Here, $f(\cdot; \theta)$ is a function with parameter $\theta$. We consider two types of $\cF$ --   two-layer neural networks (NNs)  and decision trees.
For both cases, for any $\theta$, we generate input-output examples of form $(x, f(x,\theta))$, together with a single intermediate step, i.e., $H = 2$. 
Here we assume $x \sim \mathcal{N}(0, I_{d_{\mathrm{in}}})$ where $d_{\mathrm{in} }= 10$ for two-layer NNs and $d_{\mathrm{in} }= 20$ for decision trees. 
In other words, each \ac{cot} example is of the form $\{ x, z, f(x; \theta)\} $ and each \ac{icl} example is of the form $\{x, f(x; \theta)\} $, where $z $ is the intermediate step. 

\vspace{2mm}
{\noindent \bf Two-Layer Neural Networks.} 
For a  a two-layer NN, we write $\theta = \{ W, v\}$ where $v\in \RR^4$ and $W \in \RR^{10\times 4}$. 
Under the prior distribution, $W$ and $v$ are independent, with $v \sim \cN(0, I_4) $ and $W_{ij} \overset{i.i.d}\sim N(0, 1/2) $ for all $i\in [10]$ and $j\in [4]$. 
The neural network output is $f(x; \theta ) =v^T \texttt{ReLU}(Wx) $, where $\texttt{ReLU}(x) = \max (0,x)$ is the ReLu activation function. 
We consider two kinds of \ac{cot} examples, $
\mathrm{CoT}_1: \{ x, {\color{blue} \texttt{ReLU}(Wx)} , f(x; \theta)\} $ and $  \mathrm{CoT}_2: \{ x, {\color{blue} v} , f(x; \theta)\}. 
$

\vspace{2mm}
{\noindent \bf Decision Tree.} 
We let $\theta = \{ c_j, o_j\}_{j\in[15]} $  denote the parameters of binary decision tree of depth four,
where $c_j \sim \mathrm{Unif}([20]) $ and $o_j \overset{\text{i.i.d}}{\sim}\calN(0, 1)$. 
The output $f(x; \theta)$ with input $x\in\RR^{20}$ is defined as follows. 
Let $c_1$ correspond to the root node, $c_2, c_3$ be nodes of the second layer, $c_4, \ldots, c_7$ be nodes of the third layer, and $c_8, \ldots, c_{15}$ be nodes of the last layer. Each $c_j$ indexes a coordinate of $x$ and $o_j$ is the corresponding target value. 
Note that $d_{\mathrm{in}} = 20$, and thus we sample each $c_j$ uniformly over $[20]$. 
To evaluate the decision tree,  starting from the root node $c_1$, if  
$x[c_1] < 0$, we go to the left child $c_2$. Otherwise, we go to the right child $c_3$.  
Here we let $x[i]$ denote the $i$-th coordinate of $x$.
Then we continue to look at the sign of $c_2$ or $c_3$ and go to a child node in the third layer. 
We continue this process until a leaf node, i.e., a node in the last layer, is reached, and we output the corresponding $o_j$.
In other words, at each
level, we look at the sign of the corresponding coordinate of the input $x$ move to a child.
The output $f(x, \theta)$ corresponds to the number in $\{o_{8}, \ldots, o_{15}\}$ corresponding to the leaf node that is reached.   
The intermediate reasoning steps correspond to the four entries of $x$ used to make decisions. 
Thus, a \ac{cot} example is of the form $\{ x, {\color{blue}x[\mathtt{node}_1]},\ldots, {\color{blue}x[\mathtt{node}_4]}, f(x, \theta) \}$, where $\mathtt{node}_1, \ldots, \mathtt{node}_4$ are the four nodes of the selected path, including the root and leaf nodes.

\vspace{2mm}
{\bf \noindent  Transformer Model.} 
We train a decoder-only transformer from the GPT-2 family  \citep{radford2019language} separately for \ac{cot} and \ac{icl}. 
We construct pretraining datasets following the setting in Section \ref{subsubsec: pretraining process}. 
That is, we sample $N$ i.i.d. tasks $\{ \theta_{\ell}^*\}_{\ell \in [N]} $ and $T$ examples from each task, where $T = 101$ and $N = 2.56 \times 10^7$. 
Then we build a loss function similar to that in \eqref{eq:pretraining_loss}, but with the negative  loglikelihood replaced by the mean-squared error. 
The loss function is optimized using the Adam algorithm \citep{kingma2014adam}  for $4\times 10^5$ steps, with the batch size set to $64$. 
We let $\mathtt{TF}_{\mathrm{CoT}}$ and $\mathtt{TF}_{\mathrm{ICL}}$ denote the learned transformer model.  

\vspace{2mm}
{\noindent \bf Evaluation.} 
To evaluate the performances, we sample a random task $\theta^*$ and $n$ i.i.d. examples $\{ x_i, f(x_i; \theta^*) \}_{i\in [n]}$, and ask the pretrained transformers to predict $f(x^\mathrm{test}; \theta^*)$ on a new i.i.d. input $x^\mathrm{test}$.   
We include intermediate steps to these $n$ examples to obtain \ac{cot} prompt. 
We evaluate the performance of \ac{cot} and \ac{icl} via the mean-squared error (MSE): 
\begin{align*}
\mathtt{MSE}_{\mathrm{CoT}} & = \bigl [  \mathtt{TF}_{\mathrm{CoT}} \bigl (\pt_{\mathrm{CoT}}(n) \cup \{ \hat z \} \bigr )  - f(x^\mathrm{test}; \theta ^* ) \bigr ]^2, \\
\mathtt{MSE}_{\mathrm{ICL}} & = \bigl [ \mathtt{TF}_{\mathrm{ICL}} \bigl (\pt_{\mathrm{ICL}}(n) \bigr )  - f(x^\mathrm{test}; \theta^* ) \bigr ]^2 , 
\end{align*}
where $\hat z  = \mathtt{TF}_{\mathrm{CoT}} \bigl (\pt_{\mathrm{CoT}}(n) ) $ is the intermediate step predicted by the transformer on the test example. 
We compute the MSE by 
averaging over 100 independent experiments with a fixed $\theta^*$ for each $n$, and we test for all $n$ in $\{1, \ldots, 100\}$. 
We plot the final MSE by further averaging over $10$ independent $\theta^*$'s. 

\subsection{Experiment Results} 
{\noindent \bf Two-Layer Neural Network.} 
We plot the errors of the \ac{cot} and \ac{icl} estimators in Figure \ref{fig:RELU_last_iterate_plot} against the number of demonstration examples $n \in [100]$. 
The curves with labels ``$\mathrm{CoT}_1$'' and ``$\mathrm{CoT}_2$'' are the transformer trained using the two kinds of \ac{cot} prompts, respectively, and the label `` $\mathrm{vanilla}$ $\mathrm{ICL}$'' stands for the error of the vanilla ICL. 
In Figure \ref{fig:RELU_last_iterate_plot}-(a) we plot the MSE of these methods. In all these three cases, MSE decays rapidly to zero as $n$ increases. Moreover, we observe that $\mathrm{CoT}_1$ method with $z = \mathtt{ReLU}(Wx)$ exhibits significant improvement over vanilla ICL.
Moreover,   $\mathrm{CoT}_2$ only slightly improves upon vanilla ICL, which shows that the quality of intermediate steps is crucial to the success of \ac{cot}. This finding coincides with our theory and experiment in Section \ref{subsubsec: comparison}. 
A plausible explanation for the superiority of $\mathrm{CoT}_1$ over $\mathrm{CoT}_2$ is that the major challenge of learning a two-layer NN lies in learning the nonlinear feature $\mathtt{ReLU}(Wx)$. 
Providing this piece of information in the prompt significantly simplifies the learning problem. 
Furthermore, in Figure \ref{fig:DT_last_iterate_plot}-(b) we plot the logarithm of the MSE versus $n$. As seen in the figure, there is a linear trend for all three methods when $n$ is smaller than $20$. When $n$ exceeds $20$, MSE is very close to zero. 
In this case, the pretraining error is not negligible and thus the linear trend stops. Thus, Figure (b)  shows that the statistical error of these methods decays exponentially in $n$ up to a pretraining error. This observation corroborates Theorem \ref{thm: rate_cot}.

\begin{figure}[t]%
    \centering
    \subfloat[\centering MSE]{{\includegraphics[width=0.44 \textwidth]{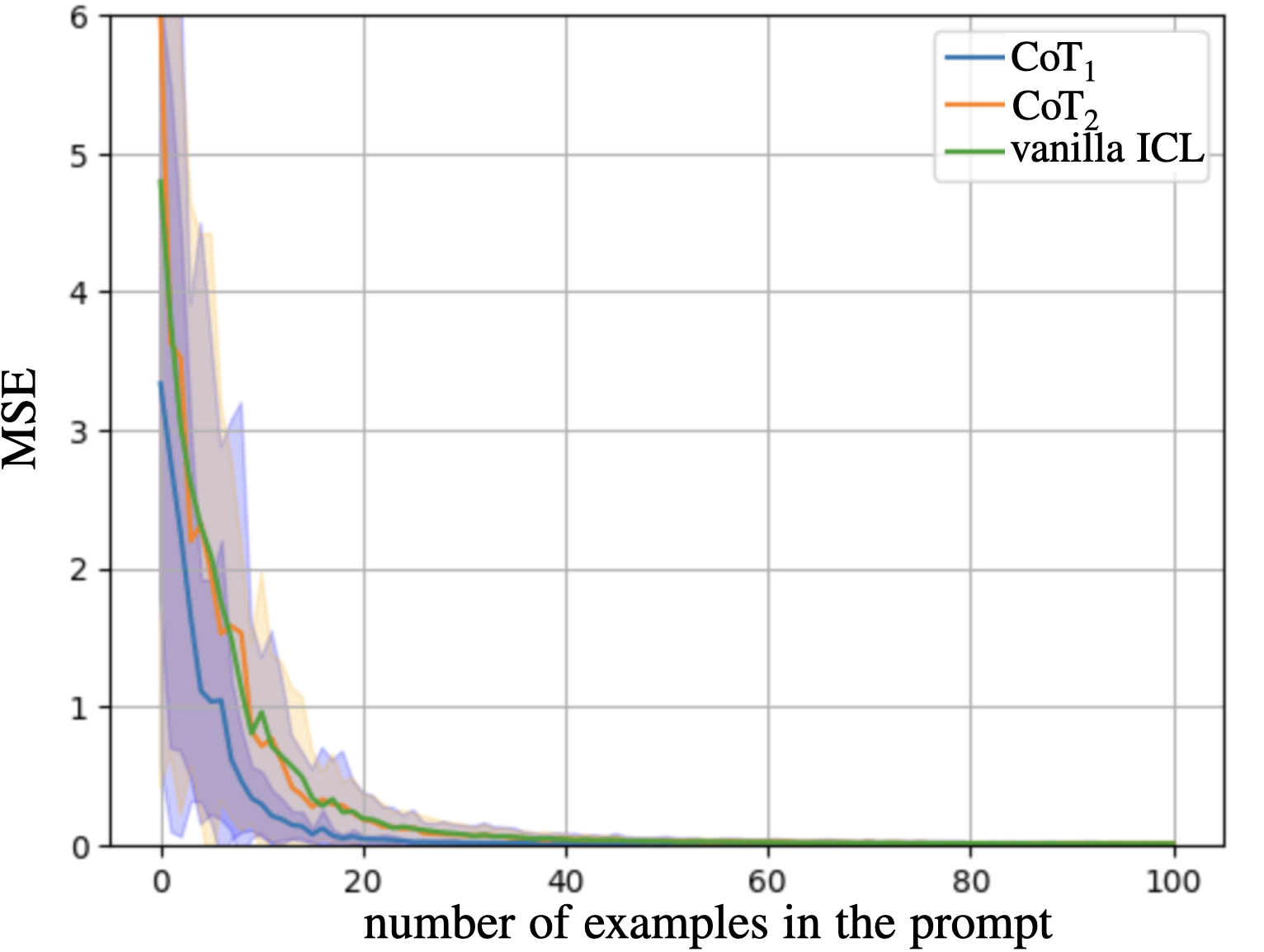} }}%
    \qquad
    \subfloat[\centering Logarithm of MSE]{{\includegraphics[width=0.45 \textwidth]{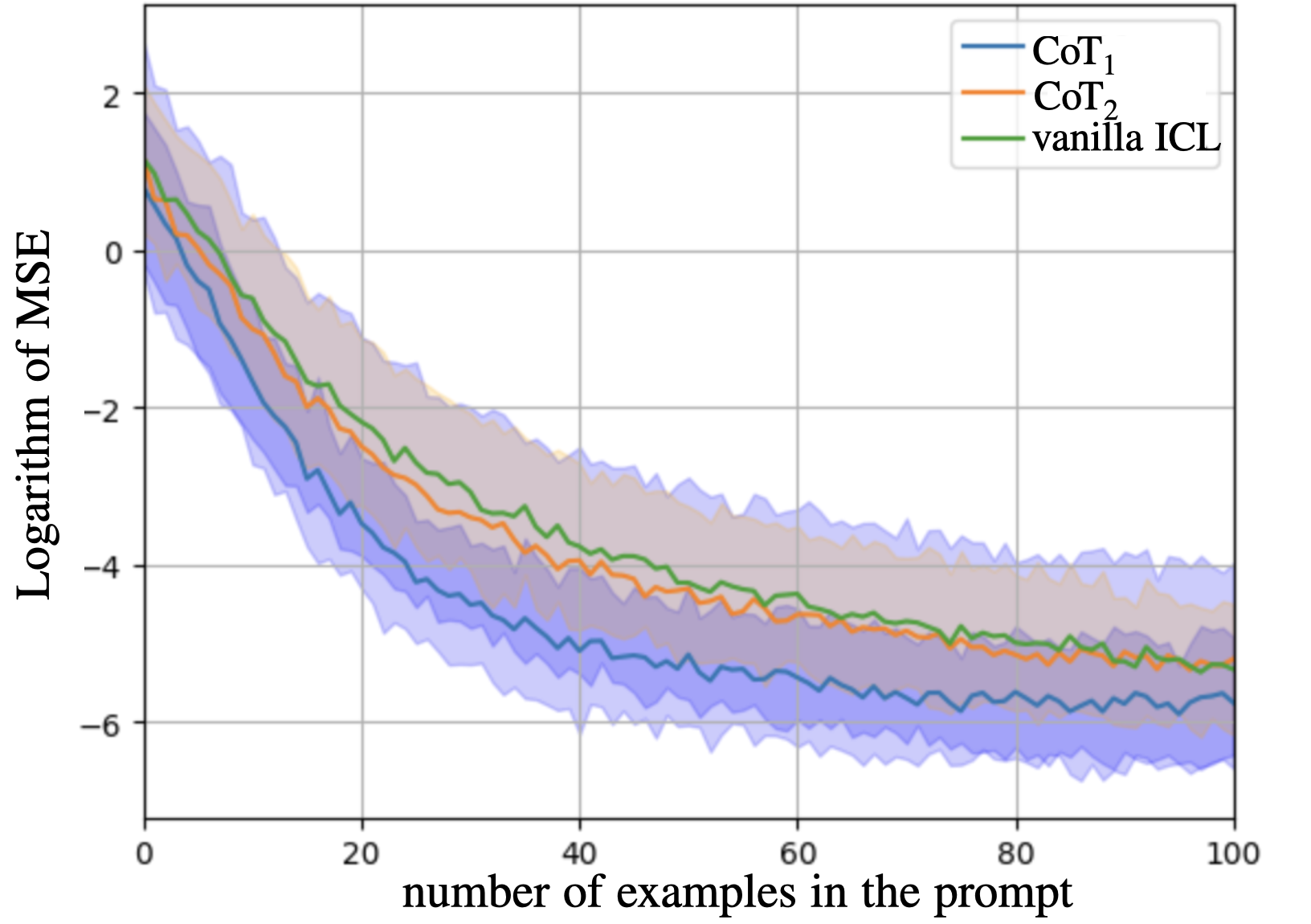} }}%
    \caption{Statistical errors of vanilla \ac{icl} and two \ac{cot} estimators on the task of learning 2-layer MLPs. The estimators are constructed by prompting the pretrained transformer models 
    with $n$ demonstration examples. In Figure (a), the MSE of the three estimators decreases rapidly as $n$ increases. In Figure (b) we plot the logarithm of MSE, which follows a strong linear trend for $n\leq 20$. Beyond $n = 20$, the MSE is dominated by the pretraining error and therefore the linear trend stops. Furthermore, the model trained with ``$\text{CoT}_1$'', using intermediate steps \( z = \mathtt{ReLU}(Wx) \), shows a significant improvement over vanilla ICL, while the ``$\text{CoT}_2$'' model, using \( z' = v \), only shows a slight improvement. In both (a) and (b), we also plot the standard deviation computed based on $10$ random experiments. }%
    \label{fig:RELU_last_iterate_plot}%
\end{figure}

\vspace{2mm}

{\noindent \bf Decision Tree.} 
We plot the errors of \ac{cot} and vanilla \ac{icl} in Figure \ref{fig:DT_last_iterate_plot}, with (a) and (b) showing the MSE and its logarithm respectively. In both cases, MSE decays rapidly as $n$ increases and there is a strong linear trend when $n \leq 65$. This aligns with our theory in Theorem \ref{thm: rate_cot}, and the MSE after $n \geq 65$ is close to the pretraining error.
Moreover, we observe that the transformer trained via the \ac{cot} method learns faster than that trained by vanilla \ac{icl}, but they achieve similar accuracy when $n$ becomes large. When given $n=100$ in-context demonstrations, both models give a testing loss of around $0.12$.
\begin{figure}[t]%
    \centering
    \subfloat[\centering MSE]{{\includegraphics[width=0.46 \textwidth]{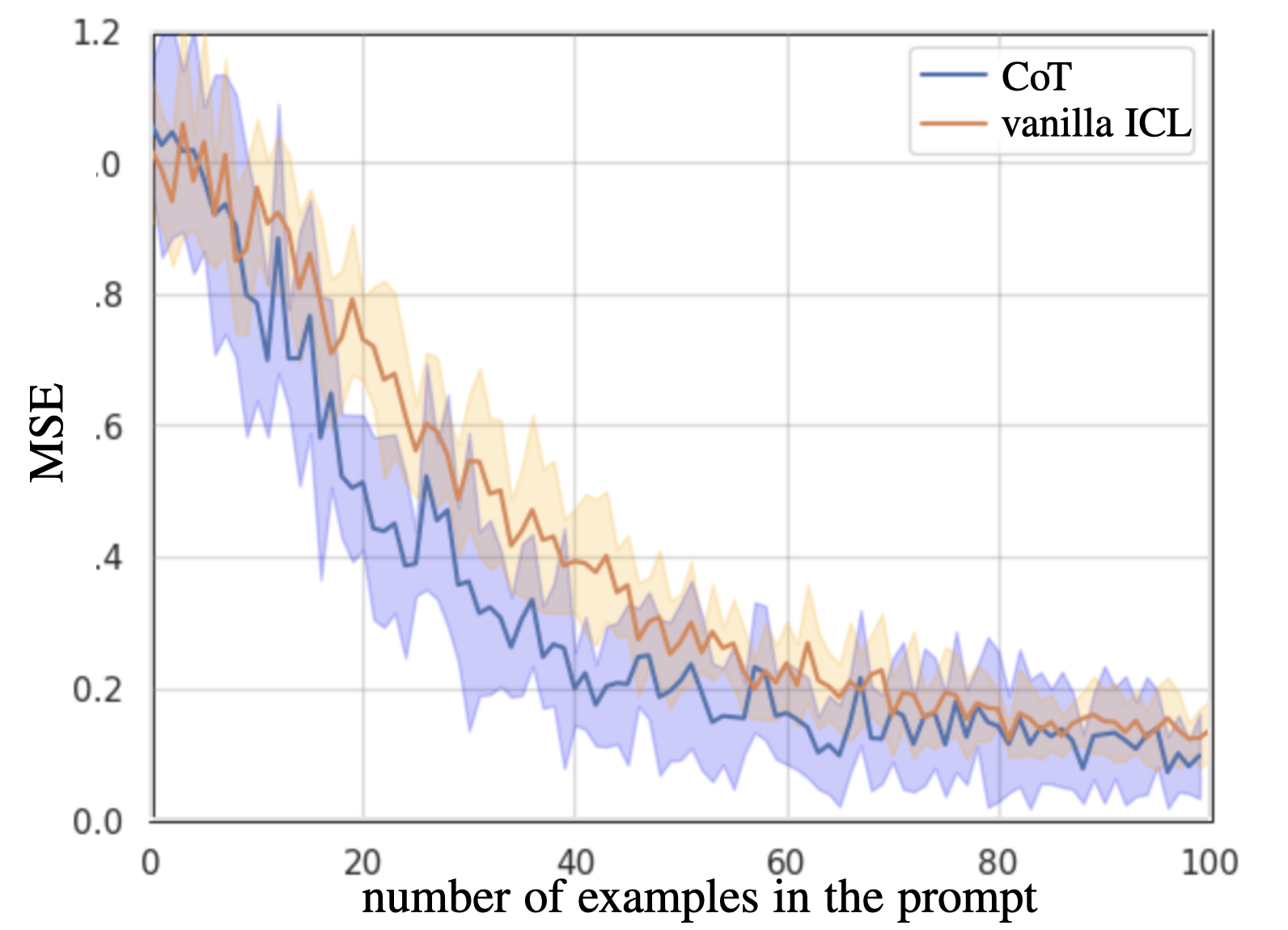} }}%
    \qquad
    \subfloat[\centering Logarithm of MSE]{{\includegraphics[width=0.45 \textwidth]{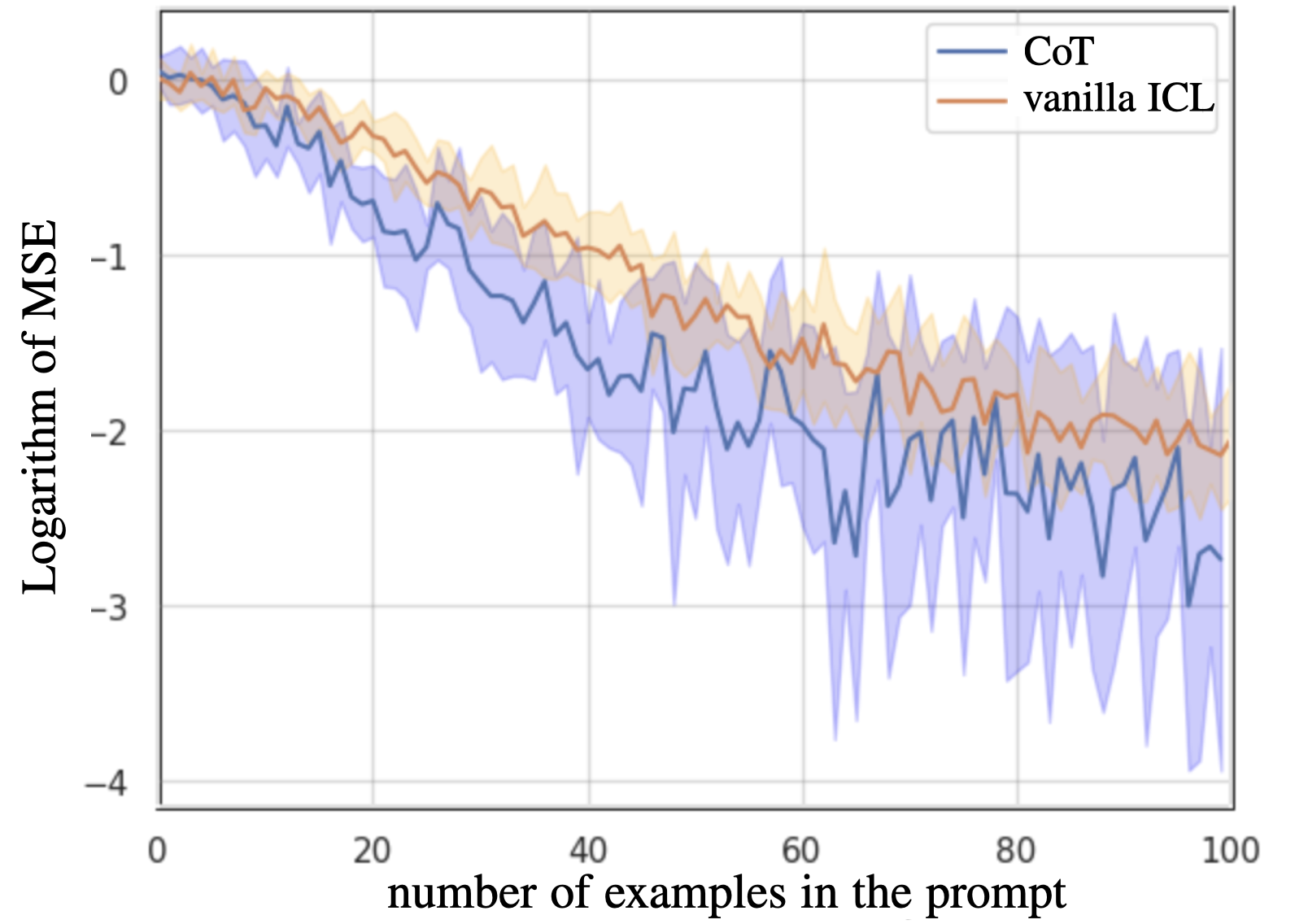} }}%
    \caption{Statistical errors of vanilla \ac{icl} and two \ac{cot} estimators on the task of  learning decision trees. 
    The estimators are constructed by prompting the pretrained transformer models 
    with $n$ demonstration examples. 
In (a), we plot the mean-squared errors of these two estimators against the number of demonstrations. The error decreases faster with in-context examples for models trained with CoT data compared to those trained with vanilla ICL. In Figure (b) we plot the logarithm of the MSE, which reveals a linear trend in errors for both models when $n \le 65$. }%
    \label{fig:DT_last_iterate_plot}%
\end{figure}

In summary, our experiments on the two-layer neural network and decision tree tasks validate the statistical error of the ``pretraining LLM + \ac{cot} prompting'' approach. 
We empirically validate the exponential error decay with respect to the number of prompt examples, and we show whether \ac{cot} significantly outperforms vanilla \ac{icl} depending on the choice of intermediate reasoning steps. 
We leave further details of the experiment setup in Appendix~\ref{sec:experiments_additional_details}.

\subsection{Additional Details of Numerical Experiments} \label{sec:experiments_additional_details}
In the following, we present the details of the numerical experiments in Section \ref{sec:experiments}. 

\vspace{2mm}
{\noindent \bf Training Data and Algorithm.} 
The transformer models are pretrained using the Adam algorithm \cite{kingma2014adam}, which is a minibatch and stochastic-gradient-based algorithm. 
We set the batch size to $64$, and in each step, we sample new training data from the model in \eqref{eq:latent_var_model}. 
That is, we sample $64$ random tasks from the prior distribution, and $T = 101$ examples from each task. 
In each task, we pack the $T$ examples into a single trajectory, and build the MSE loss function by predicting the next step autoregressive. 
For vanilla ICL, there are $2T$ steps in total, and for \ac{cot}, there are $3T$ steps in total. 
We run Adam for $4\times 10^5$ steps in total, and thus the total tasks sampled is equal to $N = 2.56\times 10^7$. 
Moreover, when implementing Adam, we set the learning rate (stepsize) to be $10^{-5}$ and the momentum parameter to be $(0.9,0.999)$. 

\vspace{2pt}
{\noindent \bf GPT-2 Transformer Model.} We adopt the GPT-2 transformer architecture. 
Here the transformer model reads in in a sequence of input vectors in $\RR^r$ and produces an output vector in the same space, where $r$ is the embedding dimension. 
 Additionally, to handle the vector-valued inputs with different sizes, 
we adopt a universal read-in function that maps the prompts into the latent embedding space of the transformer through a (learnable) linear transformation and we use separate read-out functions for the predictions of inputs $x \in \RR^{d_{\mathrm{in}}}$, intermediate steps $z\in \RR^4$, and outputs $y \in \RR$. 
Here $d_{\mathrm{in}} = 10$ for two-layer NNs and $d_{\mathrm{in}} = 20$ for decision trees. 
These read-in and read-out functions are all linear layers. 
Between these read-in and read-out functions are multiple transformer blocks stacked vertically. The details of these transformer blocks are introduced in Section \ref{subsubsec: pretraining process} and Appendix  \ref{app: pre-training process}. See Figure \ref{fig:transformer_arch_experiments} for an illustration of the transformer architecture. 

For the two-layer NN task, 
we adopt GPT-2 transformer models with $D = 12$ layers, $\eta = 4$ heads, and an embedding dimension of  $r= 128$. The embedding dimension of the queries, key, and values are $32$.
For the decision tree task, we adopt  GPT-2 transformer models consisting of $D = 12$ layers, $\eta = 8$ heads, and an embedding dimension of  $r = 512$. The embedding dimension of the queries, key, and values are blue $32$.
\begin{figure}[t]
    \centering
    \includegraphics[width = 0.95\textwidth]{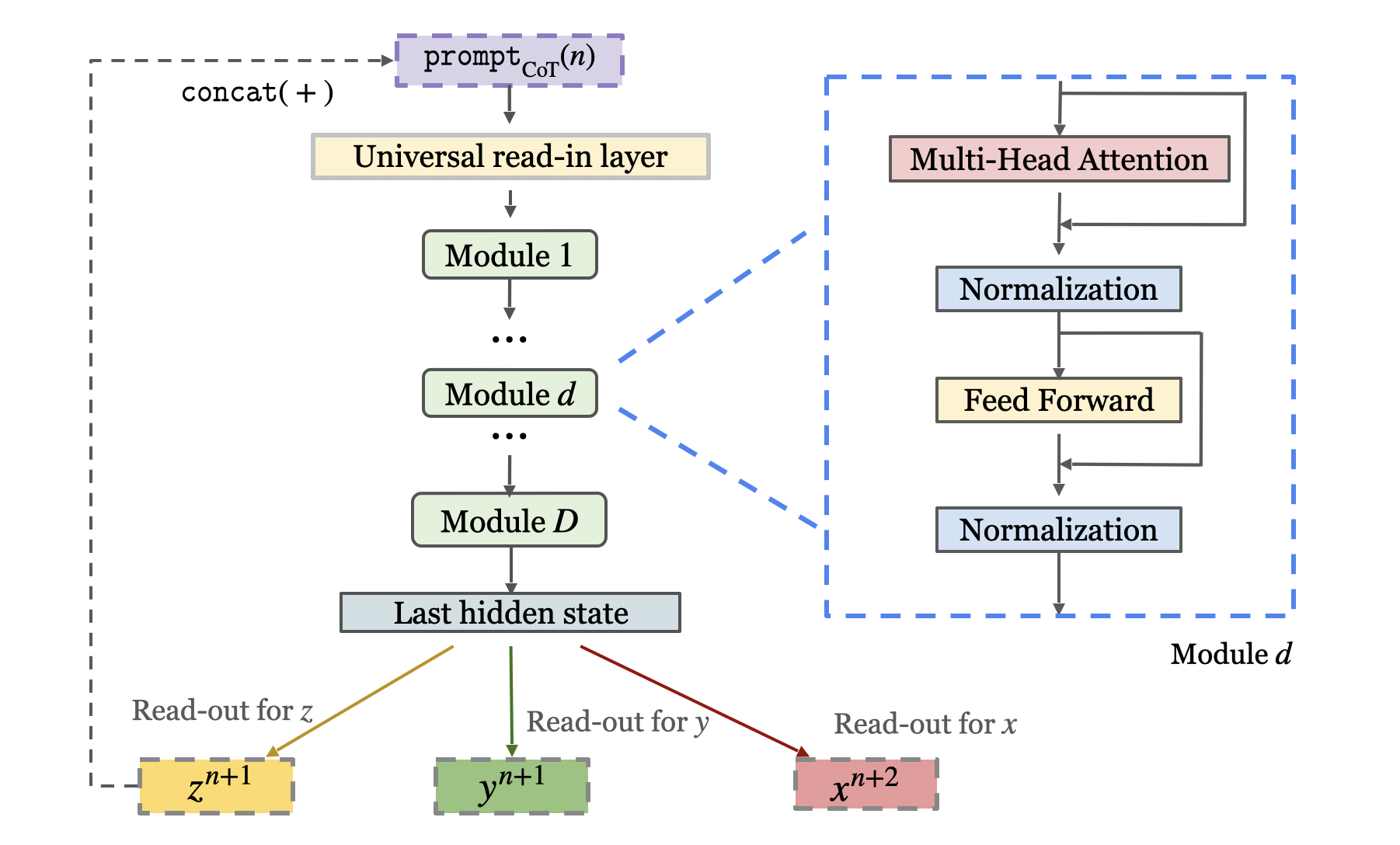}
    \caption{An illustration of the GPT-2 based transformer model. We begin by embedding the input prompt using a universal read-in function. The embedding is then processed by our backbone model, which is based on the GPT-2 architecture. We extract the last hidden state and apply separate read-out functions for inputs \( x \), intermediate steps \( z \), and outputs \( y \). During testing, we concatenate the inferred \( z^\mathrm{test} \) with \( \pt_\mathrm{CoT}(n) \) to infer \( y^\mathrm{test} \).}
    \label{fig:transformer_arch_experiments}
\end{figure}

\vspace{2cm}

\section{Conclusion}
In this paper, we explore the theoretical underpinnings of \ac{cot} prompting and its variants through a statistical lens. 
In particular, under a latent variable model that depicts multi-step reasoning,  we showed the estimators induced by  CoT prompting on a pretrained LLM is approximately equivalent to a Bayesian estimator.
More importantly, we prove that the statistical error of CoT can be upper bounded by a sum of pretraining error and promoting error, and we explicitly analyze them separately. 
In particular, we prove that prompting error decreases exponentially with the increasing number of demonstrations
included in the prompt, and the statistical error of the pretrained LLM is analyzed under the PAC-Bayes framework. 
We also extend our analysis to various CoT variants and establish exponential rates of convergence. 
Moreover, we establish both theoretical and empirical comparisons between CoT and vanilla ICL, which shed new light on the role played by intermediate reasoning steps. 
In future work,
we hope to extend our theoretical framework to better understand other prompting methods beyond CoT, shedding light on their effectiveness and potential for improvement.

\section*{Acknowledgement}
Zhuoran Yang acknowledges the support of NSF under the award DMS-2413243. 

\newpage
\bibliographystyle{ims}
\bibliography{ref.bib}

\newpage
\appendix

\newpage

\section{A Generalized Multi-Step Latent Variable Model} \label{app: generalized model}
In this section, we propose a  \textbf{generalized multi-step latent-variable model} that removes the i.i.d requirement in the model defined in~\eqref{eq:latent_var_model}. This generalized model captures \textbf{(i)} 
the evolving relationships among examples, and \textbf{(ii)} the multi-step reasoning framework of \ac{cot}. 
The rationale behind \textbf{(i)} is that 
\ac{llm}s are pretrained on trillions of reasoning steps from documents and articles from the internet \citep{openai2023gpt}. The examples in the pretraining data associated with the same task are often \textbf{not i.i.d.}
For example, imagine composing a sequence
of examples of the concept of ``animals", we typically begin with familiar
examples such as cats and dogs before progressing to less conventional ones like panthers and
meerkats. 

To capture both \textbf{(i)} and \textbf{(ii)}, we propose to model the joint distribution of the latent variable $\theta^* \in \Theta$, the $n$ examples $\{s^j\}_{j=1}^n$, and the test instance $z_{0:H}^\mathrm{test}$ as follows:
\begin{align} 
    \begin{split}
    \theta^* &\sim \pi 
       , \qquad \qquad \qquad \qquad\qquad \qquad \qquad \quad  \: ~  z_0^i  = f_{\theta^*}\big(\{s^j\}_{j=1}^{i-1},\zeta_i\big), \label{eq:generalized_latent_var_model}\\
       u_h^i &= g_{\theta^*}\big(\{u_k^j\}_{j=1,k=1}^{i-1, H},u_1^i,\cdots,u_{h-1}^i,\xi_h^i\big), \qquad 
       z_h^i  = F(z_0^{i},\cdots, z_h^i, u_h^i,\epsilon_h^i), \qquad \forall  h \in [H].  
    \end{split}
   \end{align}
   Here $f_{\theta^*}$ and $g_{\theta^*} $ are functions depending on $\theta^*$, 
 $\zeta_i$, $\xi_h^i$, and $\epsilon_h^i$ are independent noise terms, and $F$ is a function that does not depend on $\theta^*$. 
 The key assumption of \eqref{eq:generalized_latent_var_model} is that each reasoning step $z_h^i$ depends on the task parameter $\theta^*$ only through a latent variable $u_h^i$, and these latent variables  $\{u_h^i\}_{h=1, i=1}^{H, n}$ form a dynamical system in the latent space.
 The evolution of $u_h^i $ depends on $\theta^*$ and all the latent variables of the previous examples $\{ s^j \}_{j < i}$.
 Thus, each $z_h^i$ is allowed to implicitly depend on the previous examples as well. 
 Such a latent dynamical system captures the fact that the demonstration examples are created with a dependent structure. 
 In comparison, in the simpler model in 
\eqref{eq:latent_var_model}, 
the examples 
are i.i.d. given $\theta^*$. 
  We recover the simpler model in~\eqref{eq:latent_var_model} by setting $u_h^i = \theta^*$ and $z_0^i  = f_{\theta^*}\big(\zeta_i\big)$ in \eqref{eq:generalized_latent_var_model}.

  Intuitively, \(\theta^*\) represents the latent concept specifying the task, such as ``calculate twice the area code of the given country," ``solve an arithmetic problem," or ``write a science fiction novel." 
  In \eqref{eq:generalized_latent_var_model}, we model LLM-based reasoning as a hierarchical process: first, we generate a sequence of task-specific latent ``goals" \(\{u_h^i\}_{h \in [H]}\), then we translate these goals into natural language \(\{z_h^i\}_{h \in [H]}\). Each goal corresponds to a specific reasoning step. The sequence of latent variables completely determines the reasoning process, and thus we assume  \(F\) in \eqref{eq:generalized_latent_var_model} does not involve \(\theta^*\). See Figure~\ref{fig:general_model} for an illustration of this model.

\begin{figure}[t]
    \centering
    \includegraphics[width=0.9\textwidth]{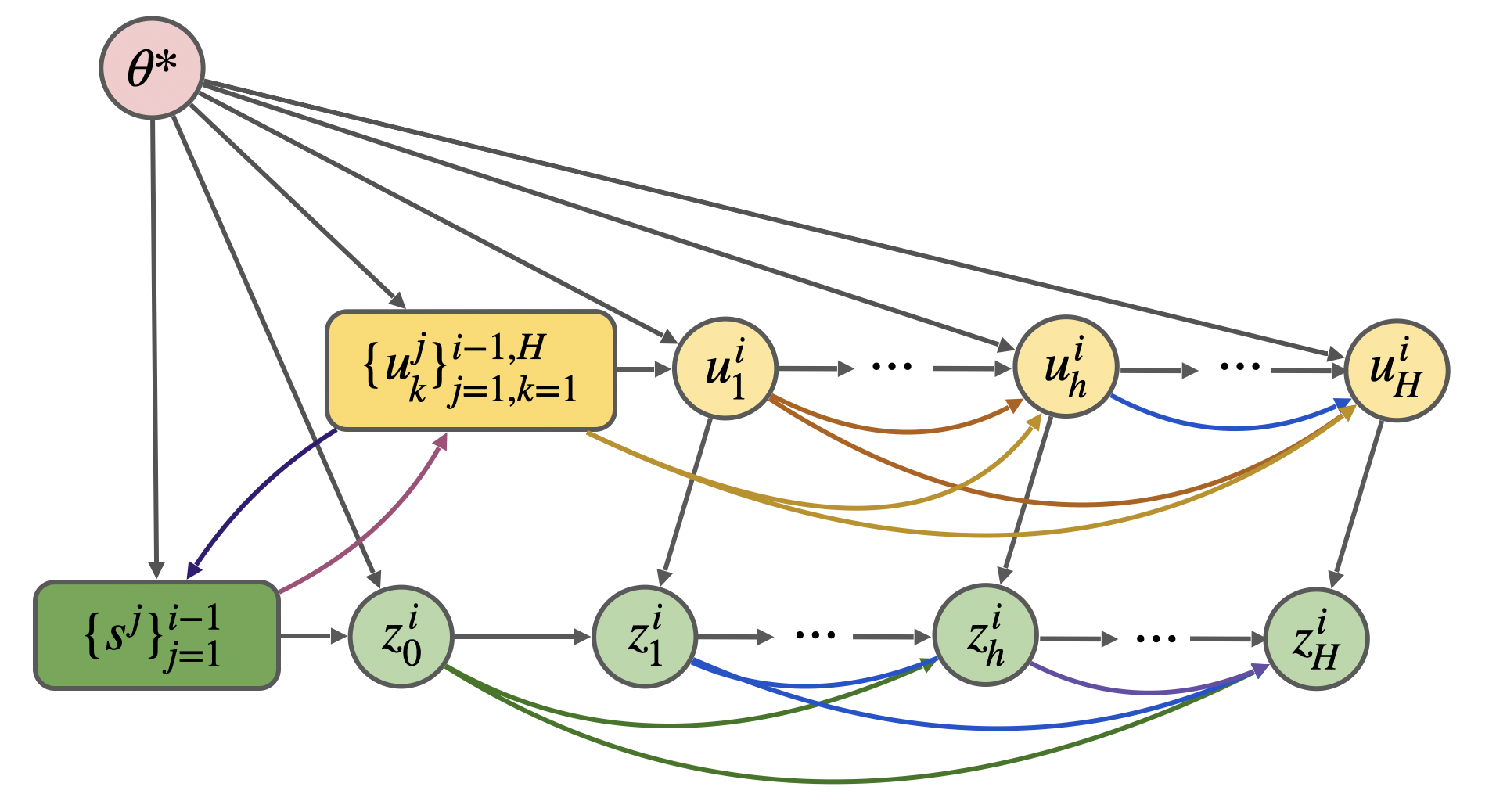}
    \caption{An illustration of the generalized multi-step latent-variable model~\eqref{eq:generalized_latent_var_model} with the fixed task $\theta^*$. Compared to the simple model~\eqref{eq:latent_var_model}, this model allows each step $z_j^i$ of $i$-th example to depend on earlier examples via $u_j^i$, which depends on all latent variables previously generated. Therefore, the examples $\{s^i\}_{i=1}^n$ do not have to be i.i.d.}
    \label{fig:general_model}
\end{figure}

As a concrete example, consider \(\theta^*\) as the task of ``solving an arithmetic problem".  The problem might be described in a specific context, such as using the number of apples. However, the underlying reasoning process is independent of this context and relies solely on a sequence of arithmetic operations. These operations can be viewed as the latent variables \(u_h^i\) in our model, while describing them in the context of apples can be seen as generating the natural language \(z_h^i\) from these latent variables. Additionally, the random variables \(\zeta^i, \epsilon_h^i\) (for \(h \in [H]\)) determine how the arithmetic formula is contextualized. The random variables \(\xi_h^i\) introduce stochasticity into the reasoning process. See Figure~\ref{fig:interpret_latent_variable_model_general} for an illustration.
\begin{figure}[t]
    \centering
    \includegraphics[width = 0.95\textwidth]{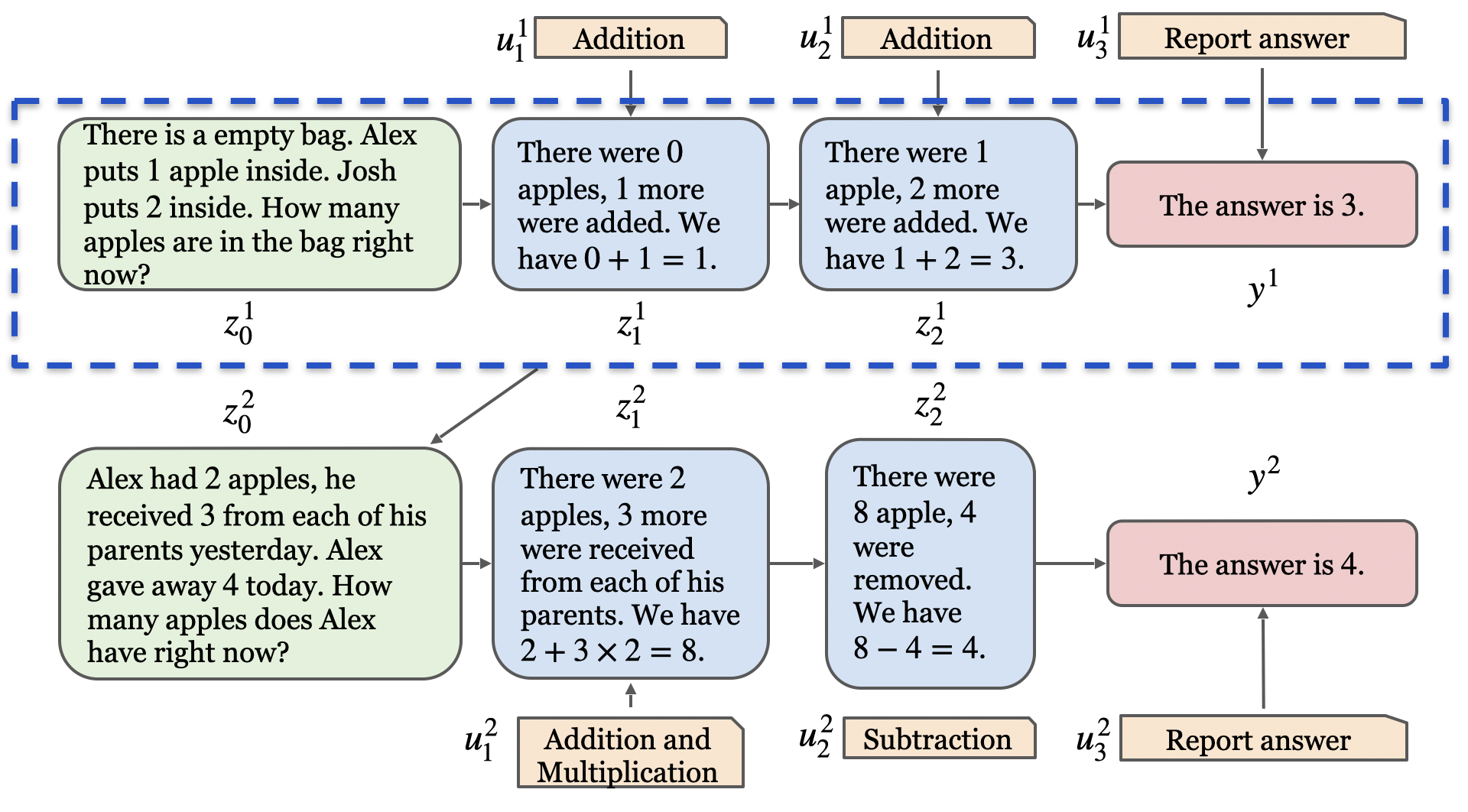}
    \caption{An example of $2$ demonstrations generated from the generalized model~\ref{eq:generalized_latent_var_model}, where we omit some edges in the graph to simplify the representation. In this example, we want to highlight two features of our model. The first and second rows correspond to two demonstrations. Note that the reasoning steps $\{z_h^i\}$ depend on latent concept $\theta^*$ implicitly via latent variables $\{u_h^i\}$, which give specific instructions for each step in the corresponding example. Furthermore, the latent variables $u^1_{1:3}$ are different from $u^2_{1:3}$, providing more flexibility and diversity among the demonstrations. In addition, the second query $z_0^2$ is not independent of the previous demonstration $z_{0:3}^1$ since it is a more complicated version of an arithmetic problem.}
    \label{fig:interpret_latent_variable_model_general}
\end{figure}


We remark that our model is a general formulation that recovers many existing models proposed in the existing works. Specifically, we recover the models in \cite{jiang2023latent} by setting the latent variables $(u_1^i,\cdots,u_H^i)$ directly as the latent variable vector $\theta^*$. Besides, we cover the models studied in \cite{zhang2023and, wang2023large} by setting $H=1$.

Finally, some of our results extend to the generalized model~\ref{eq:generalized_latent_var_model}. Specifically, Lemma~\ref{lem: bma_cot_2} demonstrates that LLMs perform BMA during CoT under this generalized model. In Section~\ref{subsec: imperfect}, we describe the pretraining process and analyze its performance within the framework of the generalized model~\eqref{eq:generalized_latent_var_model}.

\clearpage 

\section{Proofs of the Results in  Section~\ref{sec: methodology}}
In this section, we prove the results in Section \ref{sec: methodology}.
We first prove Lemma~\ref{lem: bma.cot} and its extension to the generalized multi-step latent variable model  and then prove Proposition~\ref{prop: attn}.

\subsection{Proof of Lemma~\ref{lem: bma.cot} and Its Extension} \label{proof: bma}

We prove Lemma~\ref{lem: bma.cot} and introduce its extension to the generalized multi-step latent variable in \eqref{eq:generalized_latent_var_model} with a proof.

\begin{proof}[Proof of Lemma~\ref{lem: bma.cot}]
When $N$ goes to infinity, we only need to consider the population counterpart of the likelihood loss in \eqref{eq:pretraining_loss},
which can be written as 
\begin{align}
   \cL(\rho)
    &=   -\frac{1}{T (H+1)}\sum_{t\in [T]} \sum_{h=0}^{H} \EE\left[\log \PP_\rho \big (z_{h}^{t} \given \Upsilon_{t-1},\{z_j^{t}\}_{j=0}^{h-1} \big )\right] \label{eq:population_kl_error}\\
    &= \frac{1}{T (H+1)}\sum_{t\in [T]} \sum_{h=0}^{H} \EE\Big[\kl\big(
    \PP \big (\cdot \given \Upsilon_{t-1},\{z_j^{t}\}_{j=0}^{h-1} \big )
    \,\|\,\PP_\rho \big (\cdot \given \Upsilon_{t-1},\{z_j^{t}\}_{j=0}^{h-1} \big )  \big)\Big] + \mathtt{Const}, \notag  
\end{align}
where $\mathtt{Const}$ denotes a constant that does not depend on $\rho$.
Here the second equality follows from the definition of KL divergence. 
When the LLM class is sufficiently expressive, i.e., $\PP\in \{\PP_{\rho} \given \rho\in \cP_{\mathrm{LLM}}\}$, for any $(\Upsilon_{t-1},\{z_j^{t}\}_{j=0}^{h-1})$ that has nonzero density in the pretraining distribution, 
the minimizer of $\cL(\rho)$ in \eqref{eq:population_kl_error} must satisfy 
\begin{align}
\label{eq:distribution_match_population}
    \PP_{\mathrm{LLM}} \big (z_h^{t} = \cdot \given \Upsilon_{t-1},\{z_j^{t}\}_{j=0}^{h-1} \big ) =  \PP \big (z_h^{t} = \cdot \given \Upsilon_{t-1},\{z_j^{t}\}_{j=0}^{h-1} \big )
\end{align}
for all $h\in[H]$ and $t\in [T]$,   
where we let $\PP_{\mathrm{LLM}} $ denote the distribution induced by the minimizer of $\cL(\cdot)$.
Here \eqref{eq:distribution_match_population} holds for all prompts $\Upsilon_{t-1} \cup \{z_j^{t}\}_{j=0}^{h-1} \in \cL^*$ with a nonzero density under $\PP$. 
Thus, \eqref{eq:distribution_match_population} shows that the perfectly pretrained LLM matches the pretraining distribution of predicting the \emph{next reasoning step}. 

In the following, we will prove the desired result by generalizing \eqref{eq:distribution_match_population} to the next multi-step reasoning. Now, when we generate the answer $y^\mathrm{test}$ given the \ac{cot} prompt using $\PP_{\mathrm{LLM}} $,
when $\mathtt{prompt}_{\mathrm{CoT}}(n)$ has a positive density under $\PP$,
by \eqref{eq:distribution_match_population} we have 
\begin{align} \label{eq:match_first_step}
\PP_{\mathrm{LLM}} \big (z_1^\mathrm{test} = \cdot \given \mathtt{prompt}_{\mathrm{CoT}}(n) \big ) =  \PP \big (z_1^\mathrm{test} = \cdot \given \mathtt{prompt}_{\mathrm{CoT}}(n)\big ).
\end{align}
Let $\tilde \cL_1  \subseteq \cL$ denote the subset of reasoning steps such that the conditional distribution in \eqref{eq:match_first_step} is positive when $z_1^\mathrm{test} \in \tilde \cL_1$, and let $\tilde \cL_1^ \complement $ denote its complement. 
Since $\PP\big(\mathtt{prompt}_{\mathrm{CoT}}(n)   ) > 0$, we have 
$$\PP( \mathtt{prompt}_{\mathrm{CoT}}(n) \cup \{ z_1^\mathrm{test} = z_1 \} \big) > 0, \qquad \forall z_1  \in \tilde \cL_1.
$$
Then using \eqref{eq:distribution_match_population}, 
since $\mathtt{prompt}_{\mathrm{CoT}}(n) \cup \{ z_1^\mathrm{test} = z_1 \}$ has positive density under $\PP$, 
we have 
$$
\PP_{\mathrm{LLM}} \big (z_2^\mathrm{test} = \cdot \given \mathtt{prompt}_{\mathrm{CoT}}(n), z_1^\mathrm{test} = z_1  \big ) = \PP  \big (z_2^\mathrm{test} = \cdot \given \mathtt{prompt}_{\mathrm{CoT}}(n), z_1^\mathrm{test} = z_1  \big ), \forall z_1 \in\tilde \cL_1 .
$$
Moreover, we have $\PP_{\mathrm{LLM}}  (z_1^\mathrm{test} = z_1 \given \mathtt{prompt}_{\mathrm{CoT}}(n) ) = 0$ for all $z_1 \in \tilde \cL_1^\complement$.
Therefore, by direct computation, we have 
 \begin{align*}
& \PP_{\mathrm{LLM}}\big(z_2^{\mathrm{test}}=\cdot \given \mathtt{prompt}_{\mathrm{CoT}}(n)\big) \notag \\
& \qquad = 
\int_{\tilde \cL_1} \PP_{\mathrm{LLM}}\big(z_2^{\mathrm{test}}=\cdot \given \mathtt{prompt} _{\mathrm{CoT}}(n),  z_1^{\mathrm{test}}= z_1 \big) \cdot \PP_{\mathrm{LLM}}\big(z_1^{\mathrm{test}}=z_1  \given \mathtt{prompt} _{\mathrm{CoT}}(n)     \big)\ud z_1 \\
&\qquad = 
\int_{\tilde \cL_1} \PP \big(z_2^{\mathrm{test}}=\cdot \given \mathtt{prompt} _{\mathrm{CoT}}(n),  z_1^{\mathrm{test}}= z_1 \big) \cdot \PP \big(z_1^{\mathrm{test}}=z_1  \given \mathtt{prompt} _{\mathrm{CoT}}(n)     \big)\ud z_1 \\
&\qquad = \PP \big(z_2^{\mathrm{test}}=\cdot \given \mathtt{prompt}_{\mathrm{CoT}}(n)\big).
\end{align*}

We can generalize this argument to $z_H^\mathrm{test} = y^\mathrm{test}$. 
Specifically, let $\tilde \cL^*$ denote the reasoning steps $\{z_1^\mathrm{test}, \ldots, z_{H-1}^\mathrm{test}\}$ with positive density under $ \PP  ( \cdot   \given \mathtt{prompt} _{\mathrm{CoT}}(n)      )$. 
Let each element in $\tilde \cL^*$ be denoted by $\zb = \{ z_1, \ldots, z_{H-1}\}$.
Then we have 
\begin{align}
\label{eq:finalize_matching_1}
& \PP\big(y^{\mathrm{test}}=\cdot \given \mathtt{prompt}_{\mathrm{CoT}}(n)\big) \notag \\
& \qquad = \int_{\zb \in \tilde \cL^*} \PP\big (y^{\mathrm{test}}=\cdot, \{z_1^\mathrm{test}, \ldots, z_{H-1}^\mathrm{test} \} = \zb \given  \mathtt{prompt}_{\mathrm{CoT}}(n) \big) ~\ud \zb \notag \\
& \qquad  =  \int_{\zb \in \tilde \cL^*} \prod_{h=1}^H\PP\big(z_h^{\mathrm{test}}=\cdot \given \mathtt{prompt}_{\mathrm{CoT}}(n), z_1^\mathrm{test} = z_1, \ldots, z_{h-1}^\mathrm{test} = z_{h-1} \big)\ud z_1 \cdots \ud z_{H-1}. 
\end{align}
Here the second inequality follows from the factorization of joint probability into a product of conditional probabilities. 
Since $\PP(\mathtt{prompt}_{\mathrm{CoT}}(n), z_1^\mathrm{test} = z_1, \ldots, z_{h-1}^\mathrm{test} = z_{h-1})  >0$ by  the definition of $\tilde \cL^*$, 
by \eqref{eq:distribution_match_population} we have 
\begin{align} \label{eq:finalize_matching_prob2}
    & \PP\big(z_h^{\mathrm{test}}=\cdot \given \mathtt{prompt}_{\mathrm{CoT}}(n), z_1^\mathrm{test} = z_1, \ldots, z_{h-1}^\mathrm{test} = z_{h-1} \big) \notag \\
& \qquad = \PP_{\mathrm{LLM} } \big(z_h^{\mathrm{test}}=\cdot \given \mathtt{prompt}_{\mathrm{CoT}}(n), z_1^\mathrm{test} = z_1, \ldots, z_{h-1}^\mathrm{test} = z_{h-1} \big).
\end{align}
Therefore, this equality implies that $\PP_{\mathrm{LLM}}$ cannot assign any positive density on the complement of $\tilde \cL^*$, otherwise \eqref{eq:finalize_matching_prob2} is violated. 
Thus, combining  \eqref{eq:finalize_matching_1} and \eqref{eq:finalize_matching_prob2}, we conclude that 
\begin{align*}
&\PP\big(y^{\mathrm{test}}=\cdot \given \mathtt{prompt}_{\mathrm{CoT}}(n)\big) \notag \\
& \qquad = \int_{\zb \in \tilde \cL^*} \prod_{h=1}^H\PP\big(z_h^{\mathrm{test}}=\cdot \given \mathtt{prompt}_{\mathrm{CoT}}(n), z_1^\mathrm{test} = z_1, \ldots, z_{h-1}^\mathrm{test} = z_{h-1} \big)\ud z_1 \cdots \ud z_{H-1} \notag \\
&\qquad  = 
 \int_{\zb \in \tilde \cL^*} \prod_{h=1}^H\PP_{\mathrm{LLM} }\big(z_h^{\mathrm{test}}=\cdot \given \mathtt{prompt}_{\mathrm{CoT}}(n), z_1^\mathrm{test} = z_1, \ldots, z_{h-1}^\mathrm{test} = z_{h-1} \big)\ud z_1 \cdots \ud z_{H-1} \notag\\
 & \qquad =  \int_{\zb \in \tilde \cL^*}\PP_{\mathrm{LLM} } \big(y^{\mathrm{test}}=\cdot,  \{z_1^\mathrm{test}, \ldots, z_{H-1}^\mathrm{test} \} = \zb  \given \mathtt{prompt}_{\mathrm{CoT}}(n)\big)  \ud \zb  \nonumber\\
 &\qquad= \PP_{\mathrm{LLM}} \big(y^{\mathrm{test}}=\cdot \given \mathtt{prompt}_{\mathrm{CoT}}(n)\big) .
\end{align*}
Here the second equality follows from \eqref{eq:finalize_matching_prob2} and the last equality follows from the fact that $\PP_{\mathrm{LLM}} $ is supported on $\tilde \cL^*$. 
Finally, by the Bayes' rule and the fact that $y^{\mathrm{test}}$ is independent of the $n$ examples in $\mathtt{prompt}_{\mathrm{CoT}}(n)$ conditioning on $\theta$, we have 
 \begin{align*}
\PP_{\mathrm{LLM}}\big(y^{\mathrm{test}}=\cdot \given \mathtt{prompt}_{\mathrm{CoT}}(n)\big) 
&=\PP\big(y^{\mathrm{test}}=\cdot \given \mathtt{prompt}_{\mathrm{CoT}}(n)\big)\\
&= \int_{\Theta} \pi \big(\theta \given \mathtt{prompt}_{\mathrm{CoT}}(n) \big)\PP\big(y^{\mathrm{test}}=\cdot  \given z_0^{\mathrm{test}},\theta \big) \text{d}\theta.
\end{align*}
 Therefore, we conclude the proof.
\end{proof}

In the following, we extend  Lemma  to the generalized multi-step latent variable in \eqref{eq:generalized_latent_var_model} and present its proof.  

\begin{lemma}  \label{lem: bma_cot_2}
With pretraining data sampled from the generalized model in~\eqref{eq:generalized_latent_var_model}, 
we consider the population counterpart of the MLE in \eqref{eq:pretraining_loss} with $N$ going to infinity. 
 Suppose that the CoT prompt $\mathtt{prompt}_{\mathrm{CoT}}(n)$ has nonzero density under the pretraining distribution. 
Then the perfectly pretrained \ac{llm}s perform \ac{bma} during \ac{cot} prompting. Namely, we have that 
$$
\PP_{\mathrm{LLM}}\big(y^{\mathrm{test}} = \cdot \given \mathtt{prompt}_{\mathrm{CoT}}(n)\big) =  \int_{\Theta} \PP \big(y^{\mathrm{test}} = \cdot \given \mathtt{prompt}_{\mathrm{CoT}}(n),\theta \big)\pi \big(\theta \given \mathtt{prompt}_{\mathrm{CoT}}(n)\big)\mathrm{d}\theta.
$$
    
\end{lemma}

\begin{proof}
When the pretraining dataset is  sampled according to the model in \eqref{eq:generalized_latent_var_model}, by Bayes's rule, we have  
\begin{align}
    \label{eq:apply_bayes_rule}
    \PP\big(y^{\mathrm{test}}=\cdot \given \mathtt{prompt}_{\mathrm{CoT}}(n)\big)= \int_{\Theta} \pi \big(\theta \given \mathtt{prompt}_{\mathrm{CoT}}(n) \big)\PP\big(y^{\mathrm{test}}=\cdot \given \mathtt{prompt}_{\mathrm{CoT}}(n),\theta \big) \text{d}\theta.
\end{align}
Note that the $n$ examples and testing query $z_0^\mathrm{test}$ in $\mathtt{prompt}_{\mathrm{CoT}}(n)$ are not i.i.d under the generalized model.

Since the MLE loss does not concern the underlying distribution, our analysis in the proof of Lemma \ref{lem: bma.cot} still holds. 
In particular, \eqref{eq:distribution_match_population} holds on all prompts with a positive density under the pretraining distribution. 
As a result, starting from a CoT prompt $\mathtt{prompt}_{\mathrm{CoT}}(n)$ with a positive density, we can show that 
$$
\PP\bigl ( z_{1}^\mathrm{test} = \cdot, \ldots, z_{H}^\mathrm{test} = \cdot \given \mathtt{prompt}_{\mathrm{CoT}}(n) \big) = \PP_{\mathrm{LLM}} \bigl ( z_{1}^\mathrm{test} = \cdot, \ldots, z_{H}^\mathrm{test} = \cdot \given \mathtt{prompt}_{\mathrm{CoT}}(n) \big) 
$$
by writing the joint probability as a product of $H$ conditional probabilities. 
Therefore, we similarly have 
\begin{align*}
\PP_{\mathrm{LLM}}\big(y^{\mathrm{test}}=\cdot \given \mathtt{prompt}_{\mathrm{CoT}}(n)\big) &
=\PP\big(y^{\mathrm{test}}=\cdot \given \mathtt{prompt}_{\mathrm{CoT}}(n)\big) 
\end{align*}
when $N$ does to infinity. 
We conclude the proof by combining this fact with \eqref{eq:apply_bayes_rule}.  
\end{proof}

\subsection{Proof of Proposition~\ref{prop: attn}} \label{proof: prop attn}

\begin{proof}
    We aim to show that the attention mechanism approximates \ac{bma}, meaning that $\bar v_\mathrm{test}$ and $\mathtt{attn}(q_h^\mathrm{test}, \mathtt{keys}, \mathtt{values})$ converge to the same limit as \(n \to \infty\). 
    To this end, we show that both quantities converge to a population-level estimator based on population matrices. 
    The proof is divided into three steps. First, we define the population-level estimator that bridges the BMA estimator and the attention estimator. 
    Then, in the second and the last step, we prove that these estimators converge the population-level estimator as $n$ goes to infinity. Our proof generalizes the proof in \cite{zhang2023and} for vanilla \ac{icl}.

\vspace{2mm}

\noindent \textbf{Step 1: Define the population-level estimator.}
Recall that we focus on a simplified model defined in \eqref{eq: linear model}, which involves a feature mappings $\phi$, $k$, and $v$.
Also note that we define 
$k_h^i = \big(k(z_0^i), k(z_1^i),\cdots, k(z_{h-1}^i),0,\cdots,0 \big)   \in \RR^{H\cdot d_k}$ as the features of the first $h-1$ steps of the $i$-th example and define $v_{h}^i = C_{v} \cdot v(z_h^i) $ for some constant $C_{v} > 0$. 
Both $k_h^i$ and $v_h^i$ are random variables depending on the latent variable $\theta^*$.

We define random variables \(\cK\) and \(\cV\) as uniform mixtures of \( \{ k_h^i\}_{h\in [H]} \) and \(\{ v(z_h^i) \}_{h\in [H]} \), respectively, with \(h \sim \text{Unif}([H])\).
That is, with probability with $1/H$, $(\cK, \cV)$ has the same distribution as 
$(k_h^i, v(z_h^i) ) $ for any $h \in [H]$. 
Under the model in \eqref{eq: linear model}, $\cV$ and $\cK$ are linked via a linear model, i.e., $\E[\cV \given \cK ] = \theta^* \phi( \cK)  .$
Thus, $\theta^*$ can be viewed as the parameter of a linear regression problem, with $\cK$ being the covariate and $\cV$ being the response. 
To define a population-level estimator of $\theta^*$,  we define  two population matrices
\begin{align}
    \label{eq:theta_population_level}
    \begin{split}
    C_{\cK\cK} & = 1/H\sum_{h=1}^H\EE \bigl[ \phi(k_h^i ) \phi(k_h^i )^\top \bigr] \in \RR^{d_{\phi} \times d_{\phi}} , \\
    C_{\cV\cK}&  = 1/H\sum_{h=1}^H\EE \bigl[ v(z_h^i)  \phi(k_h^i ) ^\top \bigr] \in \RR^{d_v \times d_\phi} .
    \end{split}
\end{align}
Then we have $\theta^* = C_{\cV\cK} C_{\cK\cK} ^{-1}$. 
When using this population-level estimator of $\theta^*$ to make predictions on the test instance, for any $h \geq 0$, we use $k_h^\mathrm{test} $ in \eqref{eq:define_query} as the new covariate and predict the corresponding  response, 
which is given by 
$
   C_{\cV\cK} C_{\cK\cK} ^{-1} \phi(k_h^\mathrm{test}).
$
In the rest of the proof, we relate $\bar v_h^\mathrm{test}$ defined in \eqref{eq:bma_estimator} and the $\mathtt{attn} (q_h^\mathrm{test}, \mathtt{keys} , \mathtt{values})$ defined in \eqref{eq:attention} to  this estimator, where $q_h ^\mathrm{test} = \phi(k_h^\mathrm{test})$.

\vspace{2mm}
\noindent\textbf{Step 2: Relate $\bar v_h^\mathrm{test}$ to $C_{\cV\cK} C_{\cK\cK} ^{-1} \phi(k_h^\mathrm{test})$.} 
In the following, for ease of notation, we let $\|\cdot \|_\mathrm{op}$ and $\|\cdot\|_2$  denote the operator norm of matrices and the $L^{2}$ norm of vectors. Note that according to the model in \eqref{eq: linear model}, $\theta^*$ is shared in all the $H$ steps. 
Thus, we can pool all the reasoning steps together to estimate $\theta^*$.
Notice that $\Bar v_h^\mathrm{test}$ defined in \eqref{eq:bma_estimator} is exactly the ridge regression estimator for the linear model $\cV = \theta ^* \phi(\cK) + \mathtt{noise}$, where the noise term is an independent $\cN(0, \sigma^2) $ random variable. 
To simplify the notation, 
let $L = n H$. 
Recall that we define $\phi(K^n)$ and $V^n$ in Section \ref{subsec:attention_parameterize_bma}, where $\phi(K^n)\in \RR^{d_\phi \times L }$ and $V^n \in \RR^{d_v \times L} $. 
We define the sample-based counterparts of $C_{\cK\cK}$ and $C_{\cV \cK}$  in \eqref{eq:theta_population_level} as
\begin{align}
\label{eq:define_empirical_matrices}
	    \hat C_{\cK\cK} = L^{-1} \cdot \phi(K^n) \phi(K^n) ^\top \in \RR^{d_{\phi} \times d_{\phi}} , \qquad  \hat C_{\cV\cK} = L^{-1} \cdot (V^n) \phi(K^n) ^\top \in \RR^{d_{v} \times d_{\phi}}. 
    \end{align}
In addition, we define the empirical correlation between value vectors as
\begin{align*}
    \hat C_{\cV\cV} = L^{-1} \cdot (V^n)(V^n)^\top \in \RR^{d_{v} \times d_{v}}.
\end{align*}
  By definition, for any $h\in [H]$, can write  $\Bar v_h^\mathrm{test}$ in \eqref{eq:bma_estimator} as 
    \begin{align*}
        \bar v_h^\mathrm{test} = \hatC_{\cV\cK} (\hat C_{\cK\cK} + L^{-1}\cdot \sigma^2/\lambda \cdot I)^{-1} \phi(k_{h}^\mathrm{test}).
    \end{align*}
    By triangle inequality, we have 
    \begin{align}\label{eq:tac1}
		&\norm[\big]{\bar v_h^\mathrm{test} - C_{\cV\cK} C_{\cK\cK} ^{-1} \phi(k_h^\mathrm{test})} _2 \noend 
		&\quad \le \underbrace{\norm[\big]{ \hatC_{\cV\cK} (\hat C_{\cK\cK} + L^{-1}\cdot \sigma^2/\lambda \cdot I)^{-1} \phi(k_{h}^\mathrm{test}) - C_{\cV\cK} (C_{\cK\cK} + L^{-1}\sigma^2/\lambda \cdot I)^{-1} \phi(k_{h}^\mathrm{test})}_2}_{\displaystyle\text{(i)}} \noend 
		&\qquad + \underbrace{\norm[\big]{ C_{\cV\cK} (C_{\cK\cK} + L^{-1}\cdot \sigma^2/ \lambda \cdot I)^{-1} \phi(k_{h}^\mathrm{test}) - C_{\cV\cK} C_{\cK\cK}^{-1} q_{h}^\mathrm{test}}_2}_{\displaystyle\text{(ii)}}.
	\end{align}
 Here $\textrm{(i)}$ is the variance term of ridge regression which decays with the sample size $L$, and  $\textrm{(ii)}$ is the bias term that decays with the regularization parameter. 
 Our analysis of these terms is similar to that in \cite{zhang2023and}. 
 Note that in contrast to \cite{zhang2023and}, 
 $\hat C_{\cK \cK}$ and $\hat C_{\cV \cK}$ defined in \eqref{eq:define_empirical_matrices} involve  $ \{ (k_j^i,v(z_j^i) ) \}_{j=1,i=1}^{H, n}$ that are dependent in $j$. We handle such dependency by decomposing the double sum according to different step $j\in [H]$ and apply concentration to each step. 
 We first control the norm of $\mathrm{(i)} $ in \eqref{eq:tac1} as
 \begin{align}
     \label{eq:ce001}
\mathrm{(i)} 
 & \leq  \norm[\big]{ \hatC_{\cV\cK} (\hat C_{\cK\cK} + L^{-1}\cdot \sigma^2/\lambda \cdot I)^{-1}\phi(k_h^\mathrm{test}) - \hatC_{\cV\cK} (C_{\cK\cK} + L^{-1}\cdot \sigma^2/\lambda \cdot I)^{-1} \phi(k_h^\mathrm{test}) }_{2} \notag \\
 &\quad\qquad+ \norm[\big]{ (\hatC_{\cV\cK} - C_{\cV\cK}) ( C_{\cK\cK} + L^{-1}\cdot \sigma^2/\lambda \cdot I)^{-1}\phi(k_h^\mathrm{test}) }_{2} \notag \\
 & = \norm[\big]{ \hat C_{\cV\cK} (\hat C_{\cK\cK} + L^{-1}\cdot \sigma^2/\lambda \cdot I)^{-1}(\hat C_{\cK\cK} - C_{\cK\cK} ) (C_{\cK\cK} + L^{-1}\cdot \sigma^2/\lambda \cdot I)^{-1}\phi(k_h^\mathrm{test}) }_{2} \nonumber\\
&\quad\qquad+ \norm[\big]{ (\hatC_{\cV\cK} - C_{\cV\cK}) ( C_{\cK\cK} + L^{-1}\cdot \sigma^2/\lambda \cdot I)^{-1}\phi(k_h^\mathrm{test}) }_{2},
\end{align}
where the first inequality follows from the triangle inequality and the second follows from the fact that
\begin{align*}
    &\hatC_{\cV\cK} (\hat C_{\cK\cK} + L^{-1}\cdot \sigma^2/\lambda \cdot I)^{-1} - \hatC_{\cV\cK} (C_{\cK\cK} + L^{-1}\cdot \sigma^2/\lambda \cdot I)^{-1}\nonumber\\
    &\quad = \hatC_{\cV\cK} (\hat C_{\cK\cK} + L^{-1}\cdot \sigma^2/\lambda \cdot I)^{-1}(C_{\cK\cK} + L^{-1}\cdot \sigma^2/\lambda \cdot I)(C_{\cK\cK} + L^{-1}\cdot \sigma^2/\lambda \cdot I)^{-1} \nonumber\\
    &\quad\qquad - \hatC_{\cV\cK} (\hat C_{\cK\cK} + L^{-1}\cdot \sigma^2/\lambda \cdot I)^{-1}(\hat C_{\cK\cK} + L^{-1}\cdot \sigma^2/\lambda \cdot I)(C_{\cK\cK} + L^{-1}\cdot \sigma^2/\lambda \cdot I)^{-1}\nonumber\\
    &\quad = \hatC_{\cV\cK} (\hat C_{\cK\cK} + L^{-1}\cdot \sigma^2/\lambda \cdot I)^{-1}(C_{\cK\cK} - \hat C_{\cK\cK})(C_{\cK\cK} + L^{-1}\cdot \sigma^2/\lambda \cdot I)^{-1},
\end{align*}
where the equality follows from rearranging the terms. 
We next separately bound the two terms in \eqref{eq:ce001}.
For the first term in the right-hand side of \eqref{eq:ce001}, we have that
\begin{align}\label{eq:ce002}
    & \norm[\big]{ \hat C_{\cV\cK} (\hat C_{\cK\cK} + L^{-1}\cdot \sigma^2/\lambda \cdot I)^{-1}(\hat C_{\cK\cK} - C_{\cK\cK} ) (C_{\cK\cK} + L^{-1}\cdot \sigma^2/\lambda \cdot I)^{-1} \phi(k_h^\mathrm{test})}_{2} \nonumber\\
    &  \quad\leq \norm{\hatC_{\cV\cV}}_{\oper}^{1/2} \cdot \norm[\big]{ \hat C_{\cK\cK}^{1/2} (\hat C_{\cK\cK} + L^{-1}\cdot \sigma^2/\lambda \cdot I)^{-1/2} }_{\oper} \cdot \norm[\big]{(\hat C_{\cK\cK} + L^{-1}\cdot \sigma^2/\lambda \cdot I)^{-1/2}}_{\oper} \nonumber\\
    &\quad\qquad \cdot \norm[\big]{ (\hat C_{\cK\cK} - C_{\cK\cK} ) (C_{\cK\cK} + L^{-1}\cdot \sigma^2/\lambda \cdot I)^{-1} \phi(k_h^\mathrm{test})}_{2} \noend
    & \quad\leq (L^{-1}\cdot \sigma^2/\lambda )^{-1/2} \cdot \norm[\big]{ (\hat C_{\cK\cK} - C_{\cK\cK} ) (C_{\cK\cK} + L^{-1}\cdot \sigma^2/\lambda \cdot I)^{-1}\phi(k_h^\mathrm{test}) }_{2},
\end{align}
where the first inequality follows from the cross-covariance operator decomposition (Theorem 1 in \cite{baker1973joint}) $\hatC_{\cV\cK} = \hatC_{\cV\cV}^{1/2} W \hat C_{\cK\cK}^{1/2}$ for $W$ such that  $\norm{W}_{\oper} \le 1$, and the second inequality follows from the facts that
\begin{align*}
	&\norm{\hatC_{\cV\cV}}_{\oper} ^2 \leq L^{-1} \sum_{i =1   }^{n} \sum_{h=1}^H \norm{ v(z_h^i)}_2^2 \le 1,\qquad \norm[\big]{ \hat C_{\cK\cK}^{1/2} (\hat C_{\cK\cK} + L^{-1}\cdot \sigma^2/\lambda \cdot I)^{-1/2} }_{\oper}\leq 1, \text{ and }\\
 &\qquad \qquad  \qquad \norm[\big]{(\hat C_{\cK\cK} + L^{-1}\cdot \sigma^2/\lambda \cdot I)^{-1/2}}_{\oper}\leq (L^{-1}\cdot \sigma^2/\lambda )^{-1/2}.
\end{align*}
We then upper bound the second term of the right-hand side of \eqref{eq:ce001} using concentration inequality.
To this end, based on the random variables $\cK$ and $\cV$, we 
  consider a random variable $$
  \cV \phi(\cK)^\top  (C_{\cK\cK} + L^{-1}\sigma^2/\lambda I)^{-1}\phi(k_h^\mathrm{test}) \in \RR^{d_v }.$$
  Since $\|v(z)\|_2=1$ for all input $z\in \calL$, therefore $\|\calV\|_2\leq 1$. Since the mapping $\phi$ has a  bounded image set, 
we have 
\begin{align}
    &\big\|\cV \phi(\cK)^\top  (C_{\cK\cK} + L^{-1}\sigma^2/\lambda I)^{-1}\phi(k_h^\mathrm{test})  \big\|_{2} \nonumber\\
    &\quad\leq \|\cV 
\|_2 \cdot \| \phi(\cK)^\top \|_2 \cdot \| (C_{\cK\cK} + L^{-1}\sigma^2/\lambda I)^{-1} \|_{\oper}\cdot \|\phi(k_h^\mathrm{test})\|_{2} \nonumber\\
    &\quad\leq C\cdot H\cdot (L^{-1} \sigma^2/ \lambda)^{-1},\label{eq:norm_bd}
\end{align}
where $C$ is an absolute constant such that $\|\phi(\cdot)\|_2 \leq C\cdot H$. We further bound the expected squared norm as
\begin{align}
    &\bbE\Big[\big\|\cV \phi(\cK)^\top  (C_{\cK\cK} + L^{-1}\sigma^2/\lambda I)^{-1}\phi(k_h^\mathrm{test}) \big\|_{2}^{2} \Big]\nonumber\\
    &\quad\leq \bbE\Big[\|\cV\|_{2}^{2}\cdot \big\|\phi(\cK)^\top  (C_{\cK\cK} + L^{-1}\sigma^2/\lambda I)^{-1}\big\|_{2}^{2}\cdot \big\|\phi(k_h^\mathrm{test}) \big\|_{2}^{2} \Big]\nonumber\\
    &\quad\leq C\cdot H\cdot \bbE\Big[\big\|\phi(\cK)^\top  (C_{\cK\cK} + L^{-1}\sigma^2/\lambda I)^{-1}\big\|_{2}^{2} \Big]\nonumber\\
    &\quad\leq C\cdot H\cdot (L^{-1} \sigma^2/ \lambda)^{-1} \cdot \bbE\Big[\big\langle (C_{\cK\cK} + L^{-1}\sigma^2/\lambda I)^{-1}\phi(\cK),\phi(\cK)\big\rangle\Big]. \label{eq:rhs_4.2}
\end{align}
For the expectation in the right-hand side of \eqref{eq:rhs_4.2}, we further have that
\begin{align}
    &\bbE\Big[\big\langle (C_{\cK\cK} + L^{-1}\sigma^2/\lambda I)^{-1}\phi(\cK),\phi(\cK)\big\rangle\Big]\nonumber\\
    &\quad = \bbE\Big[\text{Trace}\big( (C_{\cK\cK} + L^{-1}\sigma^2/\lambda I)^{-1}\phi(\cK)\phi(\cK)^\top \big)\Big]\nonumber\\
    &\quad
     = \text{Trace}\big( (C_{\cK\cK} + L^{-1}\sigma^2/\lambda I)^{-1}C_{\cK\cK} \big)\nonumber\\
    &\quad = d_{\phi} - L^{-1}\sigma^2/\lambda\cdot \text{Trace}\big((C_{\cK\cK} + L^{-1}\sigma^2/\lambda I)^{-1}\big)\leq d_{\phi}, \label{eq:kernel_dim}
\end{align}
where we assume $\phi(\cdot)$ to be finite dimensional. The last line follows from the direct calculation and the fact that $(C_{\cK\cK} + L^{-1}\sigma^2/\lambda I)^{-1}$ is positive definite.
 Thus, we have  
\begin{align}
    \bbE\Big[\big\|\cV \phi(\cK)^\top  (C_{\cK\cK} + L^{-1}\sigma^2/\lambda I)^{-1}\phi(k_h^\mathrm{test}) \big\|_{2}^{2} \Big]\leq C\cdot H\cdot (L^{-1} \sigma^2/ \lambda)^{-1}d_{\phi} \label{eq:var_bd}
\end{align}
for some constant $C>0$. Recall that \(\calV\) and \(\calK\) are defined as uniform mixtures of \(v(z_h^i)\) and \(k_h^i\) with \(h \sim \mathtt{Unif}[H]\). Therefore, \eqref{eq:norm_bd} and \eqref{eq:var_bd} hold for each fixed \(h \in [H]\). Specifically, note that the random variables corresponding to the same step $h\in [H]$ across different examples $i \in [n]$ are i.i.d. That is,  for any fixed $h$, $\big\{\big(k_h^i,v(z_h^i)\big)\big\}_{i=1}^{ n}$ is a sequence of i.i.d samples. Therefore, we can apply Lemma~\ref{lem:cme-concen} to each \(j \in [H]\). More specifically, we can apply Lemma~\ref{lem:cme-concen} to each sample mean with fixed step index $h$ by setting $B = 2C\cdot H \cdot (L^{-1}\sigma^2/\lambda)^{-1}$ and variance as $C\cdot H\cdot (L^{-1} \sigma^2/ \lambda)^{-1}d_{\phi}$.  
Therefore,  with probability at least $1- \delta$ we have 
 \begin{align} \label{eq:ce20}
		&\norm[\big]{\hat C_{\cV\cK} (  C_{\cK\cK} + L^{-1} \sigma^2/ \lambda I)^{-1}\phi(k_h^\mathrm{test}) -   C_{\cV\cK} (  C_{\cK\cK} + L^{-1}\sigma^2/ \lambda I)^{-1}\phi(k_h^\mathrm{test})}_{2}\\ \nonumber
  &\quad =  \bigg\| \sum_{j=1}^H \frac{1}{H} \bigg(\frac{1}{n}\sum_{i=1}^{n}  v(z_j^i)\phi(k_j^i)^\top  (C_{\cK\cK} + L^{-1} \sigma^2/ \lambda I)^{-1}\phi(k_h^\mathrm{test})\nonumber\\
  &\quad\qquad\qquad\qquad- \EE \bigl[ v(z_j^i)\phi(k_j^i)^\top \bigr] (  C_{\cK\cK} + L^{-1} \sigma^2/ \lambda I)^{-1} \phi(k_h^\mathrm{test})\bigg)\bigg\|_{2}\nonumber\\ 
  &\quad \leq \sum_{j=1}^H \frac{1}{H} \bigg\|  \bigg(\frac{1}{n}\sum_{i=1}^{n}  v(z_j^i)\phi(k_j^i)^\top  (C_{\cK\cK} + L^{-1}\sigma^2/ \lambda I)^{-1}\phi(k_h^\mathrm{test})\nonumber\\
  &\quad\qquad\qquad\qquad- \EE \bigl[ v(z_j^i)\phi(k_j^i)^\top \bigr] (  C_{\cK\cK} + L^{-1}\sigma^2/ \lambda I)^{-1} \phi(k_h^\mathrm{test})\bigg)\bigg\|_{2}\nonumber\\
  &\quad \leq C\cdot \bigg(H( L^{-1}\sigma^2/ \lambda)^{-1}n^{-1}+\sqrt{\frac{H(L^{-1} \sigma^2/ \lambda)^{-1}d_\phi}{n}}\bigg)\cdot \log\frac{2H}{\delta} \nonumber \\
  &\quad =C\cdot H\cdot \bigg(H\lambda/\sigma^2+\sqrt{\lambda \cdot d_\phi/\sigma^2}\bigg) \cdot \log\frac{2H}{\delta}, \nonumber
\end{align}
 where  $C>0$ is an absolute constant. Here the second line results from triangle inequality and the third line follows from Lemma~\ref{lem:cme-concen}.
 Similarly, we perform the same analysis for $C_{\cK\cK}$ and we obtain with  probability at least $1-\delta$ that 
 	\begin{align}\label{eq:ce221}
		&\big\|\hat C_{\cK\cK} (  C_{\cK\cK} + L^{-1}\sigma^2/ \lambda I)^{-1} -   C_{\cK\cK} (  C_{\cK\cK} + L^{-1}\sigma^2/ \lambda I)^{-1}\big\|_2 \\ \nonumber 
  &\quad \leq C'\cdot H\cdot \bigg(H\lambda/\sigma^2+\sqrt{\lambda \cdot d_\phi/\sigma^2}\bigg) \cdot \log\frac{2H}{\delta}. \nonumber
	\end{align}
Here $C'>0$ is an absolute constant.
Therefore, we bound the first term in the upper bound in \eqref{eq:tac1} using   \eqref{eq:ce001}, \eqref{eq:ce002}, \eqref{eq:ce20}, and \eqref{eq:ce221} as 
\begin{align}
    &\norm[\big]{ \hatC_{\cV\cK} (\hat C_{\cK\cK} + L^{-1} \sigma^2/ \lambda I)^{-1} - C_{\cV\cK} (C_{\cK\cK} + L^{-1}\sigma^2/ \lambda I)^{-1}}_2 \label{eq:tac1100} \\
    &\quad \le C'' \cdot \sqrt{\frac{L}{\sigma^2/ \lambda}} \cdot H\cdot \bigg(H\lambda/\sigma^2+\sqrt{\lambda \cdot d_\phi/\sigma^2}\bigg)\log\frac{2H}{\delta}.\nonumber
\end{align}
which holds with probability at least $1- \delta$.

To bound the second term in  \eqref{eq:tac1}, we can apply the bound in Equation (E.13) from the proof of \cite{zhang2023and} directly.
Namely, we have that for any $q\in \RR^{H\cdot d_k}$,
\begin{align}
    \norm[\big]{ C_{\cV\cK} (C_{\cK\cK} + L^{-1}\cdot \sigma^2/ \lambda I)^{-1} \phi(q) - C_{\cV\cK} C_{\cK\cK}^{-1} \phi(q)}_2 \le C\cdot \sigma^2/ \lambda L^{-1},\label{eq:tac223}
\end{align}
where $C > 0$ is an absolute constant.
Combing   \eqref{eq:tac1100} and \eqref{eq:tac223}, we obtain 
\begin{align}
    &\big\|\bar v_h^\mathrm{test} - \cme(k_h^\mathrm{test}, \PP_{\cK, \cV})\big\|_{2} \nonumber\\
 & \quad = \cO\biggl( \sqrt{\frac{L\cdot \lambda}{\sigma^2}} \cdot  H\cdot \bigg(H\lambda/\sigma^2+\sqrt{d_\phi \cdot \lambda/\sigma^2}\bigg)\log\frac{2H}{\delta} + \sigma^2/ \lambda L^{-1}\biggr).\label{ieq:attdaggercme}
\end{align}
Thus the error bound in \eqref{ieq:attdaggercme} goes to $0$ by selecting $\sigma^2/ \lambda=(nH)^{2/3}$.

\vspace{2mm}

\noindent \textbf{Step 3: Relate  $\mathtt{attn} (q_h^\mathrm{test}, \mathtt{keys} , \mathtt{values})$  to $C_{\cV\cK} C_{\cK\cK} ^{-1} \phi(k_h^\mathrm{test})$.} 
In this step, we want to prove that $\att(q, \mathtt{keys}, \mathtt{values})$ defined in \eqref{eq:attention} converges to $C_{\cV\cK} C_{\cK\cK} ^{-1} \phi(k_h^\mathrm{test})$ as $n\to \infty$. 
Recall that we define $v_h^i = C_{v} \cdot v(z_h^i)$ as the value of softmax attention. 
To achieve this goal, we aim to show that $C\cdot\att(q, \mathtt{keys}, \mathtt{values}) = \int_{\mathbf{S}^{d_v-1}} v \hat \PP_{\calV \given \calK}(v \given q) \text{d}v$, where $\hat \PP_{\calV \given \calK}(v \given q)$ is an estimator of the conditional distribution of $\cV$ given $\cK$, and $C$ is an absolute constant. 
Here we let  $\mathbf{S}^{d}$ denote a $d$-dimensional unit sphere.
We define the empirical kernel conditional density as follows,
\begin{align*}
    \hat \PP_{\calV \given \calK}(v \given q) = \iota \frac{\sum_{i=1}^n \sum_{h=1}^H \exp\big(\langle k_h^i,q\rangle \big) \exp\big(\langle v_h^i,v\rangle \big)}{\sum_{i=1}^n \sum_{h=1}^H \exp\big(\langle k_h^i,q\rangle \big)},
\end{align*}
where $\iota = 1/ \int_{\mathbf{S}^{d_k-1}}\exp\big(\langle v_h^i,v\rangle \big) \text{d}v = 1/ \int_{\mathbf{S}^{d_k-1}}\exp\big( \langle v(z_h^i),v\rangle \big) \text{d}v $ is a normalizing constant to make sure that $\hat \PP_{\calV \given \calK}$ is a probability measure. Furthermore, note that $\iota$ does not depend on $v_h^i$ due to the symmetry of the unit sphere. 
%

We compute the integration over $v$ as 
\begin{align*}
    \int_{\mathbf{S}^{d_v-1}} v \hat \PP_{\calV \given \calK}(v \given q)\text{d}v &= \iota \cdot \frac{\sum_{i=1}^n \sum_{h=1}^H \exp \big(\langle k_h^i, q \rangle \big) \int_{\mathbf{S}^{d_v-1}} v \exp\big(\langle v(z_h^i), v\rangle \big)\text{d}v}{\sum_{i=1}^n \sum_{h=1}^H \exp \big(\langle k_h^i, q \rangle \big)}\\
    & = \iota \cdot  \frac{\sum_{i=1}^n \sum_{h=1}^H \exp \big(\langle k_h^i, q \rangle \big) \cdot C_1\cdot v_h^i }{\sum_{i=1}^n \sum_{h=1}^H \exp \big(\langle k_h^i, q \rangle \big)}\\
    & = \frac{C_1}{\iota}\cdot\att(q, \mathtt{keys}, \mathtt{values}).
\end{align*}
Since $v, v(z_h^i) \in \mathbf{S}^{d_v-1}$, we can apply Lemma~\ref{lem: exp kernel} with $\gamma = 1$ to the integration $\int_{\mathbf{S}^{d_v-1}} v \exp(\langle C_v\cdot v(z_h^i), v\rangle )\text{d}v$, and obtain that the second line holds for some constant $C_1>0$. 
The last line follows directly from the definition of softmax attention~\ref{eq: softmax attn}.

Due to the condition where $ \hat \PP_{\calV \given \calK}(v \given q) 
 \to  \PP(v \given q) $ uniformly for any $q\in \mathbf{S}^{d_q}$ as $n\to \infty$, the integral $\int_{\mathbf{S}^{d_v-1}} v \hat \PP_{\calV \given \calK}(v \given q)\text{d}v \to \EE[\cV \given \cK = q]$ as $n\to \infty$. 
 Combining with the previous argument, we have shown that,
\begin{align}
    \frac{C_1}{\iota}\cdot \att(q, \mathtt{keys}, \mathtt{values})\rightarrow \EE[\cV \given \cK = q] \qquad \text{as} \quad n\rightarrow \infty, \label{eq:attcmeasy}
\end{align}

Combing the results from the previous two steps \eqref{ieq:attdaggercme} and \eqref{eq:attcmeasy}, we have shown that
\begin{align*}
    \lim_{n \to \infty} \max_{h\in[H]} \| \bar v_h^\mathrm{test}  - C\cdot\Att(q_h^{\mathrm{test}}, \mathtt{keys}, \mathtt{values}) \|_2  = 0
\end{align*}
for some absolute constant $C$. Therefore, we conclude the proof. Note that this proof assumes \(\phi\) is finite-dimensional, but it also holds for infinite-dimensional \(\phi\). More specifically, by replacing the trace in \eqref{eq:kernel_dim} with the effective dimension of \(\phi\), we can still balance the rate of \(\sigma^2/\lambda\) to ensure \eqref{ieq:attdaggercme} goes to zero.

\end{proof}

\newpage
\section{Proofs and Additional Results of Section~\ref{subsec: Statistical Rates of Vanilla cot}}\label{sec: main app}
This section provides proofs and additional results about the statistical properties of the vanilla \ac{cot} estimator. In Appendix~\ref{proof: err cot} we prove  Lemma~\ref{lem:error_decomp}, which establishes the error decomposition of vanilla \ac{cot} error $\mathtt{err}_\mathrm{CoT}$. In Appendix~\ref{proof: rate_cot}, we extend  Theorems~\ref{thm: rate_cot} and~\ref{thm: rate_cot2} to the scenario where $\Theta$ is continuous, and provide their corresponding proofs. We conclude this section with an additional result that characterizes   $\mathtt{err}_\mathrm{CoT}$ under the general model in~\eqref{eq:generalized_latent_var_model}. 

\subsection{Proof of Lemma~\ref{lem:error_decomp}} \label{proof: err cot}

\begin{proof}
Recall that the 
error of the \ac{cot} estimator is defined as 
$$
\mathtt{err}_\mathrm{CoT}   = \KL\bigl(\PP(y^{\mathrm{test}} = \cdot  \given z_0^{\mathrm{test}} ,\theta^*) , \PP_{\hat \rho}(y^{\mathrm{test}} = \cdot \given \pt_{\mathrm{CoT}}(n)) \bigr). 
$$
We decompose the KL divergence into three terms  by direct computation: 
\begin{align}\label{eq:appendix_c_001}
\mathtt{err}_\mathrm{CoT}   
        & = \KL\bigl(\PP(y^{\mathrm{test}} = \cdot  \given z_0^{\mathrm{test}} ,\theta^*) , \PP(y^{\mathrm{test}} = \cdot \given \pt_{\mathrm{CoT}}(n)) \bigr)\nonumber\\
        & \qquad + \KL\bigl(\PP(y^{\mathrm{test}} = \cdot \given \pt_{\mathrm{CoT}}(n)) , \PP_{\hat \rho}(y^{\mathrm{test}}= \cdot \given \pt_{\mathrm{CoT}}(n)) \bigr)\nonumber \\
        &\qquad + \int\big(\PP(y^{\mathrm{test}} = y  \given z_0^{\mathrm{test}} ,\theta^*)-\PP(y^{\mathrm{test}} = y  \given \pt_{\mathrm{CoT}}(n))\big) \notag \\
        & \qquad \qquad \qquad \cdot \log\frac{\PP(y^{\mathrm{test}} = y \given \pt_{\mathrm{CoT}}(n))}{\PP_{\hat \rho}(y^{\mathrm{test}} = y   \given \pt_{\mathrm{CoT}}(n)) \bigr)} \ud y.  
\end{align}
Note that we can upper bound the marginal log density ratio $\log(\PP(y^{\mathrm{test}} = y \given \pt_{\mathrm{CoT}}(n))/\PP_{\hat \rho}(y^{\mathrm{test}} = y   \given \pt_{\mathrm{CoT}}(n)))$ by aggregating the density ratio at each step. More specifically, we have
\begin{align}
    &\log \bigg(\frac{\PP(y^{\mathrm{test}} = y  \given \pt_{\mathrm{CoT}}(n))}{\PP_{\hat \rho}(y^{\mathrm{test}} = y  \given \pt_{\mathrm{CoT}}(n))}\bigg) \label{eq:log_decomp}\\
    &\quad = \log \bigg(\frac{\sum_{z_{1:(H-1)}\in \calL^*}\PP(z_{1:(H-1)}^\mathrm{test}=z_{1:(H-1)} , y^{\mathrm{test}} = y  \given \pt_{\mathrm{CoT}}(n))}{\sum_{z_{1:(H-1)}'\in \calL^*}\PP_{\hat \rho}(z_{1:(H-1)}^\mathrm{test}=z_{1:(H-1)}', y^{\mathrm{test}} = y\given \pt_{\mathrm{CoT}}(n))}\bigg)\nonumber\\
    &\quad \leq \log \bigg(\max_{z_{1:(H-1)}\in \calL^*}\frac{\PP(z_{1:(H-1)}^\mathrm{test}=z_{1:(H-1)} , y^{\mathrm{test}} = y  \given \pt_{\mathrm{CoT}}(n))}{\PP_{\hat \rho}(z_{1:(H-1)}^\mathrm{test}=z_{1:(H-1)}, y^{\mathrm{test}} = y\given \pt_{\mathrm{CoT}}(n))}\bigg), \nonumber
\end{align}
where the first line follows from marginalizing over the intermediate steps $z_{1:(H-1)}^\mathrm{test}$, and the second line follows from generalized mediant inequality. 

Next, we use the chain rule to decompose the joint distribution as 
\begin{align}
    &\frac{\PP(z_{1:(H-1)}^\mathrm{test}=z_{1:(H-1)} , y^{\mathrm{test}} = y  \given \pt_{\mathrm{CoT}}(n))}{\PP_{\hat \rho}(z_{1:(H-1)}^\mathrm{test}=z_{1:(H-1)}', y^{\mathrm{test}} = y\given \pt_{\mathrm{CoT}}(n))} \label{eq:log_ratio} \\
    &\quad =\prod_{h=1}^{H}\frac{\PP(z_h^\mathrm{test}=z_h  \given \pt_{\mathrm{CoT}}^h(n))}{\PP_{\hat \rho}(z_h^\mathrm{test}=z'_h  \given \pt_{\mathrm{CoT}}^h(n))}   \leq e^{H\cdot b^*}, \nonumber
\end{align}
where the first line follows directly from the chain rule of conditional probability, and the second line follows from Assumption~\ref{assump: density bd}.
Combining \eqref{eq:log_decomp} and \eqref{eq:log_ratio}, we have  $$\big|\log \PP(y^{\mathrm{test}} = y  \given \pt_{\mathrm{CoT}}(n)) -\log \PP_{\hat \rho}(y^{\mathrm{test}} = y \given \pt_{\mathrm{CoT}}(n))\big| \leq Hb^* $$ 
for all $y \in \cL$. 
Thus, we can upper bound the integral in the right-hand side of \eqref{eq:appendix_c_001} by 
    \begin{align*}
        2 \TV \big(\PP(y^{\mathrm{test}} = \cdot \given z_0^{\mathrm{test}} ,\theta^*), \PP(y^{\mathrm{test}} = \cdot  \given \pt_{\mathrm{CoT}}(n))\big) \cdot Hb^*. 
    \end{align*}
    By Pinsker's inequality, we  further have that
    \begin{align}\label{eq:appendix_c_002}
        &2 \TV \big(\PP(y^{\mathrm{test}} = \cdot \given z_0^{\mathrm{test}} ,\theta^*), \PP(y^{\mathrm{test}}= \cdot  \given \pt_{\mathrm{CoT}}(n))\big) \cdot Hb^* \notag \\
        &\quad \leq 2\sqrt{2}Hb^* \cdot  \Big(\KL\big(\PP(y^{\mathrm{test}} = \cdot \given z_0^{\mathrm{test}} ,\theta^*) , \PP(y^{\mathrm{test}} = \cdot \given \pt_{\mathrm{CoT}}(n))\big) \Big)^{1/2}.
    \end{align}
     Combining \eqref{eq:appendix_c_001} and \eqref{eq:appendix_c_002}, we conclude that 
 \begin{align*}
        \mathtt{err}_\mathrm{CoT}
        & \leq \underbrace{\KL\bigl(\PP(y^{\mathrm{test}} \given \pt_{\mathrm{CoT}}(n)) , \PP_{\hat \rho}(y^{\mathrm{test}} \given \pt_{\mathrm{CoT}}(n)) \bigr)}_{\displaystyle{\mathtt{err}_\mathrm{pre}}}\\
        &\quad +\underbrace{\KL\bigl(\PP(y^{\mathrm{test}} \given z_0^{\mathrm{test}} ,\theta^*) , \PP(y^{\mathrm{test}} \given \pt_{\mathrm{CoT}}(n)) \bigr)}_{\displaystyle{\mathtt{err}_\mathrm{prompt}\textrm{-(i)}}}\\
        &\qquad + \underbrace{2\sqrt{2}Hb^* \cdot  \KL^{1/2} \big(\PP(y^{\mathrm{test}} \given z_0^{\mathrm{test}} ,\theta^*), \PP(y^{\mathrm{test}} \given \pt_{\mathrm{CoT}}(n))\big)} _{\displaystyle{\mathtt{err}_\mathrm{prompt}\textrm{-(ii)}}}.
    \end{align*}
     Therefore, we conclude the proof. 
     Here the upper bound 
consists of three parts. The first term 
characterizes the pretraining error by comparing $\PP$ and $\PP_{\hat \rho}$. 
The second and third terms involve the 
 KL divergence between the true distribution $\PP(\cdot \given \theta^*)$ and the distribution induced by the \ac{cot} prompt and $\PP$. 
\end{proof}

\subsection{Proofs of  Theorems~\ref{thm: rate_cot} and~\ref{thm: rate_cot2} and  Extension } \label{proof: rate_cot}

In the sequel,  we generalize the results in Theorems~\ref{thm: rate_cot} and~\ref{thm: rate_cot2} to the scenario where $\Theta$ is continuous and provide the corresponding proofs. The proofs of the supporting lemmas used in this subsection are deferred to Appendix~\ref{app: supp for main thms}. 

We begin this section by specifying the distance measurement on the general $\Theta$, which can be continuous. For any two hidden concept $\theta_{1},\theta_{2}\in\Theta$, we define a  \emph{loglikelihood metric} between them as
\begin{align*}
    \|\theta_{0}-\theta_{1}\|_{\Theta} = \sup_{\substack{(Z_0,Z_1,\cdots,Z_{H-1},Y)}}\bigg|\log \frac{\PP(Z_0,Z_1,\cdots,Z_{H-1},Y \given \theta_{0})}{\PP(Z_0,Z_1,\cdots,Z_{H-1},Y \given \theta_{1})}\bigg|.
\end{align*}
We note that the loglikelihood metric is indeed a semi-metric. For any $\theta\in\Theta$, any other concepts in the neighborhood of it share a similar conditional distribution on one example. Based on this semi-metric, we then define \(\alpha\)-cover of $\cS$ and the corresponding \(\mathcal{N}(\alpha,\cS)\) for any number \(\alpha > 0\) and any set $\cS\subseteq\Theta$. 

That is, \(\mathcal{C}(\alpha, \cS) = \{\theta_i\}_{i=1}^{N} \subseteq \Theta\) is a $\alpha$-cover of $\cS$ of size $N$ if $\cS \subseteq \bigcup_{i=1}^{N} B(\theta_i, \alpha)$, where $B(\theta_i, \alpha)$ is a neighborhood of $\theta_i$ with radius $\alpha$, i.e.,
 \begin{align}
     B(\theta,\alpha)=\big\{\tilde \theta\in {\Theta^\complement} : \|\tilde \theta-\theta\|_{\Theta}\leq \alpha \big\}. \label{eq:ball}
 \end{align}
 The minimal $N$ such that there exists a $\alpha$-cover of $\cS$ of size $N$ is called the covering number of $\cS$, which is denoted as \(\mathcal{N}(\alpha,\cS)\). In the following, we consider the complement of the equivalence class of the target concept $\Theta^{\complement}=\Theta \backslash \Theta_{\mathrm{eq}}(\theta^*) $. Without misunderstanding, we adopt $\calN(\alpha)$ to denote $\calN(\alpha,\Theta^{\complement})$ in the following.
Then we restate Theorem~\ref{thm: rate_cot} with $\Theta$ allowed to be a continuous set.
\vspace{2mm}

\begin{theorem}\label{thm: rate_cot_cts}

Let $\mathtt{err}_\mathrm{pre}$ denote the pretraining error defined in \eqref{eq: err pre} and denote $\Theta^{\complement}=\Theta \backslash \Theta_{\mathrm{eq}}(\theta^*) $. 
Under 
Assumptions~\ref{assump: density bd} and \ref{assumption: seperation}, the statistical error $\mathtt{err}_\mathrm{CoT}$ defined in~\eqref{eq: err cot} is bounded under the following two cases: 
    \begin{itemize}
        \item When $\Theta$ is a  discrete and finite set, with probability $1-\delta$, we have 
        $$
        \mathtt{err}_\mathrm{CoT}\leq\mathcal{O}\big(Hb^* \cdot \pi(\theta^*)^{-1/2} \cdot \delta^{-1} \cdot \big|{\Theta^\complement}\big| \cdot e^{-\lambda n}\big)+\mathtt{err}_\mathrm{pre}.
        $$


    \item When $\Theta$ is continuous, let $\calN(\alpha)$ denote the covering number of ${\Theta^\complement}$ with precision $\alpha$ with respect to the log-likelihood metric $\|\cdot\|_{\Theta}$. Under Assumption~\ref{assumption: lower bd}, with probability $1-\delta$, we have  
        $$\mathtt{err}_\mathrm{CoT}\leq\mathcal{O}\Big(Hb^* \cdot \pi\big(\Theta_{\mathrm{eq}}(\theta^*)\big)^{-1} \cdot \sqrt{c_0^{-nH} \cdot \delta^{-2} \cdot \calN(\alpha)^{2}\cdot e^{-2n\lambda}+n\alpha} \Big)+\mathtt{err}_\mathrm{pre}.$$
    \end{itemize}
Here the probability is with respect to the randomness of the \ac{cot} prompt. 

Moreover, we remark that when the \ac{llm} is perfectly pretrained, with probability $1-\delta$ with respect to the randomness of \ac{cot} prompt, we have 
\begin{align*}
\mathtt{err}_\mathrm{CoT}\leq \begin{cases}
    & \mathcal{O}\big(\pi(\theta^*)^{-1}\delta^{-2}| {\Theta^{\complement}}|^2e^{-2\lambda n} \big) \hspace{4.2cm}\textrm{when} ~ \Theta ~\textrm{is finite,} \\
    & \mathcal{O}\big(\pi\big(\Theta_{\mathrm{eq}}(\theta^*)\big)^{-1}c_0^{-nH}\delta^{-2}\calN(\alpha)^{2}e^{-2n\lambda}+n\alpha \big) \hspace{1cm} \textrm{when} ~ \Theta ~\textrm{is continuous}. 
\end{cases}
\end{align*}
    
\end{theorem}

When the set ${\Theta}$ is continuous, the statistical rate becomes slower due to the complicated structure of the task parameter space $\Theta$.
In the theorem statement, we do not specify the value of the covering number \(\alpha\). In fact, \(\alpha\) should be chosen depending on $n$ such that  $ \alpha \cdot n $  converges to \(0\) as $n$ increases. One can easily obtain a more concrete statistical rate by specifying a concrete $\alpha$ depending on specific assumptions of the covering number. For instance, suppose $\calN(\alpha) = \mathcal{O}(\alpha^{-V})$ for some constant $V$,  by selecting   $$\alpha  =\mathcal{O}\big(n^{-1/(2V+1)}\exp\{-n(2\lambda +H\ln{c_0})/(2V+1)\}\big), $$
 prompting error $\mathtt{err}_\mathrm{prompt}$ is  $\mathcal{O}\big(n^{V/(2V+1)}\exp\{-n(\lambda +H\ln{c_0}/2)/(2V+1)\}\big)$ with probability at least $1-\delta$.

\begin{proof}[Proof of Theorem \ref{thm: rate_cot}]

We divide the proof into three parts. In Step 1, we derive an upper bound for on  KL divergence to prepare for later analysis. In Step 2, we analyze the case where \({\Theta^\complement}\) is discrete and finite. In Step 3, we extend   to continuous \({\Theta^\complement}\).

\vspace{2mm}


\noindent \textbf{Step 1: Derive an upper bound on KL divergence.} We first invoke the following proposition to establish an upper-bound of $\KL (\PP(y^{\mathrm{test}}=\cdot \given z_0^{\mathrm{test}},\theta^*)  ,   \PP(y^{\mathrm{test}}=\cdot \given \mathtt{prompt}_{\mathrm{CoT}}(n)) )$. 

\begin{proposition}\label{prop:elb}
     For some fixed task $\theta^* \in \Theta$, we provide upper bounds for the KL-divergence of the ground truth distribution $\PP(y^\mathrm{test}=\cdot \given z_0^{\mathrm{test}},\theta^*)$ from the conditional pretrained distributions $\PP(y^\mathrm{test}=\cdot \given \mathtt{prompt}_{\mathrm{CoT}}(n)) $ as follows
     \begin{align*}
      & \KL \Big(\PP(y^{\mathrm{test}}=\cdot \given z_0^{\mathrm{test}},\theta^*)  ,   \PP\big(y^{\mathrm{test}}=\cdot \given \mathtt{prompt}_{\mathrm{CoT}}(n)\big) \Big) \nonumber \\
     &\quad \leq \log \bigg(1+\cfrac{\int_{\Theta^\complement}\PP\big(\mathtt{prompt}_{\mathrm{CoT}}(n) \biggiven \theta \big) \pi(\theta) \text{d}\theta}{\int_{\Theta_{\mathrm{eq}}(\theta^*)}\PP\big(\mathtt{prompt}_{\mathrm{CoT}}(n) \biggiven  \theta \big) \pi(\theta) \text{d}\theta}  \bigg).
    \end{align*}
\end{proposition}

\begin{proof}
    See Appendix~\ref{proof: elb} for a detailed proof.
\end{proof}
This proposition applies evidence lower bound \citep{kingma2013auto} to upper bound the KL divergence using likelihood ratios.
The proof involves using a variational distribution that is only supported on $\Theta_{\mathrm{eq}}$ and proportional to the posterior distribution. 
This proposition reduces the problem of bounding the KL divergence to comparing the likelihood functions on $\Theta_{\mathrm{eq}}$ and $\Theta^\complement$. 
We consider the cases where  $\Theta$ is finite and continuous separately.


\vspace{2mm}



   

\vspace{2mm}

\noindent \textbf{Step 2: Statistical rate for the case with a discrete and finite $\Theta$.} In this step, we assume the parameter space $\Theta$ is discrete and finite. Then by Proposition \ref{prop:elb}, 
we have 
\begin{align}\label{eq:proof_thm1_1}
      & \KL \bigg(\PP(y^{\mathrm{test}}=\cdot \given z_0^{\mathrm{test}},\theta^*)  ,   \PP\big(y^{\mathrm{test}}=\cdot \given \mathtt{prompt}_{\mathrm{CoT}}(n)\big) \bigg) \nonumber \\
     &\quad \leq \log \bigg(1+\cfrac{\sum_{\theta\in\Theta^\complement}\PP\big(\mathtt{prompt}_{\mathrm{CoT}}(n) \biggiven \theta \big) \pi(\theta) }{\sum_{\theta\in\Theta_{\mathrm{eq}}(\theta^*)}\PP\big(\mathtt{prompt}_{\mathrm{CoT}}(n) \biggiven  \theta \big) \pi(\theta) }  \bigg) \notag \\
     &\quad \leq \log \bigg(1+\cfrac{\sum_{\theta\in\Theta^\complement}\PP\big(\mathtt{prompt}_{\mathrm{CoT}}(n) \biggiven \theta \big) \pi(\theta) }{\PP\big(\mathtt{prompt}_{\mathrm{CoT}}(n) \biggiven  \theta^* \big) \pi(\theta^*) }  \bigg).
    \end{align}
   Here the first inequality follows from  Proposition~\ref{prop:elb} by changing the integration signs into summations. 
    In the second inequality, we drop some terms in the denominator to get an upper bound. 

Note that we have converted the upper bound into the logarithm of a weighted sum of likelihood ratios, with the weights being the prior distribution.
We invoke the following lemma to establish an upper bound for each likelihood ratio.


\begin{lemma}\label{lemma: likelihoodbd}
    Let $S_n=\{S^i\}_{ i \in [n]}$, where $S^i=z_{0:H}^i\overset{i.i.d.}{\sim} \PP(\cdot \given \theta^*)$ denotes a set of $n$ reasoning paths of length $H$ sampled independently from the model in \eqref{eq:latent_var_model} with task $\theta^*$. Let $J_i \subseteq [H-1]$ for each $i\in [n]$, and we use $S^i_{J_i}$ to denote a truncated version of the $i$-th trajectory $S^i$ corresponding to the indices specified by $J_i$. Namely, $S^i_{J_i} = \{z_0^i\}\bigcup \{z_j^i\}_{j\in J_i}\bigcup \{ z_H^i\}$.
    Then for any $n\geq 1$, $\theta\in \Theta$, and $\delta>0$, we have 
    \begin{align*}
        \frac{\PP(\{S_{J_i}\}_{i=1}^n \given \theta)}{\PP(\{S_{J_i}\}_{i=1}^n \given \theta^*)}\leq \exp{\left(-2\sum_{i=1}^n \mathrm{H}^2 \big(\PP(S_{J_i} \given \theta^*) , \PP(S_{J_i}\given \theta) \big)+2 \log(\delta^{-1})\right)},
    \end{align*}
    with probability at least $1-\delta$. Here the probability is with respect to the randomness of  $S_n$, and $\mathrm{H}^2 (\cdot, \cdot)$ denotes the squared Hellinger distance. 
    Furthermore, let $z_{0}^{n+1} $ be the input query of the $n+1$-th reasoning path, with probability at least $1 - \delta$, we further have 
    \begin{align*}
        & \frac{\PP(\{S_{J_i}\}_{i=1}^n, z_0^{n+1}  \given \theta)}{\PP(\{S_{J_i}\}_{i=1}^n , z_0^{n+1}\given \theta^*)} \notag \\
        & \qquad \leq \exp{\left(-2\sum_{i=1}^n \mathrm{H}^2 \big(\PP(S_{J_i} , \given \theta^*) , \PP(S_{J_i}\given \theta) \big) - 2 H\bigl ( \PP(z_0^{n+1}\given \theta^*), \PP( z_0^{n+1}\given \theta) \bigr ) +2 \log(\delta^{-1})\right)},
    \end{align*}
\end{lemma}

\begin{proof}
    See Appendix~\ref{proof: likelihoodbd} for a detailed proof. 
\end{proof}

This lemma provides an upper bound on the likelihood ratio of generating trajectories $\{S_{J_i}\}_{i=1}^n$ from two distributions $\PP(\cdot \given \theta)$ and $\PP(\cdot \given \theta^*)$, where $\PP(\cdot \given \theta^*)$ is the ground truth distribution. The upper bound is related to the Hellinger distance between them.

Recall that $\pt_\mathrm{CoT}(n)$ contains $n$ complete trajectories and a testing query $z_0^\mathrm{test}$.
Applying Lemma~\ref{lemma: likelihoodbd} to any $\theta \in \Theta^{\complement}$ and $\theta^*$ and taking a union bound, we conclude that, 
\begin{align}
    &\sup_{\theta \in \Theta^{\complement}} \bigg\{ \cfrac{\ \PP\big(\mathtt{prompt}_{\mathrm{CoT}}(n)\given \theta\big) }{\PP\big(\mathtt{prompt}_{\mathrm{CoT}}(n)\given \theta^* \big) } \bigg\} \nonumber \\
    &\quad \leq \sup_{\theta \in \Theta^{\complement}} \exp\Big(-2n\mathrm{H}^2 \big(\PP(Z_{0:H} \given \theta^*) , \PP(Z_{0:H}  \given \theta) \big)-2\mathrm{H}^2 \big(\PP(z_0 \given \theta^*) , \PP(z_0 \given \theta)\big)+2\log \big(\delta^{-1}| {\Theta^{\complement}}|\big)\Big) \nonumber \\
&\quad \leq \exp\big(-2n\lambda+2\log \big(\delta^{-1}| {\Theta^{\complement}}|\big)\big) \nonumber
\end{align}
holds with probability at least $1-\delta$. The first inequality follows from Lemma~\ref{lemma: likelihoodbd} and the second inequality follows from Assumption~\ref{assumption: seperation} and the fact that $\mathrm{H}^2 \big(\PP(z_0 \given \theta^*) , \PP(z_0 \given \theta)\big)\geq 0$.
Thus, plugging this inequality into the upper 
bound in \eqref{eq:proof_thm1_1},  with probability at least $1-\delta$,  we have 
 \begin{align*}
      &\KL \bigg(\PP(y^\mathrm{test} =\cdot \given z_0^{\mathrm{test}},\theta^*)  ,   \PP\big(y^\mathrm{test} =\cdot\given \mathtt{prompt}_{\mathrm{CoT}}(n)\big) \bigg)\\
      & \quad \leq \log \bigg[1+\frac{1}{\pi(\theta^*)}\cdot \exp \big(-2n\lambda+2\log \big(\delta^{-1}| {\Theta^{\complement}}|\big)\big) \cdot \sum_{ {\theta \in \Theta^{\complement}}} \pi(\theta)\bigg]\\
      & \quad = \log \bigg[1+\frac{1-\pi\big(\Theta_{\mathrm{eq}}(\theta^*)\big)} {\pi(\theta^*)} \cdot \exp{\Big(-2n\lambda+2\log \big(\delta^{-1}| {\Theta^{\complement}}| \big)}\Big)\bigg],
      \end{align*}
    Therefore,  when ${\Theta}$ is discrete and finite, with probability at least $1-\delta$, we have 
    \begin{align} 
        \KL \Big(\PP(y^\mathrm{test} =\cdot\given z_0^{\mathrm{test}},\theta^*)  ,   \PP\big(y^\mathrm{test} =\cdot\given \mathtt{prompt}_{\mathrm{CoT}}(n)\big) \Big) =\mathcal{O}\big(\pi(\theta^*)^{-1}\delta^{-2}| {\Theta^{\complement}}|^2e^{-2\lambda n} \big),\label{eq: kl_rate}
    \end{align}
    where $\cO(\cdot) $ only omits an absolute constant. 
Here we use the fact that  $1-\pi\big(\Theta_\mathrm{eq}(\theta^*)\big)<1$.
Recall that the  prompting error is defined as  
\begin{align*}  
\mathtt{err}_\mathrm{prompt}&=\KL\bigl(\PP(y^{\mathrm{test}}=\cdot \given z_0^{\mathrm{test}} ,\theta^*) , \PP(y^{\mathrm{test}}=\cdot \given \pt_{\mathrm{CoT}}(n)) \bigr)\\
        &\qquad + 2\sqrt{2}Hb^* \sqrt{\KL \big(\PP(y^{\mathrm{test}}=\cdot \given z_0^{\mathrm{test}} ,\theta^*), \PP(y^{\mathrm{test}} =\cdot\given \pt_{\mathrm{CoT}}(n))\big)}.
    \end{align*}
Therefore, we conclude that for any task $\theta^*$ with separation from $\Theta^\complement$, the prompting error goes to zero exponentially fast at a rate of order  $\mathcal{O}\big(Hb^* \pi(\theta^*)^{-1/2}\delta^{-1}| {\Theta^{\complement}}|e^{-\lambda n} \big)$. This proves Theorem \ref{thm: rate_cot}. 

\vspace{2mm}

\noindent\textbf{Step 3: Statistical rate for the case with a continuous $\Theta$.} 
Our analysis in {\bf Step 2} requires ${\Theta}$ to be discrete and finite. To handle the continuous case, we use the cover  of $\Theta^\complement $ to discretize it, at the cost of introducing an additional error involving the covering granularity. 
      Specifically, let $\cC(\alpha)$ be an $\alpha$-cover of $\Theta$ according to the semi-metric $\| \cdot \|_{\Theta}$. For any $\theta$, let the neighborhood $B(\theta, \alpha)$ be defined in \eqref{eq:ball}. 
      Let $\cN(\alpha) $ denote the covering number $\cN(\alpha, \Theta^\complement$). 
Using the minimal $\alpha$-cover of $\Theta^\complement$, we can construct a partition of $\Theta^\complement$ into at most $\cN(\alpha)$ disjoint sets, with each set contained in a neighborhood of radius $\alpha$. 
To see this, let $\cC(\alpha ) = \{ \theta_i \}_{i=1}^{\cN(\alpha)} $ be the $\alpha$-cover of $\Theta^\complement$. 
We can construct set $C(\theta_i) \subset B(\theta_i, \alpha)$ such that \(\{C(\theta_i)\}_{i=1}^{\mathcal{N}(\alpha)}\) form a partition of $\Theta^\complement $ by shrinking each $B(\theta_i, \alpha)$ to remove overlapping parts. 
Then, we have \(C(\theta_i) \cap C(\theta_j) = \emptyset\) for all \(i \neq j\), and \(\bigcup_{i=1}^{\mathcal{N}(\alpha)} C(\theta_i) = {\Theta^\complement}\). 
We characterize the discretization error due to  the $\alpha$-cover via
\begin{align}
    & \log \bigg(\cfrac{\int_{\theta \in {\Theta^\complement}} \PP\big(\mathtt{prompt}_{\mathrm{CoT}}(n)\given \theta \big) \pi(\theta) \text{d}\theta}{\sum_{\theta  \in \calC(\alpha)} \PP\big(\mathtt{prompt}_{\mathrm{CoT}}(n)\given \theta  \big)\pi\big(C(\theta )\big) } \bigg) 
 \nonumber \\
          &\quad =\log \bigg(\cfrac{\sum_{\theta \in \calC(\alpha)}\pi\big(C(\theta)\big)\int_{\theta' \in C(\theta)} \PP\big(\mathtt{prompt}_{\mathrm{CoT}}(n)\given \theta' \big) \cdot \pi(\theta')/ \pi(C(\theta)\big)\text{d}\theta'}{\sum_{\theta \in \calC(\alpha)} \PP\big(\mathtt{prompt}_{\mathrm{CoT}}(n)\given \theta\big)\pi\big(C(\theta)\big) }\bigg) \nonumber \\
          &\quad \leq \log \bigg(\max_{\theta \in \calC(\alpha)}\cfrac{\int_{\theta' \in C(\theta)} \PP\big(\mathtt{prompt}_{\mathrm{CoT}}(n)\given \theta'\big) \pi(\theta')/ \pi\big(C(\theta)\big)\text{d}\theta'}{\PP\big(\mathtt{prompt}_{\mathrm{CoT}}(n) \given \theta \big) }\bigg) \nonumber \\
          &\quad \leq \log \bigg(\sup_{\theta \in \calC(\alpha), \theta' \in C(\theta)}\cfrac{\PP\big(\mathtt{prompt}_{\mathrm{CoT}}(n)\given \theta' \big) }{\PP\big(\mathtt{prompt}_{\mathrm{CoT}}(n)\given \theta\big) }\bigg) \leq n\alpha. \label{eq:net_err_bd}
\end{align}
The first equality follows from decomposing the integral taken over ${\Theta^\complement}$ into a double integral taken over the covering $\calC(\alpha)$ and then within the induced partition $C(\theta)$. The second and third inequality follows from generalized mediant inequality. The last inequality follows from the definition of $B(\theta,\alpha)$ in \eqref{eq:ball} and the fact that $C(\theta)\subseteq B(\theta,\alpha)$ for all $\theta \in \cC(\alpha)$.

      Now we have controlled the error introduced by approximating ${\Theta^\complement}$ using  $\calC(\alpha)$. 
      Next, we apply Assumption~\ref{assumption: lower bd} to lower bound the likelihood integrated over $\Theta_
          \mathrm{eq}(\theta^*)$:
    \begin{align}
        \int_{\Theta_
          \mathrm{eq}(\theta^*)}\PP \big(\mathtt{prompt}_{\mathrm{CoT}}(n)\given \theta'\big)\pi(\theta')\text{d}\theta' \geq  c_0^{nH} \cdot \PP \big(\mathtt{prompt}_{\mathrm{CoT}}(n)\given \theta^*\big) \cdot  \pi(\Theta_
          \mathrm{eq}(\theta^*)). \label{eq:equiv_class_bd}
    \end{align}
        The inequality follows from the fact that $ \mathtt{prompt}_{\mathrm{CoT}}(n)$ is of length $nH$ and Assumption~\ref{assumption: lower bd}, which provides a lower bound for the conditional probability of the next reasoning step.
      
Combining \eqref{eq:equiv_class_bd} and \eqref{eq:net_err_bd} and using the same technique as in {\bf Step 2}, we obtain 
      \begin{align*}
      &\KL \Big(\PP(y^\mathrm{test}= \cdot  \given z_0^{\mathrm{test}},\theta^*)  ,   \PP\big(y^\mathrm{test} = \cdot \given \mathtt{prompt}_{\mathrm{CoT}}(n)\big) \Big)\\
          & \quad \leq \log \bigg(1+\cfrac{\int_{\theta \in {\Theta^{\complement}}}\PP\big(\mathtt{prompt}_{\mathrm{CoT}}(n)\given \theta\big) \pi(\theta)\big)  \text{d}\theta}{c_0^{nH} \cdot \PP \big(\mathtt{prompt}_{\mathrm{CoT}}(n)\given \theta^*\big)\cdot \pi(\Theta_
          \mathrm{eq}(\theta^*))} \bigg) \\
          & \quad \leq  \log \bigg(1+\cfrac{\sum_{\calC(\alpha)} \PP\big(\mathtt{prompt}_{\mathrm{CoT}}(n) \given \theta \big)\pi\big(C_{n}(\theta)\big)}{c_0^{nH} \cdot \PP \big(\mathtt{prompt}_{\mathrm{CoT}}(n)\given \theta^*\big)\cdot \pi(\Theta_
          \mathrm{eq}(\theta^*)) } \bigg) + n\alpha\\
      &\quad \leq \log \bigg[1+\frac{1-\pi\big(\Theta_{\mathrm{eq}}(\theta^*)\big)}{c_0^{nH}\cdot \pi(\Theta_
          \mathrm{eq}(\theta^*))}\exp{\bigg(-2n\lambda+2\log \big(\delta^{-1}\calN(\alpha)\big) }\bigg)\bigg]+n\alpha,
    \end{align*}
    with probability at least $1-\delta$. Here the first inequality is due to Proposition~\ref{prop:elb} and \eqref{eq:equiv_class_bd}. The second inequality follows from \eqref{eq:net_err_bd}, accounting for the discretization error induced by the $\alpha$-cover. 
    The last inequality follows from the same strategy as in {\bf Step 2}, where 
apply Lemma~\ref{lemma: likelihoodbd}. 
Finally, we conclude that , with probability at least $1-\delta$, we have 
\begin{align*}
        &\KL \Big(\PP(y^\mathrm{test} = \cdot \given z_0^{\mathrm{test}},\theta^*)  ,   \PP\big(y^\mathrm{test} = \cdot \given \mathtt{prompt}_{\mathrm{CoT}}(n)\big) \Big)\nonumber \\ 
        & \quad =\mathcal{O}\Big(\pi\big(\Theta_{\mathrm{eq}}(\theta^*)\big)^{-1} \cdot c_0^{-nH} \cdot \delta^{-2}\cdot \calN(\alpha)^{2} \cdot e^{-2n\lambda}+n\alpha \Big),
    \end{align*}
    where $\calO(\cdot)$ only hides absolute constants and we use the fact that the numerator $1-\pi\big(\Theta_{\mathrm{eq}}(\theta^*)\big)<1$. Therefore, we conclude the proof.
\end{proof}

The rest of this section generalizes Theorem~\ref{thm: rate_cot2} to the case where $\Theta$ is continuous and provides the proof.

\begin{theorem} \label{thm: rate_cot_2_cts}
   Let \(\mathtt{err}_\mathrm{pre}\) denote the pretraining error defined in \eqref{eq: err pre}, and let \(\Theta^{\complement} = \Theta \backslash \Theta_{\mathrm{eq}}(\theta^*)\) with \(\Tilde{\Theta^{\complement}}\) as its representative set. Under Assumptions~\ref{assump: density bd}, \ref{assumption: seperation}, and~\ref{assumption: close}, the statistical error \(\mathtt{err}_\mathrm{CoT}\) defined in~\eqref{eq: err cot} is bounded under the following two cases: 

\begin{itemize}
        \item When $\Tilde{\Theta^{\complement}}$ is a discrete and finite set, with probability \(1-\delta\), we have $$\mathtt{err}_\mathrm{CoT}\leq\mathcal{O}\big(Hb^* \cdot \pi(\theta^*)^{-1/2} \cdot \delta^{-1} \cdot \big|\Tilde{\Theta^\complement}\big| \cdot  e^{-(\lambda-\alpha) n+\alpha_0}\big)+\mathtt{err}_\mathrm{pre}.$$ 
        
        \item When $\Theta$ is continuous, let $\calN(\alpha)$ denote the covering number of $\Tilde{\Theta^{\complement}}$ with precision $\alpha$. With probability \(1-\delta\), we have:
        $$\mathtt{err}_\mathrm{CoT}\leq\mathcal{O}\Big(Hb^*\cdot \pi\big(\Theta_{\mathrm{eq}}(\theta^*)\big)^{-1}\cdot \sqrt{\delta^{-2}\cdot \calN(\alpha)^{2}\cdot e^{-2n(\lambda-\alpha)+2\alpha_0}+n\alpha} \Big)+\mathtt{err}_\mathrm{pre}.$$
        
    \end{itemize}
    Here the probability is with respect to the randomness of \ac{cot} prompts $\pt_\mathrm{CoT}(n)$. Parameters $\alpha $ and $\alpha_0$ are introduced in Assumption \ref{assumption: close}. 

Moreover, we remark that when the \ac{llm} is perfectly pretrained, with probability $1-\delta$ with respect to the randomness of \ac{cot} prompt, we have 
\begin{align*}
\mathtt{err}_\mathrm{CoT}\leq \begin{cases}
    & \mathcal{O}\big(\pi\big(\Theta_{\mathrm{eq}}(\theta^*)\big)^{-1}\cdot \delta^{-2}\cdot | \Tilde{\Theta^{\complement}}|^2\cdot e^{-2(\lambda-\alpha) n+2\alpha_0} \big) \hspace{2.3cm}\textrm{when} ~ \Theta ~\textrm{is finite,} \\
    & \mathcal{O}\big(\delta^{-2} \cdot \pi\big(\Theta_{\mathrm{eq}}(\theta^*)\big)^{-1} \cdot\calN(\alpha)^{2}\cdot e^{-2n(\lambda-\alpha)+2\alpha_0}+n\alpha \big) \hspace{1cm} \textrm{when} ~ \Theta ~\textrm{is continuous}. 
\end{cases}
\end{align*}
These error bound follow from the fact that, when the \ac{llm} is perfectly pretrained, $\mathrm{err}_\mathrm{CoT} = \KL \big(\PP(y^{\mathrm{test}}=\cdot \given z_0^{\mathrm{test}},\theta^*)  ,   \PP\big(y^{\mathrm{test}}=\cdot \given \mathtt{prompt}_{\mathrm{CoT}}(n)\big) \big)$. 
\end{theorem}

Theorem~\ref{thm: rate_cot_2_cts} builds on Theorem~\ref{thm: rate_cot_cts} by incorporating Assumption~\ref{assumption: close}, which postulates that distributions within each equivalence class are close. This leads to improved dependency on the parameter space size, shifting from \(\Theta^\complement\) to \(\Tilde{\Theta^\complement}\), despite a slower rate of exponential decay. Note that the statistical rate of \(\mathrm{err}_\mathrm{CoT}\) with a  continuous \(\Tilde{\Theta^\complement}\) does not require Assumption~\ref{assumption: lower bd}, which is required by Theorem~\ref{thm: rate_cot_cts}. The reason is technical:  Assumption~\ref{assumption: close} plays a similar role as Assumption~\ref{assumption: lower bd} and is sufficient to establish the result.
\begin{proof}[Proof of Theorem~\ref{thm: rate_cot2}]
\label{proof: rate_cot2}
We split the proof into three parts. First,  we apply Proposition \ref{prop:elb} and Assumption~\ref{assumption: close} to derive an upper bound of KL divergence involving $\Tilde{\Theta^\complement}$. Then in the second and last part, we consider discrete and continuous cases separately.

\vspace{2mm}

\noindent \textbf{Step 1: Derive upper bound for KL divergence at an equivalence class level.} 
We fix some concept \(\theta^*\) as the true latent task parameter.
Let  \(\Theta^\complement = \Theta \backslash \Theta_{\mathrm{eq}}(\theta^*)\) and  \(\Tilde{\Theta^\complement}\) denote a representative set of \(\Theta^\complement\). In light of Proposition \ref{prop:elb}, we first write 
\begin{align}
&\int_{\Theta^{\complement}}\PP\big(\mathtt{prompt}_{\mathrm{CoT}}(n) \biggiven \theta \big) \pi(\theta) \text{d}\theta \nonumber \\
&\quad =\int_{\theta \in \Tilde{\Theta^{\complement}}}\pi\big(\Theta_{\mathrm{eq}}(\theta)\big) \int_{\theta' \in \Theta_{\mathrm{eq}}(\theta)}\PP\big(\mathtt{prompt}_{\mathrm{CoT}}(n) \given \theta'\big) \pi(\theta')/\pi\big(\Theta_{\mathrm{eq}}(\theta)\big) \text{d}\theta' \text{d}\theta. \label{eq: decomposition}
\end{align}
That is, we decompose the integral over \(\Theta^\complement\) into a double integral: the inner integral averages the likelihood \(\PP\big(\mathtt{prompt}_{\mathrm{CoT}}(n) \mid \theta'\big)\) within each equivalence class, and the outer integral averages across all equivalence classes. By Assumption~\ref{assumption: close}, there exists a representative set \(\Tilde{\Theta}\) such that \(\theta^*\) is in \(\Tilde{\Theta}\), and distributions within the same equivalence class are close to each other.
Thus, we can derive both upper and lower bounds for the averaged likelihood within each equivalence class. For any $\theta \in \Tilde \Theta$, we have 
\begin{align}
    \exp(-n\alpha-\alpha_0) \leq \int_{\theta' \in \Theta_{\mathrm{eq}}(\theta)}\frac{\PP\big(\mathtt{prompt}_{\mathrm{CoT}}(n)\given \theta' \big) }{\PP\big(\mathtt{prompt}_{\mathrm{CoT}}(n)\given \theta\big)}\cdot \frac{\pi(\theta')} {\pi\big(\Theta_{\mathrm{eq}}(\theta)\big) } \text{d}\theta'\leq \exp(n\alpha+\alpha_0). \label{eq: sandwhich} 
\end{align}
 Combing \eqref{eq: sandwhich} and \eqref{eq: decomposition} with Proposition \ref{prop:elb}, we have
    \begin{align*}
        &\KL \Big(\PP(y^\mathrm{test} = \cdot \given z_0^{\mathrm{test}},\theta^*)  ,   \PP\big(y^\mathrm{test}= \cdot \given \mathtt{prompt}_{\mathrm{CoT}}(n)\big) \Big) \\
      & \quad \leq \log \bigg(1+\cfrac{\int_{\theta \in \Tilde{\Theta^{\complement}}}\pi\big(\Theta_{\mathrm{eq}}(\theta)\big) \int_{\theta' \in \Theta_{\mathrm{eq}}(\theta)}\PP\big(\mathtt{prompt}_{\mathrm{CoT}}(n) \given \theta'\big) \pi(\theta')/\pi\big(\Theta_{\mathrm{eq}}(\theta)\big) \text{d}\theta' \text{d}\theta}{\int_{\theta '' \in \Theta_{\mathrm{eq}}(\theta^*)}\PP\big(\mathtt{prompt}_{\mathrm{CoT}}(n) \biggiven  \theta'' \big) \pi(\theta'') \text{d}\theta''} \bigg)\\
      & \quad \leq \log \bigg(1+\cfrac{e^{n\alpha +\alpha_0} \cdot \int_{\theta \in \Tilde{\Theta^{\complement}}}\pi\big(\Theta_{\mathrm{eq}}(\theta)\big) \PP\big(\mathtt{prompt}_{\mathrm{CoT}}(n)\given \theta \big) \text{d}\theta}{e^{-n\alpha -\alpha_0} \cdot \PP \big(\mathtt{prompt}_{\mathrm{CoT}}(n)\given \theta^* \big)\pi\big(\Theta_\mathrm{eq}(\theta^*) \big)}\bigg).
      \end{align*}
The first inequality follows from Proposition \ref{prop:elb} and \eqref{eq: decomposition} and the second inequality is due to  \eqref{eq: sandwhich}. Note that the representative set \(\Tilde{\Theta^\complement}\) may not be unique, but this does not affect our analysis because we use \(\theta \in \Tilde{\Theta^\complement}\) only as a reference and  the value of $\PP  (\mathtt{prompt}_{\mathrm{CoT}}(n)\given \theta)$ stays the same within each equivalence class. The definition is consistent under any selection of the representative set. Therefore, we rewrite the upper bound for the KL divergence at the equivalence class level:
\begin{align}
        &\KL \Big(\PP(y^\mathrm{test}= \cdot \given z_0^{\mathrm{test}},\theta^*)  ,   \PP\big(y^\mathrm{test}= \cdot \given \mathtt{prompt}_{\mathrm{CoT}}(n)\big)\Big) \nonumber \\
      & \quad \leq  \log \bigg(1+e^{2n\alpha+2\alpha_0}\cfrac{\int_{\theta \in \Tilde{\Theta^{\complement}}}\pi \big(\Theta_{\mathrm{eq}}(\theta)\big) \PP\big(\mathtt{prompt}_{\mathrm{CoT}}(n) \given \theta\big) \text{d}\theta}{\PP\big(\mathtt{prompt}_{\mathrm{CoT}}(n)\given \theta^*\big)\pi\big(\Theta_\mathrm{eq}(\theta^*)\big) } \bigg). \label{eq: upper bd for kl2}
      \end{align}
Here, Assumption~\ref{assumption: close} reduces the integration region from $\Theta^\complement$ to $\Tilde{\Theta^\complement}$, at the cost of introducing the terms $e^{2\alpha n +2\alpha_0}$.

\vspace{2mm}

\noindent \textbf{Step 2: Statistical rate for the discrete case.} In this step, we assume the parameter space $\Tilde{\Theta^{\complement}}$ is discrete and finite. 
For any $\theta \in \Tilde{\Theta^\complement}$,  with probability at least $1-\delta$, we  obtain an upper bound for the averaged likelihood ratio  as follows: 
\begin{align}
    &\frac{\PP\big(\mathtt{prompt}_{\mathrm{CoT}}(n) \given \theta\big)}{\PP\big(\mathtt{prompt}_{\mathrm{CoT}}(n)\given \theta^* \big)} \nonumber \\
    &\quad \leq \exp \bigg(-2n \text{H}^2\big(\PP(Z_0,Z_1,\cdots,Z_{H-1}, Y\given \theta^*) , \PP(Z_0,Z_1,\cdots,Z_{H-1},Y\given \theta)\big) \nonumber \\
    &\quad \qquad -2\text{H}^2\big(\PP(Z_0 \given \theta^*) , \PP(Z_0\given \theta)\big)+2\log \delta^{-1}\bigg) \nonumber \\
    &\quad \leq \exp \big(-2n\lambda+2\log \delta^{-1}\big). \label{eq:likelihood_2}
\end{align} 
Here the first inequality follows from the Lemma \ref{lemma: likelihoodbd}, where we leverage the conditional independence between the $n$ CoT demonstrations and the query given task parameter $\theta$. 
The second inequality is due to the Assumption \ref{assumption: seperation}, which specifies that  $\theta$ and $\theta^*$ are strictly separated with a margin $\lambda$, and the non-negativity of the Hellinger distance.
Combing \eqref{eq:likelihood_2} with \eqref{eq: upper bd for kl2},  with probability at least $1-\delta$, we have 
      \begin{align*}
      &\KL \Big(\PP(y^\mathrm{test}= \cdot \given z_0^{\mathrm{test}},\theta^*)  ,   \PP\big(y^\mathrm{test}= \cdot \given \mathtt{prompt}_{\mathrm{CoT}}(n)\big) \Big) \nonumber \\
      & \quad \leq  \log \bigg(1+e^{2n\alpha+2\alpha_0} \cdot \cfrac{\sum_{\theta \in \Tilde{\Theta^{\complement}}}\pi \big(\Theta_{\mathrm{eq}}(\theta)\big) \PP\big(\mathtt{prompt}_{\mathrm{CoT}}(n) \given \theta\big) }{\PP\big(\mathtt{prompt}_{\mathrm{CoT}}(n)\given \theta^*\big)\pi\big(\Theta_\mathrm{eq}(\theta^*)\big) } \bigg)\\
      & \quad \leq \log \bigg[1+\frac{e^{2n\alpha+2\alpha_0}}{\pi\big(\Theta_{\mathrm{eq}}(\theta^*)\big)} \cdot \sum_{ \theta \in \Tilde{\Theta^{\complement}}} \exp \big(-2n\lambda+2\log \big(\delta^{-1}| \Tilde{\Theta^{\complement}}|\big)\big) \pi\big(\Theta_{\mathrm{eq}}(\theta)\big)\bigg]\\
      & \quad = \log \bigg[1+\frac{1-\pi\big(\Theta_{\mathrm{eq}}(\theta^*)\big)}{\pi\big(\Theta_{\mathrm{eq}}(\theta^*)\big)} \cdot \exp{\bigg(-2n(\lambda-\alpha)+2\alpha_0+2\log \big(\delta^{-1}| \Tilde{\Theta^{\complement}}| \big)}\bigg)\bigg].
      \end{align*}
    Here the  first inequality follows from \eqref{eq: upper bd for kl2} and the fact that \(\Tilde{\Theta^\complement}\) is discrete and finite. The second inequality is due to \eqref{eq:likelihood_2}, and the final line results from rearranging terms.

    In sum, we have that when $\Tilde{\Theta^\complement}$ is discrete and finite,   with probability at least $1-\delta$ we have 
    \begin{align}
        &\KL \Big(\PP(y^\mathrm{test}= \cdot \given z_0^{\mathrm{test}},\theta^*)  ,   \PP\big(y^\mathrm{test}= \cdot \given \mathtt{prompt}_{\mathrm{CoT}}(n)\big) \Big) \nonumber\\
        &=\mathcal{O}\big(\pi\big(\Theta_{\mathrm{eq}}(\theta^*)\big)^{-1} \cdot \delta^{-2} \cdot | \Tilde{\Theta^{\complement}}|^2 \cdot e^{-2(\lambda-\alpha) n+2\alpha_0} \big), \label{eq: kl_rate_2}
    \end{align}
where $\calO(\cdot)$ only hides absolute constants. We conclude that when \( \Theta \) is discrete and finite,  the \ac{cot} prompting error decays exponentially to zero and it depends on $\Theta^\complement$ only through $|\Tilde{\Theta^\complement}|$. The upper bound in \eqref{eq: kl_rate_2} establishes Theorem \ref{thm: rate_cot2}. 

\vspace{2mm}

\noindent \textbf{Step 3: Convergence rate for the continuous case.} It remains to consider the case where $\Theta$ is continuous.  Similar to the proof of Theorem \ref{thm: rate_cot_cts}, we the $\alpha$-cover with respect to the likelihood metric $\| \cdot \|_{\Theta}$  to discretize $\Tilde{\Theta^\complement}$ at the cost of introducing an additional error.

  For any \(\alpha > 0\), let \(\mathcal{C}(\alpha)\) denote an \(\alpha\)-covering of \(\Tilde{\Theta^\complement}\) with covering number \(\mathcal{N}(\alpha)\) and let  $\{C(\theta_i)\}_{\theta\in \calN(\alpha)}$ denote the partition of $\Tilde{\Theta^\complement}$ induced by $\cC(\alpha)$. We bound the discretization error due to the $\alpha$-cover as follows: 
\begin{align}
    & \log \bigg(\cfrac{\int_{\theta \in \Tilde{\Theta^\complement}} \PP\big(\mathtt{prompt}_{\mathrm{CoT}}(n)\given \theta \big) \pi\big(\Theta_\mathrm{eq}(\theta)\big) \text{d}\theta}{\sum_{\theta \in \calC(\alpha)} \PP\big(\mathtt{prompt}_{\mathrm{CoT}}(n)\given \theta \big)\pi\big(C_{n}(\theta)\big) } \bigg) \nonumber \\
          &\quad =\log \bigg(\cfrac{\sum_{\theta \in \calC(\alpha)}\pi\big(C(\theta )\big)\int_{\theta' \in C(\theta )} \PP\big(\mathtt{prompt}_{\mathrm{CoT}}(n)\given \theta' \big) \pi\big(\Theta_\mathrm{eq}(\theta')\big)/ \pi(C(\theta )\big)\text{d}\theta'}{\sum_{\theta \in \calC(\alpha)} \PP\big(\mathtt{prompt}_{\mathrm{CoT}}(n)\given \theta\big)\pi\big(C(\theta )\big) }\bigg) \nonumber \\
          &\quad \leq \log \bigg(\max_{\theta \in \calC(\alpha)}\cfrac{\int_{\theta' \in C(\theta )} \PP\big(\mathtt{prompt}_{\mathrm{CoT}}(n)\given \theta'\big) \pi\big(\Theta_\mathrm{eq}(\theta')\big)/ \pi\big(C(\theta )\big)\text{d}\theta'}{\PP\big(\mathtt{prompt}_{\mathrm{CoT}}(n) \given \theta \big) }\bigg) \nonumber \\
          &\quad \leq \log \bigg(\sup_{\theta \in \calC(\alpha), \theta' \in C(\theta)}\cfrac{\PP\big(\mathtt{prompt}_{\mathrm{CoT}}(n)\given \theta' \big) }{\PP\big(\mathtt{prompt}_{\mathrm{CoT}}(n)\given \theta\big) }\bigg)\quad \leq n\alpha. \label{eq:net_err_bd_2}
\end{align}
Here the  first equality follows from decomposing the integral over \(\Tilde{\Theta^\complement}\) into a double integral  using the partition structure. The second and third inequalities follow from the generalized mediant inequality. The last inequality is derived from the definition of \(\calC(\alpha)\).

Now we have controlled the discretization error, combing with analysis from {\bf Step 2}, we obtain that the following inequality holds with probability at least $1-\delta$:
      \begin{align*}
      &\KL \Big(\PP(y^\mathrm{test}= \cdot \given z_0^{\mathrm{test}},\theta^*)  ,   \PP\big(y^\mathrm{test}= \cdot \given \mathtt{prompt}_{\mathrm{CoT}}(n)\big) \Big) \nonumber \\
      & \quad \leq  \log \bigg(1+e^{2n\alpha+2\alpha_0}\cdot \cfrac{\int_{\theta \in \Tilde{\Theta^{\complement}}}\pi \big(\Theta_{\mathrm{eq}}(\theta)\big) \PP\big(\mathtt{prompt}_{\mathrm{CoT}}(n) \given \theta\big) \text{d}\theta}{\PP\big(\mathtt{prompt}_{\mathrm{CoT}}(n)\given \theta^*\big)\pi\big(\Theta_\mathrm{eq}(\theta^*)\big) } \bigg)\\
          & \quad \leq  \log \bigg(1+e^{2n\alpha+2\alpha_0}\cdot \cfrac{\sum_{\theta \in \calC(\alpha)} \PP\big(\mathtt{prompt}_{\mathrm{CoT}}(n) \given \theta \big)\pi\big(C_{n}(\theta)\big)}{\PP\big(\mathtt{prompt}_{\mathrm{CoT}}(n)\given \theta^*\big)\pi\big(\Theta_\mathrm{eq}(\theta^*)\big) } \bigg) + n\alpha\\
      &\quad \leq \log \bigg[1+\frac{1-\pi\big(\Theta_{\mathrm{eq}}(\theta^*)\big)}{\pi\big(\Theta_\mathrm{eq}(\theta^*)\big)}\cdot \exp{\bigg(-2n(\lambda-\alpha)+2\alpha_0+2\log \big(\delta^{-1}\calN(\alpha)\big) }\bigg)\bigg]+n\alpha,
    \end{align*}
    where $\calO(\cdot)$ only hides absolute constants, and the randomness comes from the stochasticity of \ac{cot} prompts $\pt_\mathrm{CoT}(n)$.
   The first inequality follows from \eqref{eq: upper bd for kl2}, and the second inequality follows from \eqref{eq:net_err_bd_2}. The final inequality is due to Lemma~\ref{lemma: likelihoodbd}. Therefore, we conclude that with probability at least $1-\delta$ we have  
\begin{align*}
        &\KL \Bigl(\PP(y^\mathrm{test}= \cdot \given z_0^{\mathrm{test}},\theta^*)  ,   \PP\big(y^\mathrm{test}= \cdot \given \mathtt{prompt}_{\mathrm{CoT}}(n)\big) \Big) \nonumber \\
        & \quad =\mathcal{O}\big(\delta^{-2} \cdot \pi\big(\Theta_{\mathrm{eq}}(\theta^*)\big)^{-1}\cdot \calN(\alpha)^{2}\cdot 
 e^{-2n(\lambda-\alpha)+2\alpha_0}+n\alpha \big), 
    \end{align*}
    where $\calO(\cdot)$ omits absolute constants. Thus, we conclude the proof.
\end{proof}

\newpage

\section{Proofs of the Results in  Section~\ref{subsubsec: variants}}

\subsection{Proof of Corollary \ref{cor: self-consistency cot}}
\label{proof: self-consistency cot}
\begin{proof}

There are two notions of sample size in self-consistency \ac{cot}: the number of examples in \ac{cot} prompt $n$, and the number of reasoning paths $K$. 
These two notions have different roles. 
A large $n$ ensures that the distribution of the perfectly pretrained \ac{llm}, 
$\PP \big(y^{\mathrm{test}} = \cdot \given \pt_{\mathrm{CoT}}(n)\big)$, approximates the desired distribution 
$\PP \big(y^{\mathrm{test}} = \cdot \given z_0^\mathrm{test},\theta^* \big)$.
Whereas a large $K$ ensures that the sample mode $y_k^*$ approximates the population mode. 

In the following, we prove the corollary in two steps. 
We first show that the population mode of $\PP \big(y^{\mathrm{test}} = \cdot \given \pt_{\mathrm{CoT}}(n)\big)$ coincides with $y^*$ when $n$ is sufficiently large.
Then we prove that $y_K^*$ finds the population mode of $\PP  (y^{\mathrm{test}} = \cdot \given \pt_{\mathrm{CoT}}(n) )$ when $K$  is sufficiently large. 
The final statistical error can be obtained by combining these two steps.


\vspace{2mm}

\noindent \textbf{Step 1: Mode of $\PP  (y^{\mathrm{test}} = \cdot \given \pt_{\mathrm{CoT}}(n) )$ converges to $y^*$.} 
In the first step, we show that there exists $n^*$ such that, as long as $n \geq n^*$, the mode of $\PP  (y^{\mathrm{test}} = \cdot \given \pt_{\mathrm{CoT}}(n) )$ coincides with $y^*$, the mode of  $\PP \big(y^{\mathrm{test}} = \cdot  \given  z_0^\mathrm{test} ,\theta^* \big)$. 
This step leverages Theorem \ref{thm: rate_cot}.


Suppose there exists some $\xi>0$ such that
\begin{align*}
\KL \Big(\PP(y^{\mathrm{test}} = \cdot  \given z_0^{\mathrm{test}},\theta^*)  ,   \PP\big(y^{\mathrm{test}}  = \cdot \given \mathtt{prompt}_{\mathrm{CoT}}(n)\big) \Big)\leq \xi,
\end{align*}
Then by Pinsker's inequality, we obtain a bound on the TV distance,
\begin{align*}
    \TV\Big((\PP(y^{\mathrm{test}}  = \cdot \given z_0^{\mathrm{test}},\theta^*) ,   \PP\big(y^{\mathrm{test}} = \cdot \given \mathtt{prompt}_{\mathrm{CoT}}(n)\big)\Big) \leq \sqrt{2\xi}.
\end{align*}
Recall that we assume that $y^{\mathrm{test}} $ belongs to a finite set $\cY$.
Then we have 
\begin{align*}
    \sqrt{2\xi} &\geq \TV\Big((\PP(y^{\mathrm{test}}  = \cdot\given z_0^{\mathrm{test}},\theta^*) ,   \PP\big(y^{\mathrm{test}} = \cdot \given \mathtt{prompt}_{\mathrm{CoT}}(n)\big)\Big)\\
    &=\frac{1}{2}\sum_{y \in \cY }\Big|\PP(y^{\mathrm{test}}  = y \given z_0^{\mathrm{test}},\theta^*) -   \PP\big(y^{\mathrm{test}} = y \given \mathtt{prompt}_{\mathrm{CoT}}(n)\big)\Big|\\
    &\geq \frac{1}{2}\max_{y \in \cY }\Big| \PP(y^{\mathrm{test}} = y\given z_0^{\mathrm{test}},\theta^*) -   \PP\big(y^{\mathrm{test}} = y  \given \mathtt{prompt}_{\mathrm{CoT}}(n)\big)\Big|. 
\end{align*}
Thus we can sandwich $\PP (y^{\mathrm{test}} = \cdot \given \mathtt{prompt}_{\mathrm{CoT}}(n) )$ for any $y$ by
\begin{align*} 
    \PP(y^\mathrm{test} = y \given z_0^{\mathrm{test}},\theta^*) - 2\sqrt{2\xi}\leq \PP\big(y^\mathrm{test} = y \given \mathtt{prompt}_{\mathrm{CoT}}(n)\big)\leq \PP(y^\mathrm{test} = y \given z_0^{\mathrm{test}},\theta^*) + 2\sqrt{2\xi}.
\end{align*}
We plug in different values of $y$ in the above inequality  and obtain 
\begin{align}
    \max_{y \neq y^*}\PP\big(y^\mathrm{test}  = y \given \mathtt{prompt}_{\mathrm{CoT}}(n)\big)&\leq \max_{y \neq y^*}\PP(y^\mathrm{test} =  y \given z_0^{\mathrm{test}},\theta^*) + 2\sqrt{2\xi},  \label{eq: 1}\\
    \PP\big(y^\mathrm{test} =  y^* \given \mathtt{prompt}_{\mathrm{CoT}}(n)\big)&\geq \PP(y^\mathrm{test} = y^* \given z_0^{\mathrm{test}},\theta^*) - 2\sqrt{2\xi}. \label{eq: 2}
\end{align}

 Recall the definition of $\epsilon$ in Assumption \ref{assump: sc}. Combining \eqref{eq: 1} and \eqref{eq: 2} we have 
\begin{align}
    \PP\big(y^\mathrm{test} =  y^* \given \mathtt{prompt}_{\mathrm{CoT}}(n)\big)&\geq \PP(y^\mathrm{test} = 
 y^* \given z_0^{\mathrm{test}},\theta^*) - 2\sqrt{2\xi} \notag\\
    &\geq\max_{y \neq y^*}\PP(y^\mathrm{test} =  y \given z_0^{\mathrm{test}},\theta^*) + \epsilon -2\sqrt{2\xi} \notag \\
    &\geq \max_{y \neq y^*}\PP\big(y^\mathrm{test} =  y \given \mathtt{prompt}_{\mathrm{CoT}}(n)\big)+ \epsilon -4\sqrt{2\xi}, \label{eq:bound_cor1_1_final}
\end{align}
where the first inequality follows from \eqref{eq: 1}, the second follows from the definition of  $\epsilon$, and the last one follows from \eqref{eq: 2}. 
Hence, as long as $\epsilon - 4 \sqrt{2\xi}  > 0$, we ensure that $y^*$ is also the unique mode of $\PP (y^\mathrm{test} = \cdot \given \mathtt{prompt}_{\mathrm{CoT}}(n) )$.
Now if we set the prompt size to be
\begin{align}
    \label{eq:define_n_star}
    n^* = C\cdot  \Big( \log \bigl( |\Theta^\complement| / \pi(\theta^*)\bigr) + \log (1/ \epsilon  )  \Big)\Big/ \lambda ,
\end{align}
where  $C$ is a sufficiently large absolute constant. 
We now 
leverage Theorem \ref{thm: rate_cot} with a perfectly pretrained \ac{llm}.
Specifically, 
setting $\delta =e^{-\lambda n /2}$ in  \eqref{eq: kl_rate}, we conclude that, with probability at least $1-e^{-\lambda n /2}$, when $n \geq n^*$, 
we have 
\begin{align}\label{eq:apply_kl_bound}
\KL \Big(\PP(y^{\mathrm{test}} = \cdot  \given z_0^{\mathrm{test}},\theta^*)  ,   \PP\big(y^{\mathrm{test}}  = \cdot \given \mathtt{prompt}_{\mathrm{CoT}}(n)\big) \Big) \leq \epsilon^2 / 128. 
\end{align}
Combining \eqref{eq:bound_cor1_1_final} and \eqref{eq:apply_kl_bound}, we conclude that, when $n$ is sufficiently large such that $n \geq n^*$, with probability $1-e^{-\lambda n /2}$, 
\begin{align}
    \label{eq:cor1_1_final_bound}
    \PP\big(y^\mathrm{test} =  y^* \given \mathtt{prompt}_{\mathrm{CoT}}(n)\big)&\geq  \max_{y \neq y^*}\PP\big(y^\mathrm{test} =  y \given \mathtt{prompt}_{\mathrm{CoT}}(n)\big)+ \epsilon/2 ,
\end{align}
where $n^*$ is defined in \eqref{eq:define_n_star}. 
Thus, $y^*$ is also the mode of $\PP (y^\mathrm{test} =  y^* \given \mathtt{prompt}_{\mathrm{CoT}}(n) )$ with high probability. Now we conclude {\bf Step 1}.


\vspace{2mm}

\noindent \textbf{Step 2: Sample mode $y_K^*$ converges to population mode $y^*$.} 
In this step, we utilize concentration to show that $y_K^*$ converges to $y^*$ when $K$ is sufficiently large. 
Recall that we assume $y^{\mathrm{test}}$ takes values in a finite set $\cY$ and also recall that we define an empirical distribution $p_K(y) = K^{-1} \sum_{i=1}^K \mathbf{1} \{y^{\mathrm{test},i} = y\}$ for all $y \in \cY$. 
Thus, for any $y\in \cY$, $Kp_K(y)$ is a binomial variable with distribution  $\mathtt{Bin} (K, p \big)$, where $p = \PP  (y^{\mathrm{test}} = y \given \pt_{\mathrm{CoT}}(n) )$. 
Thus by Bernstein's inequality for binomial distribution, for any $t>0$ we have 
\begin{align*}
    &\PP \big(\big|p_K(y)-\PP \big(y^{\mathrm{test}} = y \given \pt_{\mathrm{CoT}}(n)\big)\big|\geq t \big) \\
    & \quad \leq 2\exp \bigg(-\frac{3K t^2}{6\cdot \PP \big(y \given \pt_{\mathrm{CoT}}(n)\big) \big(1-\PP (y \given \pt_{\mathrm{CoT}}(n))\big)+2 t }\bigg)\leq 2\exp \bigg(-\frac{6Kt^2 }{3+4 t}\bigg)
\end{align*}
where the second inequality follows from the fact that $p(1-p) \leq 1 /4$ for any $p \in \RR$.
Now we set $ t = \epsilon/ 4$ and take a union bound over $y \in \cY$ to obtain that 
\begin{align}\label{eq:cor1_union_bound}
     \PP \Big(\max_{y \in \cY} \big|p_K(y)-\PP \big(y^{\mathrm{test}} = y \given \pt_{\mathrm{CoT}}(n)\big)\big|\geq \epsilon / 4  \Big)  
    &    \leq 2 |\cY| \cdot \exp \bigg(-\frac{3K \epsilon^2  }{24+ 8 \epsilon}\bigg).
\end{align}

Recall that $y_K^*$ and $y^*$ are the modes of $p_K$ and $\PP  (y^{\mathrm{test}} = \cdot  \given \pt_{\mathrm{CoT}}(n))$, respectively. 
Also note that \eqref{eq:cor1_1_final_bound} implies that there is a gap of $\epsilon/2 $ between the mode of $\PP  (y^{\mathrm{test}} = \cdot  \given \pt_{\mathrm{CoT}}(n))$ and its second largest probability mass. 
If $y_K^*\neq y^*$, there must exist some $y \in \cY$ such that 
$$ |p_K(y)-\PP \big(y^{\mathrm{test}} = y \given \pt_{\mathrm{CoT}}(n)\big)\big|\geq \epsilon / 4 .$$
Therefore,  we can upper bound    $\PP\big(y_K^*\neq y^*\given \pt_{\mathrm{CoT}}(n)\big)$ using \eqref{eq:cor1_union_bound}.

Therefore, we conclude that 
with a perfectly pretrained \ac{llm} and a discrete and finite ${\Theta}$, when is sufficiently large such that  $n \geq n^*$, we have 
 \begin{align*}
        \PP\big(y_K^*\neq y^*\given \pt_{\mathrm{CoT}}(n)\big) \leq 2 |\cY| \cdot \exp \bigg(-\frac{3K \epsilon^2  }{24+ 8 \epsilon}\bigg) 
    \end{align*}
    with probability at least $1-e^{-\lambda n/2}$. 
    Here $n^*$ is define in \eqref{eq:define_n_star}. 
  Thus we conclude the proof.
\end{proof}

\subsection{Proof of Proposition \ref{cor: tot-bfs}} \label{proof: tot-bfs}
\begin{proof} 

In \ac{tot} prompting, there are two notions of sample size: the number of examples in the \ac{cot} prompt $n$ and the number of candidates generated at each step \( K \).  
Additionally, there is a breadth limit parameter \( B\), which controls the number of candidates that continue to the next step.  
For the ease of notation, we write $z_h^{\mathrm{test},*}$ as $z^*_h$  and $t_h^{\mathrm{test},*}$ as $t_h^*$.

Intuitively, a large $n$ ensures that the distribution induced by a perfectly pretrained \ac{llm} $\PP(z^\mathrm{test}_h = \cdot \given \pt_h(n), t^*_{h-1})$ approximates the true distribution $\PP(z_h^\mathrm{test} = \cdot \given t^*_{h-1}, \theta^*)$ for each $h\in [H]$. A large \( K \) ensures that the optimal \( z_h^* \) appears in the samples for each \( h \in [H] \), which is then selected by BFS.

The proof involves two steps. First, we show that for any \( h \in [H] \), \(p_h = \PP(z_h^\mathrm{test} = z_h^* \mid \pt_h(n), t_{h-1}^*)\) is close to \(p_h^* = \PP(z_h^\mathrm{test} = z_h^* \mid t_{h-1}^*, \theta^*)\) when \( n \) is large enough, where \( z_h^* \) is the optimal next step $z_h^* = \argmax_{z_h}V_{\theta^*}(t_{h-1}^*, z_h)$. Second, we demonstrate that a large \( K \) helps find the optimal trajectory \( t_H^* \) by iteratively sampling from the \ac{llm}-induced distribution \(\PP(z_h^\mathrm{test} = \cdot \mid \pt_h(n), t_{h-1}^*)\). Finally, we combine both arguments to present the statistical error for \ac{tot}  prompting.


\vspace{2mm}

\noindent \textbf{Step 1: \ac{llm}-induced probability $p_h$ approximates $ p_h^* $ for large $n$.} 
In this step, we show \( p_h \) approaches \( p_h^* \) as \( n \) increases, which is similar to Step 1 of the proof of Theorem~\ref{thm: rate_cot}.

For each $h\in [H]$, $\pt_h(n)$  contains the \ac{cot} examples in terms of the prediction in the $h$-th step. We want to lower bound the probability of outputting $z_h^*$ by prompting a perfectly pretrained \ac{llm}. For simplicity, we assume $\Theta$ is finite and discrete. For each $h\in [H]$, using Bayes rule, we write  the posterior as 
$$
        \pi(\theta^* \given \pt_h(n), t_{h-1}^*)   =\bigg(1+\sum_{\theta\neq \theta^*}\frac{\PP(\pt_h(n) \given \theta)\PP(t_{h-1}^* \given \theta)\pi(\theta)}{\PP(\pt_h(n) \given \theta^*)\PP(t_{h-1}^* \given \theta^*)\pi(\theta^*)}\bigg)^{-1} . 
$$
Under Assumption~\ref{assump:tot-1}, $\theta^*$ maximizes $\PP(t_{h}^*\given \theta)$ for each $h\in [H]$. Thus, we have 
    \begin{align}
        \pi(\theta^* \given \pt_h(n), t_{h-1}^*) 
        & \geq\bigg(1+\sum_{\theta\neq \theta^*}\frac{\PP(\pt_h(n) \given \theta)\pi(\theta)}{\PP(\pt_h(n)\given \theta^*)\pi(\theta^*)}\bigg)^{-1} \nonumber \\
        &\geq \bigg(1+\sum_{\theta\neq \theta^*}\frac{ \pi(\theta)}{\pi(\theta^*)}\cdot\exp\big[-2n \text{H}^2 \big(\PP(t_h\given \theta),\PP(t_h \given \theta^*)\big)+2\log(\delta^{-1}|\Theta|)\big]\bigg)^{-1} \nonumber \\
        &\geq \bigg(1+\frac{1-\pi(\theta^*)}{\pi(\theta^*)}\cdot\exp\big(-2n\lambda_h+2\log(\delta^{-1}|\Theta|)\big)\bigg)^{-1} \label{eq:posterior_bd}
    \end{align}
    with probability at least $1-\delta$. 
    Here the second inequality follows from Lemma~\ref{lemma: likelihoodbd} in the proof of Theorem~\ref{thm: rate_cot}  and the third inequality follows from  Assumption~\ref{assump:tot-1}. Thus we derive a lower bound for $p_h$ as 
\begin{align*}
    p_h&=\sum_{\theta\in \Theta}\PP(z_h^*\given t_{h-1}^*,\theta)\pi(\theta \given  t_{h-1}^*,\pt_h(n)) \geq \PP(z_h^*\given t_{h-1}^*,\theta^*)\pi(\theta^* \given  t_{h-1}^*,\pt_h(n)) \nonumber \\
    &\geq p_h^* \cdot \bigg(1+\frac{1-\pi(\theta^*)}{\pi(\theta^*)} \cdot\exp\big(-2n\lambda_h+2\log(\delta^{-1}|\Theta|)\big)\bigg)^{-1}.
\end{align*}
with probability at least $1-\delta$ with respect to the randomness of the \ac{cot} prompt. The first inequality follows from omitting terms corresponding to $\theta\neq \theta^*$. The second inequality is due to \eqref{eq:posterior_bd}. 
Now we use the fact that $(1+x)^{-1}\geq 1-x$ for $x\in [0,1]$ to obtain that 
\begin{align}
   p_h^* -  p_h
     \leq    p_h^* \cdot  \frac{1-\pi(\theta^*)}{\pi(\theta^*)}\cdot \exp\big(-2n\lambda_h+2\log(\delta^{-1}|\Theta|)\big)  .\label{eq:mode_bd2}
\end{align}
Therefore, we conclude that \( p_h^* - p_h \), the difference in the probabilities evaluated at \( z_h^* \) for \ac{llm}-induced distribution and the true distribution, decreases exponentially with the number of examples \( n \). 
Recall that we define $\lambda^* = \min_{h\in[H]}\lambda_h$. 
For any $\epsilon \in (0, 1)$, we define 
\begin{align}
    n^* = \frac{1}{\lambda^*}\bigg(2\log \big(H|\Theta| \big) + \log \big((1-\pi(\theta^*))/\pi(\theta^*)\big) +\log(1/\epsilon)\bigg). \label{eq:n*_tot}
\end{align}
Then by \revise{applying Theorem~\ref{thm: rate_cot} with}{taking} $\delta = e^{-n\lambda^*/2}$ and combining \eqref{eq:mode_bd2},  we conclude that with probability at least $1-  e^{-n\lambda^*/2}$, when $n\geq n^*$, 
$p_h \geq p_h^*\cdot (1-\epsilon)$ holds for 
 every $h\in [H]$. Thus we conclude \textbf{Step 1}.

\vspace{2mm}

\noindent \textbf{Step 2: Large \( K \) improves the selection of  \( z_h^* \).} 
For any $h\in [H]$, let $$ C_h = \sum_{b = 1}^B \sum_{i=1}^{K } \mathbf{1}{\{ (t_{h-1}^{b }, z_{h}^{b,i} )  = t_h^*\} } $$ 
denote the number of candidates in $\cT_h$ that match the optimal partial history \( t_h^* \) at step $h\in [H]$, where we define $\cT_h $ in \eqref{eq:define_candiates_tree_search}. Recall that we set \(B= 1 \). 
Therefore, each $C_h$, $h\in [H]$ is a binomial random variable, where 
    \begin{align*}
        C_1  \sim \mathtt{Bin}(K, p_1), \qquad 
        C_h  \sim \mathtt{Bin}(\min\{B=1,C_{h-1}\}K,p_h) \text{ for }2\leq h\leq H.
    \end{align*}
     The algorithm outputs $t_H^*$ if and only if $C_h\geq 1$ for all $h\in [H]$. 
    The following gives the probability of outputting the optimal trajectory from the search tree: 
    \begin{align}
        \PP\big(\hat t_H = t_H^* \given \pt_\mathrm{CoT}(n)\big)&= \PP\big(C_1,\cdots,C_H\geq 1 \given \pt_\mathrm{CoT}(n)\big) \nonumber \\
        & \geq  \prod_{h=1}^{H}\PP\big(C_{h}\geq 1\given \pt_h(n), t_{h-1}^*\big) \nonumber \\
        &=\prod_{h=1}^H \big(1-(1-p_h)^K\big) \nonumber \\
        &\geq 1-\sum_{h=1}^H\big(1-p_h\big)^K. \label{eq:k_bd}
    \end{align}
The last inequality is because $\prod_{i=1}^m(1-x_i)\geq 1-\sum_{i=1}^m x_i$ for $x_i\in [0,1]$.

Combining \eqref{eq:k_bd} and the conclusion of {\bf Step 1}, we conclude that with a perfectly pretrained \ac{llm} and a discrete and finite $\Theta$, \ac{tot} prompting using BFS with $B=1$ 
incurs the following statistical error with probability at least $1-e^{-n\lambda^*/2}$:
\begin{align*}
    \PP\big(\hat t_H \neq t_H^* \given \pt_\mathrm{CoT}(n)\big)&\leq \sum_{h=1}^H\Big(1-p_h^* + p_h^*\epsilon \Big)^{K}.
\end{align*} 
Here $n$ is sufficiently large such that $n\geq n^*$ where $n^*$ is defined  in \eqref{eq:n*_tot} and $\epsilon \in (0,1)$ is an arbitrary number.   
 Therefore, we conclude the proof.
\end{proof}

\subsection{Proof of Corollary \ref{cor: selection-inference cot}} \label{proof: selection-inference cot}
\begin{proof}
The main idea of selection-inference prompting is to break down each step in vanilla \ac{cot} into two separate stages: selection and inference. We decompose each task $\theta$ into two components: $\theta_\mathrm{se}$ and $\theta_\mathrm{in}$, with underlying distributions $\PP(\tau_h =\cdot \given t_{h-1}, \theta_\mathrm{se})$ and $\PP(z_h =\cdot\given \tau_h, \theta_\mathrm{in})$, respectively. We follow and modify the proof for vanilla \ac{cot} in Appendix~\ref{proof: rate_cot} to derive the statistical rate of \ac{si} prompting.

The proof consists of two steps. 
We first derive an upper bound for the KL divergence using a ratio of two integrals. Then we bound such a ratio using the separation among the probability distributions with different parameters in $\Theta$.  

\vspace{2mm}

\noindent \textbf{Step 1: Deriving an upper bound for the KL divergence.} We invoke Proposition~\ref{prop:elb} to derive an upper bound for the KL divergence as follows: 
\begin{align}
      & \KL \bigg(\PP(y^{\mathrm{test}}=\cdot \given z_0^{\mathrm{test}},  \theta^*)  ,   \PP_\mathrm{SI}\big(y^{\mathrm{test}}=\cdot  \given S_{\mathrm{se}}(n),S_{\mathrm{in}}(n),z_0^{\mathrm{test}}\big) \bigg) \nonumber \\
     &\quad \leq \log \bigg(1+\cfrac{\sum_{\Theta^\complement}\PP\big(S_{\mathrm{se}}(n),S_{\mathrm{in}}(n),z_0^{\mathrm{test}} \biggiven \theta \big) \pi(\theta) \text{d}\theta}{\sum_{\Theta_{\mathrm{eq}}(\theta^*)}\PP\big(S_{\mathrm{se}}(n),S_{\mathrm{in}}(n), z_0^{\mathrm{test}} \biggiven  \theta' \big) \pi(\theta') \text{d}\theta'}  \bigg) \notag \\
     &\quad \leq \log \bigg(1+\cfrac{\sum_{\Theta^\complement}\PP\big(S_{\mathrm{se}}(n),S_{\mathrm{in}}(n),z_0^{\mathrm{test}} \biggiven \theta \big) \pi(\theta) }{\PP\big(S_{\mathrm{se}}(n),S_{\mathrm{in}}(n), z_0^{\mathrm{test}} \biggiven  \theta^*\big) \pi(\theta^*) }\bigg).\label{eq: elb_si}
    \end{align}
The first inequality is obtained by applying Proposition~\ref{prop:elb} with substituting $\pt_\mathrm{CoT}(n)$ by $(S_{\mathrm{se}}(n),S_{\mathrm{in}}(n),z_0^{\mathrm{test}})$, and in the second inequality we exclude all terms corresponding to \(\theta \neq \theta^*\) in the denominator. We can directly apply Proposition~\ref{prop:elb} since this proposition only requires the prompt to be generated from $\PP(\cdot\given \theta^*)$, but does not assume a specific statistical dependency relationship in the prompt. 
Here $\PP_{\mathrm{SI} }(y^{\mathrm{test}} \given S_\mathrm{se}(n),S_\mathrm{in}(n), z_0^{\mathrm{test}})  $
        is the marginal distribution of $y^{\mathrm{test}} $ according to \eqref{eq:si_prompt_output}.
\vspace{2mm}

\noindent \textbf{Step 2: Statistical rate for discrete and finite $\Theta$.} The key to analyzing  \eqref{eq: elb_si} is to derive an upper bound of the likelihood ratio  $$\PP\big(S_{\mathrm{se}}(n),S_{\mathrm{in}}(n),z_0^{\mathrm{test}} \biggiven \theta \big)/\PP\big(S_{\mathrm{se}}(n),S_{\mathrm{in}}(n),z_0^{\mathrm{test}} \biggiven \theta^* \big).$$
Recall that we define $S_{\mathrm{se}}(n)$ and $S_{\mathrm{in}}(n)$ in \eqref{eq:define_si_examples}. 
 By construction, the likelihood of a single piece of trajectory can be decomposed as
\begin{align}
    &\PP\big(S_{\mathrm{se}}(1),S_{\mathrm{in}}(1) \biggiven \theta \big) \nonumber \\
    &\quad = \PP(z_0\given \theta)\prod_{h=1}^H\PP\big(\tau_h \biggiven \theta_\mathrm{se},  t_{h-1}\big) \cdot \prod_{j=1}^H\PP\big( z_j \biggiven \theta_\mathrm{in},\tau_j \big). \label{eq:si-decomp}
\end{align}
Note that the $n$ reasoning paths are independent conditioning on $\theta$. 
We decompose the likelihood ratios into a sum of independent terms according to \eqref{eq:si-decomp} and then apply Lemma~\ref{lemma: hellingerbd}. 
With probability at least $1-\delta$, we have 
\begin{align}
        &\frac{1}{2}\log \frac{\PP(S_{\mathrm{se}}(n),S_{\mathrm{in}}(n) \given \theta)}{\PP(S_{\mathrm{se}}(n),S_{\mathrm{in}}(n) \given \theta^*)} \nonumber \\
        &\quad =\frac{1}{2}\sum_{i=1}^n \bigg[\log\frac{\PP( z_0^{i}\given \theta)}{\PP( z_0^{i}\given \theta^*)}+\sum_{h=1}^H \log\frac{\PP\big(\tau_h^{i} \biggiven \theta_\mathrm{se}, t_{h-1}^{i} \big) }{\PP\big(\tau_h^{i} \biggiven \theta^*_{\mathrm{se}}, t_{h-1}^{i} \big) }+ \sum_{j=1}^H\log \frac{\PP\big( z_j^{i}  \biggiven \theta_\mathrm{in},\tau_j^{i}  \big)}{\PP\big( z_j^{i}  \biggiven \theta_\mathrm{in}^*,\tau_j^{i}  \big)}\bigg] \nonumber \\
        & \quad \leq \sum_{i=1}^n \log\bigg[\E_{\theta^*}\bigg(\frac{\PP(z_0^{i}\given \theta)}{\PP(z_0^{i}\given \theta^*)}\bigg)^{1/2}\bigg]+\sum_{h=1}^H \sum_{k=1}^n \log \bigg[\E_{\theta^*}\bigg(\frac{\PP(\tau_h^{k} \biggiven \theta_\mathrm{se}, t_{h-1}^{k})}{\PP(\tau_h^{k} \biggiven \theta^*_{\mathrm{se}}, t_{h-1}^{k})}\bigg)^{1/2} \bigg] \nonumber  \\
        &\quad \qquad + \sum_{j=1}^H \sum_{m=1}^n \log \bigg[\E_{\theta^*}\bigg(\frac{\PP\big( z_j^{m}  \biggiven \theta_\mathrm{in},\tau_j^{m}  \big)}{\PP\big( z_j^{m}  \biggiven \theta_\mathrm{in}^*,\tau_j^{m}  \big)}\bigg)^{1/2}\bigg]+ \log((2H+1)\delta^{-1}). \label{eq:si_bd_1}
        \end{align}
      Here the first equality follows from summing over the decomposition of likelihood ratios as shown in \eqref{eq:si-decomp}, and inequality follows is obtained by applying Lemma \ref{lemma: hellingerbd} to each sum. 
      Using \eqref{eq:si_bd_1} and the fact that $x-1 \geq \log(x)$, we further obtain that 
        \begin{align}
        &\frac{1}{2}\log \frac{\PP(S_{\mathrm{se}}(n),S_{\mathrm{in}}(n) \given \theta)}{\PP(S_{\mathrm{se}}(n),S_{\mathrm{in}}(n) \given \theta^*)}\nonumber \\
        & \quad \leq \sum_{i=1}^n \bigg[\E_{\theta^*}\bigg(\frac{\PP( z_0^{i}\given \theta)}{\PP( z_0^{i}\given \theta^*)}\bigg)^{1/2}-1+\sum_{h=1}^H \bigg(\E_{\theta^*}\frac{\PP(\tau_h^{i} \biggiven \theta_\mathrm{se},  t_{h-1}^{i})}{\PP(\tau_h^{i} \biggiven \theta^*_{\mathrm{se}},  t_{h-1}^{i})}\bigg)^{1/2} -1\bigg)\nonumber \\
        &\quad \qquad + \sum_{j=1}^H\bigg(\E_{\theta^*}\bigg(\frac{\PP\big( z_j^{i}  \biggiven \theta_\mathrm{in},\tau_j^{i}  \big)}{\PP\big( z_j^{i}  \biggiven \theta_\mathrm{in}^*,\tau_j^{i}  \big)}\bigg)^{1/2}-1\bigg)\bigg]+ \log\big((2H+1)\delta^{-1}\big) \nonumber \\
        & \quad = -n \sum_{h=1}^H \Big[\E_{\theta^*}\text{H}^2 \big(\PP(\tau_h \given \theta^*, t_{h-1}) , \PP(\tau_h\given \theta,t_{h-1}) \big) \nonumber \\
        &\quad \qquad +\E_{\theta^*}\text{H}^2 \big(\PP(z_h \given \theta^*,\tau_h) , \PP(z_h\given \theta,\tau_h) \big)\Big] -n\text{H}^2 \big(\PP(z_0 \given \theta^*) , \PP( z_0\given \theta) \big) \nonumber \\
        & \quad \qquad+\log\big((2H+1)\delta^{-1}\big). \label{eq:si_bd_2}
    \end{align}
    Here the final line follows from the definition of Hellinger distance.

    Now we apply Assumption~\ref{assumption: seperation2} to derive an upper bound of the likelihood ratio using the constant $\lambda_{\mathrm{SI}} = \lambda_{\mathrm{S}} + \lambda_{\mathrm{I}} + \lambda_{\mathrm{q}}$, replacing the constant $\lambda$ from Theorem~\ref{thm: rate_cot}. 
    Combing \eqref{eq:si_bd_2} and Assumption~\ref{assumption: seperation2}, we have that when ${\Theta}$ is discrete and finite,  
    \begin{align*}
        &\frac{\PP\big(S_{\mathrm{se}}(n),S_{\mathrm{in}}(n),z_0^{\mathrm{test}} \given \theta\big)}{\PP\big(S_{\mathrm{se}}(n),S_{\mathrm{in}}(n),z_0^{\mathrm{test}}\given \theta^* \big)}\\
        &\quad \leq \exp \big(-2n\lambda_{\mathrm{SI}}+2\log\big((2H+1)\delta^{-1}\big)\big) \text{ with probability at least $1-\delta$}. 
    \end{align*}
Therefore, we have that when ${\Theta^\complement}$ is discrete and finite, then with probability at least $1-\delta$,
    \begin{align*}
        &\KL \Big(\PP(y^\mathrm{test} = \cdot \given  z_0^{\mathrm{test}}, \theta^*)  ,   \PP_\mathrm{SI}\big(y^\mathrm{test} =\cdot \given S_{\mathrm{se}}(n),S_{\mathrm{in}}(n), z_0^{\mathrm{test}}\big) \Big)\\
        &\quad =\mathcal{O}\big(\pi(\theta^*)^{-1} \cdot \delta^{-2} \cdot | {\Theta^{\complement}}| ^2  \cdot e^{-2\lambda_{\mathrm{SI}} n} \big),
    \end{align*}
    where $\calO (\cdot) $ omits only absolute constants. Therefore, we conclude the proof.
\end{proof}

\newpage 
\section{Proofs  and Auxiliary Results of Section~\ref{subsubsec: comparison}}

\subsection{Proof of a Generalized Version of Proposition~\ref{prop: dom}} \label{subsec: gen_dom}

In the following, we generalize Proposition~\ref{prop: dom} to handle comparisons of different truncated \ac{cot} methods, which covers vanilla \ac{icl} as a special case. 
For simplicity, we assume
zero pretraining error and input query does not have a distributional shift, i.e., $\PP_\mathrm{LLM} = \PP$ and $z_0^\mathrm{test} \sim \PP(\cdot \given \theta^*)$. Before presenting the result, we state the following regularity assumption.

\begin{assumption}\label{assump: compare2}
For a fixed \ac{cot} prompt $\pt_\mathrm{CoT}(n)$, we define truncated \ac{cot} prompts with fixed intermediate step indices as  $\pt_\calJ (n)= \{z_0^i,y^i\}_{i=1}^n \cup \{z_{j}^{i}\}_{j\in \calJ}\cup\{ z_0^\mathrm{test}\}$ with fixed index set $\calJ\subseteq [H-1]$, where $\pt_\mathrm{CoT}(n) = \{z_{0}^i,\cdots, z_H^i\}_{i=1}^n \cup \{z_0^\mathrm{test}\} $. 
We assume that \(\PP\big(y^{\mathrm{test}} \given \mathtt{prompt}_{\calJ}(n)\big)\) is a mixture of \(\PP \big(y^{\mathrm{test}} \given \theta, z_0^\mathrm{test}\big)\) aggregated with respect to the posterior of \(\theta\) based on truncated \ac{cot} demonstrations. 
Specifically, the density \(\PP \big(y^{\mathrm{test}} \given \theta, z_0^\mathrm{test}\big)\) is obtained by marginalizing the omitted intermediate steps $z_{i\notin \calJ}$ from the joint distribution induced by the \ac{cot} model~\eqref{eq:latent_var_model}.
\end{assumption}

The truncated \ac{cot} method recovers \ac{cot} by setting \(\calJ = [H-1]\) and vanilla \ac{icl} by setting \(\calJ = \emptyset\). This assumption ensures that the estimators induced by different truncated \ac{cot} methods are comparable. This can be achieved by training \ac{llm}s using a truncated dataset and then prompted via the truncated CoT. 
Specifically, 
each \(\calJ\), we define \(\calD_\calJ\) as a truncated version of the \ac{cot} data \(\calD\) obtained by omitting intermediate steps with indices \(j \notin \calJ\). 
Then we can pretrain an LLM using \(\calD_\calJ\) using the same MLE loss as in \eqref{eq:pretraining_loss}, and then prompt the learned model using $\pt_\calJ (n)$. 
Assumption~\ref{assump: compare2} requires \ac{llm}s to process different truncated \ac{cot} prompts by making posterior inferences based on their respective pretraining data. 
This is the case when the LLM is pretrained using \(\calD_\calJ\), following a similar argument as in Section \ref{subsec: attention parameterizes bma}.  
Based on this assumption, we establish a hierarchy of CoT methods in terms of statistical error.


\begin{corollary}[Comparison of CoT Methods]\label{prop: dom2}
Let \(\pi\) represent the task distribution over space \(\Theta\). Let \(\pt_\mathrm{CoT}(n)\) be a \ac{cot} prompt. Given \(\calJ \subseteq \calJ' \subseteq [H-1]\), define \(\mathtt{prompt}_\calJ(n) = \{z_0^i,y^i\}_{i=1}^n \cup \{z_{j}^{i}\}_{j\in \calJ}\cup\{ z_0^\mathrm{test}\}\) and \(\mathtt{prompt}_{\calJ'}(n) = \{z_0^i,y^i\}_{i=1}^n \cup \{z_{j}^{i}\}_{j\in \calJ'}\cup\{ z_0^\mathrm{test}\}\), where $ \pt_\mathrm{CoT}(n) = \{z_{0}^i,\cdots, z_H^i\}_{i=1}^n \cup \{z_0^\mathrm{test}\}$. Under the Assumption~\ref{assump: compare2}, for any number of examples $n\geq 0$,  we have 
\begin{align*}
    &\E_{\theta^* \sim \pi} \E_{\mathtt{prompt}_{\calJ'}(n)\sim \PP(\cdot \given \theta^*)} \bigg[\KL \bigg(\PP(y^{\mathrm{test}}=\cdot \given z_0^{\mathrm{test}},\theta^*)  ,   \PP\big(y^{\mathrm{test}}=\cdot \given  \mathtt{prompt}_{\calJ'}(n)\big) \bigg) \bigg]\\
    & \quad \leq \E_{\theta^* \sim \pi}\E_{\mathtt{prompt}_{\calJ}(n)\sim \PP(\cdot \given \theta^*)}\bigg[\KL \bigg(\PP(y^{\mathrm{test}}=\cdot \given z_0^{\mathrm{test}},\theta^*)  ,   \PP\big(y^{\mathrm{test}}=\cdot \given \mathtt{prompt}_\calJ(n)\big) \bigg)\bigg].
\end{align*}
 We notice that such an inequality only holds in an average case by taking an expectation with respect to $\theta ^* \sim \pi$.    
\end{corollary}

\subsection{Proof of Corollary \ref{prop: dom2}}  \label{proof: dom2}
\begin{proof}

To simplify the notation, let $\E_{\theta^*}$ denote $\E_{\mathtt{prompt}_{\mathrm{CoT}}(n)\sim \PP(\cdot \given \theta^*)}$. First, we compute the difference of KL divergences of two index sets $\cJ$ and $\cJ'$ on  a fixed task $\theta^*$:
\begin{align}
    &\E_{\theta^*} \bigg[\KL \big(\PP(y^{\mathrm{test}} = \cdot \given z_0^{\mathrm{test}},\theta^*), \PP\big(y^{\mathrm{test}}= \cdot \given \mathtt{prompt}_\calJ(n)\big)\big) \bigg] \nonumber \\
    &\qquad\qquad - \E_{\theta^*} \bigg[\KL \big(\PP(y^{\mathrm{test}} = \cdot\given z_0^{\mathrm{test}},\theta^*),\PP\big(y^{\mathrm{test}} = \cdot\given \mathtt{prompt}_{\calJ'}(n)\big)\big) \bigg] \nonumber \\
    &\quad = \E_{\theta^*} \E_{y^{\mathrm{test}} \sim \PP(\cdot \given z_0^{\mathrm{test}},\theta^*)} \bigg[\log \frac{\PP\big(y^{\mathrm{test}} \given \mathtt{prompt}_{\calJ'}(n)\big)}{\PP\big(y^{\mathrm{test}} \given \mathtt{prompt}_\calJ(n)\big)}\bigg]. \label{eq:comparison_1}
\end{align}

Next, we take expectation of \eqref{eq:comparison_1} with respect to $\theta^*\sim \pi$ to obtain 
\begin{align}
    &\E_{\pi} \E_{\theta^*} \E_{y^{\mathrm{test}} \sim \PP(\cdot \given z_0^{\mathrm{test}},\theta^*)} \bigg[\log \frac{\PP\big(y^{\mathrm{test}} \given \mathtt{prompt}_{\calJ'}(n)\big)}{\PP\big(y^{\mathrm{test}} \given \mathtt{prompt}_\calJ(n)\big)}\bigg] \nonumber \\
    &\quad = \int_{\calL\times\calL^*} \PP\big(\mathtt{prompt}_{\calJ'}(n)\big) \bigg[\log \frac{\PP\big(y^{\mathrm{test}} \given \mathtt{prompt}_{\calJ'}(n)\big)}{\PP\big(y^{\mathrm{test}} \given \mathtt{prompt}_\calJ(n)\big)}\bigg] \nonumber \\
    &\qquad \cdot \bigg(\int_{\Theta} \pi\big(\theta^* \given \mathtt{prompt}_{\calJ'}(n)\big) \PP(y^{\mathrm{test}} \given \mathtt{prompt}_{\calJ'}(n),\theta^*) \text{d}\theta^*\bigg) \text{d}y^{\mathrm{test}} \text{d}\mathtt{prompt}_{\calJ'}(n).\label{eq:comparison_2}
\end{align}
Applying the Bayes' rule, we have 
\begin{align}
    \label{eq:comparison_20}
    & \int_{\calL}\bigg(\int_{\Theta} \pi\big(\theta^* \given \mathtt{prompt}_{\calJ'}(n)\big) \cdot \PP(y^{\mathrm{test}} \given \mathtt{prompt}_{\calJ'}(n),\theta^*)  \cdot  \bigg[\log \frac{\PP\big(y^{\mathrm{test}} \given \mathtt{prompt}_{\calJ'}(n)\big)}{\PP\big(y^{\mathrm{test}} \given \mathtt{prompt}_\calJ(n)\big)}\bigg] \text{d}\theta^*\bigg) \text{d}y^{\mathrm{test}}   \notag \\
    & \qquad =  \E_{y^{\mathrm{test}} \sim \PP(\cdot \given \mathtt{prompt}_{\calJ'}(n))} \bigg[\log \frac{\PP\big(y^{\mathrm{test}} \given \mathtt{prompt}_{\calJ'}(n)\big)}{\PP\big(y^{\mathrm{test}} \given \mathtt{prompt}_\calJ(n)\big)}\bigg].
\end{align}
Thus, by \eqref{eq:comparison_20} and  interchanging the order of integration in \eqref{eq:comparison_2}, we have 
\begin{align}
   &  \E_{\pi} \E_{\theta^*} \E_{y^{\mathrm{test}} \sim \PP(\cdot \given z_0^{\mathrm{test}},\theta^*)} \bigg[\log \frac{\PP\big(y^{\mathrm{test}} \given \mathtt{prompt}_{\calJ'}(n)\big)}{\PP\big(y^{\mathrm{test}} \given \mathtt{prompt}_\calJ(n)\big)}\bigg] \notag \\
   &\qquad = \E_{\mathtt{prompt}_{\calJ'}(n) \sim \PP} \E_{y^{\mathrm{test}} \sim \PP(\cdot \given \mathtt{prompt}_{\calJ'}(n))} \bigg[\log \frac{\PP\big(y^{\mathrm{test}} \given \mathtt{prompt}_{\calJ'}(n)\big)}{\PP\big(y^{\mathrm{test}} \given \mathtt{prompt}_\calJ(n)\big)}\bigg] \nonumber\\
    &\qquad = \E_{\PP} \bigg[\KL \big(\PP\big(y^\mathrm{test} = \cdot \given \mathtt{prompt}_{\calJ'}(n)\big) , \PP\big(y^\mathrm{test} = \cdot \given \mathtt{prompt}_\calJ(n)\big)\bigg] \geq 0. \label{eq:kl_diff}
\end{align}
Therefore, the expectation of the difference between the two KL divergences in \eqref{eq:comparison_2} can be expressed as \eqref{eq:kl_diff}, another KL divergence and thus is nonnegative for any number of samples \( n \). 
Therefore, we conclude that on average, truncated \ac{cot} methods with steps $\calJ' \subseteq [H-1]$ is no worse than truncated \ac{cot} methods with fewer intermediate steps \(\calJ \subset \calJ'\). Since vanilla \ac{icl} corresponds to the special case where \( \calJ = \emptyset \), we conclude that on average, \ac{cot} is no worse than vanilla \ac{icl}. This completes the proof.
\end{proof}

\newpage

\section{Proof and Auxiliary Results of Section~\ref{subsec: imperfect}} \label{app: pretrain}
In this section, we prove the results in Section ~\ref{subsec: imperfect}. 
In particular, in Section \ref{proof: pre-train} we prove Proposition \ref{prop: pre-train}, in Section \ref{app: construction} we introduce the details of the approximation error analysis for pretraining, and in Section \ref{proof: rate with error} we prove Corollary \ref{cor: rate with error}.

\subsection{Proof of Proposition \ref{prop: pre-train}} \label{proof: pre-train}
In the following, we prove 
Proposition~\ref{prop: pre-train} under the generalized multi-step latent variable model introduced in Section \ref{app: generalized model}, which contains the model in \eqref{eq:latent_var_model} as a special case.

\begin{proof}
In this proof, we adopt the PAC-Bayes framework \citep{mcallester1998some, alquier2021user} to decompose the error and control each component. 
 This proof consists of two steps. We first decompose the pretraining error into three parts using the PAC-Bayes framework, and then control each term to conclude the proof.
 Our proof is adapted from \cite{zhang2023and}, which analyze the generalization error for pretraining an LLM based with ICL data, i.e., $H = 1$. 
 We explain the structure of the proof in detail and highlight the similarities and differences from \cite{zhang2023and}. 

\vspace{2mm}

\noindent \textbf{Step 1: Error decomposition using the PAC-Bayes framework.} We fist  decompose the pretraining error $\E_\calD \TV \big(\PP(\cdot \given S), \PP_{\hat \rho}(\cdot\given S)\big)^2$ to prepare for further analysis.
Recall that the pretraining dataset with $N$ trajectories with $T$ examples is $\mathcal{D}_{N,T} = \{(S_h^{t,\ell}, z_h^{t,\ell})\}_{h=0, t=1, \ell=1}^{H, T, N}$, where $z_h^{t,\ell} \sim \PP(\cdot \given S_h^{t,\ell})$, and $S_{h+1}^{t,\ell} =  S_h^{t,\ell} \cup \{z_h^{t,\ell}\} $. 
Under the construction of $\cD_{N, T}$ under the general model in \eqref{eq:generalized_latent_var_model}, 
the training data admits a sequential structure. 
For any $h \in \{0, \ldots, H\}$, $t \in [T]$ and $\ell \in [N]$, we let 
$\cF_{h}^{t, \ell}$ denote a  $\sigma$-algebra defined as   \begin{align*}
   \cF_{h}^{t, \ell} = \mathtt{\sigma}\texttt{-algebra}\Big(\bigl\{ z_{h'}^{t', \ell'} \colon \ell < \ell, \mathrm{or~} \ell' = \ell, t' < t, \mathrm{or~}\ell' = \ell, t' = t, h' < h\bigr\}\Big),
\end{align*}
which is the $\sigma$-algebra generated by all the random variables in $\cD_{N,T}$ appearing before $z_h^{t, \ell} $.
Moroever, we  define a sequence of ghost samples as $\Tilde{\mathcal{D}}_{N,T} = \{(\Tilde S_h^{t,\ell}, \Tilde z_h^{t,\ell})\}_{h=0, t=1, \ell=1}^{H, T, N}$, where $\Tilde z_h^{t,\ell} \sim \PP(\cdot \given \Tilde S_h^{t,\ell})$, and $ \Tilde S_h^{t,\ell}= S_h^{t,\ell}$. 
Here $\tilde z_h^{t, \ell}$ is independent of $z_h^{t, \ell} $ and all random variables in $\cD_{N, T}$ generated later than $S_h^{t,\ell}$.
In the following, we use this ghost sample and Donsker-Varadhan representation to decompose the error.

Donsker-Varadhan representation \citep{mackay2003information} states that, for any distribution $P, Q \in \calP_{\mathrm{LLM}}$ and for any function $g: \calP_{\mathrm{LLM}} \rightarrow \mathbb{R}$ such that $\E_{\rho\sim Q}[\exp\big(g(\rho)\big)]<\infty$, we have 
\begin{align}\label{eq: dv}
    \E_{\rho \sim P}[g(\rho)] \leq \KL(P , Q) + \log\E_{\rho\sim Q}[\exp \big(g(\rho) \big)].
\end{align}
To proceed, we choose $Q\in\calP_{\mathrm{LLM}}$ independent of both the dataset $\calD_{N,T}$ and ghost dataset $\tilde \calD_{N,T}$ and $P\in\calP_{\mathrm{LLM}}$ to be potentially dependent on the dataset $\calD_{N,T}$ but independent of the newly sampled $\{\Tilde z_h^{t,\ell}\}$ in the ghost dataset.
To simplify the notation, we omit the subscripts and write the datasets as $\cD$ and $\tilde \cD$ respectively. 
In the sequel, we use $\EE_{\calD}$ or $\EE_{\tilde \calD}$ to denote the expectation with respect to the joint distribution of $\cD$ and $\tilde \cD$, respectively.
We set the function $g$ as $g(\rho) = L(\rho,\calD)-\log \bbE_{\tilde{\calD}}[\exp(L(\rho,\tilde{\calD})) \given \cD]$, where 
    \begin{align}\label{eq:define_L_rho_D}
        L(\rho,\tilde \calD)=-\frac{1}{4}\sum_{\ell=1}^{N}\sum_{t=1}^T\sum_{h=0}^{H}\log\frac{\PP(\tilz_h^{t,\ell}\given S_h^{t,\ell})}{\PP_{\rho}(\tilz_h^{t,\ell}\given S_h^{t,\ell})}.
    \end{align}
Moreover, for any $\ell \in [N]$, $t\in [T]$, and $ h \in \{ 0, \ldots , H\}$, we let 
$L^{\ell,t, h} (\rho, \tilde D)$ denote the partial sum of $L(\rho, \tilde D)$ with the last term being $ - 1/ 4 \cdot (\log \PP(\tilz_h^{t,\ell}\given S_h^{t,\ell})  - \log \PP_{\rho}(\tilz_h^{t,\ell}\given S_h^{t,\ell})  ) $.
    We note that $L(\rho, \tilde D)$ itself is a random variable where the randomness stems from both $\rho$ and $\tilde D$.

    Exponentiating both sides of \eqref{eq: dv} and taking expectations with respect to $\calD$ on both sides, we have  
    \begin{align*}
        & \E_{\calD}\Big[\exp\Big(\E_{\rho\sim P}\big[L(\rho,\calD)-\log \bbE_{\tilde{\calD}}[\exp(L(\rho,\tilde{\calD}))]\big]-\KL(P, Q)\Big)\Big] \notag \\  &\quad \leq \E_{\cD}\E_{\rho\sim Q}[\exp \big(L(\rho,\calD)-\log \bbE_{\tilde{\calD}}[\exp(L(\rho,\tilde{\calD})) \given \cD ] \big)] \\
        &\quad = \E_{\rho\sim Q}\E_{\cD}[\exp \big(L(\rho,\calD)-\log \bbE_{\tilde{\calD}}[\exp(L(\rho,\tilde{\calD}))]  \given \cD \big)]  = \E_{\rho\sim Q}\E_{\cD}\bigg[\frac{\exp(L(\rho,\calD))}{ \bbE_{\tilde{\calD}}[\exp(L(\rho,\tilde{\calD}))\given \cD]}   \bigg] ,
    \end{align*}
where in the first equality we exchange the order of expectations due to the independence between the dataset $\calD$ and the prior $Q$. 
By the construction of $L(\rho, \tilde \cD)$   in \eqref{eq:define_L_rho_D}, we have 
$$
\E_{\tilde{\calD}}[\exp(L(\rho,\tilde{\calD}))\given \cD] = \prod_{\ell=1}^N \prod _{t=1}^T \prod_{h=0}^H \E_{\tilde z_{h}^{t , \ell }} \Big[  \exp\Big ( -1/4 \cdot \big ( \log \PP(\tilde z_{h}^{t, \ell} \given S_{h}^{ t, \ell} ) - \log \PP_{\rho} (\tilde z_{h}^{t, \ell} \given S_{h}^{ t, \ell} )\big)\Big)\Big]. 
$$
Notice that this is a random variable that is measurable under $\cF_{N,T, H}$. 
Moreover, conditioning on $\cF_{N,T, H}$, 
we have 
\begin{align*}
& \E_{\cD} \bigl [ \exp(L(\rho,\calD))\biggiven \cF_{N, T, H} \big] \notag \\
& \quad =    \exp \bigl ( L^{N, T, H-1} (\rho, \cD) \bigr) \cdot    \E_{  z_{H}^{T , N }}  \Big[  \exp\Big ( -1/4 \cdot \big ( \log \PP(z_{H}^{T , N } \given S_{H}^{ T, N} ) - \log \PP_{\rho} (z_{H}^{T , N } \given S_{H}^{ T, N} )\big)\Big)\Big].
\end{align*}
Since $\tilde z_{H}^{T , N }$ and $z_{H}^{T , N }$ have the same conditional distribution, we obtain that 
\begin{align*} 
\E_{\cD}\bigg[\frac{\exp(L(\rho,\calD))}{ \bbE_{\tilde{\calD}}[\exp(L(\rho,\tilde{\calD}))\given \cD]}   \bigggiven \cF_{N, T, H} \bigg]   =  \frac{\exp \bigl( L^{N, T, H-1} (\rho, \cD)  \big) }{ \E_{\tilde \cD} \bigl [ \exp( L^{N, T, H-1} (\rho, \tilde \cD)\given \cD ]   }     .   
\end{align*}
Then, using the tower property, we  similarly have 
\begin{align}\label{eq:recursion_expectation}
& \E_{\cD} \biggl [ \E_{\cD}\bigg[\frac{\exp(L(\rho,\calD))}{ \bbE_{\tilde{\calD}}[\exp(L(\rho,\tilde{\calD}))\given \cD]}   \bigggiven \cF_{N, T, H} \bigg] \bigggiven \cF_{ N, T, H-1} \biggr ]    \notag \\
& \qquad = \EE _{\cD} \bigg[ \frac{\exp \bigl( L^{N, T, H-1} (\rho, \cD)  \big) }{ \E_{\tilde \cD} \bigl [ \exp( L^{N, T, H-1} (\rho, \tilde \cD)\given \cD ]   }  \bigggiven \cF_{ N, T, H-1} \biggr ]  =  \frac{\exp \bigl( L^{N, T, H-2} (\rho, \cD)  \big) }{ \E_{\tilde \cD} \bigl [ \exp( L^{N, T, H-2} (\rho, \tilde \cD)\given \cD ]   }. 
\end{align} 
Recursively apply conditional expectations to \eqref{eq:recursion_expectation} with respect to  the filtration $\{\cF_{\ell, t, h} \}$, we obtain that  
$$
\E_{\cD}\Big[ \exp(L(\rho,\calD)) \big /  \bbE_{\tilde{\calD}}[\exp(L(\rho,\tilde{\calD}))\given \cD]    \Big] =  1 .
$$
Therefore, we have 
    \begin{align*}
        \E_{\calD}\Big[\exp\Big(\E_{\rho\sim P}\big[L(\rho,\calD)-\log \bbE_{\tilde{\calD}}[\exp(L(\rho,\tilde{\calD}))]\big]-\KL(P, Q)\Big) \Big ] &\leq 1.
    \end{align*}
    Applying the Chernoff bound to it, we obtain a high probability bound as follows. With probability at least $1-\delta$, we have 
\begin{align} \label{eq: ch bd}
    -\E_{\rho\sim P}\Big[\log\E_{\tilde{\calD}}\big[ \exp\big( L(\rho,\tilde{\calD})\big)\big]\Big]\leq -\E_{\rho\sim P}\big[L(\rho,\calD)\big]+ \KL(P, Q)+\log\frac{1}{\delta}.
\end{align}

Now we separately bound the left-hand side and the right-hand side of \eqref{eq: ch bd}. 
Similar to the derivation in \cite{zhang2023and}, 
for the left-hand side of \eqref{eq: ch bd}, 
using the definition of $L(\rho, \tilde \cD)$ and Cauchy-Schwarz inequality, 
we have 
\begin{align} \label{eq:combine_ineq}
    &\log\bbE_{\tilde{\calD}}\big[ \exp\big( L(\rho,\tilde{\calD})\big)  \given \calD\big]  \\
    & \quad = \log\bbE_{\tilde{\calD}}\bigg[\exp\bigg(-\frac{1}{4}\sum_{\ell=1}^{N}\sum_{t=1}^T\sum_{h=0}^{H}\log\frac{\PP(\tilz_h^{t,\ell}\given S_h^{t,\ell})}{\PP_{\rho}(\tilz_h^{t,\ell}\given S_h^{t,\ell})}\bigg) \Biggiven \calD\bigg] \nonumber\\
    & \quad = \log\bbE_{\tilde{\calD}}\bigg[\exp\bigg(-\frac{1}{4}\sum_{\ell=1}^{N}\sum_{t=1}^T\sum_{h=0}^{H}\Bigl(\log\frac{\PP(\tilz_h^{t,\ell}\given S_h^{t,\ell})}{\PP_{\hat \rho}(\tilz_h^{t,\ell}\given S_h^{t,\ell})} + \log\frac{\PP_{\hat \rho}(\tilz_h^{t,\ell}\given S_h^{t,\ell})}{\PP_{\rho}(\tilz_h^{t,\ell}\given S_h^{t,\ell})}\Bigr)\bigg)  \bigggiven \calD \bigg] \nonumber \\
    & \quad \leq \frac{1}{2}\log\bbE_{\tilde{\calD}}\bigg[\exp\bigg(-\frac{1}{2}\sum_{\ell=1}^{N}\sum_{t=1}^T\sum_{h=0}^{H}\log\frac{\PP(\tilz_h^{t,\ell}\given S_h^{t,\ell})}{\PP_{\hat \rho}(\tilz_h^{t,\ell}\given S_h^{t,\ell})}\bigg)  \Biggiven \calD \bigg] \nonumber \\
    &\qquad \qquad \qquad 
 + \frac{1}{2}\log\bbE_{\tilde{\calD}}\bigg[\exp\bigg(-\frac{1}{2}\sum_{\ell=1}^{N}\sum_{t=1}^T\sum_{h=0}^{H}\log\frac{\PP_{\hat \rho}(\tilz_h^{t,\ell}\given S_h^{t,\ell})}{\PP_{\rho}(\tilz_h^{t,\ell}\given S_h^{t,\ell})}\bigg)  \biggiven \calD \bigg],   \notag 
\end{align}
where the last inequality follows from Cauchy-Schwarz inequality. 
To see this, note that for two random variables $X$ and $Y$,
Cauchy-Schwarz inequality implies that 
$$
\E [ \exp((X+ Y)/2) ] \leq \sqrt{ \E [ \exp(X) ] \cdot \E [ \exp(Y)]}. 
$$
According to the definition of Hellinger distance, we have $1-H^2(Q_1, Q_2) = \int_x \sqrt{q_1(x)q_2(x)} \text{d}x = \E_{x\sim Q_2} \sqrt{q_1(x)/q_2(x)}$.
Therefore, we can rewrite the first term on the right-hand side of \eqref{eq:combine_ineq} using Hellinger distance as follows,
\begin{align}
    &\frac{1}{2}\log\bbE_{\tilde{\calD}}\bigg[\exp\bigg(-\frac{1}{2}\sum_{\ell=1}^{N}\sum_{t=1}^T\sum_{h=0}^{H}\log\frac{\PP(\tilz_h^{t,\ell}\given S_h^{t,\ell})}{\PP_{\hat \rho}(\tilz_h^{t,\ell}\given S_h^{t,\ell})}\bigg) \bigggiven \calD \bigg] \nonumber \\
    &\quad = \frac 12 \log\bbE_{\tilde{\calD}}\left[
        \prod_{\ell=1}^{N}\prod_{t=1}^T\prod_{h=0}^{H}\sqrt{\frac{\PP_{\hat \rho}(\tilz_h^{t,\ell}\given S_h^{t,\ell})}{\PP(\tilz_h^{t,\ell}\given S_h^{t,\ell})}}  \bigggiven \calD
    \right], \nonumber\\
    &\quad = \frac 12   \sum_{\ell=1}^{N}\sum_{t=1}^T\sum_{h=0}^{H} \log\left( 1 -  \text{H}^2(\PP_{\hat \rho}(\cdot\given S_h^{t,\ell}), \PP(\cdot\given S_h^{t,\ell}))\right).\label{eq:hellinger}
\end{align}

Due to the fact that $\log(1-x)\leq -x$ for $x\in [0,1)$, we further upper bound \eqref{eq:hellinger} as follows,
\begin{align}
\textrm{RHS of}~\eqref{eq:hellinger}
    &\leq  -\frac 12   \sum_{\ell=1}^{N}\sum_{t=1}^T\sum_{h=0}^{H} \text{H}^2(\PP_{\hat \rho}(\cdot\given S_h^{t,\ell}), \PP(\cdot\given S_h^{t,\ell})) \nonumber \\
    &\leq -\frac{1}{4}\sum_{\ell=1}^{N}\sum_{t=1}^{T}\sum_{h=0}^H \text{TV}^2 \big(\PP(\cdot\given S_h^{t,\ell}),\PP_{\hat \rho}(\cdot\given S_h^{t,\ell})\big), \label{eq:tv_bd}
\end{align}
where the second line follows from the fact that $2\text{H}^2(Q_1, Q_2) \geq \TV^2(Q_1, Q_2)$. Applying \eqref{eq:tv_bd} to the left-hand side of \eqref{eq: ch bd}, we thus have
\begin{align}\label{eq:lhs}
    -\E_{\rho\sim P}\Big[\log\bbE_{\tilde{\calD}}\bigl[ \exp\big( L(\rho,\tilde{\calD})\big) \given \calD\bigr]\Big]
    & \geq \frac{1}{4}\sum_{\ell=1}^{N}\sum_{t=1}^{T}\sum_{h=0}^H\text{TV}^2 \big(\PP(\cdot\given S_h^{t,\ell}),\PP_{\hat \rho}(\cdot\given S_h^{t,\ell})\big)\\
    &\hspace{-3cm} -\frac{1}{2}\bbE_{\rho\sim P}\bigg[\log\bbE_{\tilde{\calD}}\bigg[ \exp\bigg(-\frac{1}{2}\sum_{\ell=1}^{N}\sum_{t=1}^{T}\sum_{h=0}^H\log\frac{\PP_{\hat \rho}(\tilz_h^{t,\ell}\given S_h^{t,\ell})}{\PP_{\rho}(\tilz_h^{t,\ell}\given S_h^{t,\ell})}\bigg)\,\bigg|\,\calD\bigg]\bigg]. \nonumber
\end{align}
Next, we upper bound   the right-hand side of \eqref{eq: ch bd}. For any $\rho' \in \calP_{\mathrm{LLM}}$,  we have 
\begin{align}\label{eq:rhs}
    &-\E_{\rho\sim P}\big[L(\rho,\calD)\big]+\KL(P , Q) +\log(\frac{1}{\delta})  \\
    & \quad = \frac{1}{4}\sum_{\ell=1}^{N}\sum_{t=1}^{T}\sum_{h=0}^H\log\frac{\PP(z_h^{t,\ell}\given S_h^{t,\ell})}{\PP_{\rho'}(z_h^{t,\ell}\given S_h^{t,\ell})}+\frac{1}{4}\sum_{\ell=1}^{N}\sum_{t=1}^{T}\sum_{h=0}^H\E_{\rho\sim P}\bigg[\log\frac{\PP_{\rho'}(z_h^{t,\ell}\given S_h^{t,\ell})}{\PP_{\rho}(z_h^{t,\ell}\given S_h^{t,\ell})}\bigg]+\KL(P , Q) +\log(\frac{1}{\delta})\nonumber\\
    & \quad \leq \frac{1}{4}\sum_{\ell=1}^{N}\sum_{t=1}^{T}\sum_{h=0}^H\log\frac{\PP(z_h^{t,\ell}\given S_h^{t,\ell})}{\PP_{\rho'}(z_h^{t,\ell}\given S_h^{t,\ell})}+\frac{1}{4}\sum_{\ell=1}^{N}\sum_{t=1}^{T}\sum_{h=0}^H\E_{\rho\sim P}\bigg[\log\frac{\PP_{\hat \rho}(z_h^{t,\ell}\given S_h^{t,\ell})}{\PP_{\rho}(z_h^{t,\ell}\given S_h^{t,\ell})}\bigg]+\KL(P , Q) +\log(\frac{1}{\delta}), \nonumber
\end{align}
where the last inequality holds by noting that $\hat\rho$ maximizes the likelihood function.

We next choose $\rho'$ as the projection of $\PP$ (in terms of the KL divergence) onto the space of all parameterized learnable models $\{\PP_\rho \given \rho \in \calP_{\mathrm{LLM}}\}$, i.e.,
\begin{align*}
    \rho' = \argmin_{\rho^{*}\in \calP_{\mathrm{LLM}}}{\E_{ S \sim\mathcal{D}}\KL\big(\PP(\cdot| S )\|\PP_{\rho^{*}}(\cdot| S )\big)}.
\end{align*}
Combining  inequalities \eqref{eq:lhs} and \eqref{eq:rhs}, we thus upper bound the desired pretraining error as a sum of a few terms as follows 
\begin{align} \label{eq: decomp}
    &  \frac{1}{4}\sum_{\ell=1}^{N}\sum_{t=1}^{T}\sum_{h=0}^H{\TV}^2 \big(\PP(\cdot\given S_h^{t,\ell}),\PP_{\hat \rho}(\cdot\given S_h^{t,\ell})\big) \nonumber\\
        &\quad\leq \underbrace{\frac{1}{2}\bbE_{\rho\sim P}\bigg[\log\bbE_{\tilde{\calD}}\bigg[ \exp\bigg(-\frac{1}{2}\sum_{\ell=1}^{N}\sum_{t=1}^{T}\sum_{h=0}^H\log\frac{\PP_{\hat \rho}(z_h^{t,\ell}\given S_h^{t,\ell})}{\PP_{\rho}(z_h^{t,\ell}\given S_h^{t,\ell})}\bigg)\,\bigg|\,\calD\bigg]\bigg]}_{\displaystyle(\rm I.i)}  \\
        &\quad \qquad +\underbrace{\frac{1}{4}\sum_{\ell=1}^{N}\sum_{t=1}^{T}\sum_{h=0}^H\bbE_{\rho\sim P}\bigg[\log\frac{\PP_{\hat \rho}(z_h^{t,\ell}\given S_h^{t,\ell})}{\PP_{\rho}(z_h^{t,\ell}\given S_h^{t,\ell})}\bigg]}_{\displaystyle (\rm I.ii)} \notag \\
        & \quad \qquad +\underbrace{\frac{1}{4}\sum_{\ell=1}^{N}\sum_{t=1}^{T}\sum_{h=0}^H\log\frac{\PP(z_h^{t,\ell}\given S_h^{t,\ell})}{\PP_{\rho'}(z_h^{t,\ell}\given S_h^{t,\ell})}}_{\displaystyle (\rm II)}+ \underbrace{\KL(P , Q)}_{\displaystyle (\rm III)}+\log\frac{1}{\delta}. \notag 
\end{align}
Here the 
The first two errors (I.i) and (I.ii) represent  the fluctuation error due to the randomness of \(\rho \sim P\),   (II) is the approximation error that characterizes the discrepancy between the true distribution \(\PP\) and its best approximator \(\PP_{\rho'}\), and   (III) is the KL divergence between \(P\) and \(Q\). 
Note that the left-hand side of \eqref{eq: decomp} can be written as 
$$
N   T \cdot (H+1) /4 \cdot \EE_{S\sim \cD} \bigl[ {\TV}^2 \big(\PP(\cdot\given S ),\PP_{\hat \rho}(\cdot\given S )\big) \bigr]. 
$$
With the error decomposition in \eqref{eq: decomp}, we conclude 
\textbf{Step 1}. In the following, we will specify distributions $ P$  and $Q$.

\vspace{2mm}

\noindent \textbf{Step 2: Control each term in the decomposition of pretraining error.} In this step, we control each term in the error decomposition~\eqref{eq: decomp}.

Our first step is to control the fluctuation errors (\rm I.i) and (\rm I.ii), which describe the log density ratio between \(\PP_{\hat{\rho}}\) and \(\PP_{\rho}\). The errors (\rm I.i) and (\rm I.ii) are small when \(\rho\) is close to \(\hat{\rho}\). Therefore we control these two terms by setting the support of \(P\) to be a neighborhood around \(\hat{\rho}\). Specifically, for each weight matrix and residual link specified by $$\hat \rho = \bigg(W_\mathtt{softmax}, \{W_{\mathrm{ff},1}^d,W_{\mathrm{ff},2}^d, \{W_i^{Q,d},W_i^{K,d},W_i^{V,d}\}_{i=1}^\eta, \gamma_1^d, \gamma_2^d\}_{d=1}^D\bigg),$$ 
we construct a ball with a radius shrinking at rate $1/(NT(H+1))$. More specifically, we define
\begin{align}\label{eq:define_distribution_P} 
P = \calB_S\cdot \prod_{d=1}^D \calB_M(d)\cdot \calB_F(d)\cdot \calB_R(d),
\end{align} 
where we define the balls around each weight matrix in each layer $d$ as
\begin{align}
\calB_{S} & = \mathtt{Unif}\big(B(W_{\mathrm{softmax}}, r_{S},\|\cdot\|_{1,2})\big), \nonumber\\
\calB_{R}(d) & = \mathtt{Unif}\big( B(\gamma_1^d, r_{\gamma,1}^{(d)},|\cdot|)\big)\cdot \mathtt{Unif}\big( B(\gamma_2^d, r_{\gamma,2}^{(d)},|\cdot|)\big), \nonumber\\
\calB_{F}(d) &= \mathtt{Unif}\big(B(W_{\mathrm{ff},1}^{d}, r_{F,1}^{(d)},\|\cdot\|_\mathrm{F})\big) \cdot \mathtt{Unif}\big(B(W_{\mathrm{ff},2}^{d}, r_{F,2}^{(d)},\|\cdot\|_\mathrm{F})\big), \nonumber\\
    \calB_{M}(d) &= \prod_{i=1}^{\eta}\mathtt{Unif}\big(B(W_{i}^{Q,d}, r_V^{(d)},\|\cdot\|_\mathrm{F})\big) \cdot \mathtt{Unif}\big(B(W_{i}^{Q,d}, r_Q^{(d)},\|\cdot\|_\mathrm{F})\big)\cdot \mathtt{Unif}\big(B(W_{i}^{Q,d}, r_K^{(d)},\|\cdot\|_\mathrm{F})\big). \nonumber 
\end{align}
The ball around center $x$ with radius $r$ is defined as $B(x, r,\|\cdot\|) = \{y\given \|x-y\|\leq r\}$. And $\mathtt{Unif}(\calS)$ denotes uniform distribution over the set $\calS$. Finally, we specify the radius as follows,
\begin{align*}
    r_{K}^{(d)} = r_{Q}^{(d)} &= R^{-1}\eta^{-1}(1+B_F^2)^{-1}B_M^{-2}\alpha_d^{-1}/(NT(H+1)),\\
    r_{F,1}^{(d)} = r_{F,2}^{(d)}&=R^{-1}B_F^{-1}\alpha_d^{-1}/(NT(H+1)),\\
    r_{V}^{(d)} &= R^{-1}\eta^{-1}(1+B_F^2)^{-1}\alpha_d^{-1}/(NT(H+1)),\\
    r_{\gamma,1 }^{(d)}& = R^{-1}(1+B_F^2)^{-1}\alpha_d^{-1}/(NT(H+1)),\\
    r_{\gamma,2 }^{(d)} &=R^{-1}\alpha_d^{-1}/(NT(H+1)),\\
    r_S&=\tau B_s^{-1}/(NT(H+1)),\\
    \text{where } \alpha_d &= \frac{2}{\tau}B_S(1+B_F^2)\big(1+\eta B_M(1+4B_M^2)\big)^{D-d}.
\end{align*}

Under this assignment of $P$, we can control  (\rm I.i) and (\rm I.ii) by invoking the following lemma from \cite{zhang2023and}.

\begin{lemma}\label{lemma: error 1}
We set the distribution $P$ to be \eqref{eq:define_distribution_P}, which is a uniform distribution over a  neighborhood around $\hat \rho$ with radius proportional to
$1/NT(H+1)$. Under Assumptions~\ref{assumption: bd magnitude} and \ref{assumption: lower bd}, we have
    \begin{align*}
        &\frac{1}{2}\bbE_{\rho\sim P}\bigg[\log\bbE_{\tilde{\calD}}\bigg[ \exp\bigg(-\frac{1}{2}\sum_{\ell=1}^{N}\sum_{t=1}^{T}\sum_{h=0}^H\log\frac{\PP_{\hat \rho}(z_h^{t,\ell}\given S_h^{t,\ell})}{\PP_{\rho}(z_h^{t,\ell}\given S_h^{t,\ell})}\bigg)\,\bigg|\,\calD\bigg]\bigg]\\
        &\quad +\frac{1}{4}\sum_{\ell=1}^{N}\sum_{t=1}^{T}\sum_{h=0}^H\bbE_{\rho\sim P}\bigg[\log\frac{\PP_{\hat \rho}(z_h^{t,\ell}\given S_h^{t,\ell})}{\PP_{\rho}(z_h^{t,\ell}\given S_h^{t,\ell})}\bigg]=\mathcal{O}(1).
    \end{align*}
\end{lemma}
\begin{proof}
See Appendix F.2 by \cite{zhang2023and} for a detailed proof.
\end{proof}
This lemma quantifies how \(\PP_\rho\) changes when \(\rho\) is getting closer to $\hat \rho$. In the following, we briefly outline the proof and refer readers to the original work by \cite{zhang2023and} for more details. The proof consists of two steps. The first step is to control the TV distance between \(\PP_\rho\) and \(\PP_{\hat \rho}\) using the differences between the layer parameters specified in $\rho$ and $\hat \rho$. The second step sets $\rho \sim P$, where the distribution $P$ in \eqref{eq:define_distribution_P} is supported on a neighborhood around $\hat \rho$. Then for any $\rho\in \mathtt{supp}(P)$, we can control the log density ratio between $\PP_{\hat \rho}$ and $\PP_{\rho}$ using the radius defined in \eqref{eq:define_distribution_P} as follows:
\begin{align}
    \log \big(\PP_{\hat \rho}(z_h^{t,\ell} \mid S_h^{t,\ell}) /\PP_{\rho}(z_h^{t,\ell} \mid S_h^{t,\ell})\big) = \mathcal{O}\big(1/(NT(H+1))\big) \label{eq:err_1_density_bd}
\end{align}
for any $(z_h^{t,\ell}, S_h^{t,\ell})$. Therefore, we conclude that the fluctuation error (I) has a rate of \(\mathcal{O}(1)\).

Next, we control error (II) using the following lemma.

\begin{lemma}\label{lemma: error 2}
Under Assumptions~\ref{assumption: bd magnitude} and \ref{assumption: lower bd}, with probability at least $1-\delta$, we have   
    \begin{align}
        &\frac{1}{NT(H+1)}\sum_{\ell=1}^{N}\sum_{t=1}^{T}\sum_{h=0}^H \bigg( \log\frac{\PP(z_h^{t,\ell}\given S_h^{t,\ell})}{\PP_{\rho'}(z_h^{t,\ell}\given S_h^{t,\ell})}-\bbE_{S_h^{t,\ell}}\KL\big(\PP(\cdot \given S_h^{t,\ell}) \given \PP_{\rho'}(\cdot\given S_h^{t,\ell})\big) \bigg) \nonumber \\
        &\qquad \leq b^*\sqrt{\frac{1}{2N}}\log \frac{T(H+1)}{\delta}. \label{eq:error_2}
    \end{align}
 
\end{lemma}

\begin{proof}
    See Appendix \ref{proof: lem error 2} for details.
\end{proof}
This lemma controls error \rm(II). A key part of the proof is to derive the log-density bound
\begin{align}
    \bigl|\log \PP(z \given S)  -\log \PP_{\hat \rho}(z \given S )\big| \leq b^* = \log \max\{c_0^{-1}, 1+|\calL|\exp(B_S/\tau)\}, \label{eq: log density difference}
\end{align}
which provides the explicit form of $b^*$ mentioned earlier in Assumption~\ref{assump: density bd}. The proof involves applying Hoeffding's inequality, along with the log-density bound in \eqref{eq: log density difference}, to the left-hand side of \eqref{eq:error_2}.

Furthermore, we control (III), the  
KL divergence between \( P \) and \( Q \), using the following lemma obtained from \cite{zhang2023and}. To make sure that $\mathtt{supp}(P)\subseteq \mathtt{supp}(Q)$, we set $Q$ to be uniformly distributed over $\calP_\mathrm{LLM}$. More specifially, we have
\begin{align}\label{eq:define_distribution_Q} 
Q = \calB'_S\cdot \prod_{d=1}^D \calB'_M(d)\cdot \calB'_F(d)\cdot \calB'_R(d),
\end{align} 
where we define the balls around each weight matrix in each layer $d$ as
\begin{align}
\calB'_{F}(d) &= \prod_{i=1}^2\mathtt{Unif}\big(B(0, B_F,\|\cdot\|_\mathrm{F})\big),\quad 
    \calB'_{M}(d) = \prod_{i=1}^{3\eta}\mathtt{Unif}\big(0, B_M,\|\cdot\|_\mathrm{F})\big), \nonumber \\
    \calB'_{S} & = \mathtt{Unif}\big(B(0, B_S,\|\cdot\|_{1,2})\big), \quad \hspace{5mm}
\calB'_{R}(d) = \prod_{i=1}^2\mathtt{Unif}\big( B(1/2, 1/2,|\cdot|)\big). \nonumber
\end{align}

\begin{lemma} \label{lemma: error 3}

Let $\bar D = D^{2}\cdot r\cdot (d_{F}+d_k+r)+r\cdot |\calL|$ and $\bar B = \tau^{-1}RhB_{S}B_F^2B_M^3$. Let distributions \(P\) and \(Q\) be defined as in \eqref{eq:define_distribution_P} and \eqref{eq:define_distribution_Q}, respectively. We have that under Assumptions~\ref{assumption: bd magnitude} and \ref{assumption: lower bd},
    \begin{align*}
        \KL(P,Q)= \cO \big( \bar D \log(1+NTH\bar B) \big).
    \end{align*}
\end{lemma}
\begin{proof}
    See Equation (F.9) in Appendix F.2 of \cite{zhang2023and} for a detailed proof. 
\end{proof}

This Lemma is proved by directly computing the KL divergence between two uniform distributions $P$ and $Q$, where $\mathtt{supp}(P)\subseteq \mathtt{supp}(Q)$. The calculation can be found in Appendix F.2 of \cite{zhang2023and}.

Applying Lemmas \ref{lemma: error 1}, \ref{lemma: error 2}, and \ref{lemma: error 3} to the three errors in \eqref{eq: decomp}, we have that with probability at least $1-\delta$,
\begin{align}
    &\frac{1}{NT(H+1)}\sum_{\ell=1}^{N}\sum_{t=1}^{T}\sum_{h=0}^H{\TV} \big(\PP(\cdot\given S_h^{t,\ell}),\PP_{\hat \rho}(\cdot\given S_h^{t,\ell})\big) \label{eq:combine_lems}\\
    &\quad \leq \bigg(\frac{1}{NT(H+1)}\sum_{\ell=1}^{N}\sum_{t=1}^{T}\sum_{h=0}^H{\TV}^2 \big(\PP(\cdot\given S_h^{t,\ell}),\PP_{\hat \rho}(\cdot\given S_h^{t,\ell})\big) \bigg)^{1/2} \nonumber\\
    &\quad \leq \calO\bigg(\underbrace{\frac{\sqrt{b^*}}{N^{1/4}}\log \frac{TH}{\delta} + \inf_{\rho'\in \calP_\mathrm{LLM}}\sqrt{\frac{1}{NTH}\sum_{\ell=1}^{N}\sum_{t=1}^{T}\sum_{h=0}^H\bbE_{S_h^{t,\ell}}\KL\big(\PP(\cdot \given S_h^{t,\ell}) \given \PP_{\rho'}(\cdot\given S_h^{t,\ell})\big)}}_{\displaystyle (\rm II)} \nonumber\\
    &\quad \qquad  +\underbrace{\sqrt{1/(NTH)}}_{\displaystyle (\rm I)} + \underbrace{\sqrt{\frac{\bar D}{NTH}} \log(1+NTH\bar B)}_{\displaystyle (\rm III)} + \sqrt{1/(NTH)}  \log(1/\delta)\bigg), \nonumber\\
    &\quad \leq \calO\bigg(\frac{\sqrt{b^*}}{N^{1/4}}\log \frac{TH}{\delta} + \inf_{\rho'\in \calP_\mathrm{LLM}}\sqrt{\frac{1}{NTH}\sum_{\ell=1}^{N}\sum_{t=1}^{T}\sum_{h=0}^H\bbE_{S_h^{t,\ell}}\KL\big(\PP(\cdot \given S_h^{t,\ell}) \given \PP_{\rho'}(\cdot\given S_h^{t,\ell})\big)} \nonumber \\
    &\quad \qquad + \sqrt{\frac{\bar D}{NTH}} \log(1+NTH\bar B)\bigg), \nonumber
\end{align}
where the first line follows from Cauchy-Schwarz inequality, the second line follows from upper bounding the three errors in \eqref{eq: decomp} using Lemmas \ref{lemma: error 1}, \ref{lemma: error 2}, and \ref{lemma: error 3}. The last line drops two terms that are dominated by the rest.

In the final step, we will change the left-hand side of \eqref{eq:combine_lems} to its expectation. We control the difference between \eqref{eq:combine_lems} and its expectation using the following lemma.
\begin{lemma} \label{lemma: err 4}
Under Assumptions~\ref{assumption: bd magnitude} and \ref{assumption: lower bd}, with probability at least $1-\delta$, we have 
    \begin{align*}
        &\frac{1}{NT(H+1)}\sum_{\ell=1}^{N}\sum_{t=1}^{T}\sum_{h=0}^H \bigg( \bbE_{S_h^{t,\ell}}\Big[\TV\big(\PP(\cdot\given S_h^{t,\ell}),\PP_{\hat \rho}(\cdot\given S_h^{t,\ell})\big)\Big]-\TV\big(\PP(\cdot\given S_h^{t,\ell}),\PP_{\hat \rho}(\cdot\given S_h^{t,\ell})\big) \bigg) \\
        &\quad=\cO\bigg(\frac{1}{\sqrt{N}}\Big(\barD\log(1+NTH\barB)+\log\frac{TH}{\delta}\Big)\bigg).
    \end{align*}

\end{lemma}

\begin{proof}
    See Appendix~\ref{proof: lem error 4} for details.
\end{proof}

This lemma follows from establishing uniform convergence between the TV distances and their expectations that hold for any distribution \( P \).
Adding Lemma~\ref{lemma: err 4} to \eqref{eq:combine_lems}, we obtain the rate for pretraining error:
\begin{align*}
    &\EE_{S\sim \cD} \bigl[ {\TV} \big(\PP(\cdot\given S ),\PP_{\hat \rho}(\cdot\given S )\big) \bigr]\\
    &\quad =\calO\bigg(\frac{\sqrt{b^*}}{N^{1/4}}\log \frac{TH}{\delta} + \inf_{\rho'\in \calP_\mathrm{LLM}}\sqrt{\frac{1}{NTH}\sum_{\ell=1}^{N}\sum_{t=1}^{T}\sum_{h=0}^H\bbE_{S_h^{t,\ell}}\KL\big(\PP(\cdot \given S_h^{t,\ell}) \given \PP_{\rho'}(\cdot\given S_h^{t,\ell})\big)} \nonumber \\
    &\quad \qquad + \frac{1}{\sqrt{N}}\Big(\barD\log(1+NTH\barB)+\log\frac{TH}{\delta}\Big) +\sqrt{\frac{\bar D}{NTH}} \log(1+NTH\bar B)\bigg) \nonumber \\
    &\quad =\calO\bigg(\frac{\sqrt{b^*}}{N^{1/4}}\log \frac{TH}{\delta} + \inf_{\rho'\in \calP_\mathrm{LLM}}\sqrt{\frac{1}{NTH}\sum_{\ell=1}^{N}\sum_{t=1}^{T}\sum_{h=0}^H\bbE_{S_h^{t,\ell}}\KL\big(\PP(\cdot \given S_h^{t,\ell}) \given \PP_{\rho'}(\cdot\given S_h^{t,\ell})\big)} \nonumber \\
    &\quad \qquad + \frac{1}{\sqrt{N}}\Big(\barD\log(1+NTH\barB)\bigg),\nonumber
    \end{align*}
    where the final line follows from dropping the last term in the second line, which is dominated by the rest. Therefore, we conclude the proof.

\end{proof}

\subsection{Formal Statement of Proposition \ref{prop: kl bd}}  \label{app: construction}

In this section, we formally state Proposition~\ref{prop: kl bd} and provide its proof. 
For simplicity, we derive the approximation error bound for reasoning steps of dimension one, i.e., we regard $\cL$ as a subset of $\RR$.  Our method can be readily generalized to higher-dimensional cases~\citep{elbrachter2021deep}. 
In this proof, we construct networks with specific parameters such that the KL divergence between the target distribution \(\PP\) and its best transformer neural network approximation \(\PP_{\hat \rho}\)  decays exponentially as the network depth increases.


We let $T$ be the maximal number of examples included in the prompt. 
Thus, $L=T(H+1)$ is the largest number of reasoning steps included in the prompt. 
We let $S_h^t = \{\Upsilon_{t-1}, z_{0:(h-1)}^t\}$ denote the collection of  $t-1$ examples of reasoning paths and a partial trajectory of length $h$ of the $t$-th example. 
Here $h \in \{0, \ldots, H \}$. 
Note that the desired transformer neural network takes each $S_h^{t} $ as the input and outputs an element in the probability distribution over $\cL$ as the conditional distribution of $z_h^t$. 
That is, the transformer takes a sequence of reasoning steps as the input and outputs a probability distribution.

Since each position of the input is indexed by $(t, h)$, to simplify the notation, 
we use $L'(t,h) = (t-1)(H+1)+h $ to denote the length of $S_h^t$. 
For any $t, t' \in [T]$ and $h, h' \in \{0, \ldots , H\}$, we write 
\begin{align} \label{eq:order_index}
(t', h') < (t, h)\qquad \textrm{if and only if} \qquad L'(t', h') = (t'-1) (H+1) + h' < L'(t,h). 
\end{align}
That is, $(t', h') < (t, h)$ if and only if $z^{t'}_{h'}$ appears earlier than $z_{h}^t$.

In the sequel, we fix $t\in [T]$ and $0\leq h \leq H$ and 
focus on the problem of approximating the conditional distribution of $z_h^t$. 
We 
abbreviate $L'(t,h)$ as $L'$ when the meaning is clear from the context. 
The target distribution is denoted by a function $ g_h^*: \calL^{L'} \rightarrow \mathbb{R}^{|\calL|}$, i.e.,
$ g_h^*(S_h^t)   = \PP(z_h^{t} = \cdot \given S_h^t)$.  
That is, for any $z \in \cL$, the $z$-th entry of $g_h^*(S_h^t)$ is equal to $ \PP(z_h^{t} = z \given S_h^t)$. 
Functions $\{ g_h^* \}_{h= 0 }^H$ are the target functions and we want to {\bf construct a single transformer that approximates all of them}.

\vspace{2mm}
{\bf \noindent Function class containing $\{ g_h^* \}_{h= 0 }^H$.}
Under the general model introduced in Appendix \ref{app: generalized model}, by the Bayes' rule,  function $g_h^*(\cdot)$ is invariant to permutations of the $L(t, h) $ reasoning steps in $S_h^t$. 
Theorem $2$ in \cite{zaheer2017deep} proves that any permutation invariant function of a function admits a factorization structure. 
In particular, there exist  $w_h^*:\bbR\rightarrow\bbR^{|\calL|}$ and $\psi_h^{*}:\calL \rightarrow\bbR$ for all $h \in \{0, \ldots, H\}$ and $t\in[T]$ such that
\begin{align}
    g_h^{*}(S_h^t)=w_h^*\bigg(\frac{1}{L'}\bigg(\sum_{i=1}^{t-1}\sum_{j=0}^{H}\psi_h^{*}(z_{j}^i) + \sum_{j'=0}^{h-1}\psi_h^{*}(z_{j'}^{t})\bigg)\bigg). \label{eq:g*}
\end{align}
In particular, if $h = 0$, the second summation $\sum_{j'=0}^{h-1}\psi_h^{*}(z_{j'}^{t})$ is set to zero. Let \( w_{h,i}^* \) denote the \(i\)-th component of \( w_h^* \) for all $i \in [|\cL|]$.  


In the following, we let 
\( \calS^{\infty}([-B,B], \mathbb{R}) \) denote the set of real-valued smooth functions on \([-B,B]\) equipped with the $\ell_{\infty}$-norm \(\|f\|_\infty = \sup_{x \in [-B,B]} |f(x)| \). We define \(\calS_B\) as the set of smooth functions with bounded derivatives:
\begin{align*}
    \calS_{B}=\Big\{f\in\calS^{\infty}([-B,B],\bbR)\given  \big\|f^{(n)}\big\|_\infty \leq C_{\cS} \cdot n! \text{ for all }n\in\bbN_{+}, \text{ and }\|f\|_{\infty}\leq C_A \Big\},
\end{align*}
where $f^{(n)}$ is the $n$-th order derivative of $f$, $C_{\cS}$ is a constant, and $\bbN_{+}$ is the set of positive integers. 
Here $\cS_{B}$ contains functions whose high-order derivatives grow moderately fast in magnitude. 
We impose some regularity assumptions on functions $\{ g_h^*\}_{h=0}^H$ as follows.
\begin{assumption}\label{assumption: smooth}
     We assume that there exists $B, C_A>0$ such that for any $h \in \{0, \ldots, H \} $, 
     we have $\psi_h^{*},\tau\log w_{h,i}^{*}\in\calS_{B}$ for $i\in[|\calL|]$, where $\tau$ is the temperature of the \ac{llm}s and $w_{h, i}^* $ is the $i$-th entry of $w_h^* $ in \eqref{eq:g*}. Moreover, without loss of generality, we assume $C_{\cS} = 1$. 
\end{assumption}
This assumption states that the target functions $\{ g_h^*\}_{h=0}^H $ are sufficiently smooth in the sense that all functions appearing in the factorization in \eqref{eq:g*} are smooth. 
We establish the approximation error in the following proposition.


\begin{proposition} [Formal Statement of Proposition~\ref{prop: kl bd}]  \label{prop:formal_construction}

Let $S_h^t= (\Upsilon_{t-1}, \{ z_j^t\}_{j=0}^{h-1} )$ be the sequence of reasoning steps that includes $t-1$ examples of reasoning paths $\Upsilon_{t-1}$  and the first $h-1$ steps of the $t$-th example $ \{ z_j^t\}_{j=0}^{h-1}$. 
Let $D$ denote a sufficiently large integer, consider the parameter class $\calP_\mathrm{LLM}$ in \eqref{eq:P_LLM_class} with $d_{F}\geq 18|\calL|+4$,  $B_{S}\geq (C_A+1)\cdot \sqrt{|\calL|}$, $B_{M}\geq \max\{8\log(8TH), \sqrt{B^2+(H+1)^2+1}\}$, and 
$$
B_F\geq C_F\cdot\sqrt{B^2+H^2+C_A^2\cdot |\calL|} \cdot 16^{D'}\cdot |\calL|^{3/2},
$$
where $D'=(D-C_p\log(3H))/(H+1)$, $C_F, C_p>0$ are absolute constants, and $C_A$ is from Assumption~\ref{assumption: smooth}. Under Assumptions~\ref{assumption: bd magnitude}, \ref{assumption: lower bd} and~\ref{assumption: smooth}, there exists 
a transformer with at most $\calO(D)$ transformer blocks and
parameter $\rho^* \in \cP_{\mathrm{LLM}}$ satisfying
\begin{align*} 
 \max_{ S_h^t\in \calL^*}\KL\big(\PP(z_h^t = \cdot\,|\,S_h^t),\PP_{\rho^{*}}(z_h^t = \cdot\,|\, S_h^t)\big) =\cO
        \bigg(  \exp\bigg(-\frac{ \big(D-C\log(2H))/H\big)^{1/4}}{5B}\bigg)\bigg),
\end{align*}
for all $t \in [T] $ and $h \in \{ 0, \ldots, H\}$, where $C>0$ is a absolute constant. The integer $H$ is the length of a reasoning trajectory,  $B$ is the parameter from Assumption~\ref{assumption: smooth}, and $|\calL|$ is the alphabet size of the output distribution. 
We note that 
$\cO(\cdot) $  is with respect to the asymptotic regime where $D$ goes to infinity.

    
    
\end{proposition}

This proposition shows that the approximation error decays exponentially to zero as $D$ increases. 
The proof is based on an explicit construction of a transformer neural network that estimates $\{ g_h^*\}_{h=0}^H$ altogether. 
The transformer architecture follows the one described in Appendix \ref{app: pre-training process}. 
In particular, the transformer has $H+1$ submodules that approximate each $g_h^*$ separately. 
Besides, 
we  assume $C_{\cS} = 1$ in Assumption \ref{assumption: smooth} only to simplify the presentation. Our approximation result can be modified for a general $C_{\cS}$ by changing the constants in the upper bound correspondingly. 

The proof of this proposition is technical and lengthy. 
We present a detailed proof in Appendix \ref{proof:prop:formal_construction} and give an overview as follows.


\vspace{2mm}
{\noindent \bf Overview of the Proof of Proposition \ref{prop:formal_construction}.} 
Note that the transformer takes $S_h^t$ as the input and outputs a probability distribution over $\cL$, where $h \in \{0, \ldots, H\}$ and $t \in [T]$. 
We fix some arbitrary $(t, h)$ and consider the problem of predicting $z_{h}^t$.

 As introduced in Appendix \ref{app: pre-training process}, in
 the transformer architecture, the input sequence is first embedded in an Euclidean space and then passed through a series of transformer blocks. Then the output goes through a softmax output layer to generate a probability distribution.  
Intuitively,  
when predicting $z_h^t $ using $S_h^t$, 
we need to first extract the step-index $h$ and then apply an approximation of $g_h^*$. 
To achieve this goal, our transformer includes an extraction module \(\mathtt{NN}_\mathrm{J}\) followed by \(H+1\) approximation and selection modules  \(\{G_{h'}, F_{h'}\}_{h' =0}^H\).
Here $\mathtt{NN}_\mathrm{J}$ adds the desired step-index, i.e., $h$, to all the $L'(t, h)$ locations.
The approximation module $G_{h'}$ approximates the target distribution $g_{h'}^*$ for all $  h' \in \{0, \ldots, H\}$.
Selection modules  $\{ F_{h'}\}_{h' =0}^H$ are used to select the particular  approximation module $G_h $ with step-index $h$. 
The output of the final selection module \(F_H\) is then passed to a softmax layer, which produces the output distribution.
We list the components of the transformer architecture as follows.
Also see Figure~\ref{fig:network_sketch} for an illustration.


\begin{itemize}
    \item \textbf{Input embedding}: Given a prompt $S_h^t \in \cL^*$, we construct an \textit{input embedding to} prepare for further processing of the input, which is defined as $X_\mathrm{NN}^{(0)}$  in~\eqref{eq:input_layer_0}. 
    Here $X_\mathrm{NN}^{(0)}$ is a sequence of vectors of length $L'(t, h)$, where each vector has length $3$, including the value of the reasoning step and its step-index. 
    
    \item \textbf{Extraction module $\mathtt{NN}_\mathrm{J}$}: The extraction module $\mathtt{NN}_\mathrm{J}$ extracts the step-index $h$ from the last reasoning step $z_{h-1}^{t}$ and copy it to all previous reasoning steps $\{z_{h'}^{t'}\}_{(t',h')<(t,h-1)}$. 
    Specifically, this module takes input embedding $X_\mathrm{NN}^{(0)}\in \mathbb{R}^{L'\times 3}$ in~\eqref{eq:input_layer_0} and outputs a vector sequence of length $L'(t,h)$, where each vector is in $\RR^{2+2|\cL|}$. 
    In each vector indexed by $(t', h')$, the first entry is the value of the reasoning step $z_{h'}^{t'}$, and the second entry is approximately equal to $h$, the step-index of the desired output $z_h^t$. 
    The remaining $2|\cL|$ entries of each vector are all set to zero. 
    The output of $\mathtt{NN}_\mathrm{J}$  is then fed into a sequence of $H+1$ approximation and selection modules. 

    \item \textbf{Approximation module $G_{h'}$}: 
    For any $h ' \in \{0, \ldots, H\}$, $G_{h'}$ computes an embedding of $S_h^t$, denoted by \(\mathtt{embed}_{h'}(S_h^t)\), which is a vector-valued function in $|\cL|$.
     In particular, $\mathtt{embed}_{h'} (S_h^t)$ is 
    used to approximate the target distributions \( g_{h'}^* (S_h^t) \) after a softmax transformation. 
    Each $G_{h'}$ is a mapping that maps a sequence of $L' = L'(t,h)$ vectors in $\RR^{2 + 2 | \cL|}$ to a vector sequence of the same shape, i.e., a function between  $\RR^{L' \times (2 + 2|\cL| ) }$ to itself. 
    Here $ G_{h'}$
      only changes the columns with indices in $\{3, \ldots, 2+2|\cL|\}$ and sets them to $\mathtt{embed}_{h'} (S_h^t) \in \RR^{L' \times |\cL|}$,
      where $\mathtt{embed}_{h'}$ is a matrix-valued mapping that maps $S_h^t \in \mathbb{R}^{L' \times 1}$ to a matrix in $\RR^{L' \times |\cL|}$. 

    \item \textbf{Selection module $F_{h'}$}: 
    For any $h ' \in \{0, \ldots, H\}$, 
  $F_{h'}$ checks if its index \( h' \) matches the extracted index \( h \). 
  When viewing each $F_{h'}$ as a matrix-valued  mapping 
  from $ \RR^{L' \times (2 + 2|\cL| ) }$
  to itself, 
it only changes the last $|\cL|$ columns of the matrix and uses them as a ``memory''. 
In particular, it approximately adds $\mathtt{embed}_{h'} (S_h^t) \cdot  \mathbbm{1} \{ h' = h\}$ to the memory and passes it to the subsequent modules. 
As a result, after the last selection module, $F_{H}$, the last $|\cL|$ columns of the output matrix are given by $\sum_{h' = 0}^H \mathtt{embed}_{h'} (S_h^t) \cdot  \mathbbm{1} \{ h' = h\} \approx \mathtt{embed}_{h} (S_h^t). $ Thus, by combining the approximation and selection modules, we eventually obtain $\mathtt{embed}_{h} (S_h^t)$ approximately. 
  
    \item \textbf{Output softmax layer}: Finally, we pass the output of $F_H$, $\mathtt{embed}_{h} (S_h^t)$, to a  $\mathtt{softmax}$ function to produce the output distribution $\hat g_h^* (S_h^t)$, which is closed to the desired output $ g_h^* (S_h^t) $.
    \end{itemize}


\subsection{Proof of Proposition \ref{prop:formal_construction}} \label{proof:prop:formal_construction}
Before the formal proof, we would like to highlight that our construction in the proof is based on a slightly generalized version of the transformer structure in Section~\ref{subsubsec: pretraining process}. We note that this slight generalization can be easily taken into account in the generalization error in Proposition~\ref{prop: pre-train}. Here, we first define a single {\bf transformer block} as follows, which takes $X\in\bbR^{L\times r}$ as input and output $Y\in\bbR^{L\times r}$.
\begin{align}\label{eq:transformer_block}
\begin{split}
    Z & = \mathtt{NL} \bigl (\mathtt{mha}(X, W_{\mathrm{mha}}) + X\gamma_1 \big), \\
    Y & =  \mathtt{NL} \big (\mathtt{ff}(Z, W_{\mathrm{ff}},b_{\mathrm{ff}}) + Z\gamma_2\big), 
    \end{split} 
\end{align}
where $\gamma_1$ and $\gamma_2$ are diagonal matrices, and the fully connected   feed-forward (FF) network $\mathtt{ff}$ is defined as
\begin{align} 
    \mathtt{ff}(X_{\mathrm{in}}, W_{\mathrm{ff}}, b_{\mathrm{ff} }) = \mathtt{ReLU}(X_{\mathrm{in}} W_{\mathrm{ff},1} + \mathbf{1}^\top b_{\mathrm{ff} , 1 })W_{\mathrm{ff},2} + \mathbf{1}^\top b_{\mathrm{ff} , 2 }.\label{eq:ffn_bias}
\end{align}
Compared to the \ac{ff} layer in Section~\ref{subsubsec: pretraining process}, this \ac{ff} layer has two additional bias terms $b_{\mathrm{ff},1} \in \RR^{d_{F}}$  and $b_{\mathrm{ff},1} \in \RR^{r}$. 
Here, $\mathtt{NL}$ is the row-wise $\ell_{2}$-normalization layer, which is defined in \eqref{eq:nl} in Appendix~\ref{app: pre-training process}. 
This function projects each row of the input matrix into the unit $\ell_2$-ball. 
Moreover, the multi-head attention (MHA) layer   $\mathtt{mha}$ is defined in \eqref{eq: mha}.
In particular, a transformer block can be viewed as a four-layer neural network, where both $\mathtt{ff}$ and $\mathtt{mha}$ have two neural network layers. 
In this proof, we use  ``module'' to refer to a sequence of transformer blocks that achieves certain functionality. 

Throughout this proof, we construct a transformer neural network that includes an input embedding module, a sequence of transformer blocks, and the output softmax layer. 
Instead of counting the number of neural network layers in the transformer, we keep track of the number of transformer blocks. 
 
In this proof, we often construct neural network components that are solely based on the   MHA or FF layers. 
These layers themselves can be regarded as special cases of the transformer block, as shown below.


\vspace{2mm}
{\bf \noindent Multi-Head Attention Layer as a Transformer Block.} 
Let $W_{\mathrm{mha}}$ be the weight matrices of a MHA layer. 
To view   $   \mathtt{mha} (\cdot, W_{\mathrm{mha}})$ as a single transformer block, we can set $W_{\mathrm{ff}} $ and $b_{\mathrm{ff}} $ as zero matrices and vectors respectively. 
Then $\mathtt{ff}$ in \eqref{eq:ffn_bias} becomes a zero function. 
We also set $\gamma _1 = {\bf 0}$
and $\gamma_{2}=I$ in \eqref{eq:transformer_block}.

It remains to consider the normalization layer, 
which plays a role when row-wise $\ell_2$-norm of $\mathtt{mha}(X, W_{\mathrm{mha}})$ exceeds one. 
To handle this, we introduce a scaling trick as follows. 
When the input matrix $X $ has bounded rows and $W_{\mathrm{mha}}$ is bounded, 
we know that each $X W_i^V \in \RR^{L \times d_{v} } $ has bounded rows. 
We let $B \geq 1$ be an upper bound on the $\ell_2$-norm of the rows of $ \mathtt{mha} (X, W_{\mathrm{mha}})$ for all bounded input matrix $X$.
Then we define another set of MHA parameters $\overline W_{\mathrm{mha}}$ as $  \{ W_{i}^Q, W_{i}^{K}, {\overline W}_{i}^V \}_{i =1}^\eta$, where $\overline W_i^V = W_{i}^V / B$. 
Thus, for any input matrix $X$, we have 
$
\mathtt{mha} ( X, \overline W_{\mathrm{mha}}) =  \mathtt{mha} ( X,   W_{\mathrm{mha}})  / B,
$
whose row-wise $\ell_2$-norm is no more than one. 
Therefore, for any input matrix $X$ and any weight matrix $W$ of a proper size, we have 
\begin{align}
    \label{eq:scaling_trick}
\mathtt{NL} \big ( \mathtt{mha} (X, \overline W_{\mathrm{mha}} ) \big) \cdot \overline W =  \mathtt{mha} (X,  W_{\mathrm{mha}} )  W, 
\end{align}
where we set $\overline W = B \cdot W$.
Here $W $ is some weight matrix that is multiplied to the output of MHA layer, i.e., a weight matrix of the next layer. 
The equality in \eqref{eq:scaling_trick} shows that suppose our constructed neural network involves a softmax layer, we can scale the weight matrices to ensure that it is equivalent to a transformer block.
Moreover, the norms of these matrices are scaled by a factor of $B$.

\vspace{2mm}
{\bf \noindent Fully Connected Layer as a Transformer Block.} Similarly, consider a FF layer with parameters $\{ W_{\mathrm{ff}}, b_{\mathrm{ff}} \}$. 
We set $ \{W_i^{V} \} _{i=1}^\eta$ to be a zero matrix and thus $\mathtt{mha}(\cdot , W_{\mathrm{mha}})$ becomes a zero function. 
We then set $\gamma _ 1 = I$ and $ \gamma _2 = {\bf 0}$ in \eqref{eq:transformer_block}. 
Similarly, we can apply the scaling trick by multiplying  $W_{\mathrm{ff}, 2}$ and $b_{\mathrm{ff}, 2} $ by $ 1  / B$ for some parameter $B$. 
This ensures that the output of the FF layer in \eqref{eq:ffn_bias} has row-wise $\ell_2$-norm bounded by one, and thus the normalization $\mathtt{NN}(\cdot)$ does not take effect. 
We can multiply the weight matrix in the subsequent layer by $B$ and get the desired output.

\vspace{2mm}
{\bf \noindent Multi-Layer Perceptron as Transformer Blocks.} 
The above argument can be extended to multi-layer perceptions (MLPs), i.e., a multi-layer feed-forward neural network. 
We can show that an MLP can be written as a composition of multiple transformer blocks. 
This is achieved by (i) setting $W_{i}^V ={\bf  0}$ in the MHA layer and setting $\gamma_1  = I$ and $\gamma_2 = {\bf 0}$  and (ii) 
applying the scaling trick in each transformer block.  

Specifically, we define an MLP as a composition of $L$ feed-forward layers with parameters 
$\{ W_{\mathrm{ff}}^\ell, b_{\mathrm{ff}}^d\}_{\ell \in [L]}$.
Given the input matrix $X^0 \in \RR^{L\times r}$, the output of each layer is given by 
\begin{align}
    X^{\ell} = \mathtt{ReLU}(X^{\ell-1} W_{\mathrm{ff}}^{\ell}+\mathbf{1}^\top b_{\mathrm{ff}}^{\ell}) , \qquad \forall \ell \in [L] \label{eq:int_ffn}
\end{align}
Here  $W_{\mathrm{ff}}^{\ell} $ and $b_{\mathrm{ff}}^{\ell} $ are the weight matrix and bias vector of a proper dimension. 
We have the following result showing that a $L$-layer MLP can be represented as a transformer with $L$ blocks. 
\begin{proposition} \label{prop:ffn_transformer}
    We consider a row-wise fully-connected network defined in \eqref{eq:int_ffn}. 
    Let $X^0 \in \mathbb{R}^{L\times r}$ denote the input of this network,
    and the intermediate outputs are given by \eqref{eq:int_ffn}. 
We assume there exist positive numbers $\{B_{\ell}, 0 \leq \ell \leq L\}$ such that 
$\| (X^{\ell})^\top   \|_{2,\infty}\leq B_\ell$ for all $\ell \in \{0, \ldots, L\}$
with $B_{\ell} \geq 1$. 
Consider a transformer with input $Y^{0} = X^{0}$. 
Let $Y^{\ell}$ denote the output of the $\ell$-th transformer block for all $\ell \geq 1$. 
Then, we can construct a transformer with $L$ transformer blocks   such that $Y^{\ell} = X^{\ell} / B_{\ell}$ for all $\ell \in [L]$. 
    Moreover, 
    let $\{\overline W_{\mathrm{ff}}^{\ell}, \overline b_{\mathrm{ff}}^{\ell}\} $ denote the parameters of FF layer of  the $\ell$-th transformer block. 
    We have 
    \begin{align}
        \label{eq:define_weights_scaling_trick}
  \overline W_{\mathrm{ff},1}^{\ell}  =  B_{\ell-1}\cdot   W_{\mathrm{ff}}^{\ell} , \quad   W_{\mathrm{ff},2}^{\ell} = I /B_\ell, \quad      \overline b_{\mathrm{ff},1}^{\ell} = b_{\mathrm{ff} }^{\ell},  \quad \overline b_{\mathrm{ff},2}^{\ell} =  {\bf 0}. 
    \end{align} 
  
\end{proposition}
Suppose the weight matrices of a fully connected network with \(L\) layers have a maximum width \(d\) and maximum weight \(\alpha\), and the biases have maximum weight $\beta$. In that case, the magnitude of the intermediate output can increase at most exponentially with \(d \cdot \alpha\). More specifically, by direct calculation, we have
\begin{align*}
    B_\ell &\leq \sqrt{d} \cdot \left(B_0 \cdot (d \cdot \alpha)^\ell + \beta \cdot ((d\cdot \alpha)^\ell-1)/(d\cdot\alpha-1)\right) \text{ for } \ell \in [L]. 
\end{align*}

\begin{proof}
    See Appendix~\ref{app:ffn_transformer} for a detailed proof. Here $W_{\mathrm{ff},2}^{\ell}$ in \eqref{eq:define_weights_scaling_trick} is proportional to an identity matrix of a proper dimension, and $b_{\mathrm{ff},2}^{\ell}$ is a zero vector.  The details of the other parameters of the transformer can be found in the proof. 
\end{proof}

\subsubsection{Rigorous Proof of Proposition \ref{prop:formal_construction}} \label{proof:rig_construction}
\begin{proof}
Throughout this proof, we focus on the problem of approximating $\PP( z_{h}^t  = \cdot \given S_h^t)$ for some fixed $(t,h)$, which is denoted by $g_h^*(S_h^t)
$. We write $L'(t, h)$ as $L'$ for simplicity. 
To prove this proposition, we first introduce the transformer architecture and then establish the desired approximation error. 
As outlined above, the transformer as five components. 
We first introduce the input embedding as follows.

\vspace{2mm}

\noindent \textbf{Input embedding.} For each reasoning step in $S_h^t$, we define the input embedding as
\begin{align}\label{eq:input_layer_0}
    \tilde z_{h'}^{t', (0)}  = \begin{cases}
        (z_{h'}^{t'},  h',0) & \text{for } (t', h')< (t,h-1),\\
        (z_{h-1}^{t}, h-1,1) & \text{for }(t', h')= (t,h-1). 
        \end{cases}
\end{align}
Here the ordering between index tuples is specified in \eqref{eq:order_index}. 
Recall that $L'(t', h')$ is the index of the reasoning step $z_{h'}^{t'}$ in $S_h^t$. 
In the embedding in \eqref{eq:input_layer_0},  
the first coordinate $z_{h'}^{t'}$ is the \textit{content embedding}, which stores the actual reasoning step.
The second coordinate of $\tilde{z}_{h'}^{t'}$ indicates the step-index of each reasoning step \(z_{h'}^{t'}\), and the last coordinate specifies if it is the last reasoning step. 
This last coordinate acts as an indicator function because we want to extract the index $h$ and approximate the target function \(g_{h}^*\). 
The indicator function is \(0\) for all steps except the last reasoning step \(z_{h-1}^{t}\) of $S_h^t$,
which helps to locate the step-index \(h\). 
Thus, the last two coordinates are the \emph{positional embedding} which carries the positional information. 
In the sequel,
we let $X_{\mathrm{NN}}^{(0)} \in \RR^{L' \times 3}$ denote the embedding matrix, whose rows are the embedding vectors defined in \eqref{eq:input_layer_0}.

Using $X_{\mathrm{NN}}^{(0)}$ as the input, we present the other components of the transformer as follows. 
Our construction is decomposed into five steps as follows.

\begin{itemize}
    \item In \textbf{Step 1}, we design the extraction module $\mathtt{NN}_\mathrm{J}$ to extract the step-index $h$ from the last reasoning step and copy it to all previous reasoning steps in the prompt $S_h^t$.
    This module takes $X_{\mathrm{NN}}^{(0)}$ as the input and outputs a matrix in $\RR^{L' \times (2 + 2|\cL|)}$ in \eqref{eq:output_nn_j}. 
    Specifically, the first column of the 
  output matrix corresponds to the context embedding $S_h^t$. 
  The entries of the second column are all approximately equal to $h$, the step-index of $z_{h}^t$. 
  The rest of the  $2|\cL | $ columns are all equal to zero vectors.

    \item In \textbf{Step 2}, for all $\tilde h \in \{0, \ldots, H\}$, we construct the approximation modules \( G_{\tilde{h}} \) that produce an  \textit{approximation embedding} \(\mathtt{embed}_{\tilde{h}}(S_h^t)\), which is used to approximate the target distributions \( g_{\tilde{h}}^* (S_h^t) \)  after a softmax transformation.

    \item In \textbf{Step 3}, for all $\tilde h \in \{0, \ldots, H\}$, we build the selection module  \( F_{\tilde{h}} \) to check if the index of the current module, i.e.,  \( \tilde{h} \), 
    matches the extracted index \( h \). 
    This module approximately adds $\mathtt{embed}_{\tilde h} (S_h^t) \cdot  \mathbbm{1} \{ \tilde h = h\}$ to a memory. 
    As a result, $F_H$ outputs a desired output $\mathtt{embed}_{h} (S_h^t) $.

    \item In \textbf{Step 4}, we combine the constructions introduced in the first three steps
    with a softmax output layer to complete the final transformer. 
    Then we analyze the approximation error of the transformer network.
    \item  Finally, we conclude the proof in \textbf{Step 5} by verifying that the constructed transformer network belongs to the function class $\calP_\mathrm{LLM}$ by verifying that the transformer parameters satisfy  \eqref{eq:P_LLM_class}.   
\end{itemize}

\vspace{2mm}
\noindent\textbf{Step 1: Extract and copy step index $h$ using module $\mathtt{NN}_\mathrm{J}$.} In this step, we construct the extraction module \(\mathtt{NN}_\mathrm{J}\) to extract the step-index \(h\)  from the input $X_\mathrm{NN}^{(0)}$ and copy it to each reasoning step \(z_{h'}^{t'}\).
Here $X_{\mathrm{NN}}^{(0)}$ is defined in \eqref{eq:input_layer_0}. 
This step is achieved by four transformer submodules. 
In particular,   $\mathtt{NN}_\mathrm{J}$ takes $X_\mathrm{NN}^{(0)}$ as the input and outputs $X_\mathrm{NN}^{(4)}=(S_h^t, \hat p_h, \mathbf{0}) \in \RR^{L'\times (2 + 2 |\cL| ) }$. 
Here the first column is $S_{h}^t$ which stores all the reasoning steps. 
The second column $\hat p_h$ is close to $h\cdot \mathbf{1}_{L'}$, which copies the step-index $h$ to every reasoning step. Here $\mathbf{1}_{L'}$ denotes an all-one vector in $\RR^{|\cL|} $. 
The last $2|\cL|$ columns are all equal to zero vectors.

More specifically, we let $\{\mathtt{NN}_{\mathrm{J}, a} \}_{a \in [4]}$ denote the four submodules of $\mathtt{NN}_\mathrm{J}$. 
We define $X_{\mathrm{NN}}^{(a)} = \mathtt{NN}_{\mathrm{J}, a}(X_{\mathrm{NN}} ^{(a-1)})$ for all $a\in [4]$ as the output matrices of each submodule. 
These matrices are in $\RR^{L' \times 3}$ for $a\in[3]$ and in $\RR^{L'\times (2 + 2 |\cL| )}$ for $a=4$.
We let $\tilde z_{h'}^{t', (a)}$, $a \in [4]$, to denote the rows of these matrices. 
For any $(t', h') $ with $(t',h')<(t,h-1)$,  
$z_{h'}^{t'}$ is the $L'(t', h'+1)$-th element in $S_h^t$. 
Then,
starting from 
$\tilde z_{h'}^{t', (0)}$ defined in \eqref{eq:input_layer_0}, 
the $L'(t', h' + 1 )$-th rows of these matrices are given by
\begin{align*}
&\begin{array}{c r}
\tilde z_{h' }^{t', (0)} = (z_{h'}^{t'}, h', 0) & \qquad  L'(t', h'+1)\textrm{-th row of~} X_{\mathrm{NN}}^{(0)}, \\
 \Downarrow \mathtt{NN}_{\mathrm{J},1} & \\
    \tilde z_{h' }^{t', (1)} = 
 (z_{h'}^{t'}, h' + 1, 0),  &\qquad  L'(t', h'+1)\textrm{-th row of~} X_{\mathrm{NN}}^{(1)}, \\
    \Downarrow \mathtt{NN}_{\mathrm{J},2} & \\
   \tilde z_{h' }^{t', (2)} =   \big(z_{h'}^{t'}, f_\mathrm{product}(h' + 1, 0), 0\big), & \qquad L'(t', h' + 1 )\textrm{-th row of~} X_{\mathrm{NN}}^{(2)}, \\
    \Downarrow \mathtt{NN}_{\mathrm{J},3} &  \\
   \tilde z_{h' }^{t', (3)} =   \big(z_{h'}^{t'}, f_\mathrm{product}(h' + 1, 0), 1\big), & \qquad  L'(t', h' + 1 )\textrm{-th row of~} X_{\mathrm{NN}}^{(3)},\\
      \Downarrow \mathtt{NN}_{\mathrm{J},4} &  \\
   \tilde z_{h' }^{t', (4)} =   \big(z_{h'}^{t'}, p_{h'}^{t'}, \mathbf{0}\big),& \qquad  L'( t', h'+ 1 )\textrm{-th row of~} X_{\mathrm{NN}}^{(4)}.
       \end{array}
\end{align*}
Here $f_\mathrm{product}(a,b)\approx ab$ is a neural network that approximately implements the product operation using transformer blocks. 
We specify the construction of $f_\mathrm{product}$ in Lemma~\ref{lemma: product_module}. 
We use $p_{h'}^{t'}\approx h$ for each $(t',h') <  (t,h-1)$ to copy the step index $h$ to the embedding of each reasoning step $z_{h'}^{n'}$. 
Moreover, the last row, i.e., the $L'(t,h)$-th row of these matrices are 
\begin{align*}
&\begin{array}{c r}
    \tilde z_{h-1}^{t, (0)}  =  (z_{h-1}^t, h-1, 1) &  \qquad  L'(t, h)\textrm{-th row of~} X_{\mathrm{NN}}^{(0)}, \\
     \Downarrow \mathtt{NN}_{\mathrm{J},1}& \\
      \tilde z_{h-1}^{t, (1)} =(z_{h-1}^{t},h,1)    ,  &\qquad  L'(t, h)\textrm{-th row of~} X_{\mathrm{NN}}^{(1)}, \\
    \Downarrow \mathtt{NN}_{\mathrm{J},2}& \\
      \tilde z_{h-1}^{t, (2)} = \big(z_{h-1}^{t},f_\mathrm{product}(h, 1),1\big), & \qquad  L'(t, h)\textrm{-th row of~} X_{\mathrm{NN}}^{(2)}, \\
    \Downarrow \mathtt{NN}_{\mathrm{J},3}&  \\
      \tilde z_{h-1}^{t, (3)} = \big(z_{h-1}^{t},f_\mathrm{product}(h, 1),1/2\big), & \qquad  L'(t, h)\textrm{-th row of~} X_{\mathrm{NN}}^{(3)},\\
      \Downarrow \mathtt{NN}_{\mathrm{J},4} &  \\
      \tilde z_{h-1}^{t, (4)} = \big(z_{h-1}^{t},p_{h-1}^{t},\mathbf{0}\big),& \qquad   L'(t, h)\textrm{-th row of~} X_{\mathrm{NN}}^{(4)},
       \end{array}
\end{align*}
where $p_{h-1}^t $ is close to $h$.

\vspace{2mm}

In the rest of {\bf Step 1}, we prove present lemmas proving that $\mathtt{NN}_{\mathrm{J}, 1}, \ldots, \mathtt{NN}_{\mathrm{J}, 4}$ can be realized by FF  or MHA layers. The first submodule $\mathtt{NN}_{\mathrm{J},1}$ adds an one to the second coordinate of each input vector. This operation can be realized by a FF layer exactly.
As shown in the beginning of Appendix \ref{proof:prop:formal_construction}, this can be realized by a single transformer block. 



\begin{lemma}[Submodule $\mathtt{NN}_{\mathrm{J},1}$] \label{lem:layer_1}
    There exists a FF layer $\mathtt{NN}_{\mathrm{J},1}$ such that 
    \begin{align*}
    \mathtt{NN}_{\mathrm{J},1}\big(X_\mathtt{NN}^{(0)}\big)=X_\mathtt{NN}^{(1)}, \text{ where }
    \tilde z_{h'}^{t',(1)}  = \begin{cases}
        (z_{h'}^{t'}, h'+1,0) & \text{for } (t', h')< (t,h-1),\\
        (z_{h-1}^{t},h,1) & \text{for }(t', h')= (t,h-1).
        \end{cases}
\end{align*}
    Moreover, the Frobenius norms of the weight matrices are bounded by $\sqrt{B^2+H^2+1}\cdot\sqrt{6}$.
 Thus, this function can represented as a single transformer block.

    
\end{lemma}
\begin{proof}
    See Appendix~\ref{app:layer_1} for details.
\end{proof}


    

Next, we aim to substitute the second coordinate of each $\tilde z_{h'}^{t',(1)}$ with the product of itself and the third coordinate using $\mathtt{NN}_{\mathrm{J}, 2} $. Namely, we aim to compute $(h'+1)\cdot 0 =0$ for $(t', h')< (t,h-1)$ and $h\cdot 1 =h$ for $(t', h')= (t,h-1)$.
The product operation can be approximately realized by a fully connected neural network with an arbitrarily small error. 
As a result, $\mathtt{NN}_{\mathrm{J}, 2}$ can be implemented by  a composition of multiple transformer blocks.


\begin{lemma}[Submodule $\mathtt{NN}_{\mathrm{J},2}$] \label{lem:layer_2}
   Let $\epsilon' \in (0, 1)$ be a desired accuracy level. 
    There exists  fully connected MLP $\mathtt{NN}_{\mathrm{J},2}$ with at most $C_p\cdot (\log(H)+\log(1/\epsilon'))$   layers  such that 
\begin{align}\label{eq:output_layer_1}
    \mathtt{NN}_{\mathrm{J},2}\big(X_\mathtt{NN}^{(1)}\big)=X_\mathtt{NN}^{(2)}, \text{ where }
    \tilde z_{h'}^{t',(2)}  = \begin{cases}
        \big(z_{h'}^{t'},f_\mathrm{product}(h'+1,0),0\big) & \text{for } (t', h')< (t,h-1), \\
        \big(z_{h-1}^{t},f_\mathrm{product}(h,1),1 \big) & \text{for }(t', h')= (t,h-1),
        \end{cases}
\end{align}
Here $C_p$ is an absolute constant and $f_\mathrm{product}:\mathbb{R} \times \mathbb{R} \to \mathbb{R}$ is an approximation of the product operation in the sense that 
$|f_\mathrm{product}(h,1)-h|<\epsilon'$, and $|f_\mathrm{product}(h'+1,0)-0|<\epsilon'$ for each $(t',h')<(t,h-1)$. Thus, by Proposition \ref{prop:ffn_transformer}, 
$ \mathtt{NN}_{\mathrm{J},2}$ can be written as a composition of $C_p\cdot (\log(H)+\log(1/\epsilon')) $ transformer blocks up to a scaling factor. Moreover, the Frobenius norms of the weight matrices are all bounded by $\sqrt{B^2+H^2+6}\cdot \sqrt{H^4\cdot5 +4}$.
\end{lemma}
\begin{proof}
    See Appendix~\ref{app:layer_2} for details.
\end{proof}

The third submodule $\mathtt{NN}_{\mathrm{J},3}$ modifies the last coordinates of the input vectors by adding a $-1/2$ or $1$, which is a simple linear operation and thus can be implemented by a FF layer.

\begin{lemma}[Submodule $\mathtt{NN}_{\mathrm{J},3}$] \label{lem:layer_3}
    There exists a FF layer  $\mathtt{NN}_{\mathrm{J},3}$ 
     such that 
    \begin{align*}
    \mathtt{NN}_{\mathrm{J},3}\big(X_\mathtt{NN}^{(2)}\big)=X_\mathtt{NN}^{(3)}, \text{ where }
    \tilde z_{h'}^{t',(3)}  = \begin{cases}
        \big(z_{h'}^{t'},f_\mathrm{product}(h'+1,0),1\big) & \text{for } (t', h')< (t,h-1),\\
        \big(z_{h-1}^{t},f_\mathrm{product}(h,1),1/2 \big) & \text{for }(t', h')= (t,h-1). 
        \end{cases}
\end{align*}
Thus, as a single FF layer,  $\mathtt{NN}_{\mathrm{J},3}$ can be represented by a single transformer block. The Frobenius norms of the weight matrices are bounded by $\sqrt{B^2+(H+1)^2+1}\cdot \sqrt{6}$.
\end{lemma}
\begin{proof}
    See Appendix~\ref{app:layer_3} for a detailed proof.
\end{proof}

In addition to the weights constructed in Lemma~\ref{lem:layer_3}, the parameter $W_{\mathrm{ff},1}$ in $\mathtt{NN}_{\mathrm{J},3}$ is first multiplied with a diagonal matrix to compensate the scaling factor in Lemma~\ref{lem:layer_2}.
Finally, we use an attention layer to copy the step-index $h$ of $z_{h}^t$ to 
all previous steps. 
Moreover, we also use a residual link to ensure the first coordinate remains unchanged.

\begin{lemma}[Submodule $\mathtt{NN}_{\mathrm{J},4}$]\label{lemma: extraction_module}
Let $\epsilon'$ denote the error induced by $f_\mathrm{product}(\cdot)$ in Lemma~\ref{lem:layer_2}, then for any $\epsilon \in (2  \epsilon', 1) $, there exists a submodule $\mathtt{NN}_{\mathrm{J},4}$ such that 
\begin{align}\label{eq:output_layer_3}
    \mathtt{NN}_{\mathrm{J},4}\big(X_\mathtt{NN}^{(3)}\big)=X_\mathtt{NN}^{(4)},\text{ where }\tilde z_{h'}^{t',(4)}  = \begin{cases}
        (z_{h'}^{t'},p_{h'}^{t'},\mathbf{0}) & \text{for } (t', h')< (t,h-1),\\
        (z_{h-1}^{t},p_{h-1}^{t},\mathbf{0} ) & \text{for }(t', h')= (t,h-1), 
        \end{cases}
\end{align}
where $p_{h'}^{t'}$ is the approximation of the step-index $h$ such that 
$|p_{h'}^{t'}-h|<\epsilon$ for all $(t',h') \leq (t, h-1)$.
Moreover, $\mathtt{NN}_{\mathrm{J},4}$ is a transformer block with a single-head attention (MHA with $\eta$ = 1). The parameters $\{W^Q, W^K, W^V \}$ satisfying
$\|W^Q\|_\mathrm{F}=8\log(TH/(\epsilon-2\epsilon'))$,  $\|W^K\|_\mathrm{F}= 1$, and $\|W^V\|_\mathrm{F}= \sqrt{B^2+(H+1)^2+1}$, and the Frobenius norms of the weight matrices in \ac{ff} layers are bounded by $\sqrt{B^2+(H+1)^2+6}\cdot \sqrt{6}$.
\end{lemma}

\begin{proof}
    See Appendix~\ref{app: extraction_module} for details.
\end{proof}
Setting $\epsilon=1/4$ and $\epsilon'=1/16$ in Lemma \ref{lemma: extraction_module}
we conclude that we can use a transformer block $\mathtt{NN}  _{\mathrm{J},4}$ to generate the output in \eqref{eq:output_layer_3}.
The weights of the transformer are bounded by $8\log(8TH)$ in the Frobenius norm. Moreover, for any $(t', h') \leq  (t, h-1)$, we have $|p_{h'}^{t'}-h|< \epsilon=1/4$.  
Then we write the composition of the four submodules above as
\begin{align}
    \mathtt{NN}_\mathrm{J}(S_h^t) = (S_h^t, \hat p_h, \mathbf{0})\in \mathrm{R}^{ L'\times (2+2|\calL|)}, \text{ where }\mathbf{0}\in \mathbb{R}^{ L'\times 2|\calL|}, \label{eq:output_nn_j}
\end{align}
where we use $\hat p_h = (p_0^1,\cdots, p_{h-1}^t)^\top \in \mathbb{R}^{L'\times 1}$ to denote the vector that is the second column of $X_{\mathtt{NN}}^{(4)}$. 
The entries of $\hat p_h$ are all close to $h$ in the sense that $\|\hat p_h-h\cdot \mathbf{1}_{L'}\|_\infty<1/4$.

To summarize, in {\bf Step 1},  we have successfully designed a transformer module $\mathtt{NN}_{\mathbf{J}}$ to extract the target step-index $h$ from the last reasoning step of the prompt $S_h^t$, and approximately copy it to all previous reasoning steps in the prompt. We defer the summary of parameters of $\mathtt{NN}_\mathrm{J}$ to \textbf{Step 5}. This concludes \textbf{Step 1}.


\vspace{2mm}

We outline how  $\mathtt{NN}_\mathrm{J}(S_h^t)$ is processed by the subsequent transformer blocks.
In \textbf{Step 2} and \textbf{3} we construct modules $G_{\tilde h}$ and $F_{\tilde h}$. Submodule \( G_{\tilde h} \) approximates each target function \( g_{\tilde h}^* \) and submodule \( F_{\tilde h} \) checks if the module index \( \tilde h \) matches the step index \( h \). If \( \tilde h = h \), it passes along the approximation produced by \( G_{\tilde h} \); otherwise, it discards the output. In the final network, blocks $(G_{\tilde h}, F_{\tilde h})_{\tilde h=0}^H$ are chained sequentially.

With a slight abuse of notation, we also use $G_{\tilde h}$ and $F_{\tilde h}$ to refer to their outputs, respectively. 
The input and output of each module are listed as follows:
\begin{itemize}
    \item \( G_{\tilde h}: \mathbb{R}^{L' \times (2+2|\calL|)} \to \mathbb{R}^{L' \times (2+2|\calL|)} \) takes a matrix \( (S_h^t, \hat p_h, F_{\tilde h-1}[3], F_{\tilde h-1}[4]) \)  as the input.
    The columns of this matrix have four components, 
    where \( S_h^t \) and  \( \hat p_h \) are the same as in \eqref{eq:output_nn_j}. 
The last two components  \( F_{\tilde h-1}[3]\) and \( F_{\tilde h-1}[4] \in \mathbb{R}^{L' \times |\calL|} \) are the third and fourth components of the output of the previous module \( F_{\tilde h-1} \). If $\tilde h=0$, the input matrix is \eqref{eq:output_nn_j}, the output of the extraction module $\mathtt{NN}_\mathrm{J}$. The output of $G_{\tilde h }$ \( (S_h^t, \hat p_h, G_{\tilde h }[3], G_{\tilde h }[4]) \) keeps the first two and the last components unchanged and only changes the third component.
That is, $ G_{\tilde h }[4] = F_{\tilde h-1}[4]$, and $G_{\tilde h }[3]$
computes the approximation of \( g_{\tilde h}^* \) with the third component, which will be specified in \textbf{Step 2}.
\item Similarly, \( F_{\tilde h}: \mathbb{R}^{L' \times (2+2|\calL|)} \to \mathbb{R}^{L'\times (2+2|\calL|)} \) takes vector \( (S_h^t, \hat p_h, G_{\tilde h}[3], G_{\tilde h}[4]) \)  as the input and produces \( (S_h^t, \hat p_h,F_{\tilde h}[3], F_{\tilde h}[4]) \). 
Here, only the last component of columns is changed, i.e., $F_{\tilde h}[3] = G_{\tilde h}[3]$. 
The last component $F_{\tilde h}[4]$
is constructed iteratively via 
\begin{align} \label{eq:define_f_block_fou4}
F_{\tilde h}[4] \approx   G_{\tilde h}[3] \cdot \mathbbm{1}\{\tilde h=h\} + F_{\tilde h - 1}[4].
\end{align}
We will introduce how to use a transformer to implement \eqref{eq:define_f_block_fou4}  in  \textbf{Step 3}.
\end{itemize}


\noindent \textbf{Step 2: Construct approximation module $G_{\tilde h}$ that approximates $g_{\tilde h}^*$.} In this step, we introduce the submodules $G_{\tilde h}: \mathbb{R}^{L' \times (2+2|\calL|)} \to \mathbb{R}^{L' \times (2+2|\calL|)}$ to separately approximate each target function $g_{\tilde h}^*$ in \eqref{eq:g*} for all $\tilde h \in \{ 0, \ldots, H\} $. 
The input and output of each  \( G_{\tilde h} \) are given by
\begin{align*}
    G_{\tilde h}\big(S_h^t, \hat p_h, F_{\tilde h-1}[3], F_{\tilde h-1}[4]\big) =\big(S_h^t, \hat p_h, \mathbf{1}_{L'}^\top \mathtt{embed}_{\tilde h}(S_h^t), F_{\tilde h-1}[4]\big),
\end{align*}
where $\mathtt{embed}_{\tilde h}(\cdot)$ is used for the approximation of the target distribution $g_{\tilde h}^*(\cdot)$ in \eqref{eq:g*}. More specifically, for each $ \tilde h  $, we use $\hat g_{\tilde h}^*(\cdot) = \mathtt{softmax}\big(\mathtt{embed}_{\tilde h}(\cdot)/\tau\big)$ to denote the output distribution by passing $\mathtt{embed}_{\tilde h}(\cdot)$ through a $\mathtt{softmax}$ function with a temperature $\tau$. We expect $\hat g_h^*(S_h^t) \approx g_h^*(S_h^t)$ for all $h \in \{0, \ldots, H\}$.   Intuitively, we want each module \(G_{\tilde h}\) to handle prompts at different steps during the testing stage.
We summarize the construction of  \( G_{\tilde h} \) in the following proposition.

 \begin{proposition}\label{prop: approx of submodule} 
 
 Let $\epsilon_{\psi} , \epsilon_{w} \in (0,1) $ be two accuracy levels. 
    Under Assumptions~\ref{assumption: bd magnitude}, \ref{assumption: lower bd}, and \ref{assumption: smooth}, for any $  h \in \{ 0, \ldots,  H\}$, there exists 
    a module  
    $G_{ h}\colon \mathbb{R}^{L' \times (2+2|\calL|)} \to \mathbb{R}^{L' \times (2+2|\calL|)}$
    such that 
   $G_{h}(S_h^t, \hat p_h, F_3, F_4) = (S_h^t, \hat p_h, \mathtt{embed}_{h}(S_h^t), F_4)$.
Here $G_h$ contains $(D_{w} + D_{\psi} + 2)$ transformer blocks with 
\begin{align}
\label{eq:approximation_define_depth}    D_{\psi} =    2C_dB \cdot \big (\log (1/ \epsilon_\psi ) \bigr )^2 + \log B, \qquad D_{w} = 2C_dB \cdot \big (\log (1/  \epsilon_w   ) \bigr )^2 + \log B, 
\end{align}
    where $C_d > 0$ is an absolute constant and $B$  is the smoothness parameter appearing in   Assumption~\ref{assumption: smooth}. 
We define a function $\hat g_h^*( S_h^t)=\mathtt{softmax}(\mathtt{embed}_h(S_h^t)/\tau)$ as the output distribution approximated by the network $G_h$. 
Then $\hat  g_h^*$  satisfies 
    \begin{align} \label{eq:approximation_error_G_h}
\max_{\substack{S_h^t\in \calL^*}} \TV\big(g_h^*( S_h^t), \hat g_h^*( S_h^t)\big)  = \cO\big (  \epsilon_{w}  + 256 ^{D_{w}} \cdot \epsilon_{\psi} \big ) .
    \end{align}
For each \(G_h\), the maximum width of the \ac{ff} layers is \(18 |\cL| + 4\), the maximum Frobenius norm of weight matrices is
\begin{align*}
    C_F\cdot B_0\cdot16^{\max \{D_\psi,D_w\}}\cdot |\calL|^{3/2},
\end{align*}
where $B_0 = \sqrt{B^2+H^2+|\calL|\cdot C_A^2}$ and $C_F>1$ is a absolute constant.


\end{proposition}

\begin{proof}
    See Appendix~\ref{app: approx of submodule} for a detailed proof.
\end{proof}
This proposition states that for any $h \in [H]$, we can use a transformer $G_h$ to approximate $g_h^*(S_h^t)$ accurately. 
The number of transformer blocks in $G_h$ is determined by the desired accuracy levels $\epsilon_{\psi} $ and $\epsilon_w$. 
Note that the approximation accuracy grows exponentially in $D_{w}$. 
This will not be a problem when $D_{w}$ is small compared to $D_{\psi}$. 
To see this, we can rewrite the upper bound in \eqref{eq:approximation_error_G_h}  in terms of the depth of the transformer. 
Specifically, let $\overline D_{\psi} $ and  $\overline D_{w}$ be two sufficiently large integers. 
Setting 
\begin{align}\label{eq:define_error_psi_w}
\epsilon_{\psi} = \exp \Bigl ( -   \sqrt{ ( \overline D_{\psi} - \log B) / B } \Bigr ) \qquad \textrm{and} \qquad \epsilon_{w } = \exp \Bigl (  -    \sqrt{ ( \overline D_{w}  - \log B) / B } \Bigr ) 
\end{align}
in Proposition \eqref{prop: approx of submodule},
we know that there exists   $\{ G_h\}_{h=0}^{H}$ such that 
  \begin{align}\label{eq:invert_upper_bound}
      \max_{\substack{S_h^t\in \calL^*}} \TV\big(g_h^*( S_h^t), \hat g_h^*( S_h^t)\big)  = \cO\bigg (   \exp\Bigl ( - \sqrt{ \overline D_{w} / B } \Big)  + \exp\Bigl ( - \sqrt{ \overline D_{\psi} / B} + 6 \cdot D_{w} \Bigr)     \bigg ) 
  \end{align}
  for all $h \in \{0, \ldots, H\}$.
  Here we use the fact that $\log (256) < 6$. 
  Moreover, 
\eqref{eq:approximation_define_depth} implies that  the number of transformer blocks satisfies $ D_{\psi} \leq C_d \cdot \overline D_{\psi} $ and $D_{w} \leq C_d \cdot \overline D_{w} $ for some absolute constant $C_d$. 
Note that $\overline D_{\psi} $ and $\overline D_{w} $ can be chosen arbitrarily. 
We can set $\overline D_{\psi} =  \cO(\overline D_{w}^2)$ so that the second term in \eqref{eq:invert_upper_bound} becomes negligible compared to the first term. 
We will determine $\overline D_{\psi}$ and $\overline D_{w} $ to obtain the final error in  {\bf Step 4}.





\vspace{2mm}

\noindent\textbf{Step 3: Construct selection module $F_{\tilde h}$.}
In this step, we introduce a sequence of transformer modules \(\{F_{\tilde h}\}_{\tilde h=1}^H \).
Each $F_{\tilde h}$ is a mapping from $ \RR^{L'\times (2 + 2 |\cL| )} $  to itself, and its input and output are given by 
\begin{align}
    F_{\tilde h}(S_h^t, \hat p_h, G_{\tilde h}[3], G_{\tilde h}[4]) &=\big(S_h^t, \hat p_h, G_{\tilde h}[3], f_\mathrm{product}\big(G_{\tilde h}[3], \mathbbm{1}\{\tilde h=h\}\cdot \mathbf{1}_{L'}^\top\big) + G_{\tilde h}[4]\big) \label{eq:f_output} .
\end{align}
That is, $F_{\tilde h}$ takes the output of $G_{\tilde h}$ as the input, and it keeps the first three components of the columns unchanged. 
The last component, i.e., the last $|\cL|$ columns, are used a ``memory''. 
Note that $G_{\tilde h} [4] = F_{\tilde h -1} [4]$. 
The last component of the output of $F_{\tilde h} $, $F_{\tilde h} [4]$, can be written as 
\begin{align}\label{eq:f_tilde_4}
F_{\tilde h} [4] = \mathbf{1}_{L'}^\top f_\mathrm{product}\big(\mathtt{embed}_{\tilde h}(S_h^t), \mathbbm{1}\{\tilde h=h\}\big) + F_{\tilde h-1}[4],
\end{align}
where $f_\mathrm{product} $
is a transformer that approximately implements the product operation. 
Thus, 
$F_{\tilde h}$ first checks if the module index \(\tilde h\) matches the target step index \(h\), and then writes $\mathtt{embed}_{\tilde h}(\cdot)$  to the memory if $\tilde h  = h$. 
Thus, by \eqref{eq:f_tilde_4} we have 
\begin{align}
\label{eq:f_tilde_4_H}
F_{H}[4] = \sum_{\tilde h =0}^H \mathbf{1}_{L'}^\top f_\mathrm{product}\big(\mathtt{embed}_{\tilde h}(S_h^t), \mathbbm{1}\{\tilde h=h\}\big) \approx \mathbf{1}_{L'}^\top \mathtt{embed}_{h}(S_h^t). 
\end{align}

To implement each $F_{\tilde h}$, we starting from the input $(S_h^t, \hat p_h, G_{\tilde h}[3], G_{\tilde h}[4])$, denoted by $X_{F}^{\tilde h, (0)}$, we perform the following three steps:
\begin{itemize}
    \item [(i)] First, we use a sequence of transformer blocks to represent the indicator $\mathbbm{1}\{\tilde h=h\}$, and then append it to the  end of $X_{F}^{\tilde h, (0)}$. Thus, we have 
    \begin{align}
        \label{eq:X_F_tilde_1}
        X_{F}^{\tilde h, (1)} = \big(S_h^t, \hat p_h, G_{\tilde h}[3], G_{\tilde h}[4],  \mathbbm{1}\{\tilde h=h\}\cdot \mathbf{1}_{L'}^\top \big).
    \end{align}
    Here the indicator is obtained by feeding $\hat p_h$ to a trapezoid-shaped function. 
    \item [(ii)] Then we feed $X_{F}^{\tilde h, (1)}$ to the product module introduced in Lemma \ref{lemma: product_module} to multiply each entry of $G_{\tilde h} [3] $ with $ \mathbbm{1}\{\tilde h=h\}$. 
    The resulting output is 
    \begin{align}
        \label{eq:X_F_tilde_2}
X_{F}^{\tilde h, (2)} = \big(S_h^t, \hat p_h, G_{\tilde h}[3], G_{\tilde h}[4],  f_\mathrm{product}\big(G_{\tilde h}[3], \mathbbm{1}\{\tilde h=h\}\cdot \mathbf{1}_{L'}^\top\big)\big).
    \end{align}
    \item [(iii)] 
    Finally, we pass $X_{F}^{\tilde h, (2)}$ to a linear layer, which adds the last two components of $X_{F}^{\tilde h, (2)}$ and obtain 
    \begin{align*}
        X_{F}^{\tilde h,(3)} =\big(S_h^t, \hat p_h, G_{\tilde h}[3], f_\mathrm{product}\big(G_{\tilde h}[4], \mathbbm{1}\{\tilde h=h\}\cdot \mathbf{1}_{L'}^\top\big) + G_{\tilde h}[4]\big).
    \end{align*}
    \end{itemize}


\vspace{2mm}
{\bf\noindent  Details of (i).} We present the details of these three steps as follows. 
We first focus on how to construct the indicator 
$\mathbbm{1}\{h=\tilde h\}$. 
Recall that by Lemma \ref{lemma: extraction_module}  we show that $\hat p_h $ satisfies \(\|\hat p_h - h\cdot \mathbf{1}_{L'}\|_\infty < 1/4\). 
Thus, each entry of $\hat p_h $ is in $(h-1/4, h+1/4)$. 
For any $\tilde h$, we want to construct a neural network $f_{\tilde h} \colon \RR \rightarrow [0, 1] $ such that $f_{\tilde h} (x) = 1 $ if $| x - \tilde h | \leq 1/4 $ and $f(x) = 0$ if $|x - \tilde h| \geq 3/4$. 
Then applying $f_{\tilde h}$ to each entry of $\hat p_h$, we have 
$ f_{\tilde h} (\hat p_h) = \mathbbm{1}\{h=\tilde h\} \cdot \mathbf{1}_{L'} $.

Such a $f_{\tilde h} $ can be constructed by a trapezoid-shaped function, which has value one in $[h-\epsilon, h+\epsilon ]$, zero when $|x- h | > 1 - \epsilon $, and a linear function in between. 
Here we can set $\epsilon = 1/4$. 
See  Figure~\ref{fig:trapezoid} for an illustration of two trapezoid-shape functions. 
The following lemma shows that such trapezoid-shaped functions can be implemented by a \ac{ff} layer.

\begin{lemma}[Trapezoid module] \label{lemma: trapezoid_module}
For any $h \in \{0, \ldots, H\}$ and any $\epsilon \in (0, 1/2) $, we define a trapezoid-shaped function $f_h \colon \RR \rightarrow [ 0, 1]$  as  
\begin{align*}
    f_h(x)=\begin{cases}
        1 &\text{for } |x-h|\leq \epsilon,\\
        -(|x-h|-\epsilon)/(1-2\epsilon) +1 &\text{for } \epsilon< |x-h|\leq 1-\epsilon, \\
        0 &\text{otherwise.} 
    \end{cases}
\end{align*}
Then there exists a neural network $f$ that is identical to $f_h$.
Moreover, $f$ is a composition of two FF layers, each with no more than $10$ neurons, and the entries of the  weight matrices and bias vectors are bounded by $H + 1$ in magnitude.



\end{lemma}
\begin{proof}
    See Appendix~\ref{app: trapezoid_module} for details.
\end{proof}

\begin{figure}[h]
    \centering
    \includegraphics[width = 0.95\textwidth]{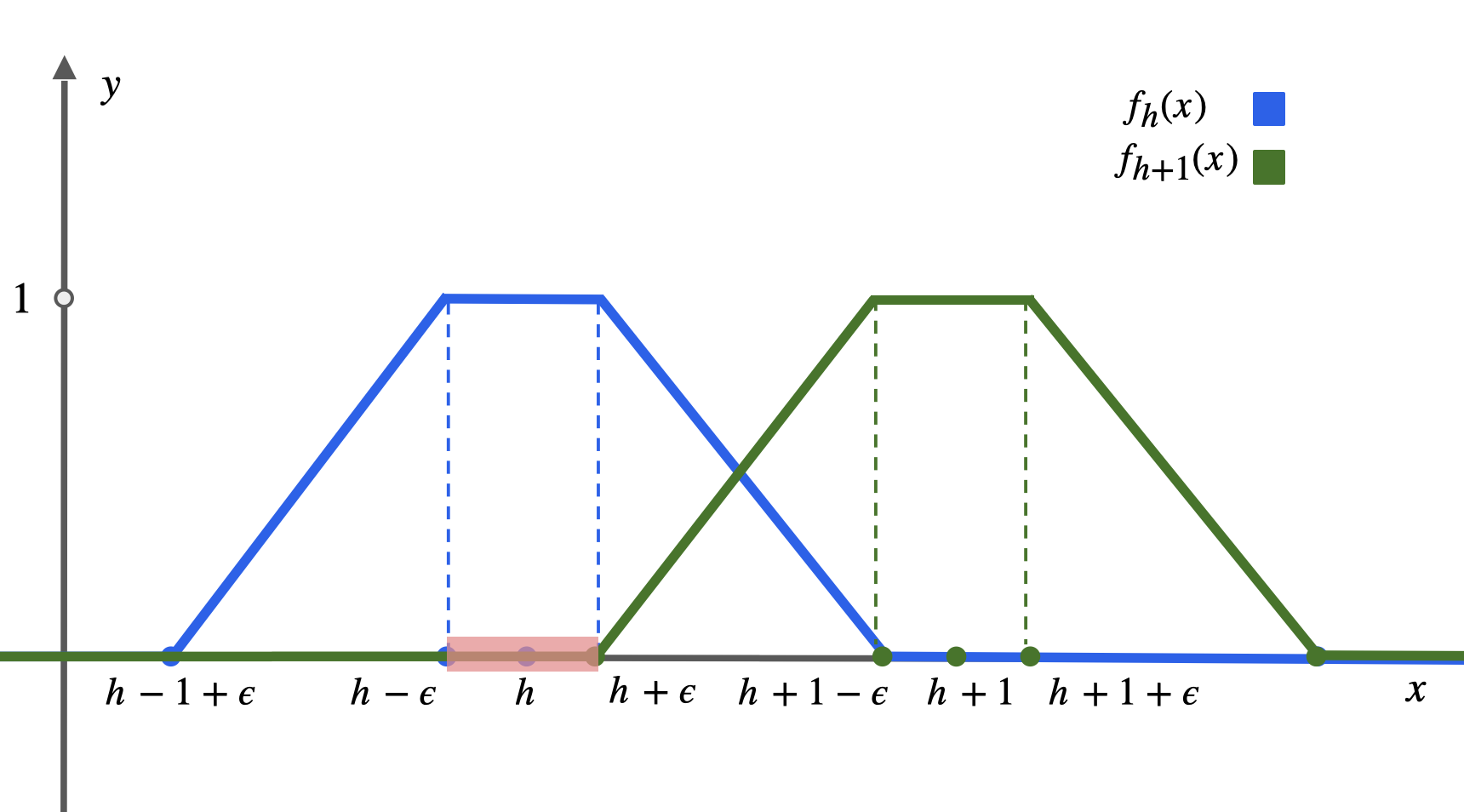}
    \caption{An illustration of two trapezoid-shaped functions $f_h$ (blue) and $f_{h+1}$ (green), i.e., $f_{\tilde h} $ with $\tilde h = h$ and $\tilde h = h+1$. 
    Observe that for any $x\in [h-\epsilon, h+\epsilon]$, $f_{\tilde h}(x)=1$ if and only if $\tilde h=h$, and $f_{\tilde h}(x)=0$ for any other $\tilde h\neq h$.
    Thus, when we apply each $f_{\tilde h}$ to entries of $\hat p_h$, when $\epsilon = 1/4$, we have $\mathbbm{1}\{h=\tilde h\}$. Each entry of \(\hat p_h\) falls within the highlighted region on the x-axis. }
    \label{fig:trapezoid}
\end{figure}




We apply this lemma with $\epsilon = 1/4$ and apply  $f_{\tilde h} $ to each entry of $\hat p_h$ to obtain  
$ \mathbbm{1} \{h = \tilde{h}\}\cdot \mathbf{1}_{L'}$,which becomes the last component of $X_{F}^{\tilde h, (1)} $ in \eqref{eq:X_F_tilde_1}.
Moreover, to preserve the first three components of the input  $X_{F} ^{\tilde h, (0)}$, we apply Lemma~\ref{lem:residual_relu},
which implies that we can use a single FF layer to map $X_{F}^{\tilde h, (0)} $ to its first three components. 
The number of neurons in this FF layer is bounded by $2 ( 2 + 2 |\cL|)$. 
Thus, we can concatenate these two networks and obtain a larger network that maps $X_{F}^{\tilde h, (0)}$ to $X_{F}^{\tilde h, (1)}$.
Moreover, such a network has one FF layer, 
and the maximum width of the weight matrices is bounded by  
$10 + 2 (2 + 2|\cL|)$.
As a result, the weight matrices of this feed-forward neural network are all bounded by $\sqrt{B^2+H^2+2|\calL|\cdot(C_A+1)^2}\cdot\sqrt{10 +2(2+2|\calL|)}\leq \sqrt{B^2+H^2+2|\calL|\cdot(C_A+1)^2}\cdot(4+2\sqrt{|\calL|)}$ in terms of the Frobenius norm.

\vspace{2mm}
{\bf\noindent  Details of  (ii).}  
Then, to get $X_{F}^{\tilde h, (2)}$ defined in \eqref{eq:X_F_tilde_2}, 
we pass $X_{F}^{\tilde h, (1)}$ to the product module in Lemma~\ref{lemma: product_module} to  multiply $G_{\tilde h}[3]$ and $\mathbbm{1}\{\tilde h=h\}\cdot \mathbf{1}_{L'}^\top$  in an elementwise fashion. 
Similar to the implementation of \(\mathtt{NN}_{\mathrm{J},2}\) introduced in Lemma \ref{lem:layer_2}, the product module  $f_{\mathrm{product}} $  here can be implemented as an MLP.
More concretely,
recall that Assumption \ref{assumption: smooth}
states that each $\tau \cdot \log w_{h, i}^*
$ is bounded by $C_A$ in terms of the $\ell_\infty$-norm. 
Then, as shown in the proof of Proposition \ref{prop: approx of submodule} in Appendix \ref{app: approx of submodule}, we have 
\begin{align}\label{eq:error_embed_map}
\bigg| \bigl( \mathtt{embed}_{\tilde h}  (S_h^t) \bigr ) [i]- \tau \cdot \log  w_{\tilde h,  i} \bigg ( \frac{1}{L'} \bigg( \sum_{(i, j) = (1,0)}^{(t, h-1) } \psi_{\tilde h} (z_{h}^i ) )
    \bigg)  \bigg| \leq \epsilon_{w} + 256^{D_{w}} \cdot \epsilon _{\psi}
\end{align}
for all $i \in [|\cL| ]$. 
Here $ ( \mathtt{embed}_{\tilde h}  (S_h^t)  ) [i]$  is the $i$-th entry of $ \mathtt{embed}_{\tilde h}  (S_h^t)$, 
which is defined in the same way as in \eqref{eq:embed_approx}, but with $\hat f_h^{w}$ and $\hat \psi_h^*$ replaced by $\hat f_{\tilde h}^w$ and $\hat \psi_ {\tilde h}^*$. 
Besides, as we will show later, we set $\epsilon_\psi$ and $\epsilon_w$ as in \eqref{eq:define_error_psi_w} so that the right-hand side of \eqref{eq:error_embed_map} is much smaller than one.  
As a result, we have 
 $\| \mathtt{embed}_{\tilde h}(S_h^t)\|_\infty \leq 1+C_A$.
Thus, combining this fact with Lemma \ref{lemma: product_module}, we conclude that,  there exists an MLP $f^{\tilde h}_{\mathrm{product}}$ such that  for any $\epsilon_p \in (0, 1)$,  
\begin{align}\label{eq:error_f_prod_F}
\|f^{\tilde h}_\mathrm{product}\big(\mathtt{embed}_{\tilde h}(S_h^t), \mathbbm{1}\{\tilde h=h\}\big)-\mathtt{embed}_{\tilde h}(S_h^t)\cdot\mathbbm{1}\{\tilde h=h\}\|_\infty <\epsilon_p
\end{align}
 and any $\tilde h \in \{0, \ldots, H\}$,
where $f^{\tilde h}_\mathrm{product}$ has at most 
$D_{p} = C_p \cdot (\log(C_A+1)+\log(1/\epsilon_p))$ FF layers, where $C_p$ is an absolute constant. 
Moreover, we remark that the module $f^{\tilde h}_\mathrm{product}$ is different from the one constructed in Lemma \ref{lem:layer_2}, and different for each $0\leq \tilde h \leq H$.


 Moreover, we need to write such a product module as a composition of transformer blocks using Proposition~\ref{prop:ffn_transformer}. At the same time, we adopt Lemma~\ref{lem:residual_relu} to preserve the first four components of $X_{F}^{\tilde h , (1)}$, namely $S_h^t, \hat p_h, G_{\tilde h}[3], G_{\tilde h}[4]$ using \ac{ff} layers. Specifically, we apply Proposition~\ref{prop:ffn_transformer} by setting the scaling factors $B_0 = \sqrt{B^2+H^2+2\cdot |\calL|\cdot (C_A+1)^2}$, $\{B_\ell = \sqrt{B_0^2+5|\calL|}\}_{\ell=1}^{D_p-1}$, $B_{D_p}=\sqrt{B_0^2+(C_A+1)^2}$.
As a result, the function that maps $X_{F}^{\tilde h , (1)}$ to $X_{F}^{\tilde h, (2)} $ can be implemented by a composition of $D_{p}$ transformer blocks. 
 Each block has at most $5|\cL|+2(2+2|\calL|)$ neurons in the FF layer, where $5|\cL|$ are used for the product operation, and $2(2+2|\calL|)$ are used for preserving the input of $X_F^{\tilde h, (1)}$. According to Lemma~\ref{lemma: product_module}, the Frobenius norm of weight matrices are bounded by $\sqrt{B_0^2+5|\calL|}\cdot \sqrt{2(2+2|\calL|)+ (C_A+1)^4\cdot 5|\calL|} \leq C'\cdot (C_A+1)^3\cdot |\calL|$ for some absolute constant $C'$.

\vspace{2mm}
{\bf \noindent Details of (iii).}
Finally, we pass $X_{F}^{\tilde h,(2)}$ through a linear layer to add    $f_\mathrm{product}\big(G_{\tilde h}[3], \mathbbm{1}\{\tilde h=h\}\cdot \mathbf{1}_{L'}^\top\big)$ with $G_{\tilde h}[4]$. 
Moreover, we  adopt Lemma~\ref{lem:residual_relu} to preserve the first three components of $X_{F}^{\tilde h, (2)}$. Therefore, the maximum Frobenius norm of the weight matrix of this module is $\sqrt{B^2+H^2+3\cdot |\calL|\cdot (C_A+1)^2}\cdot \sqrt{2(2+|\calL|)+4|\calL|}$, where the first term results from the scaling trick, and the second term results form the linear operation.

\vspace{2mm}

{\noindent \bf Combining (i)--(iii).}
Combining these three steps above, 
we obtain the selection module $F_{\tilde h}$ shown in \eqref{eq:f_output} and \eqref{eq:f_tilde_4}.
By \eqref{eq:f_tilde_4_H},  the output of $F_{H}$ is given by 
\begin{align}
  \bigg(S_h^t, \hat p_h, \mathbf{1}_{L'}^\top\mathtt{embed}_{H}(S_h^t), \sum_{\tilde h =1}^H  f_\mathrm{product}\big(\mathtt{embed}_{\tilde h}(S_h^t), \mathbbm{1}\{\tilde h=h\}  \cdot \mathbf{1}_{L'}^\top\big)\bigg). \label{eq:last_f}
\end{align}
For the ease of notation, we use $\Tilde{\mathtt{embed}}(S_h^t)$  to denote last component of \eqref{eq:last_f}, and we expect $\Tilde{\mathtt{embed}}(S_h^t)\approx \mathtt{embed}_h(S_h^t)$.

Finally, we calculate the depth and norms of the weight matrices of \( F_{\tilde h} \). 
According to Lemma~\ref{lemma: trapezoid_module}, the trapezoid module has a depth of 1 and a maximum Frobenius norm of \( \sqrt{B^2+H^2+2\cdot|\calL|\cdot(C_A+1)^2}(4+2\sqrt{|\calL|})\). The linear module also has a depth of 1 and a maximum Frobenius norm of \( \sqrt{B^2+H^2+3\cdot|\calL|\cdot(C_A+1)^2}(2+\sqrt{6|\calL|}) \). The product module has a maximum Frobenius norm of $C'\cdot (C_A+1)^3\cdot |\calL|$ and depth \( D_p \), which will be determined in \textbf{Step 4}. Overall, each \( F_{\tilde h} \) module has a depth of \( D_f = D_p + 2 \) and maximum Frobenius norm of weight matrices as $C'\cdot (C_A+1)^3\cdot |\calL|$.

\vspace{2mm}

{\noindent  \bf Step 4: Compute the approximation error.} 
The last component of $F_{H}$, $\Tilde{\mathtt{embed}}(S_h^t)$, is  then fed into a softmax layer to generate the final output. 
In this step, we characterize the approximation error 
 $\norm{g^*_h(S_h^t) -\mathtt{softmax}(\Tilde{\mathtt{embed}}(S_h^t)  /\tau)}_1$. 
 Here $g_h^*$ in \eqref{eq:g*} refers to the target distribution, and $\mathtt{softmax}(\Tilde{\mathtt{embed}}(S_h^t)  /\tau)$ refers to the output distribution.

First, we handle the error induced by the product operation in $F_{\tilde h}$, which is established  in  
 \eqref{eq:error_f_prod_F}. 
Let $\overline D_{p}$ be an integer and we set  \(\epsilon_p = (C_A + 1) \cdot \exp(-\overline{D}_p/C_p)\) in \eqref{eq:error_f_prod_F}, by triangle inequality we conclude that 
\begin{align}
    \big \| \Tilde{\mathtt{embed}} (S_h^t)- \mathtt{embed}_{h}(S_h^t)\big\|_\infty<H\epsilon_p. \label{eq:err_embed}
\end{align}
Moreover, with number of transformer blocks used to implemnt $f_{\mathrm{product}}$, $D_{p}$, satisfies  $D_p\leq \overline D_p$.
Since $\mathtt{softmax}$ is Lipschitz continuous, as shown in 
Lemma~\ref{lem:softmax_lip}, we conclude that 
\begin{align}
    &\big\|\mathtt{softmax}\big(\Tilde{\mathtt{embed}}(S_h^t) /\tau \big) - \mathtt{softmax}\big( \mathtt{embed}_h (S_h^t)/ \tau \big)\big\|_1 \nonumber\\
        &\qquad \leq 2\big\| \Tilde{\mathtt{embed}}(S_h^t)/\tau-  \mathtt{embed}_h(S_h^t)/\tau  \big\|_\infty  \leq 2H\epsilon_p/\tau,\label{eq:softmax_err}
    \end{align}
where the second inequality follows from \eqref{eq:err_embed}.

Therefore, combining   Proposition~\ref{prop: approx of submodule} and \eqref{eq:softmax_err}, for any prompt $S_h^t$ with length $L'\leq L$, the approximation error of the transformer is bounded by 
\begin{align}
&\norm[\big]{g^*_h (S_h^t) -\mathtt{softmax}\big(\Tilde{\mathtt{embed}}(S_h^t)  /\tau\big)}_1 \nonumber\\
&\qquad \leq \norm[\big]{g^*_h(S_h^t) -\mathtt{softmax}\big( \mathtt{embed}_h(S_h^t)/\tau \big) }_1 \nonumber\\
&\qquad \qquad +\big\|\mathtt{softmax}\big( \mathtt{embed}_h(S_h^t)/\tau \big) - \mathtt{softmax}\big(\Tilde{\mathtt{embed}}(S_h^t)/\tau\big) \big\|_1 \nonumber \\
&\qquad  \leq 2\epsilon_w + 2\cdot 256^{\overline D_w}\cdot \epsilon_\psi + 2H \epsilon_p/\tau. \label{eq:3 errors}
\end{align}
Here $\epsilon_w$ and $\epsilon_\psi$  are chosen as in \eqref{eq:define_error_psi_w} and $\epsilon_p$ is defined above. 
Thus, we can equivalently write the approximation error in terms of the parameters $\overline D_{\psi}$, $\overline D_{w}$, and $\overline D_{p}$ as 
\begin{align}
    &\norm[\big]{g^*_h (S_h^t) -\mathtt{softmax}\big(\Tilde{\mathtt{embed}}(S_h^t)  /\tau\big)}_1 \label{eq:3 bds}\\
    &\quad = \calO \bigg(\exp\Big(-\sqrt{\overline D_w/B}\Big) + \exp\Big(-\sqrt{\overline D_\psi/B}+6\cdot \overline D_w\Big)+ H/\tau\cdot \exp \big(-\overline D_p/C_p +\log(C_A+1)\big)\bigg),\nonumber
\end{align}
where we use \eqref{eq:invert_upper_bound}, \eqref{eq:3 errors}, and the definition of $\epsilon_p$. 
Here  the number of transformer blocks of each module among 
$\{G_j, F_j\}_{j=1}^H$ 
satisfy  $D_w = \cO(\overline D_w) $, $D_\psi \leq C\overline D_\psi$, and $D_p \leq \overline D_p$, where $C>0$ is a absolute constant.

To get an explicit upper bound, we choose $\overline D_\psi$, $\overline D_w$, and $\overline D_p$ properly to balance the three terms in the right-hand side of \eqref{eq:3 bds}. 
Specifically, we require 
\begin{align}\label{eq:choose_overline_D}
    \overline D_\psi \geq \big(6\sqrt{B}\cdot \overline D_w +\sqrt{\overline D_w}\big)^2,   \qquad (\overline D_p/C_p-\log(C_A+1))^2 \geq  \overline D_w/B,
\end{align}
and let 
$\overline D_g = (\overline D_\psi +\overline D_w+2) + (\overline D_p +2)$ denote a parameter that characterizes the total number of transformer blocks in each $G_{\tilde h}$ and $ F_{\tilde h}$ together. Note that the actual total number of transformer blocks in $G_{\tilde h}$ and $ F_{\tilde h}$ is $\cO(\overline D_g)$. 
To satisfy \eqref{eq:choose_overline_D},  we can set 

\begin{align}
    \overline D_w = \sqrt{\overline D_g}/(24\sqrt{B}), \qquad  \overline D_p = \frac{C_p }{\sqrt{24}B^{3/4}}\cdot \sqrt{\overline D_g}+C_p\cdot \log(C_A+1) , \qquad \overline D_\psi = \overline D_g-\overline D_w-\overline D_p-4. \label{eq:depths_balanced}
\end{align}
Then for sufficiently large $\overline D_g$, we have that \eqref{eq:3 bds} is dominated by the first term:
\begin{align}
    \norm[\big]{g^*_h(S_h^t) -\mathtt{softmax}\big(\Tilde{\mathtt{embed}}(S_h^t)  /\tau\big)}_1 &=\calO\bigg(\exp\big(-\sqrt{\overline D_w/B}\big)\bigg) \nonumber\\
    &=\calO\bigg(\exp\bigg(-\frac{\overline D_g^{1/4}}{5B}\bigg)\bigg),\label{eq:l1_bd}
\end{align}
where the second inequality follows from \eqref{eq:depths_balanced}, the fact that $\sqrt{24}<5$, and the relaxation of the exponent of \( B \) for notational clarity, assuming \( B \geq 1 \).

Finally, we convert the $\ell_1$-norm upper bound  in \eqref{eq:l1_bd} into a bound in terms of the KL divergence. Let $ \PP_{\hat \rho} (\cdot\given S_h^t)$ denote $\mathtt{softmax}\big(\Tilde{\mathtt{embed}}(S_h^t)  /\tau\big)$, where $\hat \rho$ refers to the parameter that specifies the transformer we constructed in the first three steps of the proof, which consists of a step-index extraction module $\mathtt{NN}_\mathrm{J}$, and $H$ pairs of modules $\{G_{\tilde h}, F_{\tilde h}\}_{\tilde h=0}^H$. 
We first note that if $\TV(\PP(\cdot\given S_h^t),  \PP_{\hat \rho}(\cdot\given S_h^t)) = \varepsilon$ with $\varepsilon$ sufficiently small such that $\varepsilon <c_0/2$, where $c_0$ comes from Assumption~\ref{assumption: lower bd}.
Under this lemma, 
we can bound the likelihood ratio $\PP(z=\cdot\given S_h^t)/\PP_{\hat \rho}(z=\cdot\given S_h^t)$ for each $z\in \calL$ by 
$$
\log \frac{c_0}{c_0+2\epsilon}\leq \log \frac{\PP(\cdot\given S_h^t)}{ \PP_{\hat \rho}(\cdot\given S_h^t)} \leq \log\frac{c_0+2\epsilon}{c_0}\leq  \frac{2\epsilon}{c_0}.
$$
Therefore we conclude that, 
when $\overline D_g$ is sufficiently large, there exists a transformer with parameter $\hat \rho$ such that, for an 
any $t\in [T]$ and $h \in \{0, \ldots H\}$,
\begin{align*}
    \max_{\substack{S_h^t\in \calL^*}}\KL\big(\PP(\cdot\given S_h^t),  \PP_{\hat \rho}(\cdot\given S_h^t)\big)=\cO\bigg( \exp\bigg(-\frac{\overline D_g^{1/4}}{5B}\bigg)\bigg),
\end{align*}
where the number of transformer blocks is at most 
  $ C\cdot \overline D_g$  for some absolute constant $C$.

\vspace{2mm}

\noindent \textbf{Step 5: Verify that the constructed transformer is in $\cP_{\mathrm{LLM}}$.} 
Finally, we verify that the constructed transformer is in $\cP_{\mathrm{LLM}}$. To this end, for each module of the transformer, we explicitly characterize the width and norms of the weight matrices. Recall that 
the transformer contains an  
embedding module $\mathtt{NN}_\mathrm{J}$,  $\{G_{\tilde h}, F_{\tilde h}\}_{\tilde h=0}^H$,  and a softmax output layer.

\begin{itemize}
    \item \textbf{Module $\mathtt{NN}_\mathrm{J}$.} The module $\mathtt{NN}_\mathrm{J}$ consists of four submodules, which are two linear modules, a product module realized via \ac{ff} layers, and a \ac{mha} module. 
    The two linear modules in Lemmas~\ref{lem:layer_1} and \ref{lem:layer_3} have Frobenius norm of weight matrices upper bounded by $\sqrt{B^2+(H+1)^2+1}\cdot\sqrt{6}$. According to Lemma~\ref{lem:layer_2}, the product module consists of at most $C(\log(H)+\log(16))$ layers with maximum Frobenius norm $\sqrt{B^2+H^2+6}\cdot \sqrt{H^4\cdot 5+4}$. 
    According to Lemma \ref{lemma: extraction_module}, the extraction module has the maximum Frobenius norm of weight matrices as $\max\{8\log(8TH), \sqrt{B^2+(H+1)^2+1}\}$ for the \ac{mha} layer and $\sqrt{B^2+(H+1)^2+6}\cdot\sqrt{6}$ for the \ac{ff} layer. In conclusion, the module $\mathtt{NN}_\mathrm{J}$ has total depth of at most $\overline D_j= C_p\log(3H)+3$ for some constant $C_p>0$, the Frobenius norms for \ac{ff} layers are upper bounded by $\sqrt{B^2+(H+1)^2+6}\cdot \sqrt{H^4\cdot 5+4}$ and those for \ac{mha} layers are upper bounded by $\max\{8\log(8TH), \sqrt{B^2+(H+1)^2+1}\}$.
    \item \textbf{Module $G_{\tilde h}$ and $F_{\tilde h}$.} The parameter constraint for any $G_{\tilde h}$ is specified in Proposition~\ref{prop: approx of submodule}. According to \textbf{Step 3}, each $F_{\tilde h}$ has maximum Frobenius norm of weight matrices as $C'\cdot (C_A+1)^3\cdot |\calL|$. The depth of each submodule pair $G_{\tilde h}$ and $F_{\tilde h}$ is specified in \eqref{eq:depths_balanced}.
    
\end{itemize}

Let $\overline D$ denote a sufficient large integer. Consider any parameter class $\calP_\mathrm{LLM}$ in \eqref{eq:P_LLM_class} with $B_S \geq (C_A+1)\cdot \sqrt{|\calL|}$, $B_M\geq \max\{8\log(8TH), \sqrt{B^2+(H+1)^2+1}\}$, $d_F\geq 4+18|\calL|$ and 
\begin{align*}
    B_F &\geq C_F\cdot \sqrt{B^2+H^2+C_A^2\cdot |\calL|}\cdot 16^{\overline D'}\cdot |\calL|^{3/2},
\end{align*}
where $\overline D' = (\overline D-C_p\log(3H))/(H+1)$,  $C_p, C_F>0$ are absolute constants, and $C_A$ comes from Assumption~\ref{assumption: smooth}. Then under Assumptions~\ref{assumption: bd magnitude},\ref{assumption: lower bd}, and \ref{assumption: smooth}, there exists a transformer with at most $\cO(\overline D)$ transformer blocks and parameter $\rho^*\in \calP_\mathrm{LLM}$ such that
\begin{align*}
\max_{\substack{S_h^t\in \calL^*}}\KL\big(\PP(z_h^t = \cdot\,|\,S_h^t),\PP_{\rho^{*}}(z_h^t = \cdot\,|\, S_h^t)\big) =O\bigg( \exp\bigg(-\frac{ \big(\overline D-C\log(2H))/H\big)^{1/4}}{5B}\bigg)\bigg),
    \end{align*}
for any $t\in [T], 0\leq h\leq H$, where $C>0$ is some absolute constant. Therefore, we conclude the proof.


\end{proof}

\subsection{Proof of Corollary~\ref{cor: rate with error}}  \label{proof: rate with error}


\begin{proof}
Recall that Lemma~\ref{lem:error_decomp} decomposes the \ac{cot} error \(\mathtt{err}_\mathrm{CoT}\) into \(\mathtt{err}_\mathrm{pre}\) and \(\mathtt{err}_\mathrm{prompt}\). This proof consists of two steps. 
We first control the expected pretraining error \(\E_{\PP_\mathrm{CoT}} [ \mathtt{err}_\mathrm{pre} ] \) under the distribution $\PP_\mathrm{CoT}$ with OOD queries.  
Then we consider the prompting error \(\E_{\PP_\mathrm{CoT}} [ \mathtt{err}_\mathrm{prompt} ] \) in the second step.

\vspace{2mm}

\noindent \textbf{Step 1: Control expected pretraining error under distribution shift.}
In this step, we evaluate 
$\E_{\PP_\mathrm{CoT}}[\mathtt{err}_\mathrm{pre}]$, the expected pretraining error \(\mathtt{err}_\mathrm{pre}\) defined in \eqref{eq: err pre} with the expectation taken under \(\pt_\mathrm{CoT}(n) \sim \PP_\mathrm{CoT}\). We first decompose \(\mathtt{err}_\mathrm{pre}\) into a sum of errors incurred in each reasoning step and then take expectations with respect to $ \PP_\mathrm{CoT}$. We adopt the following lemma to decompose $\mathtt{err}_\mathrm{pre}$ into a sum of KL divergences.

\begin{lemma}[KL decomposition] \label{lem: kl decomp}
Recall that $\pt_\mathrm{CoT}^h(n) = \{\Upsilon_n , z_{0:h-1}^\mathrm{test}\}$ consists of $n$ examples and the first $h-1$-th inferred steps for the testing example. Then we have
\begin{align}
    &\KL\bigl(\PP(y^{\mathrm{test}} =\cdot \given \pt_{\mathrm{CoT}}(n)) , \PP_{\hat \rho}(y^{\mathrm{test}}=\cdot \given \pt_{\mathrm{CoT}}(n)) \bigr) \label{eq:kl_sum}\\
        &\quad \leq \sum_{h=1}^H \E_{{z_{1:h-1}^\mathrm{test} \sim \PP(\cdot \given \pt_\mathrm{CoT}(n))}} \Big[\KL\bigl(\PP(z_{h}^{\mathrm{test}} =\cdot \given \pt_{\mathrm{CoT}}^h(n)) , \PP_{\hat \rho}(z_{h}^{\mathrm{test}} =\cdot \given \pt_{\mathrm{CoT}}^h(n)) \bigr)\Big]. \nonumber
\end{align}
    
\end{lemma}
\begin{proof}
    See Appendix~\ref{app: kl decomp} for detailed proof.
\end{proof}
This Lemma states that we can upper bound the pretraining error $\mathtt{err}_\mathrm{pre}$ by aggregating the pretraining error at each step of inference. We apply \eqref{eq: log density difference} with Lemma~\ref{lem: tv-kl} to convert each KL divergence in \eqref{eq:kl_sum} into TV distances. Namely, we have 
\begin{align}\label{eq:tv_sum}
&\KL\bigl(\PP(y^{\mathrm{test}} =\cdot \given \pt_{\mathrm{CoT}}(n)) , \PP_{\hat \rho}(y^{\mathrm{test}}=\cdot \given \pt_{\mathrm{CoT}}(n)) \bigr) \\
    &\quad \leq \sum_{h=1}^H \E_{{z_{1:h-1}^\mathrm{test} \sim \PP(\cdot \given \pt_\mathrm{CoT}(n))}} \bigg[2(3+b^*) \cdot \TV\bigl(\PP(z_{h}^{\mathrm{test}} = \cdot \given \pt_{\mathrm{CoT}}^h(n)) , \PP_{\hat \rho}(z_{h}^{\mathrm{test}} =\cdot \given \pt_{\mathrm{CoT}}^h(n)) \bigr) \bigg]. \notag 
\end{align}
The number $b^*$ comes from \eqref{eq: log density difference}, which upper bounds the log density difference $|\log \PP(z|S)-\log \PP_{\hat \rho}(z|S)$ for any $z\in \calL, S\in \calL^*$.

We take the expectation of \eqref{eq:tv_sum} with respect to \(\pt_\mathrm{CoT}(n) \sim \PP_\mathrm{CoT}\). Notice that different parts of $\pt_\mathrm{CoT}^h(n)$ have different distributions: \(z_{1:(h-1)}^\mathrm{test} \sim \PP(\cdot \mid \pt_\mathrm{CoT}(n))\), \(z_0^\mathrm{test} \sim \mu(\cdot \mid \Upsilon_n)\),  \(\Upsilon_n \sim \PP(\cdot\given \theta^*)\). We take these expectations sequentially:
    \begin{align}\label{eq:bd_1_tv}
        &\sum_{h=1}^H\E_{ \theta^*} \E_{\mu}\E_{z_{1:(h-1)}^\mathrm{test}\sim \PP(\cdot\given \pt_\mathrm{CoT}(n))} \Big[\KL\bigl(\PP(z_{h}^{\mathrm{test}}=\cdot \given \pt_{\mathrm{CoT}}^h(n)) , \PP_{\hat \rho}(z_{h}^{\mathrm{test}}=\cdot \given \pt_{\mathrm{CoT}}^h(n)) \bigr)\Big]\nonumber \\
        &\quad \leq \sum_{h=1}^H 2(3+b^*) \cdot \int_{\calL^*}  \Big(\TV\bigl(\PP(z_{h}^{\mathrm{test}}=\cdot \given \pt_{\mathrm{CoT}}^h(n)) , \PP_{\hat \rho}(z_{h}^{\mathrm{test}}=\cdot \given \pt_{\mathrm{CoT}}^h(n)) \bigr)\\
        &\quad \qquad  \cdot \pi(\theta^*)^{-1}\cdot \PP(\Upsilon_n) \cdot \kappa\cdot \PP(z_0^\mathrm{test}=\cdot\given \Upsilon_n)  \cdot \PP(z_{1:(h-1)}^\mathrm{test}=\cdot\given \pt_\mathrm{CoT}(n))\Big)\text{d} \pt_{\mathrm{CoT}}^h(n)\nonumber \\
        &\quad = 2(3+b^*) \kappa \pi(\theta^*)^{-1}\sum_{h=1}^H\E_\PP \Big[\TV\bigl(\PP(z_{h}^{\mathrm{test}}=\cdot \given \pt_{\mathrm{CoT}}^h(n)) , \PP_{\hat \rho}(z_{h}^{\mathrm{test}}=\cdot \given \pt_{\mathrm{CoT}}^h(n)) \bigr)\Big] \nonumber
    \end{align}
    with probability at least $1-\delta$. 
    In the first inequality, we integrate over $\pt_{\mathrm{CoT}}^h(n) \in \cL^*$. 
    This inequality is a result of \eqref{eq:tv_sum}, which transforms the KL distances into TV distances, and a change of distributions from  $\PP_{\mathrm{CoT}}$ to $\PP$. 
    Notice that we have $\PP(\Upsilon_n\given \theta^*)\pi(\theta^*)\leq \PP(\Upsilon_n)$ due to the discreteness of $\Theta$. 
    Also note that, by Assumption~\ref{assump: cover},  $\mu (z_{0}^\mathrm{test} = \cdot  \given \Upsilon_n) $ can be bounded by $\PP(z_{0}^\mathrm{test} = \cdot  \given \Upsilon_n)  $ by introducing an additional factor $\kappa$. 
    Consequently, \eqref{eq:bd_1_tv} shifts the expectation under \(\PP_\mathrm{CoT}\) towards \(\PP\) by scaling some constants.

Next, we upper bound \eqref{eq:bd_1_tv} using the analysis of pretraining error in Proposition~\ref{prop: pre-train}.
Recall that we introduce the notation $\E_{S\sim \calD}$ in \eqref{eq:avg_ex}, which involves an expectation over $N$ i.i.d. documents, each has $T$ examples. 
Since $ \pt_{\mathrm{CoT}}(n)$ only has $n$ examples and $T \geq n+1$ and the $N$ documents in $\cD_{N,T}$ are i.i.d., 
we have 
    \begin{align}\label{eq:bd_2_tv} 
        &\sum_{h=1}^H\E_\PP \Big[\TV\bigl(\PP(z_{h}^{\mathrm{test}} =\cdot\given \pt_{\mathrm{CoT}}^h(n)) , \PP_{\hat \rho}(z_{h}^{\mathrm{test}} =\cdot\given \pt_{\mathrm{CoT}}^h(n)) \bigr)\Big]  \\
        & \quad \leq \sum_{t=1}^T \sum_{h=1}^H \Big[\TV\bigl(\PP(z_{h}^{ t, 1} =\cdot\given  S_{h}^{t,1}) , \PP_{\hat \rho}(\PP(z_{h}^{ t, 1}  =\cdot\given  S_{h}^{t,1} ) \bigr)\Big] 
        \leq T \cdot (H+1) \cdot \Delta_{\mathrm{pre}}(N,T,\delta), \nonumber
    \end{align}
with probability with at least $1-\delta$. 
Here the first inequality holds because the left-hand side is only a single term in the right-hand side summation with $t = n$, and the second inequality follows from  Proposition~\ref{prop: pre-train}, 
Combing \eqref{eq:bd_1_tv} and \eqref{eq:bd_2_tv}, we upper bound the expectation of pretraining error as 
    \begin{align}
        &\E_{\PP_\mathrm{CoT}}\big[\mathtt{err}_\mathrm{pre}(\PP,\PP_{\hat \rho};\pt_\mathrm{CoT}(n))\big] \nonumber \\
        &\quad =\E_{\PP_\mathrm{CoT}}\Big[\KL\bigl(\PP(y^{\mathrm{test}} =\cdot \given \pt_{\mathrm{CoT}}(n)) , \PP_{\hat \rho}(y^{\mathrm{test}}=\cdot \given \pt_{\mathrm{CoT}}(n)) \bigr)\Big]\nonumber\\
        &\quad \leq  2(3+b^*) \kappa \cdot \pi(\theta^*)^{-1} \cdot T(H+1) \cdot \Delta_{\mathrm{pre}}(N,T,\delta), \label{eq:pretrain_bd}
    \end{align}
    with probability with at least $1-\delta$. The randomness comes from the pretrained model $\PP_{\hat \rho}$. This concludes \textbf{Step 1}.


\vspace{2mm}

\noindent \textbf{Step 2: Control expected prompting error under distribution shift.} In this step, we evaluate the expected prompting error \(\mathtt{err}_\mathrm{prompt}\) \eqref{eq: err pt} with the expectation taken under \(\pt_\mathrm{CoT}(n) \sim \PP_\mathrm{CoT}\). Similar to {\bf Step 1}, we first take the expectation of the shifted testing query \(z_0^\mathrm{test} \sim \mu(\cdot \mid \pt_\mathrm{CoT}(n))\), followed by the expectation of the demonstrations \(\Upsilon_n \sim \PP(\cdot \given \theta^*)\).

We first compute the expected KL divergence with respect to the query $z_0^\mathrm{test}\sim \mu(\cdot \given \Upsilon_n)$:
\begin{align}
    &\E_{ z_0^\mathrm{test}\sim \mu(\cdot\given \Upsilon_n)}\bigg[\KL \big(\PP(y^{\mathrm{test}} =\cdot \given z_0^{\mathrm{test}} ,\theta^*), \PP(y^{\mathrm{test}}=\cdot \given \pt_{\mathrm{CoT}}(n))\big)\bigg] \nonumber \\
    &\quad \leq \E_{z_0^\mathrm{test}\sim \mu(\cdot\given \Upsilon_n)}\bigg[\log\bigg(1+\frac{\sum_{\theta \in \Theta^\complement}\PP(\Upsilon_n, z_0^\mathrm{test}\given \theta)\pi(\theta)}{\PP(\Upsilon_n, z_0^\mathrm{test}\given \theta^*)\pi(\theta^*)}\bigg)\bigg] \nonumber\\
    &\quad \leq \log\bigg(1+\E_{z_0^\mathrm{test}\sim \mu(\cdot\given \Upsilon_n)}\bigg[\frac{\sum_{\theta \in \Theta^\complement}\PP(\Upsilon_n, z_0^\mathrm{test}\given \theta)\pi(\theta)}{\PP(\Upsilon_n, z_0^\mathrm{test}\given \theta^*)\pi(\theta^*)}\bigg]\bigg). \label{eq:log_bd}
\end{align}
The first inequality follows from Proposition~\ref{prop:elb}, and the second is due to Jensen's inequality. Next, we rewrite the expected likelihood ratio in \eqref{eq:log_bd} as
\begin{align}
&\E_{z_0^\mathrm{test}\sim \mu(\cdot\given \Upsilon_n)}\bigg[\frac{\sum_{\theta \in \Theta^\complement}\PP(\Upsilon_n, z_0^\mathrm{test}\given \theta)\pi(\theta)}{\PP(\Upsilon_n, z_0^\mathrm{test}\given \theta^*)\pi(\theta^*)} \bigg]\nonumber\\
&\quad = \sum_{\theta \in \Theta^\complement}\frac{\PP(\Upsilon_n\given \theta)}{\PP(\Upsilon_n\given \theta^*)} \cdot \E_{z_0^\mathrm{test}\sim \mu(\cdot\given \Upsilon_n)}\bigg[\frac{\PP(z_0^\mathrm{test} \given \theta)\pi(\theta)}{\PP(z_0^\mathrm{test} \given \theta^*)\pi(\theta^*)} \bigg], \label{eq: log_sum}
\end{align}
where the equality follows from the independence between $\Upsilon_n$ and $z_0^\mathrm{test}$ conditioning on any task $\theta$,
which enables us to exchange $\sum_{\theta \in \Theta^{\complement} }  $ with the expectation with respect to $\mu (\cdot \given \Upsilon_n) $.
According to Lemma~\ref{lemma: likelihoodbd} and Assumption~\ref{assumption: seperation}, we have $\PP(\Upsilon_n\given \theta)/\PP(\Upsilon_n\given \theta^*)\leq \exp(-2n\lambda+2\log(\bar \xi^{-1}))$ with probability at least $1-\bar \xi$ for any $\theta\in \Theta^\complement$. 
Setting $\bar \xi = \xi / |\Theta | $ and taking a union bound,   we therefore have
\begin{align}
\mathrm{RHS~of}~\eqref{eq: log_sum}\leq \big|\Theta^\complement \big|^2 \xi^{-2} \cdot \exp(-2n\lambda) \cdot \sum_{\theta \in \Theta^\complement}\E_{z_0^\mathrm{test}\sim \mu(\cdot\given \Upsilon_n)}\bigg[\frac{ \PP(z_0^\mathrm{test}\given \theta) \pi(\theta)}{\PP( z_0^\mathrm{test}\given \theta^*)\pi(\theta^*)}\bigg], \label{eq:log_bd_2}
\end{align}
which holds with  probability at least $1-\xi$ with respect to the  randomness comes from $\Upsilon_n\sim \PP(\cdot \given \theta^*)$. 

Next we upper bound each term in the summation in \eqref{eq:log_bd_2}.
By changing the probability measure and Cauchy-Schwarz inequality, we have 
\begin{align}
   \E_{z_0^\mathrm{test}\sim \mu(\cdot\given \Upsilon_n)}\bigg[\frac{\PP(z_0^\mathrm{test}\given \theta)}{\PP( z_0^\mathrm{test}\given \theta^*)}\bigg]&= \E_{z_0^\mathrm{test}\sim \PP(\cdot\given \theta^*)}\bigg[\frac{\mu(z_0^\mathrm{test}\given \Upsilon_n)\cdot \PP(z_0^\mathrm{test}\given \theta)}{\PP( z_0^\mathrm{test}\given \theta^*) \cdot \PP( z_0^\mathrm{test}\given \theta^*) }\bigg] \nonumber \\
   &\leq \sqrt{\E_{z_0^\mathrm{test}\sim \PP(\cdot\given \theta^*)}\left[\bigg(\frac{\PP(z_0^\mathrm{test}\given \theta)}{\PP( z_0^\mathrm{test}\given \theta^*)}\bigg)^2 \right]}\sqrt{\E_{z_0^\mathrm{test}\sim \PP(\cdot\given \theta^*)}\left[\bigg(\frac{\mu(z_0^\mathrm{test}\given \Upsilon_n)}{\PP( z_0^\mathrm{test}\given \theta^*)}\bigg)^2 \right]} \nonumber \\
   &=\sqrt{\big(\chi^2\big(\PP( z_0^\mathrm{test} =\cdot \given \theta), \PP( z_0^\mathrm{test}=\cdot \given \theta^*)\big)+1\big) \cdot \kappa^2}, \label{eq:log_bd_3}
\end{align}
where the second line follows from the Cauchy–Shwarz inequality. The final line follows because $$\E_{ z_0^\mathrm{test}\sim \PP(\cdot \mid \theta^*)}\left[ \biggl( 
\frac{\mu(z_0^\mathrm{test} \mid \Upsilon_n)}{\PP( z_0^\mathrm{test} \mid \theta^*)} \bigg)^2  \right] \leq \kappa^2,$$ which results from Assumption~\ref{assump: cover}.
Besides, we  define $C(\theta^*)$ as  
\begin{align}
    C(\theta^*)=\sup_{\theta\in \Theta^\complement}\sqrt{\chi^2\big(\PP( z_0^\mathrm{test}=\cdot\given \theta), \PP( z_0^\mathrm{test}=\cdot\given \theta^*)\big)+1}.  \label{eq:c_def}
\end{align}
Thus, combining \eqref{eq:log_bd_2}, \eqref{eq:log_bd_3},   \eqref{eq:c_def}, with probability at least  $1-\xi$,  we have 
\begin{align}
    \mathrm{RHS~of}~ \eqref{eq:log_bd_2}\leq C(n)\cdot \xi^{-2},  \text{ where }C(n) = \frac{\pi(\Theta^\complement)}{\pi(\theta^*)}\cdot \kappa\big| \Theta^\complement \big|^2 \cdot C(\theta^*) \cdot \exp(-2n\lambda), \label{eq:c_n}
\end{align}
and $C(\theta^*)$ is defined in \eqref{eq:c_def}.

Applying \eqref{eq:c_n} to \eqref{eq:log_bd} gives us the following tail probability bound:
\begin{align}
    \PP \bigg(\E_{ z_0^\mathrm{test}\sim \mu(\cdot\given \Upsilon_n)}\Big[\KL \big(\PP(y^{\mathrm{test}}=\cdot \given z_0^{\mathrm{test}} ,\theta^*), \PP(y^{\mathrm{test}}=\cdot \given \pt_{\mathrm{CoT}}(n))\big)\Big]>\log \big(1+C(n) \xi^{-2} \big)\bigg)< \xi, \label{eq:tail_bd}
\end{align}
where $\xi\in (0,1)$, and the randomness comes from $\Upsilon_n \sim \PP(\cdot \given \theta^*)$.
By replacing $ x = \log ( 1 + C(n) \cdot \xi ^{-2} ) $ in \eqref{eq:tail_bd}, 
for any $x \in [0, \infty)$, 
we have 
\begin{align}
    \label{eq:tail_bd_2}
   &  \PP \bigg(\E_{ z_0^\mathrm{test}\sim \mu(\cdot\given \Upsilon_n)}\Big[\KL \big(\PP(y^{\mathrm{test}}=\cdot \given z_0^{\mathrm{test}} ,\theta^*), \PP(y^{\mathrm{test}}=\cdot \given \pt_{\mathrm{CoT}}(n))\big)\Big] > x \bigg) \notag \\
    & \qquad  \leq \bigl ( C(n) /  ( \exp(x) - 1 ) \bigr )^{1/2}. 
 \end{align}

We provide an upper bound of the expected KL divergence by integrating the tail probability  in \eqref{eq:tail_bd_2} as follows, 
\begin{align}
    &\E_{\PP_{\mathrm{CoT}}}\Big[\KL \big(\PP(y^{\mathrm{test}}=\cdot \given z_0^{\mathrm{test}} ,\theta^*), \PP(y^{\mathrm{test}}=\cdot \given \pt_{\mathrm{CoT}}(n))\big)\Big] \nonumber \\
    &\quad = \int_{0}^{\infty}\PP \bigg(\E_{ z_0^\mathrm{test}\sim \mu(\cdot\given \Upsilon_n)}\Big[\KL \big(\PP(y^{\mathrm{test}} =\cdot\given z_0^{\mathrm{test}} ,\theta^*), \PP(y^{\mathrm{test}} =\cdot\given \pt_{\mathrm{CoT}}(n))\big)\Big]>x\bigg)\text{d}x\nonumber.  
\end{align}
Note that we can split the integration of $x$ over $[0,\infty)$ into two regions: $[0,\log(1+C(n))$ and $[\log(1+C(n),\infty)$, where the probability in~\eqref{eq:tail_bd_2} is bounded by one  in  $[0,\log(1+C(n))$. Therefore, we have
\begin{align}
& \E_{\PP_{\mathrm{CoT}}}\Big[\KL \big(\PP(y^{\mathrm{test}}=\cdot \given z_0^{\mathrm{test}} ,\theta^*), \PP(y^{\mathrm{test}}=\cdot \given \pt_{\mathrm{CoT}}(n))\big)\Big] \nonumber \\
    &\quad \leq \log\big(1+C(n)\big) + \int_{\log(1+C(n))}^\infty \bigl ( C(n) /  ( \exp(x) - 1 ) \bigr )^{1/2}  \text{d} x \nonumber \\
    &\quad =\log\big(1+C(n)\big) + C(n)^{1/2} \cdot \Big(\pi-2\arctan  \big(C(n)^{1/2}\big) \Big) = \cO\big ( C(n)^{1/2} \big)  \nonumber \\
    &\quad =\calO\bigg(\bigg(\frac{\pi(\Theta^\complement)}{\pi(\theta^*)}\cdot C(\theta^*)\cdot \kappa\bigg)^{1/2}\cdot \big| \Theta^\complement \big| \cdot \exp(-n\lambda)\bigg). \label{eq:pt_bd}
\end{align}
Here in the third line we plug in the closed-form   $$\int \bigl (\exp(x) - 1 \big ) ^{-1/2} \ud x = 2 \arctan  \big (\sqrt{\exp(x) - 1 } \big).  $$
When $n$ is large, $C(n) $ is sufficiently small. 
The second equality  follows from the first-order Taylor approximations $\log (1 + u) \approx u $ and $\arctan (u) \approx u$ when $u$ is close to zero. 

The convergence rate in the last line is dominated by the rate of $C(n)^{1/2}$.
Therefore, we control the rate of expected prompting error defined in  \eqref{eq: err pt} by applying \eqref{eq:pt_bd}: 
\begin{align}
    \E_{\PP_\mathrm{CoT}}\Big[\mathtt{err}_\mathrm{prompt}\big(\PP,\theta^*,\pt_\mathrm{CoT}(n)\big)\Big] = \calO\bigg(Hb^*\bigg(\frac{\pi(\Theta^\complement)}{\pi(\theta^*)}C(\theta^*)\kappa\bigg)^{1/4}\cdot \big| \Theta^\complement \big|^{1/2}\cdot \exp(-n\lambda/2)\bigg). \label{eq:err_pt_bd}
\end{align}

Combining \eqref{eq:pretrain_bd} and \eqref{eq:err_pt_bd}, we have that under the Assumptions~\ref{assumption: seperation}, \ref{assump: cover}, with probability at least $1-\delta$, 
\begin{align*}
        &\E_{\PP_{\mathrm{CoT}}} \Big[\KL\bigl(\PP(y^{\mathrm{test}}=\cdot \given z_0^{\mathrm{test}} ,\theta^*) , \PP_{\hat \rho}(y^{\mathrm{test}}=\cdot \given \pt_{\mathrm{CoT}}(n)) \bigr) \Big] \nonumber \\
        &\quad  = \calO\bigg(Hb^* \bigg(\frac{\pi(\Theta^\complement)}{\pi(\theta^*)}\cdot C(\theta^*)\cdot \kappa\bigg)^{1/4}\cdot \big| \Theta^\complement \big|^{1/2} \cdot \exp(-n\lambda/2) + \kappa T H \pi(\theta^*)^{-1}(1+b^*) \cdot \Delta_{\mathrm{pre}}(N,T,\delta) \bigg),
\end{align*}
where the first term corresponds to the prompting error $\mathtt{err}_\mathrm{prompt}$ \eqref{eq: err pt}, and the second corresponds to the pre-training error $\mathtt{err}_\mathrm{pre}$ \eqref{eq: err pre}. The randomness comes from the pretrained model $\PP_{\hat \rho}$. Therefore, we conclude the proof.
\end{proof}

\newpage 

\section{Supplemental Materials}
This section consists of three subsections. The first Subsection~\ref{app: pre-training process} gives a more detailed description of the pretraining process. The Subsections~\ref{app: supp for main thms} and \ref{app: pretrain supp} prove the lemmas and propositions used in Sections~\ref{sec: main app} and \ref{app: pretrain}.

\subsection{Additional Details about Pretraining} \label{app: pre-training process}
In this section, we provide a detailed description of the pretraining process of autoregressive \ac{llm}. Specifically, we focus on pretraining with data sampled from the generalized model described in \eqref{eq:generalized_latent_var_model}. An autoregressive \ac{llm} is a transformer that maps a reasoning step sequence $S\in \calL^*$ to a probability distribution for predicting the next reasoning step $z\in \calL$.

\vspace{2mm}
\noindent \textbf{Transformer Architecture.} 
We focus on a transformer with   $D$  transformer blocks stacked sequentially followed by a final softmax layer. Let $X^0\in \mathbb{R}^{L\times r}$ denote the initial input embedding for the entire network, which contains both the content embedding and the positional encoding. 
The $d$-th block takes in $X^{d-1}\in \mathbb{R}^{L\times r}$, produces $X^{d}\in \mathbb{R}^{L\times r}$, and feeds it to the next module until arriving at the last one. Each transformer block consists of four components: a \ac{mha} and a \ac{ff} layer. Each component has a residual connection around it, followed by layer normalization, which prepares the raw output of the current layer to be forwarded as input toward the next layer. See Figure~\ref{fig:pre-train-intext} for an illustration of the architecture. 

\vspace{2mm}

{\bf \noindent  Input Embedding.}
Specifically, the input of the transformer is a sequence of reasoning steps of length $L$, with each step taking values in $\cL$.
Since the attention mechanism is permutation invariant but the sequential order matters in \ac{cot} reasoning, to encode an order, the transformer incorporates positional embeddings that map the positional information of each reasoning step into the Euclidean space. In addition, the values in $\cL$ are also mapped to a vector space. 
Thus, the transformer first maps the input sequence of length $L$ into a sequence of $L$ vectors in $\RR^r$, which involves both content and positional embedding.

\vspace{2mm} 
{\bf \noindent Multi-Head Attention (MHA).}  We let $X^{0} \in \RR^{L\times r}$ denote the output after the embedding module, which is passed to $D$ transformer blocks.   
Each block 
consists of a MHA layer, a FF layer, and two normalization layers. 
For any \(d \in [D]\), the parameters of the $d$-th transformer block are
 \(\rho^d = (W^d_\mathrm{mha}, W^d_\mathrm{ff}, \gamma^d_1, \gamma^d_2)\). The weight matrices of the \ac{mha} layer are \(W_{\mathrm{mha}}^d = (W^{Q,d}_i, W^{K,d}_i, W^{V,d}_i)_{i=1}^\eta\), where \(\eta\) is the number of heads. Here, \(W^{Q,d}_i \in \mathbb{R}^{r \times d_q}\), \(W^{K,d}_i \in \mathbb{R}^{r \times d_k}\), and \(W^{V,d}_i \in \mathbb{R}^{r \times d_v}\) convert the input \(X^{d-1}\) into queries, keys, and values, respectively. We set \(d_v = r\) to ensure the \ac{mha} output is also in $\RR^r$. 
Specifically, a MHA layer with $\eta$ heads output a vector sequence given by \eqref{eq: mha}. We denote the output of the $d$-th attention layer by 
\(\overline  X^d = \mathtt{mha}(X^{d-1}, W_{\mathrm{mha}}^d) \in \mathbb{R}^{L \times r}\). In particular, for any $\ell \in [L]$, the $\ell$-th vector of $\overline  X^d$ is given by 
\begin{align} \label{eq:attention_output}
\overline X^d[\ell] = \sum_{i=1}^\eta 
\mathtt{attn}(q_{i}^\ell, K_i, V_i),
\end{align} 
where the query, key, and value  of the $i$-th head are $q_{i}^\ell = (W^{Q,d}_i )^\top X^{d-1} [\ell] $, $K_i = X^{d-1}  W^{K,d}_i$, and $V = X^{d-1}W^{V,d}_i $. 
Here $\mathtt{attn}$ in \eqref{eq:attention_output} is the softmax attention defined in \eqref{eq: softmax attn}.

\vspace{2mm} 
{\bf \noindent First Residual Link and Normalization.}  The raw output of MHA layer is then passed through a residual link with diagonal weight matrix \(\gamma_1^d \in \RR^{r\times r} \) and a normalization layer \(\mathtt{NL}(\cdot)\), resulting in the intermediate output 
$$
Y^d = \mathtt{NL}(\overline {X}^d +  X^{d-1}\gamma_1^d) = \mathtt{NL} \bigl (\mathtt{mha}(X^{d-1}, W_{\mathrm{mha}}^d) + X^{d-1}\gamma_1^d \big) \in \mathbb{R}^{L \times r}.$$
Here the multiplication of $\gamma_1^d$ should be understood as a columnwise operation for all $r' \in [r]$. Note that each of the $r$ vector of $X^{d-1}$ is in $\RR^{L}$, which is mapped to another vector in $\RR^{L}$ by scaling with $\gamma_1^d[r']$. 
For the ease of analysis, we adopt the normalization function that maps each row of the input into the unit $\ell_{2}$-ball as follows.
\begin{align}
    [\mathtt{NL}(X)]_{i,:} = \begin{cases}
X_{i,:}& \text{ if }\|X_{i,:}\|_{2}\leq 1 \\
X_{i,:}/\|X_{i,:}\|_{2}& \text{ otherwise.} \label{eq:nl}
\end{cases}
\end{align}
Another popular normalization function is layer normalization \citep{xiong2020layer}, which standardizes the vectors of $\overline {X}^d +  X^{d-1}\gamma_1^d$ by subtracting the mean and dividing by the variance. 

\vspace{2mm} 
{\bf \noindent Feed Forward (FF) Layer.} 
The \ac{ff} layer is parameterized by 
\(W^d_\mathrm{ff} = (W^d_\mathrm{ff,1}, W^d_\mathrm{ff,2})\), where \(W^d_{\mathrm{ff,1}} \in \mathbb{R}^{r \times d_F}\) and \(W^d_{\mathrm{ff,2}} \in \mathbb{R}^{d_F \times r}\), where $d_F $ is the number of neurons of the FF layer. 
In particular,  \(Y^d\) is passed through the FF  layer and the output is another sequence of vectors in $\RR^{L\times d}$. 
To get the output, for any $\ell \in [L]$, we pass $Y^d[\ell]$ into a two layer neural network and obtain 
$
\mathtt{ReLU} \bigl(  (Y^d[\ell]) ^\top (W^d_\mathrm{ff,1} ) \big) W^d_\mathrm{ff,2},
$
where $\mathtt{ReLU}(\cdot)$ is the ReLU activation function and we regard $Y^d[\ell]$ as a column vector in $\RR^{r}$. 
Here we omit the intercepts to simplify the presentation. 
The output of FF  layer, denoted by $\mathtt{ff}(Y^d, W_{\mathrm{ff}}^d)$, concatenates all these $L$ output vectors.

\vspace{2mm} 
{\bf \noindent Second Residual Link and Normalization.}  The  output of FF layer is then passed through a second residual link with weight \(\gamma_2^d \in \RR^{r\times r} \) and a normalization layer \(\mathtt{NL}(\cdot)\).  This concludes the $d$-th transformer block and the resulting output is given by 
$$
 X^d = \mathtt{NL} \big (\mathtt{ff}(Y^d, W_{\mathrm{ff}}^d) + Y^d\gamma_2^d  \big )   \in \mathbb{R}^{L \times r}.
$$


{\noindent \bf Softmax Output Layer} 
After processing through all \(D\)  transformer blocks, the output \(X^D\) is fed into a softmax layer to generate the probability distribution of the next reasoning step. This softmax layer is parameterized by \(\rho^{D+1} = (W_\mathrm{softmax}, \tau)\), where \(0 < \tau \leq 1\) is the temperature parameter and \(W_\mathrm{softmax} \in \mathbb{R}^{r \times |\calL|}\) is the weight matrix. The softmax layer takes \(X^D \in \mathbb{R}^{L \times r}\) and produces the output distribution 
\begin{align}
    O^{D+1} = \mathtt{softmax}\bigl(\mathbf{1}^\top X^D W_\mathrm{softmax} / (L \cdot\tau) \bigr) \in \mathbb{R}^{|\calL|}, \label{eq:softmax_layer}
\end{align}
where \(\mathbf{1} \in \mathbb{R}^L\) is an all-one  vector. Here $\mathtt{softmax}(\cdot)$ is the softmax function which maps a vector in $\mathbb{R}^{|\calL|}$ to a distribution over $\cL$. 

We concatenate all $\rho^d$, $d\in [D+1]$, to form $\rho$, which parameterizes the whole pretraining network. We assume the parameters are bounded, i.e., we consider transformers in the following parameter space:
\begin{align*}
    \cP_{\mathrm{LLM}} &= \Big\{\rho : \| \gamma_1^d\|_\infty, \|\gamma_2^d\|_\infty \leq 1, \|W^{Q, d}_i\|_F ,\|W^{K, d}_i\|_F ,\|W^{V, d}_i\|_F \leq B_M, \| 
    W_{\mathrm{ff},1}^{d}\|_F,\|W_{\mathrm{ff},2}^{d}\|_F \leq B_{F}, \\
    & \qquad  
     \|W_{\mathrm{softmax}}\|_{1,2}\leq B_S, \forall d \in [D], i \in [\eta], D\geq C\log(2H), \eta\geq 1 \Big\},
\end{align*}
where $C>0$ is a constant, $B_M, B_{F}, B_S$ are the upper bounds on the norms. We assume these bounds to be larger than $1$. 

\vspace{2mm}

\noindent \textbf{Pretraining Dataset under the Generalized Model.} We describe the pretraining dataset generated according to the generalized model in \eqref{eq:generalized_latent_var_model}. We denote the pretraining dataset using $\calD_{N,T}$, which consists of $N$ independent trajectories with $T$ examples in each trajectory. For each trajectory $\ell\in [N]$, we first sample a task $\theta_\ell^* \overset{i.i.d}{\sim} \pi$. Conditioning on this task, we sequentially generate $T$ examples $\{s^{k,\ell}\}_{k=1}^T$ according to the model \eqref{eq:generalized_latent_var_model}, i.e., we iteratively generate the next reasoning step $z_h^{t,\ell} \sim \PP(\cdot \given S_h^{t,\ell})$, where we use $S_h^{t,\ell} =(\Upsilon_{t-1,\ell},\{z_j^{t,\ell}\}_{j=0}^{h-1})$ to denote the sequence with all previous reasoning steps of the $\ell$-th trajectory. Since \ac{llm}s make prediction autoregressively, we divide each trajectory into $T(H+1)$ pieces and collect all $N$ independent trajectories and use $\mathcal{D}_{N,T} = \{(S_h^{t,\ell} ,z_{h}^{t,\ell})\}_{h=0,t=1, \ell=1}^{H, T, N}$ to denote the pretraining dataset. 

\vspace{2mm}
{\noindent \bf Maximum Likelihood Estimation (MLE).} 
 We obtain the pretrained \ac{llm} by minimizing the negative likelihood loss computed based on $\calD_{N,T}$,
\begin{align}
 \hat \rho = \argmin_{\rho \in \cP_{\mathrm{LLM}}} -\frac{1}{NT (H+1)}\sum_{\ell=1}^N \sum_{t=1}^T \sum_{h=0}^{H}\log \PP_\rho \big (z_{h}^{t,\ell} \given S_h^{t,\ell} \big ) \label{eq:pretrain_rho_appendix}
\end{align}
and  set $\PP_\mathrm{LLM} = \PP_{\hat \rho}$. Here $\PP_\rho $ denotes the conditional distribution specified by the transformer with parameter $\rho$. 
We neglect the optimization issue and assume that the MLE in \eqref{eq:pretrain_rho_appendix} can be obtained.  
We note that when the transformer class is sufficiently expressive, we expect that $\PP_\mathrm{LLM} $ learns the conditional distribution of \(z_h^{t,\ell}\) given   \(S_h^{t,\ell}\), which is given in \eqref{eq:bayes_factorization}.

\subsection{Proofs of the Auxiliary Results in  Appendix~\ref{proof: rate_cot}} \label{app: supp for main thms}

\subsubsection{Proof of Proposition \ref{prop:elb}}\label{proof: elb}

\begin{proof}[Proof of Proposition \ref{prop:elb}] 
Using  any distribution $q $ over $\Theta$, we  bound  the loglikelihood $\log \PP\big(y^{\mathrm{test}} = \cdot \given \mathtt{prompt}_\mathrm{CoT}(n)\big)$ by  
\begin{align}\label{eq: elb lower bd}
    \log \PP\big(y^{\mathrm{test}} = \cdot  \given \mathtt{prompt}_\mathrm{CoT}(n)\big)&\geq \E_{\theta \sim q(\theta)} \bigg[\log \frac{\PP\big(y^{\mathrm{test}},\theta \given \mathtt{prompt}_\mathrm{CoT}(n)\big)}{q(\theta)}\bigg]  \\
    &=\E_{q}\bigg[\log \frac{\PP\big(y^{\mathrm{test}} \given \theta, \mathtt{prompt}_\mathrm{CoT}(n) \big)\pi(\theta)\PP \big(\mathtt{prompt}_\mathrm{CoT}(n) \given \theta \big)}{q(\theta)\PP\big(\mathtt{prompt}_\mathrm{CoT}(n) \big)}\bigg], \nonumber 
\end{align}
where we take expectation with respect to an arbitrary distribution $q$ over $\Theta$. The inequality follows from the evidence lower bound, and the equality follows from decomposing the numerator. Conditioning on any $\mathtt{prompt}_\mathrm{CoT}(n)$ consisting of $n$ examples generated from the true distribution and a new testing input $z_0^{\mathrm{test}}$, we compute the KL divergence with respect to the final output $y^{\mathrm{test}}$. For simplicity, we write  $\E_{y^{\mathrm{test}} \sim \PP(\cdot \given z_0^{\mathrm{test}},\theta^*)}$ as $\E_{\theta^*}$. Then, we have 
\begin{align}\label{eq:proof_prop_c2_1}
& \KL \Big(\PP(y^{\mathrm{test}}=\cdot \given z_0^{\mathrm{test}},\theta^*) , \PP\big(y^{\mathrm{test}}=\cdot \given \mathtt{prompt}_\mathrm{CoT}(n)\big) \Big)\\
&\quad \leq \E_{\theta^*} \log \PP(y^{\mathrm{test}} \given z_0^{\mathrm{test}},\theta^*) -\E_{ \theta^*}\E_{q}\bigg[\log \frac{\PP\big(y^{\mathrm{test}} \given \theta, \mathtt{prompt}_\mathrm{CoT}(n) \big)\pi(\theta)\PP \big(\mathtt{prompt}_\mathrm{CoT}(n) \given \theta \big)}{q(\theta)\PP\big(\mathtt{prompt}_\mathrm{CoT}(n) \big)}\bigg] \nonumber\\
& \quad = \E_{ \theta^*}\E_{\theta \sim  q} \bigg[\log \frac{\PP\big(\mathtt{prompt}_\mathrm{CoT}(n)\big) q(\theta)}{\PP(\mathtt{prompt}_\mathrm{CoT}(n)\given \theta)\pi(\theta)} \bigg]- \E_{\theta^*}\E_{\theta \sim q} \bigg[ \log \frac{\PP\big(y^{\mathrm{test}} \given z_0^{\mathrm{test}},\theta\big)}{\PP(y^{\mathrm{test}} \given z_0^{\mathrm{test}},\theta^*)} \bigg]. \nonumber
\end{align}
The  inequality follows from the definition of KL divergence and the lower bound in \eqref{eq: elb lower bd}.
The equality follows from the fact that $\PP (y^{\mathrm{test}} \given \theta, \mathtt{prompt}_\mathrm{CoT}(n)  ) = \PP (y^{\mathrm{test}} \given z_0^{\mathrm{test}},\theta )$ under the model in \eqref{eq:latent_var_model} and 
 rearranging terms.

Now we set  $q$ as 
$q(\theta) \propto \mathbf{1}\{\theta \in \Theta_{\mathrm{eq}}(\theta^*)\} \cdot \pi\big(\theta \given \mathtt{prompt}_\mathrm{CoT}(n)\big)$, where $\pi\big(\theta \given \mathtt{prompt}_\mathrm{CoT}(n)\big)$ is the posterior distribution over $\Theta$ after observing the prompt. 
Note that this $q$ assigns zero probability to any $\theta$ outside the equivalence class $\Theta_{\mathrm{eq}}(\theta^*)$. 
By taking an expectation with respect to this $q$, we have  
$$ \E_{\theta^*}\E_{\theta \sim q} \big[ \log\PP(y^{\mathrm{test}} \given z_0^{\mathrm{test}},\theta)-\log\PP(y^{\mathrm{test}} \given z_0^{\mathrm{test}},\theta^*) \big] =0  $$ 
by the construction of the equivalence classes in Definition \eqref{def: eqv class}. 
Thus the KL divergence in \eqref{eq:proof_prop_c2_1} is further bounded as follows: 
\begin{align}\label{eq:proof_prop_c2_2}
& \KL \Big(\PP\big(y^{\mathrm{test}}=\cdot \given z_0^{\mathrm{test}},\theta^*\big) , \PP\big(y^{\mathrm{test}} =\cdot\given \mathtt{prompt}_\mathrm{CoT}(n) \big)\Big) \notag \\
& \quad \leq  \E_{\theta^*}\E_{\theta \sim q} \bigg[ \log \frac{\PP\big(\mathtt{prompt}_\mathrm{CoT}(n)\big) q(\theta) }{\PP \big(\mathtt{prompt}_\mathrm{CoT}(n)\given \theta\big)\cdot \pi(\theta) } \bigg] \notag \\
& \quad =  \E_{\theta^*}\E_{\theta \sim q} \bigg[\log  \frac{\PP\big(\mathtt{prompt}_\mathrm{CoT}(n)  \big) \cdot \pi\big(\theta \given  \mathtt{prompt}_\mathrm{CoT}(n)\big)}{\PP\big(\mathtt{prompt}_\mathrm{CoT}(n) \given \theta\big)\pi(\theta)\cdot \int_{\Theta_{\mathrm{eq}}(\theta^*)} \pi\big(\theta' \given  \mathtt{prompt}_\mathrm{CoT}(n)\big) \text{d}\theta'} \bigg] \notag \\
& \quad = -  \log \Bigl (    \int_{\Theta_{\mathrm{eq}}(\theta^*)} \pi\big(\theta \given  \mathtt{prompt}_\mathrm{CoT}(n)\big) \text{d}\theta \Bigr) . 
\end{align}
Here in the first equality, we plug in the closed form of $q(\theta)$, and in the second equality, we use the fact that 
$$
\PP\big(\mathtt{prompt}_\mathrm{CoT}(n)  \big) \cdot \pi\big(\theta \given  \mathtt{prompt}_\mathrm{CoT}(n)\big) = \PP\big(\mathtt{prompt}_\mathrm{CoT}(n)  \given \theta  \big) \cdot \pi (\theta  )
$$
to cancel terms. 
We can interpret the last term in \eqref{eq:proof_prop_c2_2}  as an integrated version of posterior contraction. Intuitively, a better CoT prompt yields a higher posterior concentration on $\Theta_{\mathrm{eq}}(\theta^*)$, leading to a smaller upper bound. 
Finally, we plug in the closed-form expression of $\pi (\theta \given  \mathtt{prompt}_\mathrm{CoT}(n) )$, i.e., 
$$
\pi\big(\theta \given  \mathtt{prompt}_\mathrm{CoT}(n)\big) = \frac{ \PP\big(\mathtt{prompt}_\mathrm{CoT}(n) \given \theta \big) \pi(\theta) } { \int_{\Theta} \PP\big(\mathtt{prompt}_\mathrm{CoT}(n) \given \theta \big) \pi(\theta)  \ud \theta},
$$
in \eqref{eq:proof_prop_c2_2} and obtain that 
\begin{align*}
    &\KL\Big(\PP(y^{\mathrm{test}}=\cdot \given z_0^{\mathrm{test}},\theta^*) , \PP\big(y^{\mathrm{test}} =\cdot \given \mathtt{prompt}_\mathrm{CoT}(n)  \big)\Big)\\
    &\quad \leq  \log \frac{\int_{\Theta}\PP\big(\mathtt{prompt}_\mathrm{CoT}(n) \given \theta \big) \pi(\theta) \text{d}\theta}{\int_{\Theta_{\mathrm{eq}}(\theta^*)}\PP \big(\mathtt{prompt}_\mathrm{CoT}(n) \given \theta \big) \pi(\theta) \text{d}\theta}\\
    & \quad = \log \bigg(1+\cfrac{\int_{\Theta^{\complement}}\PP\big(\mathtt{prompt}_{\mathrm{CoT}}(n) \biggiven \theta \big) \pi(\theta) \text{d}\theta}{\int_{\Theta_{\mathrm{eq}}(\theta^*)}\PP\big(\mathtt{prompt}_{\mathrm{CoT}}(n) \biggiven  \theta \big) \pi(\theta) \text{d}\theta} \bigg).
\end{align*}
Here the equality follows from the definition of $\Theta^{\complement}$.
Therefore, we conclude the proof.
\end{proof}

\subsubsection{Proof of Lemma \ref{lemma: likelihoodbd}}\label{proof: likelihoodbd}

\begin{proof}[Proof of Lemma \ref{lemma: likelihoodbd}] 
Recall that we define $J_i \subseteq [H-1]$ for each $i\in [n]$, and we use $S^i_{J_i}$ to denote a truncated version of the $i$-th trajectory $S^i$ corresponding to the indices specified by $J_i$. Namely, $S^i_{J_i} = \{z_0^i\} \cup \{z_{j}\}_{j\in J_i} \cup \{ z_H^i\}$.
We begin by applying the previous Lemma \ref{lemma: hellingerbd}:
    \begin{align*}
        \frac{1}{2}\log \frac{\PP(\{S_{J_i}^i\}_{i=1}^n \given \theta)}{\PP(\{S_{J_i}^i\}_{i=1}^n \given \theta^*)}
        &=\sum_{i=1}^n \frac{1}{2}\log \frac{\PP(S_{J_i}^i\given \theta)}{\PP(S_{J_i}^i \given \theta^*)}\\
        &\leq \sum_{i=1}^n \log \bigg[\E_{\theta^*}\bigg(\frac{\PP(S_{J_i} \given \theta)}{\PP(S_{J_i}\given \theta^*)}\bigg)^{\frac{1}{2}} \bigg]+ \log(\delta^{-1})\\
        & \leq \sum_{i=1}^n  \bigg[\E_{\theta^*}\bigg(\frac{\PP(S_{J_i}\given \theta)}{\PP(S_{J_i}\given \theta^*)}\bigg)^{\frac{1}{2}} -1 \bigg]+ \log(\delta^{-1})\\
        &= -\sum_{i=1}^n \text{H}^2 \big(\PP(S_{J_i} \given \theta^*) , \PP( S_{J_i} \given \theta) \big)+ \log(\delta^{-1}),
    \end{align*}
    with probability at least $1-\delta$. The first inequality follows from Lemma~\ref{lemma: hellingerbd}, and the second inequality follows from the fact that $x-1 \geq \log(x)$.
    Putting both sides of the inequality into the exponential function, we have 
    \begin{align*}
        \frac{\PP(\{S_{J_i}^i\}_{i=1}^n \given \theta)}{\PP(\{S_{J_i}^i\}_{i=1}^n \given \theta^*)}\leq \exp{\Big(-2\sum_{i=1}^n\text{H}^2\big(\PP(S_{J_i} \given \theta^*), \PP(S_{J_i} \given \theta)\big)+2 \log(\delta^{-1})\Big)},
    \end{align*}
    with probability at least $1-\delta$.

    Finally, to prove the second argument, by conditional independence, we have 
\begin{align*}
    \frac{1}{2}\log \frac{\PP(\{S_{J_i}^i\}_{i=1}^n, z_0^{n+1}  \given \theta)}{\PP(\{S_{J_i}^i\}_{i=1}^n, z_0^{n+1} \given \theta^*)}
        &= \frac{1}{2} \log\frac{\PP(z_0^{n+1} \given \theta)}{\PP(z_0^{n+1} \given \theta^*)} +   
    \sum_{i=1}^n \frac{1}{2}\log \frac{\PP(S_{J_i}^i\given \theta)}{\PP(S_{J_i}^i \given \theta^*)} . 
    \end{align*}
   The rest of the proof is exactly the same as above.  
    Therefore, we conclude the proof.
\end{proof}

\subsection{Proofs of the Auxiliary Lemmas in Appendix~\ref{proof: pre-train}}\label{app: pretrain supp}

\subsubsection{Proof of Lemma \ref{lemma: error 2}} \label{proof: lem error 2}

\begin{proof}


We first note that for any $z\in \calL$ and $S\in \calL^*$, we have
\begin{align}
    \PP_{\hat \rho}(z\given S) \geq 1+|\calL|\exp(B_S/\tau),\label{eq:rho_hat_bd}
\end{align}
which follows from the softmax layer in~\eqref{eq:softmax_layer}. 
Combining \eqref{eq:rho_hat_bd} with Assumption~\ref{assumption: lower bd}, we obtain the upper bound of the following log density bound:
\begin{align}
    \Bigl|\log \PP(z \given S)  -\log \PP_{\hat \rho}(z \given S )\Big| \leq b^* = \log \max\{c_0^{-1}, 1+|\calL|\exp(B_S/\tau)\}. \label{eq: log density difference supp}
\end{align}
Inequality \eqref{eq: log density difference supp} gives the specific form of the upper bound $b^*$ mentioned in Assumption~\ref{assump: density bd}.

Next, we apply concentration to each fixed $(t,h)$ with $t\in [T]$ and  $0\leq h\leq H$. 
By Hoeffding's inequality and \eqref{eq: log density difference supp} we have 
\begin{align}
        &\PP \bigg(\sum_{\ell=1}^{N} \biggl ( \log\frac{\PP(z_h^{t,\ell}\given S_h^{t,\ell})}{\PP_{\rho'}(z_h^{t,\ell}\given S_h^{t,\ell})}-\bbE_{S_h^{t,\ell}}\KL\big(\PP(\cdot \given S_h^{t,\ell}),\PP_{\rho'}(\cdot\given S_h^{t,\ell})\big) \bigg) >t\bigg) \leq 2\exp\bigg(\frac{-t^2}{2N(b^*)^2}\bigg). \label{eq:hoeffding}
    \end{align}
   Then \eqref{eq:hoeffding} implies that, with  probability at least $1-\delta$, we have 
    \begin{align}
        \frac{1}{N}\sum_{\ell=1}^{N} \biggl ( \log\frac{\PP(z_h^{t,\ell}\given S_h^{t,\ell})}{\PP_{\rho'}(z_h^{t,\ell}\given S_h^{t,\ell})}-\bbE_{S_h^{t,\ell}}\KL\big(\PP(\cdot \given S_h^{t,\ell}),\PP_{\rho'}(\cdot\given S_h^{t,\ell})\big) \bigg) \leq b^*\sqrt{\frac{1}{2N}}\log \frac{1}{\delta}. \label{eq: high_prob_bd}
    \end{align}
    Applying union bound to \eqref{eq: high_prob_bd} for all  $(t, h)$  with  $t\in [T], 0\leq h\leq H$, we have that
    \begin{align*}
        &\frac{1}{NT(H+1)}\cdot \sum_{\ell=1}^{N}\sum_{t=1}^{T}\sum_{h=0}^{H} \bigg(\log\frac{\PP(z_h^{t,\ell}\given S_h^{t,\ell})}{\PP_{\rho'}(z_h^{t,\ell}\given S_h^{t,\ell})}-\bbE_{S_h^{t,\ell}}\KL\big(\PP(\cdot \given S_h^{t,\ell}),\PP_{\rho'}(\cdot\given S_h^{t,\ell})\big)\bigg) \nonumber\\
        &\quad \leq b^*\sqrt{\frac{1}{2N}}\log \frac{T(H+1)}{\delta}.
    \end{align*}
    Therefore, we conclude the proof of this lemma.
\end{proof}

\subsubsection{Proof of Lemma \ref{lemma: err 4}}\label{proof: lem error 4}

\begin{proof}
Fix any $(t,h)$ with 
 $ t\in [T]$ and $0\leq h\leq H$, we invoke Proposition \ref{prop:pacbayes} by setting $X_\ell = S_h^{t,\ell}$ and $f(S_h^{t,\ell}) = \TV\big(\PP( \cdot \given S_h^{t, \ell}), \PP_\rho( \cdot \given S_h^{t, \ell})\big)$, which gives us that with probability at least $1-\delta$,
\begin{align}
&\frac{1}{N}\Bigl|\bbE_{\rho\sim P}\sum_{\ell=1}^N\Bigl[\bbE_{S_h^{t, \ell}}\big[\TV\big(\PP( \cdot \given S_h^{t, \ell}), \PP_\rho( \cdot \given S_h^{t, \ell})\big)\big]-\TV\big(\PP( \cdot \given S_h^{t, \ell}), \PP_\rho( \cdot \given S_h^{t, \ell})\big)\Bigr]\Bigr| \nonumber \\
&\quad \leq \sqrt{\frac{1}{2\log 2\cdot N}}\biggl[\KL(P\,\|\,Q)+\log\frac{4}{\delta}\biggr],\label{eq:err4_bd_1}
\end{align}
where $P$ refers to the distribution defined in Lemma~\ref{lemma: error 1}, and  $Q$ is the  uniform distribution over $\calP_\mathrm{LLM}$.
The right-hand side of \eqref{eq:err4_bd_1} follows from Proposition \ref{prop:pacbayes} by setting \( b = 1 \) because the  TV distance is always between $0$ and $1$.

Next, we note that by the construction of $P$ in \eqref{eq:define_distribution_P},  for $\rho\sim P$, we have 
\begin{align}
\TV\big(\PP_{\hat \rho}( \cdot \given S_h^{t, \ell}), \PP_\rho( \cdot \given S_h^{t, \ell})\big) &\leq \sqrt{\frac{1}{2}\KL\big(\PP_{\hat \rho}( \cdot \given S_h^{t, \ell}), \PP_\rho( \cdot \given S_h^{t, \ell})\big)} \nonumber \\
&= \calO\big(1/\sqrt{ NTH}\big),\label{eq:err4_bd_2}
\end{align}
for any $S_h^t$. The first line follows from the Pinsker's inequality, and the second line follows directly from \eqref{eq:err_1_density_bd}.

Combining Lemma~\ref{lemma: error 3}, \eqref{eq:err4_bd_1} and \eqref{eq:err4_bd_2} with union bound across all $(t,h), t\in [T], 0\leq h \leq H$, we obtain 
\begin{align*}
        &\frac{1}{NT(H+1)}\sum_{\ell=1}^{N}\sum_{t=1}^{T}\sum_{h=0}^H \Bigl ( \bbE_{S_h^{t,\ell}}\Big[\TV\big(\PP(\cdot\given S_h^{t,\ell}),\PP_{\hat \rho}(\cdot\given S_h^{t,\ell})\big)\Big]-\TV\big(\PP(\cdot\given S_h^{t,\ell}),\PP_{\hat \rho}(\cdot\given S_h^{t,\ell})\big) \Big) \\
        &\quad=\cO\bigg(\frac{1}{\sqrt{N}}\Big(\barD\log(1+NTH\barB)+\log\frac{TH}{\delta}\Big)\bigg).
    \end{align*}
    Therefore, we conclude the proof. 
\end{proof}

\subsection{Proofs of the Auxiliary Lemmas in Appendix \ref{app: construction}} \label{proof:app:construction} 

\subsubsection{Proof of Proposition~\ref{prop:ffn_transformer}} \label{app:ffn_transformer}
\begin{proof}
In this proof, we aim to show that a sequence of $L$ transformer blocks defined in \eqref{eq:transformer_block} can exactly represent the MLP up to a scaling. 
Recall that the intermediate layers of the MLP are denoted by $\{ X^{d}  \}_{d \in [L]}$. 
In the following, we let $Y^{0} = X^0$, and let $Y^{d} $ denote the output of the $d$-th transformer block for all $d \in [L]$. 
Our goal is to construct the transformer blocks such that $Y^{L} $ is exactly equal to $X^{L}$ up to a constant factor.

Our construction is based on two key ideas. 
First, as shown in Appendix \ref{proof:prop:formal_construction}, in each transformer block, we can set the MHA layer to a zero function and only keep the FF layer and normalization layer in \eqref{eq:transformer_block}. 
Second, 
to avoid the influence of the normalization layers,
we adopt the scaling trick by scaling the parameters of the FF layer to ensure the output matrix is bounded by one in terms of the row-wise $\ell_2$-norm. 
This normalizing scalar is then multiplied by the weight matrix of the next layer to ensure the output stays the same. 

More rigorously, we define $B_{0} = B$, which is an upper bound on $\| (X^{0})^\top  \|_{2, \infty} $. Recall that we assume the intermediate values of the MLP satisfy
$\| (X^{d})^\top \|_{2, \infty} \leq B_{d} $ for all $d \in [L]$, where $\| (X^{d})^\top \|_{2, \infty} $ is the row-wise maximum $\ell_2$-norm of $X^d$. 
We will construct a transformer such that $Y^d = X^d / B_{d}$ for all $d \in [L]$. 
We prove this argument via induction.

We will verify the base case $Y^1 = X^1 / B_1$ later. 
For any $d \geq 2$, suppose we have $Y^{d-1} = X^{d-1} / B_{d-1}$ and let us consider the $d$-th transformer block and the $d$-th layer of the MLP. 
Note that $X^{d}$ is constructed by $X^{d-1} $ via \eqref{eq:int_ffn} and $Y^{d} $ is derived from $Y^{d-1} $ via 
\begin{align} 
\label{eq:transformer_block_d_th_layer}
Y^{d} = \mathtt{NL} ( \mathtt{ff} (Z, \overline W_{\mathrm{ff}}, \overline b_{\mathrm{ff}}) + Z^d  \bar \gamma_2  \bigr ) , \qquad Z ^d = \mathtt{NL} ( \mathtt{mha} (Y^{d-1} , \overline W_{\mathrm{mha}} ) + Y^{d-1}  \bar \gamma_1  \bigr ) 
\end{align}
for some weight matrices $W_{\mathrm{ff}}$, $ \overline b_{\mathrm{ff}}$, $\overline W_{\mathrm{mha}}$, $\overline \gamma _1$, and $\overline \gamma_2$ to be determined. 
We set the number of heads of the MHA layer to be $\eta = 1 $ and set the weight matrix of the values, $\overline W^{V}  $, as a zero matrix.
Moreover, we set $\overline \gamma _1  = I $ in \eqref{eq:transformer_block_d_th_layer}. 
Thus, we have 
 \begin{align}
       Z^d&=\mathtt{NL}\big(\mathtt{mha}(Y^{d-1}, \overline W _\mathrm{mha})+Y^{d-1} \overline \gamma _1 \big)  =\mathtt{NL}( X^{d-1}/B_{d-1})=X^{d-1}/B_{d-1}. \label{eq:z}
    \end{align}
    Here the second equality follows from the induction assumption and the last equality follows from the fact that the $\ell_2$-norm of each row of $X^{d-1} / B_{d-1}$ is bounded by one. 

Now we set  
$$
\overline W_{\mathrm{ff}}=\{W^d_{\mathrm{ff}} \cdot B_{d-1} , I / B_d\}, \qquad \overline b^d_{\mathrm{ff}} = \{b^d_{\mathrm{ff} }, 0\}, \qquad \textrm{and} \qquad \overline \gamma _1 = I
$$
in \eqref{eq:transformer_block_d_th_layer}.
That is, $\overline W_{\mathrm{ff},1}$ is proportional to $W^d_{\mathrm{ff}}$ of the MLP, $\overline W_{\mathrm{ff},2}$ is proportional to an identity matrix, $\overline b_{\mathrm{ff}, 1 }$ is the same as $b^d_{\mathrm{ff}}$ of the MLP,   $ \overline b_{\mathrm{ff}, 2 }$ is a zero vector, and $\overline \gamma _1 $ as an identity matrix. 
Then, by direct calculation we have $$\mathtt{ff}( Z ^d , \overline W_\mathrm{ff}, \overline b_\mathrm{ff}) = \mathtt{ReLU} ( X^{d-1}W_\mathrm{ff}^d +\mathbf{1}^\top b_\mathrm{ff}^d) \big /B_d  = X^{d} / B_d . $$ 
As a result,  in \eqref{eq:transformer_block_d_th_layer} we have 
\begin{align}
         Y^d&=\mathtt{NL}\big(\mathtt{ff}( X^{d-1}/B_{d-1}, \overline W _\mathrm{ff}, \overline b _\mathrm{ff})\big)  =\mathtt{NL} (X^d /B_d )   
         =   X^d/B_d, \label{eq:y}
    \end{align}
    where the first equality follows from \eqref{eq:z} and the last equality follows from the fact that 
    $B_{d}\geq \| (X^{d})^{\top}\|_{2,\infty}$. 
    Also see Figure~\ref{fig:ffn-transformer} for an illustration of each transformer block.
\begin{figure}[h]
    \centering
    \includegraphics[width = 0.75\textwidth]{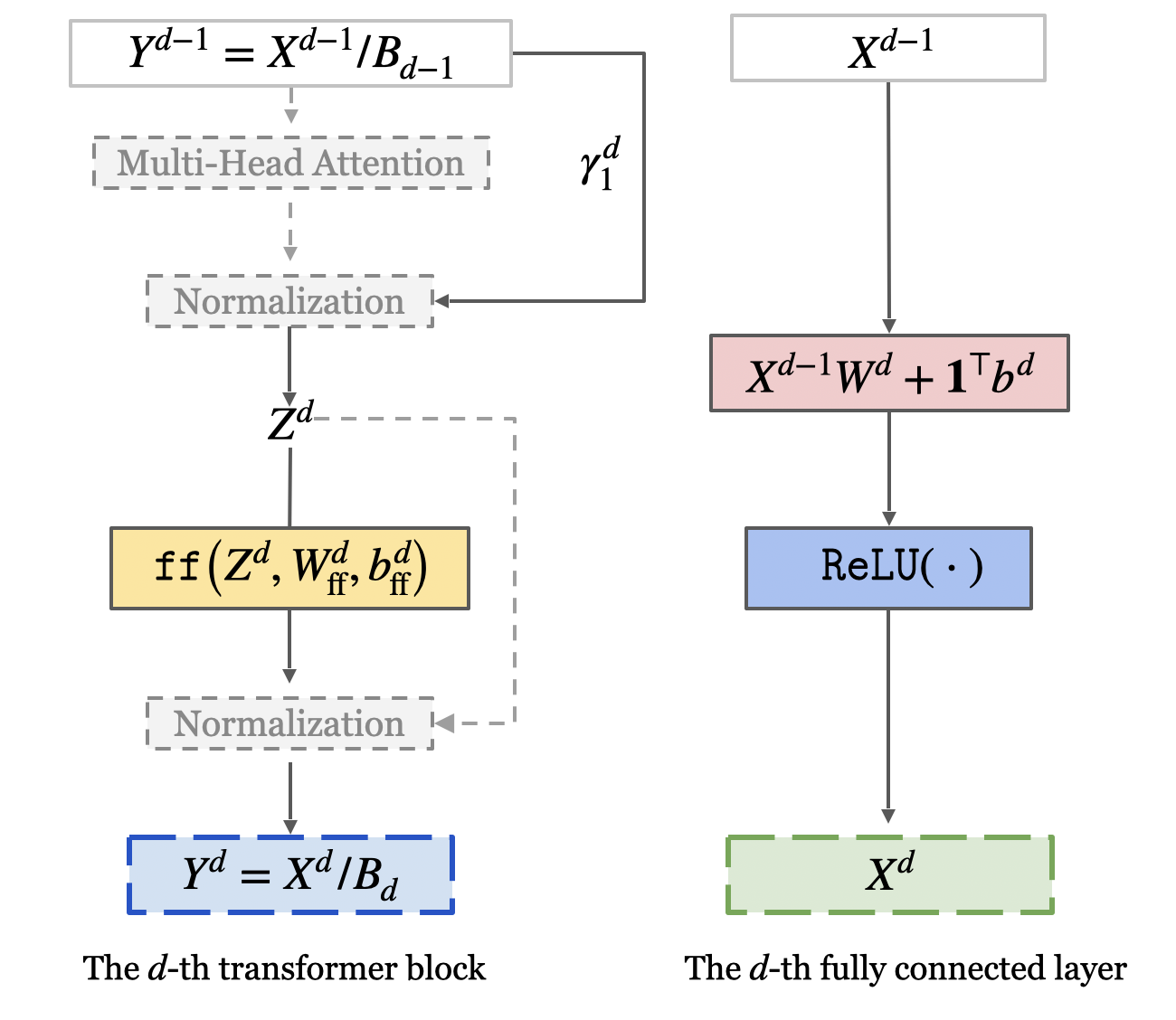}
    \caption{An illustration of the $d$-th transformer block, compared with the $d$-th fully connected layer. (a) The left figure shows the \(d\)-th transformer block, where the \ac{mha} block is omitted by setting the value weight matrix to zero.   Since the input is scaled down before each normalization layer, these two layers act as identity maps. As a result, the only ``active'' component in this transformer block is the feed-forward layer. (b) The right figure depicts the \(d\)-th fully connected layer. The output of the \(d\)-th transformer block matches the output of this fully connected layer up to a scaling factor $B_d$.}
    \label{fig:ffn-transformer}
\end{figure}

    It remains to verify the base case $Y^1 = X^1 / B_1$.
    Recall that the transformer and the MLP have the same input. Thus we have $Y^0 = X^0$. 
    For $d = 1$, in  \eqref{eq:transformer_block_d_th_layer} we set the value matrix of MHA to be zero and set $\overline \gamma_1  = I / B_0$. 
    Since each row of $X^{0}/B_0$ is in the unit ball with respect to the $\ell_2$-norm, 
we have 
$$
Z^{1} = \mathtt{NL} ( Y^{0} \overline \gamma_1 ) = \mathtt{NL} ( X ^{0} / B_0 ) = X^{0} / B_0.
$$
Thus we recover \eqref{eq:z} for $d = 1$. 
Then, similar to the derivations above, we can obtain \eqref{eq:y} for the base case. Therefore we conclude that $Y^{d} = X^{d} / B_d $ for all $d \in [L]$.

In summary, we construct a transformer with $L$ transformer blocks such that the final output $Y^{L} $ satisfies $Y^{L} = X^{L} / B_{L}$. 
The weight matrices of these transformer blocks are given by $\{\overline W_1^{Q,d}, \overline 
 W_1^{K,d}, \overline 
 W_1^{V,d}, \overline W_{\mathrm{ff},1}^d, \overline W_{\mathrm{ff},2}^d, \overline b_{\mathrm{ff},1}^d, \overline b_{\mathrm{ff},2}^d, \overline \gamma_1^d, \overline \gamma_2^d, \}_{d=1}^L$, where 
\begin{align*}
 & \overline W_{\mathrm{ff},1}^d = B_{d-1}\cdot W_{\mathrm{ff} }^d,~~ \overline 
 W_{\mathrm{ff},2}^d=I/B_d,\quad  \overline 
 b_{\mathrm{ff},1}^d = b_{\mathrm{ff}} ^d, \quad b_{\mathrm{ff},2}^d=  0 , \quad  \overline W_{\mathrm{mha}}^{V,d}=0, \quad \overline \gamma_2^d= 0   
\end{align*}
for all $d \in [L]$. 
    Moreover, we have $ \overline \gamma_1^1= I/ B_0  
 $ and $\overline \gamma_1^d= I$ for all $d \geq 2$. 
Finally, to recover $X^{L}$, it suffices to multiply $B_L \cdot I$ to $Y^L$. Therefore, we conclude the proof of this proposition.
\end{proof}

\subsubsection{Proof of Lemma~\ref{lem:layer_1}} \label{app:layer_1}
\begin{proof}
Consider the FF layer defined in \eqref{eq:ffn_bias}. 
We set the bias terms as  $b_{\mathrm{ff}, 1 }^{\mathtt{NN} , (1)}=(0,0,1,-1,0,0)$ and $b_{\mathrm{ff}, 2 }^{\mathtt{NN} , (1)} = {\bf 0} $.
We set the weight matrices as  $W_{\mathrm{ff}, 1}^{\mathtt{NN},(1)}$ and $W_{\mathrm{ff}, 2}^{\mathtt{NN},(1)}$ as 
\begin{align*}
    W_{\mathrm{ff}, 1}^{\mathtt{NN},(1)}=\big(W_{\mathrm{ff}, 2}^{\mathtt{NN},(1)}\big)^\top=
\begin{pmatrix}
1 & -1 & 0 & 0 & 0 & 0 \\
 0 & 0 & 1 & -1 & 0 & 0\\
0 & 0 & 0 & 0 & 1 & -1 
\end{pmatrix} \in \RR^{3\times 6}. 
\end{align*}
Note that fact that $x=\mathtt{ReLU}(x)-\mathtt{ReLU}(-x)$ for any $x \in \RR$. 
For any vector of the form $v = (a,b,0)^\top $ in $\RR^3$, by the direct computation 
we have 
$$
( W_{\mathrm{ff}, 2}^{\mathtt{NN},(1)} )^\top \mathtt{ReLU} \big (  (W_{\mathrm{ff}, 1}^{\mathtt{NN},(1)} )^\top v  + (b_{\mathrm{ff}, 1 }^{\mathtt{NN} , (1)  })^\top \big)  = (a,b + 1 ,0)^\top.
$$
Thus, we have 
$$ 
X_{\mathtt{NN}}^{(1)} = \mathtt{ff}(X_{\mathtt{NN}}^{(0)}, W_{\mathrm{ff} }^{\mathtt{NN},(1)}, b_{\mathrm{ff}   }^{\mathtt{NN},(1)}) = \mathtt{ReLU}(X_\mathtt{NN}^{(0)}W_{\mathrm{ff},1}^{\mathtt{NN},(1)} +\mathbf{1}^T b^{\mathtt{NN},(1)}_{\mathrm{ff},1})W_{\mathrm{ff},2}^{\mathtt{NN},(1)}  ,
$$
which means that $\mathtt{NN}_{\mathrm{J}, 1}$ is realized by a FF layer. 
Here $\mathbf{1}\in \mathbb{R}^{L'}$ is the all-one vector.
Next, we compute the scaling factors in Proposition~\ref{prop:ffn_transformer} to bypass the normalization layer in the \ac{ff} layer. Note that both the input and output magnitude is bounded by $\sqrt{B^2+H^2+1}$. Absorbing the scaling factor into the weight matrix, we have  
$$
\max \{ \| W_{\mathrm{ff}, 1 }^{\mathtt{NN},(1)} \|_{F} , \| W_{\mathrm{ff},2 }^{\mathtt{NN},(1)} \|_{F} , \| b_{\mathrm{ff}, 1 }^{\mathtt{NN} , (1)} \|_2, \| b_{\mathrm{ff},2  }^{\mathtt{NN} , (1)} \|_2  \} \leq \sqrt{B^2+H^2+1}\cdot \sqrt{6}. 
$$
Furthermore, as introduced at the beginning of Appendix \ref{prop:formal_construction},  
we can represent a feed-forward layer using a transformer block. 
Thus, we conclude the proof. 
\end{proof}

\subsubsection{Proof of Lemma~\ref{lem:layer_2}} \label{app:layer_2}
\begin{proof}
We first show that the product operation can be well approximated by a fully connected neural network, which is stated in the following lemma.

\begin{lemma} [Product Operation as Neural Network]\label{lemma: product_module}
Let $r\geq 1$ be an integer. There exists a constant $C_p>0$ such that for any $M$ and $\epsilon$, there exists a multi-layer perceptron $f_\mathrm{product}:\mathbb{R}^{r}\times \mathbb{R}^1 \to \mathbb{R}^r$ with $D$ layers  such that 
for any $x_1,\cdots,x_r,y\in [-M,M]$, $$\|f_\mathrm{product}\big(x_{1:r},y \big)-(x_1 y,\cdots,x_r y)\|_\infty\leq \epsilon.$$
Moreover, the depth $D$ satisfies $D\leq C_p\log(M)+\log(1/\epsilon))$, and the maximum number of the hidden neurons is bounded by $5r$. 
Furthermore, the parameters $\{W^i_{\mathrm{ff},1}, b^i \}_{i\in[D]} $ satisfies 
$ \max \{\|W^i_{\mathrm{ff},1}\|_\infty, \|b^i\|_\infty \} \leq 1$ for all $i \in [D-1]$, and $\|W^D_{\mathrm{ff},1}\|_\infty\leq M^2$. 
\end{lemma}
\begin{proof}
See Appendix~\ref{app: product_module} for details.
\end{proof}

    The output \eqref{eq:output_layer_1} can be achieved by applying the product operation construction from Lemma~\ref{lemma: product_module} with $r = 1$, ensuring that \(|f_\mathrm{product}(h,1) - h| < \epsilon'\) and \(|f_\mathrm{product}(h'+1,0) - 0| < \epsilon'\) for all \((t', h') < (t, h-1)\).

By setting $r=1$, this lemma says that we can realize the product operation $f_\mathrm{product}$ in \eqref{eq:output_layer_1} as the output of a fully connected network with an error at most $\epsilon'$.
Moreover, the depth of $f_\mathrm{product}$
is at most  $ C_p(\log(H) + \log(1/\epsilon'))$ for some constant $C$, and the maximum number of hidden neurons is bounded  by $5$.

To implement this fully connected network using a transformer, we employ Proposition~\ref{prop:ffn_transformer} by setting scaling factors as $B_0 = \sqrt{H^2+1}$ and $B_d = \sqrt{5}$ for all $d\in [D-1]$, and $B_D=H^2$. 
So far we have shown that we can use $D$ transformer blocks to realize the product operation $f_\mathrm{product}(x,y)\approx xy$. However, note that the target input of module $\mathtt{NN}_{\mathrm{J},2}$ is $X^{(1)}_\mathrm{NN}\in \mathbb{R}^{L'\times 3}$, and the target output is $X^{(2)}_\mathrm{NN}\in \mathbb{R}^{L'\times 3}$, where the product operation $f_\mathrm{product}(x,y)$ only substitutes the second column of the output. Thus, we have to preserve the first and last columns of the input.

To realize the output in \eqref{eq:output_layer_1}, 
we concatenate another MLP, denoted by $f_{1,3}$ to $f_\mathrm{product}$.
 MLP $f_{1,3} \colon \RR^{L'\times 3} \rightarrow \RR^{L'\times 2}$ extracts the first and last columns of the input matrix.  
 As we show in Lemma~\ref{lem:residual_relu}, such an MLP exists and $f_{1,3} $ in fact can be written as a single FF layer. The Frobenius norms of the weight matrices are bounded by $2$. 
 The intuition behind Lemma~\ref{lem:residual_relu} is that we can write and $x$ as  \(\mathtt{ReLU}(x) - \mathtt{ReLU}(-x)   \), which enables us to preserve desired columns of the input matrix using a FF layer. 

 Finally, to concatenate $f_{1,3}$ with $f_\mathrm{product}$, 
 notice that these two MLPs might have different numbers of layers.
 This does not cause trouble because by Lemma~\ref{lem:residual_relu}, we can further stack FF layers on top of $f_{1,3}$ that represents identity mappings. 
 This will enable us to write $f_{1,3}$ as an MLP that has the same depth as $ f_\mathrm{product}$. 
 Then we can concatenate the weights of these two MLPs in a layer-wise fashion. Thus, the Frobenius norm of the weight matrices of the concatenated MLP is bounded by 
 $\sqrt{B^2+H^2+6}\cdot \sqrt{H^4\cdot5 +4}$, where the first term in the multiplication comes from the scaling factors and the second term follows from the weight matrices. This concludes the proof of this lemma. 
\end{proof}

\subsubsection{Proof of Lemma~\ref{lem:layer_3}} \label{app:layer_3}
\begin{proof}
This linear operation can be realized without error by setting 
which can be realized by setting bias term $b^{\mathtt{NN},(2)}=(0,0,0,0,1,-1)$ and $b_{\mathrm{ff}, 2 }^{\mathtt{NN} , (2)} = {\bf 0} $. 
We set the weight matrices as 
\begin{align*}
\begin{split}
W_{\mathrm{ff}, 1}^{\mathtt{NN},(2)} & =
\begin{pmatrix}
1 & -1 & 0 & 0 & 0 & 0 \\
 0 & 0 & 1 & -1 & 0 & 0\\
0 & 0 & 0 & 0 & -1/2 & 1/2 
\end{pmatrix}, \\
    W_{\mathrm{ff}, 2}^{\mathtt{NN},(2)} & =
\begin{pmatrix}
1 & -1 & 0 & 0 & 0 & 0\\
 0 & 0 & 1 & -1 & 0 & 0\\
0 & 0 & 0 & 0 & 1 & -1
\end{pmatrix}^\top. 
\end{split} 
\end{align*}
For any vector of the form $v = (a,b,c)^\top $ in $\RR^3$, by direct computation, a FF layer with parameters $\{    W_{\mathrm{ff}}^{\mathtt{NN},(2)} , b_{\mathrm{ff}}^{\mathtt{NN},(2)} \}  $
 maps it to $(a,b,c')^\top$, where 
 $$
 c' = \mathtt{ReLU}(1 - c/2) +  \mathtt{ReLU}(-1 + c/2).
 $$
 Thus, $c' = 1$ if $c = 0$ and $c' = 1/2$ if $c = 1$.
 Therefore, we have 
 $$
 X_{\mathtt{NN}}^{(3 )} = \mathtt{ff}(X_{\mathtt{NN}}^{(2)}, W_{\mathrm{ff} }^{\mathtt{NN},(2)}, b_{\mathrm{ff}   }^{\mathtt{NN},(2)}) ,
 $$
 which means that $\mathtt{NN}_{\mathrm{J},3} $ can be realized by a single FF layer. 
Moreover, the weight matrices and bias vectors are all bounded by $\sqrt{B^2+(H+1)^2+1}\cdot \sqrt{6}$ in terms of the Frobenius norm. 
Noticing that a FF layer can be represented by a single transformer block, we 
conclude the proof.
\end{proof}

\subsubsection{Proof of Lemma \ref{lemma: extraction_module}}\label{app: extraction_module}
\begin{proof}
In this proof, we first construct an attention layer that takes \(\tilde{z}_{h'}^{t',(2)} \in \mathbb{R}^3\) as the input and outputs \((0, p_{h'}^{t'},0)\)
for any index $(t', h')$. 
Here 
 \(p_{h'}^{t'} \approx h\) for any \((t',h')\). Next, we explain how to maintain the first coordinate of $\tilde{z}_{h'}^{t',(2)}$, i.e., $z_{h'}^{t'}$.
 Finally, we show that this network module can be implemented by a single transformer block.  

To construct the attention layer, we introduce an auxiliary parameter \(\alpha\) to control the precision of \( p_{h'}^{t'} \). 
Recall that the sequence $S_h^t$ has length
$L' = L' (t, h) = (t-1) \cdot (H+1) + h$ and its last element is $z_{h-1}^t$. 
For any $(t', h')$, $\tilde z_{h'}^{t', (2)} $ is the $L'(t',h')$-th vector of the input $X_{\mathtt{NN}}^{(3)}$.
We define the attention matrices $W_\mathtt{mha} = \{ W^Q, W^K,W^V\}$, where $W^Q = (0,0, - \alpha)^\top \in \mathbb{R}^{3\times 1},
    W^K = (0,0,1)^\top \in \mathbb{R}^{3\times 1}$ and 
   $$
    W^V = 
    \begin{pmatrix}
        0 & 0  & 0 \\
        0 & 1  & 0 \\
        0 & 0  & 0 
    \end{pmatrix}.
    $$
    Here $\alpha > 0$ is a parameter to be determined later. 
With the input $X_{\mathtt{NN}}^{(3)} \in \RR^{L' \times 3}$, the queries, keys, and values
are given by 
\begin{align}
\begin{split}
    Q &= X _{\mathtt{NN}}^{(3)} W^Q = 
    - (\alpha,\cdots,\alpha,\alpha/2)^\top  \in \mathbb{R}^{L'},   \\
    K&= X _{\mathtt{NN}}^{(3)} W^ K = (1,\cdots,1,1/2)^\top \in \mathbb{R}^{L'},   \\
    V&= X _{\mathtt{NN}}^{(3)} W^V = 
\begin{pmatrix}
    0&\cdots&0\\
    f_\mathrm{product}(h'+1,0)&\cdots& f_\mathrm{product}(h, 1)\\
    0&\cdots&0
\end{pmatrix}^\top \in \mathbb{R}^{L' \times 3},
\end{split} \label{eq:qkv}
\end{align}
where we $ f_\mathrm{product}$ approximates the product operation in the sense that   $|f_\mathrm{product}(x,y)-xy|\leq \epsilon'$ for any $x,y\in [-H,H]$.
Using \eqref{eq:qkv}, we compute that the softmax attention score $a(i,j)$ based on the $i$-th query   and $j$-th key for all $i,j \in [L']$:
\begin{align*}
    a(i,j) = 
    \begin{cases}
        \cfrac{e^{-\alpha/2}}{e^{- \alpha/2}+ (L' - 1)\cdot e^{-\alpha}} & \text{for } i\neq L' \text{ and } j=L',\\
        \cfrac{e^{-\alpha}}{e^{-\alpha/2}+ (L'-1)\cdot e^{-\alpha}} & \text{for } i\neq L' \text{ and } j\neq L' ,\\
        \cfrac{e^{-\alpha/4}}{e^{-\alpha/4}+ (L' - 1) \cdot e^{-\alpha/2}} & \text{for } i = L' \text{ and } j= L',\\
        \cfrac{e^{-\alpha/2}}{e^{-\alpha/4}+ (L' - 1) \cdot e^{-\alpha/2}} & \text{for } i =L \text{ and } j\neq L'.
    \end{cases}
\end{align*}
Using these attention scores to aggregate the value vectors, 
for any position $(t', h')$, 
we obtain the output of the attention layer: 
\begin{align} \label{eq:attention_output_module_J}
(0,p_{h'}^{t'},0) = \sum_{j=1}^{L'} a \bigl ( L'(t', h') ,j \bigr ) \cdot v_j,
\end{align}
where \( v_j \) is the \( j \)-th row of the value matrix \( V \), and $L'(t', h')$ is the index of $z_{h'}^{t'}$ in sequence $S_h^t$. 
By the construction of $V$ in \eqref{eq:qkv}, the first and last coordinates of the attention output are both zero, and we let $p_{h'}^{t'} $ in \eqref{eq:attention_output_module_J} to denote the nonzero coordinate. 

It remains to show that $|p_{h'}^{t'} - h| $ is small for all $(t', h')$.
For any $\epsilon > 2 \epsilon'$, we can choose $\alpha $ sufficiently large such that  $(L' -1 ) \cdot \max\{e^{- \alpha/2}, e^{-\alpha/4}\}+2\epsilon'< \epsilon$.
We separately consider the cases where $(t', h')< (t,h-1)$ and $(t', h')= (t,h-1)$ as follows. 
For the first case, we have 
\begin{align}\label{eq:some_bound_p_h_1}
    |p_{h'}^{t'}-h| &=\bigg| \cfrac{e^{- \alpha/2}}{e^{- \alpha/2}+ (L' - 1) \cdot e^{- \alpha}}\cdot f_\mathrm{product}(h,1)-h  \notag \\
    &\qquad + \cfrac{e^{- \alpha}}{e^{- \alpha/2}+ (L' - 1) \cdot e^{-\alpha}}\sum_{(t', h')< (t, h-1)}f_\mathrm{product}(h'+1,0) \bigg|\\
    &  \leq ( L ' - 1) \cdot e^{ - \alpha/2} +2\epsilon' < \epsilon. \notag  
\end{align}
To see the second inequality, we note that using the fact that $|f_\mathrm{product}(h,1)-h | \leq \epsilon' $ and the fact that 
$$
\cfrac{e^{- \alpha/2}}{e^{- \alpha/2}+ (L' - 1) \cdot e^{- \alpha}} = \cfrac{1}{1 + (L'-1) \cdot e^{-\alpha / 2}},
$$
we have by direct computation that 
$$
\bigg| \cfrac{e^{- \alpha/2}}{e^{- \alpha/2}+ (L' - 1) \cdot e^{- \alpha}}\cdot f_\mathrm{product}(h,1)-h\bigg| \leq \cfrac{ |f_\mathrm{product}(h,1)-h| }{1 + (L'-1) \cdot e^{-\alpha / 2}}  +  \cfrac{ (L'-1) \cdot e^{-\alpha / 2} }{1 + (L'-1) \cdot e^{-\alpha / 2}},
$$
which is bounded by $( L ' - 1) \cdot e^{ - \alpha/2} + \epsilon'$.
Moreover, the second summation in \eqref{eq:some_bound_p_h_1} is bounded by $\epsilon'$ because 
$$
 \bigg| \cfrac{e^{- \alpha}}{e^{- \alpha/2}+ (L' - 1) \cdot e^{-\alpha}}\sum_{(t', h')< (t, h-1)}f_\mathrm{product}(h'+1,0) \bigg| \leq | f_\mathrm{product}(h'+1,0) | \leq \epsilon'. 
$$
Combining the above two inequalities yields \eqref{eq:some_bound_p_h_1}. 
Similarly, for $(t', h') = (t, h - 1)$, we use the same argument to obtain that  
\begin{align*}
    |p_{h-1}^{n+1}-h| &=\bigg| \cfrac{e^{-\alpha/4}}{e^{-\alpha/4}+ (L' - 1) \cdot e^{-\alpha/2}}\cdot f_\mathrm{product}(h,1)-h \\
    &\qquad + \cfrac{e^{-\alpha /2}}{e^{-\alpha/4}+ (L' - 1) \cdot e^{ -\alpha/2}}\sum_{(t', h')< (t, h-1)}f_\mathrm{product}(h'+1,0) \bigg|\\
    &  \leq (L' - 1) \cdot e^{-\alpha/4} +2\epsilon' < \epsilon.
\end{align*}

In conclusion, when selecting $\alpha$ such that $(L'-1) \cdot \max\{e^{- \alpha/2}, e^{- \alpha/4}\} < \epsilon - 2\epsilon'$,
we have $|p_{h'}^{t'}-h|< \epsilon$ for all $(t', h')$. 
Since $L' \leq T\cdot (H+1)$, it suffices to choose 
\begin{align*}
    \alpha = 8\cdot \log \big( TH / (\epsilon-2\epsilon') \big).
\end{align*}


We conclude that we can construct an attention layer that takes \(\tilde{z}_{h',(2)}^{t'} \in \mathbb{R}^3\) as input and outputs \((0, p_{h'}^{t'},0)\), such that \(|p_{h'}^{t'} - h|<\epsilon\) for any \((t',h') \leq (t, h-1)\).
Moreover, the norms of weight matrices satisfy 
$$\|W^Q\|_\mathrm{F} = 8\log\big(TH/(\epsilon-2\epsilon')\big), \qquad \|W^K\|_\mathrm{F} = \|W^V\|_\mathrm{F}=1.$$
Note that this single-head attention is a special of MHA layer with $\eta = 1$. 

Finally, 
to show that such a layer can be implemented by a single transformer block defined in \eqref{eq:transformer_block}, we use a  residual link by setting $\gamma_1=\mathtt{diag}(1,0,0)\in \mathbb{R}^{3\times 3}$ 
in \eqref{eq:transformer_block}. 
This enables us to 
pass along the first coordinate $z_{h'}^{t'}$. Additionally, the \ac{ff} module can append zeros to the input by taking 
$W_{\mathrm{ff}, 2}^{\mathtt{NN},(1)}$ as 
\begin{align*}
    W_{\mathrm{ff}, 1}^{\mathtt{NN},(1)}=
\begin{pmatrix}
1 & -1 & 0 & 0 & 0 & 0 & \\
 0 & 0 & 1 & -1 & 0 & 0& \\
0 & 0 & 0 & 0 & 1 & -1 & 
\end{pmatrix} \in \RR^{3\times 6},\quad
\big(W_{\mathrm{ff}, 2}^{\mathtt{NN},(1)}\big)^\top =
\begin{pmatrix}
&W_{\mathrm{ff}, 1}^{\mathtt{NN},(1)} \\
&\mathbf{0}
\end{pmatrix} \in \RR^{(2|\mathcal{L}|+2)\times 6},
\end{align*}
and $b_{\mathtt{ff},1}^{\mathtt{NN},(1)}=b_{\mathtt{ff},2}^{\mathtt{NN},(1)}=0$. As a result, 
for any $(t', h')$, the output of the transformer block is given by 
\begin{align*}
    z_{h'}^{t',(4)} 
    = \mathtt{NL}\big ( (z_{h'}^{t'}, p_{h'}^{h'},\mathbf{0}) \bigr)\in\bbR^{2+2|\mathcal{L}|} .
\end{align*}
Note that the $\ell_2$-norm of vector $(z_{h'}^{t'}, p_{h'}^{h'},\mathbf{0})$ is bounded by a constant because $| p_{h'}^{h'} | \leq H$ and $\cL$ is regarded as a compact subset of $\RR$. 
Thus, we can additionally apply the scaling trick introduced in Appendix \ref{proof:prop:formal_construction} to bypass the normalization layer $\mathtt{NL}(\cdot)$. 
To implement this scaling trick, we need only to scale $W^{V}$ by a constant factor $\sqrt{B^2+(H+1)^2+1}$, which affects the magnitude of the transformer weight matrices by a constant factor. 
Now we conclude the proof.
\end{proof}

\subsubsection{Proof of Proposition~\ref{prop: approx of submodule}} \label{app: approx of submodule}

\begin{proof}

This proof is structured in three steps. In \textbf{Step 1}, we provide a high-level overview of the network $\mathtt{embed}_h(\cdot)$, aiming for $\mathtt{softmax}(\mathtt{embed}_h(S_h^t)/\tau) \approx g_h^*(S_h^t)$, where $g_h^*(\cdot)$ is the target distribution. In \textbf{Step 2}, we provide a detailed construction of $\mathtt{embed}_h(\cdot)$ by approximating $\psi_h^*$ and $\tau \log w_{h,i}^*$ for each $i \in [|\calL|]$. Finally in \textbf{Step 3}, we apply the $\mathtt{embed}_h(\cdot)$ approximation to construct each module $G_h$, which takes $(S_h^t, \hat p_h, F_3, F_4)$ as the input and produces $(S_h^t, \hat p_h, \mathbf{1}_{L'}^\top\mathtt{embed}_h(S_h^t), F_4)$ as the output. 
We show how to modify the weight matrices in $\mathtt{embed}_h(\cdot)$ to construct $G_h$. This technique is applied repeatedly in the proof found in Appendix~\ref{proof:prop:formal_construction}, with similar approaches being used in related cases.

\vspace{2mm}

\noindent \textbf{Step 1: High-level structure of each $\mathtt{embed}_h(\cdot)$.} The module $G_{h}$ takes input 
$(S_h^t, \hat p_h, F_3, F_4) \in \mathbb{R}^{L'\times (2+2|\calL|)}$ as the input  and outputs $(S_h^t, \hat p_h, \mathbf{1}_{L'}^\top\mathtt{embed}_{h}(S_h^t), F_4)\in \mathbb{R}^{L'\times (2+2|\calL|)}$.
Note that $\psi_h^*$ is a univariate function and $\omega_h^* \colon \RR\rightarrow \RR^{|\cL|}$. 
Note that the key functionality of the module $G_{ h}$ is to use $S_h^t$ to produce $\mathtt{embed}_{h}$. For any $h\in [H]$, we want to construct networks $\hat w_h^*$ and $\hat \psi_h^*$ approximating functions  $w_h^*$ and $\psi_h^*$ in \eqref{eq:g*} separately. 
More specifically, we want to construct a network $\hat g_h^*$ such that
\begin{align}
    g_h^*(S_h^t) & = w_h^*\bigg(\frac{1}{L'}\sum_{(i,j)=(1,0)}^{(t,h-1)} \psi_h^* (z_j^i)\bigg)  \notag \\
    & \approx \mathtt{softmax}\bigg( \hat f_{h }^{w} \bigg(\frac{1}{L'}\sum_{(i,j)=(1,0)}^{(t,h-1)} \hat \psi_h^* (z_j^i)\bigg) \bigg/\tau \bigg) = \hat g_h^*(S_h^t), \label{eq:approx}
\end{align}
where $\sum_{(i,j)=(1,0)}^{(t,h-1)}$ means that we sum over all reasoning steps before $z_{h}^t$. 
Here $\hat f_{h }^{w} $ in \eqref{eq:approx} approximates the function $\tau \cdot \log w_h^*$.
 In \textbf{Step 2}, we apply Lemma~\ref{lem:network_approx} to separately bound the error induced by $\hat f_{h }^{w}$ and $\hat \psi_h^*$ using the universal approximation property of the fully-connected networks. 
Finally, in {\bf Step 3} we combine everything and construct the $G_h$ as a composition of transformer blocks. 
 
In the sequel, we introduced the rationale behind the construction of   $\hat \psi_h^*$, and $\hat f_{h }^{w}$ describe how to implement them using transformer blocks.

\begin{itemize}
    \item \textbf{Approximate $\psi_h^*$ using an MLP.} 
    We construct $\hat \psi_h^* \colon \RR\rightarrow \RR $ as a fully-connected MLP with $D_\psi$ layers and each layer has no more than $16$ neurons, where $D_{\psi}$  will be specified later. 
    The construction directly follows from Lemma~\ref{lem:network_approx}, which is a neural network approximation result established in \cite{elbrachter2021deep}. 
    As shown in Proposition \ref{prop:ffn_transformer},
    such a fully connected network can be regarded as a composition of 
    $D_{\psi}$ transformer blocks. 
    In particular, as shown in the proof of 
    Proposition \ref{prop:ffn_transformer}, 
    in each transformer block as in \eqref{eq:transformer_block}, 
    we can set the value matrices in the MHA layers to be zero,  set $\gamma_1   $ as an identity matrix, and set $\gamma_2 $ to a zero matrix. 
    This reduces the transformer block to a feed-forward layer, combined with normalization. We can apply the scaling trick introduced in Appendix \ref{proof:prop:formal_construction} to bypass the normalization layer. 
    This enables us to represent each layer of the MLP using a transformer block.


    \item \textbf{Realize the average module.} 
    After having $\hat \phi_h^*$, 
we need to compute $$1 / L' \cdot \sum_{(i,j)=(1,0)}^{(t,h-1)}\hat \psi_h^*(z_j^i)$$ using a transformer block. 
This can be achieved by having a single-head attention layer with 
 $W^Q=W^K=0$, and $W^V=1$. 
To see this, observe that when $W^Q=W^K=0$, all the attention scores become $1/ L'$ and thus the attention layer becomes an average.

    \item \textbf{Approximate $ w_h^*$ using an MLP.} Note that 
    $w_h^*$ takes values in $\RR^{|\cL|}$. 
    We let $w_{h, i}^*$ denote its $i$-th entry for all $i \in [|\cL|]$. 
    We leverage Lemma~\ref{lem:network_approx} to approximate each $\tau \cdot \log w_{h,i}^*$ using 
a fully connected MLP $\hat f_{h,i}^{w}$ with $D_{\omega}$ layers, where each layer has at most $16$ neurons. 
Here $D_{\omega}$ will be specified later. 
Similar to $\hat \psi_h$, such an MLP can be implemented by a composition of $D_{w}$ transformer blocks. 
    
\end{itemize}

\noindent \textbf{Step 2: Approximate  $\psi_h^*$ and $w_h^*$ using MLPs.} 
In this step, we employ the universal approximation properties of fully connected networks 
to construct MLPs that approximate $\psi_h^*$ and $w_{h,i}^*$ for all $i \in [|\cL|]$. 
The technical tool we leverage is 
Lemma~\ref{lem:network_approx}, obtained from  \cite{elbrachter2021deep}, 
which shows that MLP functions can approximate sufficiently smooth functions. 

Specifically, under Assumption \ref{assumption: smooth}, by Lemma~\ref{lem:network_approx}, for any desired accuracy levels $ \epsilon_\psi$ and $\epsilon_w$, 
there exist MLPs \(\hat \psi_h^*\) and $\{ \hat f_{h,i}^{w} \}_{i \in [|\cL|]} $ such that 
\begin{align}
    \label{eq:apply_lemma_approx}
    \| \hat \psi_h^* -  \psi_h^*\|_\infty <\epsilon_\psi, \qquad  \| \hat f_{h,i}^{w} -   \tau \cdot \log w_{h,i}^*\|_\infty <\epsilon_w  \text{ for all }i\in [|\calL|],
\end{align}
where $\hat \psi_h^*$ has $D_{\psi}$ layers and each $\hat f_{h,i}^{w}$ has at most $D_{w}$ layers. 
Here we have 
\begin{align}
    D_{\psi}& =    2C_dB \cdot \big (\log (1/ \epsilon_\psi ) \bigr )^2 + \log B, \qquad D_{w} = 2C_dB \big (\log (1/  \epsilon_w   ) \bigr )^2 + \log B, \label{eq:max_depth}  
\end{align}
where $C_d > 0$ is an absolute constant and  $B$ is the parameter appearing in Assumption \ref{assumption: smooth}. 
Moreover,  each layer has at most $  16$ neurons and all the neural network weights are bounded by one in magnitude, i.e., each entry of the weight matrices is bounded in $[-1, 1]$.
By this construction, define 
an embedding vector 
$\mathtt{embed}_h(S_h^t) \in \RR^{|\cL|}$ as
\begin{align}
\mathtt{embed}_h(S_h^t) = \hat f _h^{w} \bigg(\frac{1}{L'}\bigg(\sum_{i=1}^{t-1}\sum_{j=0}^{H}\hat \psi_h^{*}(z_{j}^i) + \sum_{j'=0}^{h-1}\hat \psi_h^{*}(z_{j'}^{t})\bigg)\bigg), \label{eq:embed_approx}
\end{align}
where $\hat f _h^{w}$ denotes the vector-valued mapping whose entries are $\{\hat f_{h,i}^{w}\}_{i \in [|\cL|]} $.

Next, we feed $\mathtt{embed}_h(S_h^t)$ into the softmax layer and obtain an estimator of $g_h^* (S_h^t)$. 
For any prompt $S_h^t$ with length $L'\leq L$,
the $\ell_1$-approximation error 
is bounded by 
\begin{align}
&\norm[\big]{g^*_h(S_h^t) -\mathtt{softmax}\big({\mathtt{embed}}_h(S_h^t)  /\tau\big)}_1 \nonumber\\
&\quad \leq \norm[\bigg]{g^*_h(S_h^t) -\mathtt{softmax}\bigg(  \hat f _h^{w} \bigg(\frac{1}{L'}\sum_{(i,j)=(1,0)}^{(t, h-1)} \psi_h^{*}(z_{j}^i)\bigg)  \bigg)}_1  \notag \\
&\quad \qquad + \norm[\bigg]{\mathtt{softmax}\bigg(  \hat f _h^{w} \bigg(\frac{1}{L'}\sum_{(i,j)=(1,0)}^{(t, h-1)} \psi_h^{*}(z_{j}^i)\bigg)  \bigg) - \mathtt{softmax}\bigg(  \hat f _h^{w} \bigg(\frac{1}{L'}\sum_{(i,j)=(1,0)}^{(t, h-1)} \hat \psi_h^{*}(z_{j}^i)\bigg)  \bigg)}_1 \notag \\
&\quad  \leq 2   \epsilon_w + 2 \cdot 256^{D_w}\cdot \epsilon_\psi, \label{eq:2 errors}
\end{align}
where $256$ appears because it is the total number of parameters in each layer of $\hat f_{h,i}^w$. 
Here, the first inequality follows from the triangle inequality. In the second inequality, 
we employ Lemma~\ref{lem:softmax_lip}, which states that $\mathtt{softmax}(\cdot)$
is Lipschitz continuous with parameter $2$ in terms of the $\ell_1$-$\ell_\infty$ norm pair. 
To bound the first term, 
we combine Lemma~\ref{lem:softmax_lip} and \eqref{eq:apply_lemma_approx},  
which shows that the first term is no more than  $2   \epsilon_{w}$.
To bound the second term,
we note that fact that each $f_{h, i}^{i}$, as a $D_{w}$-layer MLP, is a Lipschitz continuous function in terms of the $\ell_\infty$-norm. 
The Lipschitz parameter is bounded by $256^{D_w}$ because the vectorized $\ell_1$-norm of the weight matrix in each layer is bounded by $256$, which is a result of Lemma~\ref{lem:network_approx}. 
Then we combine Lemma~\ref{lem:softmax_lip}, Lipschitzness of $\hat f_h^w$, and \eqref{eq:apply_lemma_approx} to obtain
\begin{align*}
    &  \norm[\bigg]{\mathtt{softmax}\bigg(  \hat f _h^{w} \bigg(\frac{1}{L'}\sum_{(i,j)=(1,0)}^{(t, h-1)} \psi_h^{*}(z_{j}^i)\bigg)  \bigg) - \mathtt{softmax}\bigg(  \hat f _h^{w} \bigg(\frac{1}{L'}\sum_{(i,j)=(1,0)}^{(t, h-1)} \hat \psi_h^{*}(z_{j}^i)\bigg)  \bigg)}_1 \notag \\
    & \qquad \leq 2 \cdot \max_{i \in [|\cL| ]} \biggl |   \hat f _{h,i} ^{w} \bigg(\frac{1}{L'}\sum_{(i,j)=(1,0)}^{(t, h-1)} \psi_h^{*}(z_{j}^i)\bigg)    -  \hat f _h^{w} \bigg(\frac{1}{L'}\sum_{(i,j)=(1,0)}^{(t, h-1)} \hat \psi_h^{*}(z_{j}^i)\bigg)   \bigg|   \notag \\
    & \qquad  \leq 2 \cdot  256^{D_w} \cdot \biggl \|  \frac{1}{L'}\sum_{(i,j)=(1,0)}^{(t, h-1)} \psi_h^{*}(z_{j}^i)     -   \frac{1}{L'}\sum_{(i,j)=(1,0)}^{(t, h-1)} \hat \psi_h^{*}(z_{j}^i)   \bigg\|_{\infty} \notag\\
    &\qquad < 2 \cdot  256^{D_w}  \cdot \epsilon_{\psi},
\end{align*}

In summary, for any given $\epsilon_{w}$ and $\epsilon_{\psi}$, there exist MLPs $\hat \psi_h^*$ and $ \{ \hat f_{h,i}^{w} \}_{i \in [|\cL|]} $ such that 
$$
\norm[\big]{g^*_h(S_h^t) -\mathtt{softmax}\big({\mathtt{embed}}_h(S_h^t)  /\tau\big)}_1 \leq 2 \epsilon_w + 2 \cdot 256^{D_w}\cdot \epsilon_\psi.
$$
These MLPs have at most $D_{\psi}$ and and $D_{w}$ layers respectively, where $D_{\psi}$ and $D_w$ are defined in \eqref{eq:max_depth}.
In each layer, there are $16$ neurons and the weights are all in $[-1 , 1]$.

In conclusion, the above analysis can be extended to bound the error 
$$
\bigl \| g_{\tilde{h}}^*(S_h^t) - \mathtt{softmax}(\mathtt{embed}_{\tilde{h}}(S_h^t)/\tau) \big \|_1
$$
for all \(\tilde{h} \in \{0, \ldots, H\}\) using the same upper bound. This means that \(\mathtt{softmax}(\mathtt{embed}_{\tilde{h}}(S_h^t)/\tau)\) serves as an estimator for \(g_{\tilde{h}}^*(\cdot)\) with uniform precision across all \(\tilde{h}\). This generalization is possible because the analysis in \eqref{eq:2 errors} is based solely on the Lipschitz continuity of \(\mathtt{softmax}\) and the approximation errors established in \eqref{eq:apply_lemma_approx}. Therefore, the same error bounds apply when substituting any \(\tilde{h}\) for \(h\) in \(g_h^*(\cdot)\) and \(\mathtt{embed}_h(\cdot)\), ensuring that the error analysis is valid for any \(\tilde{h} \in \{0, \ldots, H\}\). 

\vspace{2mm}

\noindent \textbf{Step 3: Construct the transformer module $G_h$.} 
In the previous step, we construct  \( \mathtt{embed}_h(S_h^t) \) that approximates \( g^*_h(S_h^t) \). 
However, the actual input of the transformer module \( G_h \) is \( (S_h^t, \hat p_h, F_3, F_4) \in \mathbb{R}^{L' \times (2+2|\calL|)} \) and the expected output is \( (S_h^t, \hat p_h, \mathbf{1}_{L'}^\top\mathtt{embed}_h(S_h^t), F_4) \in \mathbb{R}^{L' \times (2+2|\calL|)} \). In this final step,  we explicitly construct the transformer module \( G_h \) that preserves $S_h^t, \hat p_h, F_4$ and substitute $\mathbf{1}_{L'}^\top\mathtt{embed}_h(S_h^t)$ in the third column.

To achieve such a goal, we need to first show that $\mathtt{embed}_h(\cdot )$  can be implemented by transformer blocks. 
Then we need to show that these transformer blocks can be put in a larger transformer with the desired input-output relationship. To achieve the first goal, we apply Proposition~\ref{prop:ffn_transformer} separately to the approximation modules $\hat \psi_h^*$ and $(\hat f_{h,1}^w, \cdots, \hat f_{h,|\calL|}^w)$, and connect them with the average module that is realized by a single-head attention layer. In particular, we apply the scaling trick in Proposition~\ref{prop:ffn_transformer} to bypass the normalization layers in the transformer blocks. We specify the expression of these scaling at the end of our proof. 

Our next step is to adjust weight matrices in each transformer block to preserve $S_h^t, \hat p_h, F_4$ and substitute $\mathtt{embed}_h(S_h^t)$ in the third column. 
To this end, we introduce the notion of a residual ReLU module, which is an FF layer that only keeps some desired columns of the input matrix.
Then we can concatenate \(\mathtt{embed}_h (\cdot)\) with 
a residual ReLU module to achieve the desired functionality.

\begin{lemma}[Residual ReLU module] \label{lem:residual_relu} 
    Let \(X\in \mathbb{R}^{m \times r}\) denote the input, and \(\calJ \subseteq [r]\) denote a set of indices. Then there exists a \ac{ff} layer with weight matrices \(W_{\mathrm{ff},1} \in \mathbb{R}^{r \times (2|\calJ|)}\) and \(W_{\mathrm{ff},2} \in \mathbb{R}^{(2|\calJ|) \times |\calJ|}\) such that the output matrix only keeps those columns with indices $i \in \cJ$.  Specifically, we have
$$
\mathtt{ReLU}(XW_{\mathrm{ff},1})W_{\mathrm{ff},2} = X_{:, i\in \calJ} \in \mathbb{R}^{m\times |\calJ|}.
$$
Moreover, these weight matrices satisfy
where \(\|W_{\mathrm{ff},1}\|_\mathrm{F} = \sqrt{2|\calJ|} \) and \(\|W_{\mathrm{ff},2}\|_\mathrm{F} =\sqrt{2|\calJ|}\).
    
\end{lemma}
\begin{proof}
    See Appendix~\ref{app:residual_relu} for details.
\end{proof}

To show that we can fuse $\mathtt{embed}_h(\cdot )$ with a residual ReLU module that preserves the submatrix \((S_h^t, \hat p_h, F_4) \in \mathbb{R}^{L' \times (2 + |\calL|)}\) through each FF layer, 
it suffices to show that a residual ReLU module can work together with a feed-forward layer and a MHA layer. 
The reason is that $\mathtt{embed}_h(\cdot ) $ is a composition of $D_{\psi}$ FF layers, an attention layer, and $D_{w}$ FF layers. 
If each layer of $\mathtt{embed}_h(\cdot ) $ can be added to a larger network which keeps \((S_h^t, \hat p_h, F_4) \in \mathbb{R}^{L' \times (2 + |\calL|)}\)  unchanged, then we can apply this argument to all layers of $\mathtt{embed}_h(\cdot )$  and obtain the desired network. Thus, in the following, we focus only on a FF layer and a MHA layer.

Notice that permutation of the columns can be achieved by a linear FF layer. It suffices to put columns corresponding to $\mathtt{embed}_h(\cdot )$ to the first $|\cL| $ columns. 
That is, we can study whether  the transformation
$$
(X,S_h^t, \hat p_h, F_4) \longrightarrow (X',S_h^t, \hat p_h, F_4),  
$$
can be achieved by a transformer block,
where $X \in \RR^{L'\times d}$ for some $d$, and  $X'  $ is obtained by $X $ through an FF or MHA layer. 
For ease of presentation, we denote $(S_h^t, \hat p_h, F_4)$ by a matrix $Y \in \RR^{L' \times d'}$ and study this problem with abstraction, where $d'$ is the number of columns in $Y$. 
Then  we consider  $X ' = \mathtt{ff}( (X,Y), W_{\mathrm{ff}}, b_{\mathrm{ff}})$ or $X' = \mathtt{mha} ((X,Y), W_{\mathrm{mha}})$.

First, we assume $X ' = \mathtt{ff}((X,Y), W_{\mathrm{ff}}, b_{\mathrm{ff}})$, where $W_{\mathrm{ff}} = ( W_{\mathrm{ff},1}, W_{\mathrm{ff},2}) $ and $b_{\mathrm{ff}} = ( b_{\mathrm{ff},1}, b_{\mathrm{ff},2}) $.
Using Lemma~\ref{lem:residual_relu}, we will construct weights $\overline W_{\mathrm{ff}}, \overline  b_{\mathrm{ff}}$ such that 
$
(X', Y) = \mathtt{ff} ( (X, Y), \overline W_{\mathrm{ff}}, \overline  b_{\mathrm{ff}} ) .
$
In particular, we apply Lemma \ref{lem:residual_relu} to $(X, Y)$ with $\cJ = \{d+1, \ldots, d+d'\} $, where $\calJ$ refers to the column indices corresponding to $Y$. 
Then there exist weight matrices $W_{\mathrm{ff}, 1}^{Y}$ and $W_{\mathrm{ff}, 2}^{Y}$ such that 
\begin{align*} 
Y = \mathtt{ReLU}  \bigl ( (X, Y) W_{\mathrm{ff}, 1}^{Y}) W_{\mathrm{ff}, 2}^{Y}.  
\end{align*}
Notice that $W_{\mathrm{ff}, 1}^{Y} $ has size $(d+d') \times (2d')$ and $W_{\mathrm{ff}, 2}^{Y} $ has size $2d' \times d'$. 
Whereas $W_{\mathrm{ff},1}$ has $d$ rows. 
Now we define 
$$
\overline W_{\mathrm{ff},1} = 
\begin{pmatrix}
    W_{\mathrm{ff},1}  &W_{\mathrm{ff},1}^{Y} \\   
\end{pmatrix},
\qquad  
\overline W_{\mathrm{ff},2}=
\begin{pmatrix}
    W_{\mathrm{ff},2}& 0 \\ 0& W_{\mathrm{ff},2}^{Y}
\end{pmatrix}, \qquad \overline b_{\mathrm{ff},1} = (b_{\mathrm{ff},1}, \mathbf{0}), \qquad \overline b_{\mathrm{ff},2} = (b_{\mathrm{ff},2}, \mathbf{0}). 
$$
Here in $\overline W_{\mathrm{ff},1}$ we add $d'$ all-zero rows below $W_{\mathrm{ff},1}$ to construct a valid matrix. 
As defined in \eqref{eq:ffn_bias}, we can directly calculate the FF layer with parameters $\overline W_{\mathrm{ff}} $ and $\overline b_{\mathrm{ff}} $ and  have 
\begin{align*}
    \mathtt{ff}\bigl ( (X , Y) , W_{\mathrm{ff}}, b_{\mathrm{ff} })& = \mathtt{ReLU}\bigl ( (X, Y)  \overline W_{\mathrm{ff},1} + \mathbf{1}^\top \overline b_{\mathrm{ff} , 1 } \bigr )\overline W_{\mathrm{ff},2} + \mathbf{1}^\top \overline b_{\mathrm{ff} , 2 } \notag \\
    & =   \begin{pmatrix}
    \mathtt{ReLU} ( (X,Y) W_{\mathrm{ff},1} + \mathbf{1}^\top b_{\mathrm{ff},1}), & \mathtt{ReLU} ( (X,Y) W_{\mathrm{ff},1}^Y )  
\end{pmatrix}   \overline W_{\mathrm{ff},2} + \mathbf{1}^\top \overline b_{\mathrm{ff} , 2 } \notag \\
& =(X', Y).
\end{align*}
Therefore, we construct an FF layer such that we change $X $ to $X'$ and keep $Y$ unchanged.

It remains to consider the case where $X' = \mathtt{mha} ((X,Y), W_{\mathrm{mha}})$. 
We show that $(X', Y)$ can be implemented by a transformer block starting from $(X, Y)$. 
When $X' = \mathtt{mha} (X, W_{\mathrm{mha}})$, we can augment the three matrices of $W_{\mathrm{mha}}$ by adding zeros such that 
$( X ', {\bf 0}) = \mathtt{mha} ( (X,Y), \overline W_{\mathrm{mha}})$, where 
$ \overline W_{\mathrm{mha}}$ is obtained from $W_{\mathrm{mha}}$ by adding zeros, and ${\bf 0}$ is a zero matrix that has the shape as $Y$.
Then, with a generalized residual link, we have 
$$
(X', Y) = \mathtt{mha} ( (X,Y), \overline W_{\mathrm{mha}}) + (X', Y)   \begin{pmatrix}
    {\bf 0} & {\bf 0}\\
    {\bf 0} & I
\end{pmatrix},
$$
where the $2\times 2$ block matrix plays the same role as $\gamma_1  $ in \eqref{eq:transformer_block}.

Therefore, we conclude that an FF and MHA layer that maps $X $ to $X'$ can be augmented to a layer that maps $(X,Y) $ to $(X', Y)$. 
Now we apply this argument recursively for $\mathtt{embed}_h(\cdot )$.
The input matrix is $(F_3, S_{h}^t, \hat p_h, F_4) $ and the desired output is $(\mathbf{1}_{L'}^\top\mathtt{embed}_h(S_{h}^t ) ,S_{h}^t, \hat p_h, F_4)$. 
In particular, $\mathtt{embed}_h(S_{h}^t ) $ is an MLP of $S_{h}^t$ that consists of $D_{\psi} + D_{w}$ FF layers in total and a MHA layer. 
Thus, we can apply the above argument with $X = F_3$, $Y = (S_{h}^t, \hat p_h, F_4) $ and $X' $ being the intermediate outputs of $\mathtt{embed}_h(\cdot)$. 
We conclude that such a mapping can be implemented by a transformer with $D_{\psi} + D_{w} + 1 $ blocks. 

Finally, we need to permute $(\mathbf{1}_{L'}^\top\mathtt{embed}_h(S_{h}^t ) ,S_{h}^t, \hat p_h, F_4)$ to $( S_{h}^t, \hat p_h,\mathbf{1}_{L'}^\top\mathtt{embed}_h(S_{h}^t ,  F_4)$, which can be achieved by another linear layer. 
Therefore, the desired $G_h$  can be implemented by $D_{\psi} + D_{w} + 2 $ transformer blocks.

We compute the scaling factors from Proposition~\ref{prop:ffn_transformer} when implementing each approximation module $\hat{f}^w_{h,i}$ and $\hat f_{h}^w=\{\hat f_{h,i}^w\}_{i=1}^{|\calL|}$ using transformers blocks. Furthermore, we conclude this proof by commenting on the width and norm of weight matrices of $G_h$.
 
In the construction of $G_h$, we first note that since we approximate each coordinate of the output distribution individually using \(\hat{f}^w_{h,i}\) for \(i \in [|\calL|]\), we horizontally stack the weight matrices for each \(\hat{f}^w_{h,i}\) at corresponding layers. Therefore, we derive an upper bound of the hidden layer size as $d_F \leq 16|\calL|+4+2|\calL|$, where $16|\calL|$ follows from the transformer implementation of $\hat f_h^w$, and $2(2+|\calL|)$ follows from the preservation of columns $S_h^t, \hat p_h, F_4$. 

To implement $\hat \psi_h^*$ using transformer blocks while preserving the inputs, we apply Proposition~\ref{prop:ffn_transformer} by setting the scaling factors $\{B_\ell^\psi\}_{\ell=0}^{D_\psi}$ as
 \begin{align*}
     B^\psi_0= B^\psi_{D_\psi}&=\sqrt{B^2+H^2+2|\calL|(C_A+1)^2},\\
     B^\psi_\ell&= \sqrt{d_F\cdot((B^\psi_0+1)\cdot 16^\ell)^2+B^2+H^2+|\calL|(C_A+1)^2}, \text{ for }\ell\in [D_\psi-1],
 \end{align*}
 where $B^\psi_0$ normalizes the input $(S_h^t,\hat p_h, F_3, F_4)$ row-wisely, each $B^\psi_\ell$ keeps the intermediate outputs in a unit ball. Finally, $B^\psi_{D_\psi}$ follows since $C_A$ upper bounds the magnitude of $\psi_h^*$ by Assumption~\ref{assumption: smooth}, which controls the magnitude of each row in $\hat \psi_h^*(S_h^t)$. This scaling is absorbed into the average module realized via a \ac{mha} layer. Next, we consider the transformer implementation of $\hat f_{h}^w$. Similar to the implementation of $\hat \psi_h^*$, we apply Proposition~\ref{prop:ffn_transformer} by setting the scaling factors $\{B_\ell^w\}_{\ell=0}^{D_w}$ as  
\begin{align*}
     B^w_0= B^w_{D_w}&=\sqrt{B^2+H^2+2|\calL|(C_A+1)^2},\\
     B^w_\ell&= \sqrt{d_F\cdot((B^w_0+1)\cdot 16^\ell)^2+B^2+H^2+|\calL|(C_A+1)^2}, \text{ for }\ell\in [D_w-1].
 \end{align*}
 Finally, we compute the maximum network weight for the module $G_h$ as
\begin{align*}
    \max_{0\leq i\leq D_\psi, 0\leq j\leq D_w} \{B^\psi_i,B^w_j\}\cdot\sqrt{ d_F^2+ 2(2+|\calL|)}\leq C_F\cdot B_0\cdot16^{\max \{D_\psi,D_w\}}\cdot |\calL|^{3/2},
\end{align*}
where $B_0 = \sqrt{B^2+H^2+|\calL|\cdot C_A^2}$ and $C_F>1$ is a absolute constant.
\end{proof}

\subsubsection{Proof of Lemma \ref{lemma: trapezoid_module}} \label{app: trapezoid_module}
\begin{proof}

Recall that for each $h\in [H]$ and $\epsilon\in(0,1/2)$, we want to construct a MLP to that implements the trapezoid-shaped function 
    \begin{align*}
    f_h(x)=\begin{cases}
        1 &\text{for } |x-h|\leq \epsilon,\\
        1-(|x-h|-\epsilon)/(1-2\epsilon) &\text{for } \epsilon< |x-h|\leq 1-\epsilon, \\
        0 &\text{otherwise.} 
    \end{cases}
\end{align*}
\begin{figure}[h]
    \centering
    \includegraphics[width = 0.95\textwidth]{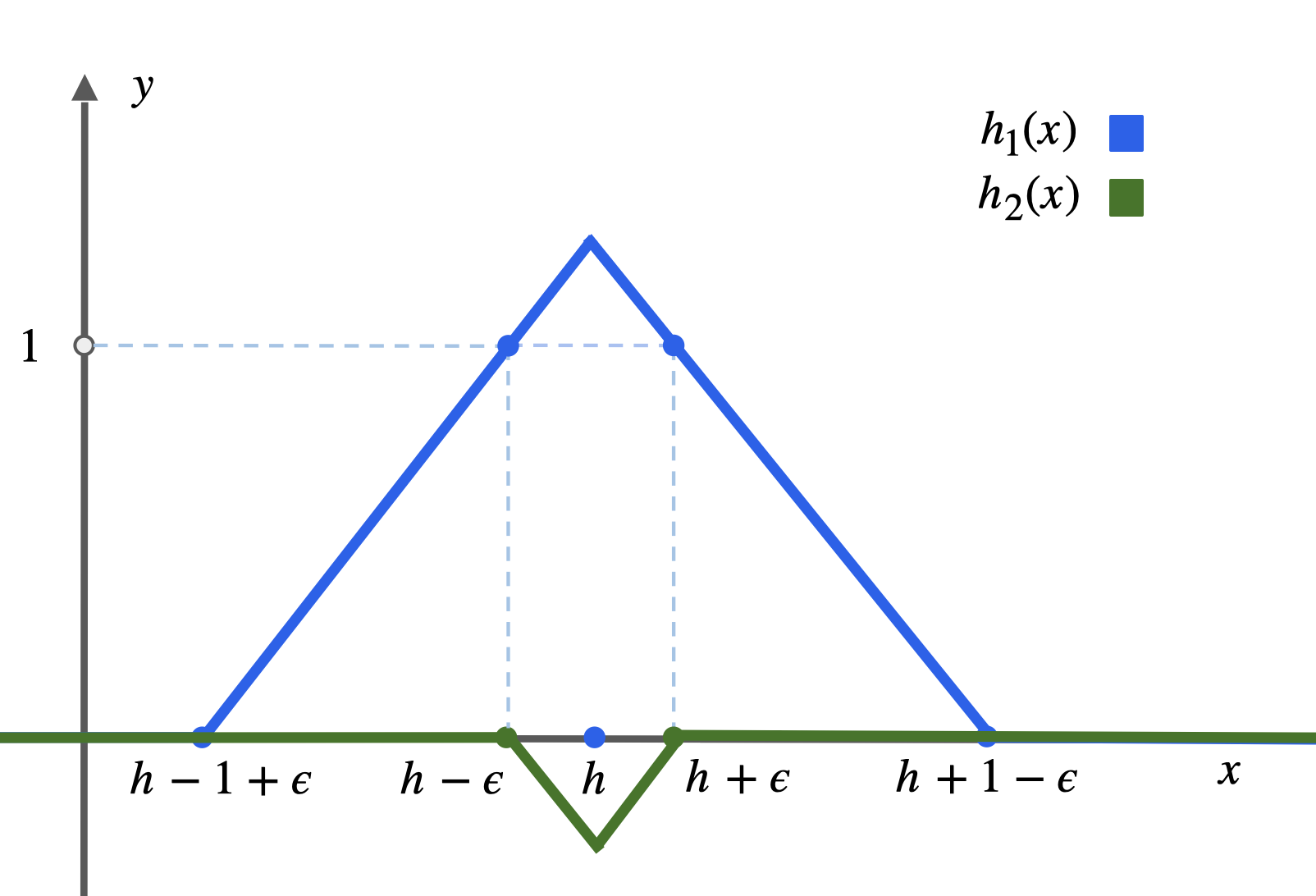}
    \caption{An illustration of the trapezoid functions $f_h(x)$ as the sum of two functions $h_1(x)$ and $h_2(x)$. Intuitively, \( h_1(x) \) is a large upward-pointing triangle, and \( h_2(x) \) is a smaller downward-pointing triangle that mirrors \( h_1(x) \) within the interval \( x \in [-\epsilon, \epsilon] \). The downward slope of \( h_2(x) \) cancels out the upward slope of \( h_1(x) \) within this interval, resulting in a flat region. Therefore, the sum of the two functions creates a trapezoid shape.}
    \label{fig:trapezoid-2-app}
\end{figure}

This function can be expressed as the sum of two triangular-shaped functions $h_1(x)+h_2(x)$, where we define $h_1$ and $h_2$ as 
\begin{align*}
h_1(x)& =\begin{cases}
        1-(|x-h|-\epsilon)/(1-2\epsilon) &\text{for } |x-h|\leq 1-\epsilon ,\\
        0 &\text{else,} 
    \end{cases} \\
h_2(x)& =\begin{cases}
     (|x-h|-\epsilon)/(1-2\epsilon) ~~~~~
 &\text{for } |x-h|\leq \epsilon, \\
        0 ~~~~~&\text{else.} 
    \end{cases}
\end{align*}
Here $h_1$ is nonzero when $|x - h| \leq 1 - \epsilon$ with $h_1 ( h ) = (1 - \epsilon )/ (1- 2 \epsilon) $, $h_1 ( h + 1 - \epsilon ) = h_1 ( h - 1 + \epsilon ) = 0$. 
Thus $h_1$ is a triangle pointing upwards.  
Similarly $h_2$ is a triangle pointing downwards with $h_2 (h) = -\epsilon / (1 -2 \epsilon)$ and  $h_2 ( h-\epsilon) = h_2 ( h + \epsilon) = 0$. See Figure~\ref{fig:trapezoid-2-app} for an illustration of $f_h$, $h_1$, and $h_2$.   

Furthermore, both $h_1$ and $h_2$ are piecewise linear functions with four linear pieces, and thus can be written as a sum of four ReLU functions. 
In particular, we can write 
$h_1 $ as 
\begin{align*}
    h_1(x) &= \frac{1}{1-2\epsilon} \Big(\mathtt{ReLU}(x-(h+1-\epsilon)) +\mathtt{ReLU}(h-1+\epsilon-x)-\mathtt{ReLU}(x-h)-\mathtt{ReLU}(h-x) +1-\epsilon \Big),   
\end{align*}
which can be verified by direct calculation. 
Thus, this function can be written as a single feed-forward layers 
with parameters 
\begin{align*}
    W_{\mathrm{ff},1}^1 = (1,-1,1,-1,0),\quad &b_{\mathrm{ff},1}^1 = (-(h+1-\epsilon),h-1+\epsilon,-h,h,1-\epsilon),\\
    W_{\mathrm{ff},2}^1 = (1,1,-1,-1,1)^\top/(1-2\epsilon), \quad &b_{\mathrm{ff},2}^1=\mathbf{0}.
\end{align*}
Similarly, we can write $h_2$ as 
$$
h_2(x) = \frac{1}{1-2\epsilon} \Big(\mathtt{ReLU}(x-(h+\epsilon)) +\mathtt{ReLU}(h-\epsilon-x)-\mathtt{ReLU}(x-h)-\mathtt{ReLU}(h-x) -\epsilon \Big), 
$$
which can be written as a feed-forward layer with parameters 
\begin{align*}
    W_{\mathrm{ff},1}^2 = (1,-1,1,-1,0),\quad &b_{\mathrm{ff},1}^2 = (-(h+\epsilon),h-\epsilon,-h,h,-\epsilon),\\
    W_{\mathrm{ff},2}^2 = (1,1,-1,-1,1)^\top/(1-2\epsilon),\quad& b_{\mathrm{ff},2}^2=\mathbf{0}.
\end{align*}

Finally, by directly concatenating the corresponding weight matrices for \(h_1(\cdot)\) and \(h_2(\cdot)\), we can implement the function \(f_h\) using a single feedforward (\ac{ff}) layer. The width of the weight matrix in this \ac{ff} layer is at most 10, and the magnitude of the weights is bounded by \(H + 1\). 
Thus, we conclude the proof.
\end{proof}

\subsection{Proofs of the Remaining Auxiliary Lemmas} \label{app:proof_other_lemmas}

In the following, we prove the remaining auxiliary lemmas, which include  Lemma~\ref{lem: kl decomp} used in the proof of Corollary \ref{cor: rate with error}, and Lemmas \ref{lemma: product_module} and \ref{lem:residual_relu} used in the proofs in Appendix \ref{proof:app:construction}.

\subsubsection{Proof of Lemma~\ref{lem: kl decomp}} \label{app: kl decomp}

\begin{proof}

Let $\pt_\mathrm{CoT}(n)$ denote a fixed prompt, then according to the chain rule of KL divergence, we have that 
\begin{align}\label{eq: chain rule}
    &\KL\bigl(\PP(z_{1:H}^{\mathrm{test}} =\cdot \given \pt_{\mathrm{CoT}}(n)) , \PP_{\hat \rho}(z_{1:H}^{\mathrm{test}}=\cdot \given \pt_{\mathrm{CoT}}(n)) \bigr)  \\
    & \quad = \sum_{h=1}^H \E_{{z_{1:h-1}^\mathrm{test} \sim \PP(\cdot \given \pt_\mathrm{CoT}(n))}} \Big[\KL\bigl(\PP(z_{h}^{\mathrm{test}} =\cdot \given \pt_{\mathrm{CoT}}^h(n)) , \PP_{\hat \rho}(z_{h}^{\mathrm{test}} =\cdot \given \pt_{\mathrm{CoT}}^h(n)) \bigr)\Big].  \nonumber
\end{align}
Here the chain rule of KL divergence states that 
\begin{align*}
& \KL\bigl(  \PP( X_2 =  \cdot , X_3 = \cdot   \given  X_1 ) , \overline   \PP ( X_2 =  \cdot , X_3 = \cdot   \given  X_1 ) \bigr)  \notag \\
& \quad = \KL( \PP( X_2 =  \cdot ,     \given  X_1 ) , \overline   \PP ( X_2 =  \cdot ,    \given  X_1 ) \bigr)  + \EE_{ X_2 \sim \PP( \cdot \given X_1 ) } \bigl [ \KL( \PP( X_3 =  \cdot      \given  X_1, X_2  ) , \overline  \PP ( X_3=  \cdot ,    \given  X_1, X_2  ) \bigr)\bigr ] 
\end{align*}
holds for any three random variables $(X_1, X_2, X_3) $ with two joint distributions $\PP$ and $\overline \PP$. 
Then according to data processing inequality, we have that
\begin{align}
    &\KL\bigl(\PP(z_{1:H}^{\mathrm{test}} =\cdot \given \pt_{\mathrm{CoT}}(n)) , \PP_{\hat \rho}(z_{1:H}^{\mathrm{test}}=\cdot \given \pt_{\mathrm{CoT}}(n)) \bigr) \label{eq: dpi}\\
    & \quad \geq \KL\bigl(\PP(z_H^{\mathrm{test}} =\cdot \given \pt_{\mathrm{CoT}}(n)) , \PP_{\hat \rho}(z_H^{\mathrm{test}}=\cdot \given \pt_{\mathrm{CoT}}(n)) \bigr). \nonumber
\end{align}
Notice that $y^\mathrm{test} 
 = z_{H} ^{\mathrm{test}} $. 
 Combing \eqref{eq: chain rule} with \eqref{eq: dpi}, we have that 
\begin{align}
    &\KL\bigl(\PP(y^{\mathrm{test}} =\cdot \given \pt_{\mathrm{CoT}}(n)) , \PP_{\hat \rho}(y^{\mathrm{test}}=\cdot \given \pt_{\mathrm{CoT}}(n)) \bigr) \nonumber \\
    & \quad \leq \sum_{h=1}^H \E_{{z_{1:h-1}^\mathrm{test} \sim \PP(\cdot \given \pt_\mathrm{CoT}(n))}} \Big[\KL\bigl(\PP(z_{h}^{\mathrm{test}} =\cdot \given \pt_{\mathrm{CoT}}^h(n)) , \PP_{\hat \rho}(z_{h}^{\mathrm{test}} =\cdot \given \pt_{\mathrm{CoT}}^h(n)) \bigr)\Big].  \notag 
\end{align}
Therefore, we conclude the proof.
\end{proof}

\subsubsection{Proof of Lemma \ref{lemma: product_module}}\label{app: product_module}

\begin{proof}
In this proof, we extend Proposition III.3 in  \cite{elbrachter2021deep} to construction a sequence of \ac{ff} modules such that $f:\mathbb{R}^{r}\times \mathbb{R} \to \mathbb{R}^{r}$, where we want $f(x_{1:r}, y)\approx (x_1 y,\cdots,x_r y)$.
For simplicity, we denote the input as $X^0=(x_1, \cdots, x_r,y)\in \mathbb{R}^{r+1}$.
By leveraging the construction by \cite{elbrachter2021deep}, we define a set of  matrices $\{A_i, b_i\}_{i=1}^{m+2}$ as follows.

First, we define 
$A_{0,1} = (1,-1,1,-1)/(2M)$ and $A_{0,2} = (1,-1,-1,1)/(2M)$.
Then for each $ i \in [r] $, we define $A_{1,i} \in \mathbb{R}^{(r+1)\times 4}$ by setting the $i$-th row as
$A_{0,1}$, $(r+1)$-th row as $A_{0,2}$ and fill the rest with zero. For example, $A_{1,1}=(A_{0,1},\mathbf{0}_{ (r-1)\times 4},A_{0,2})^\top\in \mathbb{R}^{(r+1) \times 4}$. We stack the matrices $\{A_{1,i}\}_{i=1}^r$ horizontally to form $A_1 = (A_{1,1},\cdots, A_{1,r}) \in \mathbb{R}^{(r+1)\times 4r}$.
Then we  define the bias $b_1=\mathbf{0}\in \mathbb{R}^{4r }$.

To define $\{ A_i, b_i\}_{i=2}^{m+2}$, we first  define $\{ A_{\ell}', b_{\ell}' \}_{\ell  = 2}^{m+1}$ by letting 
\begin{align*}
    A'_2 = 
\begin{pmatrix} 
1&1&1&0&0\\ 
1&1&1&0&0\\ 
0&0&-1&1&1\\ 
0&0&-1&1&1
\end{pmatrix}\in \mathbb{R}^{4\times 5}, \qquad  
& b'_2=\begin{pmatrix}
    0 &-1/2 &0 & 0&-1/2
\end{pmatrix}, \\
A'_\ell = \frac{1}{2M}
\begin{pmatrix}
    1/2&1/2&-1/2&0&0\\
    -1&-1&1&0&0\\
    0&0&1&0&0\\
    0&0&1/2&1/2&1/2\\
    0&0&-1&-1&-1
\end{pmatrix} \in \mathbb{R}^{5\times 5}, \qquad 
& b'_\ell=\begin{pmatrix}
    0 & -2^{3-2\ell} & 0 & 0 & -2^{3-2\ell}
\end{pmatrix}, 
\end{align*}
for $3\leq \ell \leq m+1$. Finally, we define  $A'_{m+2} = (-1/2, 1, 1,1/2, -1)^\top \in \mathbb{R}^{5\times 1}$ and $b'_{m+2}=0 \in \RR $.
Next,  for any  $i \in \{ 2, \ldots,   m+2\}$, we define  $A_i$ and $b_i$ as 
$$
A_i = \mathtt{diag}(\underbrace{A'_i, \cdots, A'_i}_{\displaystyle r})  , \qquad  b_i = \big (\underbrace{ b'_i,\cdots, b'_i}_{\displaystyle r} \bigr )  ,
$$
where each $A_i$ is obtained by 
constructing a block-diagonal matrix with $A_i'$ being the diagonal blocks, 
 and each $b_i$ is obtained by stacking $b'_i$ horizontally for $r$ times. Therefore we have $A_2\in \mathbb{R}^{4r\times 5r}$,  $A_{m+2}\in \mathbb{R}^{5r \times r}$, and $A_\ell\in \mathbb{R}^{5r\times 5r}$ for $3\leq \ell \leq m+1$. 
 Besides, we have $b_i \in \mathbb{R}^{1\times 5r }$ for each $2\leq \ell \leq m+1$, and $b_{m+2} = {\bf 0} \in \mathbb{R}^{r }$.
 Note that in the construction of Proposition III.3 by \cite{elbrachter2021deep}, a scalar multiplication module is used to restore the normalization \(1/2M\) introduced by the first weight matrix \(A_1\). In our approach, we instead scale the last weight matrix \(A_{m+2}\) by \(M^2\), thereby eliminating the need for a separate scalar multiplication module.

Note that by setting $r=1$, we recover the exact construction by \cite{elbrachter2021deep} in Proposition III.3. 
We use $f_1:\mathbb{R}\to \mathbb{R}$ to denote 
such a network, which is an MLP with parameters $\{A_i', b_i'\}_{i=1}^{m+2}$, where $A_1' = (A_{0,1} , A_{0,2} )^\top \in \RR^{2\times 4}$, $b'_1 = {\bf 0} \in \RR^{4}$, and $\{ A_i', b_i'\}_{i=2}^{m+2}$ are defined above. 
By the construction of the weight matrices
$\{ A'_i, b'_i\}_{i=1}^{m+2}$,
the MLP with these parameters yields a vector-valued mapping 
$f \colon \RR^{r} \times \RR \rightarrow \RR$ such that 
\begin{align} \label{eq:define_mapping_f_product}
    f(x,y) = \big(f_1(x_1,y),\cdots,f_1(x_r,y)\big), \qquad \forall x \in \RR^{r}, y \in \RR.  
\end{align}

As shown in Proposition III.3 in \cite{elbrachter2021deep}, when the depth of $f_1$, i.e., $m+2$, is bounded by 
$C_p(\log M +\log (1/\epsilon))$ for some constant $C_p>0$,
$f_1 $ is a good approximator of the product operation in the sense that 
 $|f_1(a,b)-ab|<\epsilon$ for any $a,b\in [-M,M]$.

 Therefore, $f$ constructed in \eqref{eq:define_mapping_f_product} using weight matrices $\{ A_i, b_i \}_{i=1}^{m+2}$ satisfies 
$$\|f(x,y)-(x_1y,\cdots,x_ry)\|_\infty = \max_{i\in [r]} |f(x_i,y)-x_iy|<\epsilon .
$$
The depth of $f$ is no more than $C(\log M +\log (1/\epsilon))$, and the maximum dimension of the hidden neurons is $5r$. The maximum magnitude of the intermediate weight matrices is bounded by 1, i.e., \(\max \{ \| A_i \|_{\infty}, \| b_i \|_{\infty} \} \leq 1\) for all \(i \in [m+1]\). Additionally, \(\|A_{m+2}\|_\infty \leq M^2\) due to the direct scaling, which replaces the scalar multiplication module.

To 
bound the Frobenius norms of weight matrices and bias vectors in $f$, by direct computation, we have  
$$\max_{\ell \in [m+2]}\{\|A_\ell\|_\mathrm{F}, \|b_\ell\|_\mathrm{F}\} = \max\{\|A_2\|_\mathrm{F},\|A_{m+2}\|_\mathrm{F}\}= \max\{\sqrt{12r},H^2\cdot \sqrt{5r}\}.$$ 
Finally, we compute the row-wise $\ell_2$-norm for each intermediate output 
$$
X^{\ell}  = \mathtt{ReLU}\big( X^{\ell -1}A_d+\mathbf{1}^\top b_{\ell}\big), \qquad \forall  \ell \in [m+2],
$$
where the initial input is given by $X^0=(x_1,\cdots,x_r,y)\in \mathbb{R}^{r+1}$.
Given that each entry of \( X^{\ell -1} \) lies within the interval \([0,1]\), and noting that by construction \(\|\mathbf{1} A_\ell   + b_\ell \|_\infty \leq 1\), for $\ell\geq 3$, we conclude that each entry of \( X^\ell \) is also within \([0,1]\). We can calculate \( X^1 \) as:
$$
X^1 = \frac{1}{2M} \cdot \mathtt{ReLU}\left(x_1+y, -(x_1+y), x_1-y, -(x_1-y), \ldots, -(x_r-y)\right).
$$
Since \(x_1, \ldots, x_r, y \in [-M, M]\), each coordinate of \( X^1 \) is in \([0,1]\).
Direct computation shows that
\begin{align*}
    X^2 &=  \mathtt{ReLU}\big(\frac{1}{2M} \cdot(|x_1+y|,|x_1+y|,|x_1+y|-|x_1-y|,|x_1-y|,|x_1-y|,\cdots, |x_r-y|)\\
&\qquad \qquad  -(0, 1/2,0,0,1/2,\cdots,1/2)\big),
\end{align*}
thus each coordinate of \( X^2 \) is in \([0,1]\). By induction, this implies that for each intermediate output \( X^\ell \) (with \(\ell \in [m+1]\)), every element remains within \([0,1]\).   In conclusion, we have that
\begin{align*}
    \|X^0\|_{2,\infty} &= \|(x_1,\cdots,x_r,y)\|_2 \leq M\sqrt{r+1}, \\ 
    \|X^\ell\|_{2, \infty}&\leq \sqrt{5r}, \text{ for any }\ell\in [m+1], \nonumber \\
    \|X^{m+2}\|_{2, \infty}&\leq H^2.
\end{align*}
The first line follows from the direct calculation, and the second line holds because the maximum hidden embedding size is \(5r\), thus a row in $X^\ell$ has length at most $5r$. These upper bounds on the \(\ell_2\)-norm of \(X^\ell\) will be used when implementing this fully connected network under a transformer. The total number of layers of this fully connected network is $C_p \cdot (\log M + \log (1/\epsilon) ) $, where $C_p $ is an absolute constant. 
Now we conclude the proof.
\end{proof}

\subsubsection{Proof of Lemma~\ref{lem:residual_relu}}\label{app:residual_relu}
\begin{proof}
In this proof, we first construct a pair of weight matrices $(W'_{\mathrm{ff},1}, W'_{\mathrm{ff},2})$ 
such that the output matrix keeps columns in $\cJ$ and set the other columns to a zero vector.
Thus, the output matrix is in $\RR^{m \times r}$. 
Then we modify $(W'_{\mathrm{ff},1}, W'_{\mathrm{ff},2})$ to form another pair of weight matrices $(W_{\mathrm{ff},1}, W_{\mathrm{ff},2})$ such that the FF layer truncates the zero columns generates the desired output. 

Since \(a = \mathtt{ReLU}(a) - \mathtt{ReLU}(-a)\)
for any $a\in \RR$, by defining \(W_1 = W_2^\top = (1, -1) \in \RR^{1\times 2}\), we have \(\mathtt{ReLU}(aW_1)W_2 = a\). Setting \(W'_1 = W'_2 = 0 \in \RR\), we send \(a\) to zero by \(\mathtt{ReLU}(aW'_1)W'_2 = 0\). 
For each \(i \in [r]\), define \(W_{1,i} = \mathbbm{1}\{i \in \calJ\} \cdot W_1 + \mathbbm{1}\{i \notin \calJ\} \cdot W'_1\) and \(W_{2,i} = \mathbbm{1}\{i \in \calJ\} \cdot W_2 + \mathbbm{1}\{i \notin \calJ\} \cdot W'_2\).
For any $i \in [r]$, 
using \(\{ W_{1,i}, W_{2,i}\}\) as the weight matrices of a FF layer to process
the $i$-th column $X_{:,i}$, the output is $\mathbbm{1}\{i \in \calJ\}\cdot X_{:,i}$.

Now we put these matrices in the diagonal blocks of $W'_{\mathrm{ff},1}$ and $W'_{\mathrm{ff},2}$ to form $$
W'_{\mathrm{ff},1} = \mathtt{diag}(W_{1,i},\cdots,W_{1,r}) \in \mathbb{R}^{r\times (r+|\calJ|)}, \qquad W'_{\mathrm{ff},2} = \mathtt{diag}(W_{2,i},\cdots,W_{2,r})\in \mathbb{R}^{(r+|\calJ|)\times r}.
$$ 
By direct calculation, we have 
    $$
\mathtt{ReLU}(XW'_{\mathrm{ff},1})W'_{\mathrm{ff},2} = \big(X_{:,1} \cdot \mathbbm{1}\{1 \in \calJ\}+\mathbf{0}_{m\times 1}, \ldots, X_{:, r} \cdot \mathbbm{1}\{r \in \calJ\}+\mathbf{0}_{m\times 1} \big) \in \mathbb{R}^{m\times r}.
$$
This output has the same shape as the input $X$.  
To get the final result, 

This output keeps the dimension as $m\times r$. To get the final result, we define $W_{\mathrm{ff},1} \in \mathbb{R}^{r\times 2|\calJ|}$ by removing all all-zero columns from $W'_{\mathrm{ff},1}$, and $W_{\mathrm{ff},2}\in \mathbb{R}^{ 2|\calJ|\times |\calJ|}$ by removing all all-zero rows and columns from $W'_{\mathrm{ff},2}$.
These are the submatrices of $W_{\mathrm{ff},1}'$ and $W_{\mathrm{ff},2}'$  used to process columns $X_{:,i}$'s with $i \in \cJ$. 
As a result, 
$W_{\mathrm{ff},1}$ and $W_{\mathrm{ff},2} $ have only $2|\cJ|$ nonzero entries, taking values in $\{ -1, 0, 1\}.$
Moreover, 
we have 
$$
\mathtt{ReLU}(XW_{\mathrm{ff},1})W_{\mathrm{ff},2} = X_{:, i\in \calJ} \in \mathbb{R}^{m\times |\calJ|},
$$
and the norms of the these weight matrices are  \(\|W_{\mathrm{ff},1}\|_\mathrm{F} = \sqrt{2|\calJ|}\) and \(\|W_{\mathrm{ff},2}\|_\mathrm{F} = \sqrt{2|\calJ|}\).
    \end{proof}

\newpage 

\section{Technical Lemmas}

Finally, in this appendix, we lay out the helper lemmas used in the proofs in previous appendices. These lemmas are directly obtained from existing works and we provide the references to their proofs.

\begin{lemma}[Proposition 2 in \citet{caponnetto2007optimal}]
	\label{lem:cme-concen}
	Let $(\Omega, \nu)$ be a probability space and $\xi$ be a random variable on $\Omega$ taking value in a real separable Hilbert space $\cH$. We assume that there exists constants $B, \sigma > 0$ such that
	\begin{align*}
		\bigl\|\xi(w)\bigr\|_\cH \le B/2,\ \mathrm{a.s.}, \quad \EE\bigl[\norm{\xi}_\cH^2\bigr] \le \sigma^2.
	\end{align*}
	Then, it holds with probability at least $ 1- \delta$ that
	\begin{align*}
		\biggl\| L^{-1} \sum_{i = 1}^L \xi(\omega_i) - \EE[\xi] \biggr\| \le 2\biggl( \frac{B}{L} + \frac{\sigma}{\sqrt{L}} \biggr) \log \frac{2}{\delta}.
	\end{align*}
\end{lemma}

\vspace{2mm}
\begin{lemma}[Theorem A.4  in \cite{foster2021statistical}]\label{lemma: hellingerbd}
    For any sequence of real random variables $\{X_i\}_{1\leq i \leq n}$ that adapts to a filtration $\{\mathscr{F}_{i}\}_{1\leq i \leq n}$, then for any $m\leq n$, with probability at least $1-\delta$,
    \begin{equation*}
        \sum_{i=1}^n X_i \leq \sum_{i=1}^n \log \E_{i-1} \big( e^{X_i} \big) + \log(\delta^{-1}).
    \end{equation*}
    
\end{lemma}

\vspace{2mm}

\begin{lemma}[Lemma I.10 in \cite{zhang2023and}]
    \label{lem: tv-kl}
Let $b = \sup_x \log (p(x) / q(x))$.
We have that
\begin{align*}
    \KL(p\,\|\, q)  \le 2(3 + b) \cdot \TV(p, q).
\end{align*}
\end{lemma}
\vspace{2mm}
\begin{lemma}[Proposition E.1 in \cite{zhang2023and}]\label{lem: exp kernel}
Let $\fk(a,b) = \exp(\gamma\cdot a^T b)$ denote exponential kernel with constant $\gamma>0$, where $a, b \in \bbR^{d}$ . We use $\mathbf{S}^{d-1}$ to denote a $d-1$-dimensional unit sphere. Then we have that
\begin{align*}
    \int_{\mathbf{S}^{d-1}} a\fk(a,b) \text{d}a = C_1(\gamma) b,
\end{align*}
for some constant $C_1(\gamma) = \int_{\mathbf{S}^{d-1}}(a^\top b) \exp(\gamma\cdot a^\top b) \ud a>0$ and all $b\in\mathbf{S}^{d-1}$. The constant $C_1$ does not depend on $b$ due to symmetry on the unit sphere.
    
\end{lemma}


\vspace{2mm}
\begin{proposition}[Proposition F.2 in \cite{zhang2023and}] \label{prop:pacbayes}
        Let $\calF$ be the collection of functions of $f:\bbR^{n}\rightarrow\bbR$, and we assume that $|f|\leq b $ for any function $f\in\calF$. Let $X_{1},\cdots,X_{N}$ be $N$ i.i.d. random variables.
        Let   $Q$  be a probability distributionover  $\calF$. With probability at least $1-\delta$, we have
        \begin{align}
            \Bigl|\bbE_{f \sim 
P}\Bigl[\bbE_{X_1} \big[f(X_1 )\big]-f(X)\Bigr]\Bigr|\leq \sqrt{\frac{b^{2} } {2\log 2\cdot N}}\biggl[\KL(P\,\|\,Q)+\log\frac{4}{\delta}\biggr],\nonumber
        \end{align}
        simultaneously for any distribution $P$ on $\calF$. 
    \end{proposition}

\vspace{2mm}
\begin{lemma}[Corollary A.7 in \cite{edelman2022inductive}] \label{lem:softmax_lip}
For any two vectors $x,y\in \mathbb{R}^r$, 
\begin{align*}
    \|\mathtt{softmax}(x)-\mathtt{softmax}(y)\|_1\leq 2\|x-y\|_\infty.
\end{align*}
    
\end{lemma}
\vspace{2mm}
\begin{lemma}[Lemma A.6 in \cite{elbrachter2021deep}] \label{lem:network_approx}
    
For \(a, b \in \mathbb{R}\) with \(a < b\), define
$$
\mathcal{S}_{[a, b]}=\left\{f \in \mathcal{S}^{\infty}([a, b], \mathbb{R}) \mid\left\|f^{(n)}(x)\right\| \leq n!\text { for all } n \in \mathbb{N}\right\}.
$$

There exists a constant \(C > 0\) such that for all \(a, b \in \mathbb{R}\) with \(a < b\), \(f \in \mathcal{S}_{[a, b]}\), and \(\varepsilon \in(0,1 / 2)\), there is a fully connected network \(\Psi_f\) such that
$$
\left\|f - \Psi_f\right\|_{\infty} \leq \varepsilon,
$$
where the depth of the network is upper bounded by 
$$
 C \cdot \max \{2, b - a\} (\log \varepsilon^{-1})^2 + \log (\lceil \max \{|a|, |b|\} \rceil) + \log (\lceil 1 / (b - a) \rceil),
$$
the width of the network is upper bounded by $16$, and the maximum weight of the weight matrices is bounded by one. In particular, applying Propsotion~\ref{prop:ffn_transformer}, we compute the magnitude of each intermediate output as follows. For any input $x\in [a,b]$, let $X^{\ell}$ denote the output of the $\ell$-th layer of the neural network, we have $X^\ell \leq 4\big(b+1)\cdot16^\ell$.
\end{lemma}

\lstset{breaklines=true}

\newpage
\section{Supplementary Information about Prompts} \label{app: prompts}
This section provides supplementary examples for the area code experiment in Section~\ref{subsec: bma interpretation of cot} and the CityEquation experiment in Section~\ref{subsubsec: comparison}.

\subsection{Details of the Area Code Experiment in in Section~\ref{subsec: bma interpretation of cot}} \label{subsec: area code}
The following includes the experimental details of testing the ChatGPT(gpt-3.5-turbo-16k with the temperature set to zero) on the area code task using vanilla \ac{icl} and \ac{cot} methods, respectively. The output of ChatGPT is colored in {\color{red} red}. 

\begin{mdframed}[frametitle={ The area code experiment with ICL prompt}]

Input prompt:
\vspace{2mm}

\indent Q: The US=?

\indent A: The answer is 2.

\vspace{2mm}

\indent Q: France=?

\indent A: The answer is 66.

\vspace{2mm}

\indent Q: Japan=?

\indent A: 

\vspace{2mm}

\noindent LLM output:

\vspace{2mm}

\indent {\color{red} The answer is 100.}

\end{mdframed}
As we can see,  the answer provided by ChatGPT is wrong: it should be $162$ instead of $100$


\begin{mdframed}[frametitle={ The area code experiment with CoT prompt}]

Input prompt:
\vspace{2mm}

\indent Q: The US=?

\indent A: The US has area code 1, so the answer is 2.

\vspace{2mm}

\indent Q: France=?

\indent A: France has area code 33, so the answer is 66.

\vspace{2mm}

\indent Q: Japan=?

\indent A: 

\vspace{2mm}

\noindent LLM output:

\vspace{2mm}

\indent {\color{red} Japan has area code 81, so the answer is 162.}

\end{mdframed}
The output follows the same pattern shown in the demonstrations: it starts with stating the area code and computes the final answer.


\newpage

\subsection{Details of the CityEquation Experiment}

This section provides the experimental details for the city arithmetic experiment discussed in Section~\ref{subsubsec: comparison}.

\subsubsection*{Example prompts}
We provide example prompts (2-shot) for vanilla \ac{icl}, informative \ac{cot}, Partially informative CoT-(a), Partially informative CoT-(c), and uninformative \ac{cot} methods, respectively. To highlight the differences among these five methods, we color code the useful intermediate steps in \textcolor{blue}{blue} and the irrelevant intermediate steps in \textcolor{red}{red} for better comparison.

\begin{mdframed}[frametitle={Example prompts for Vanilla ICL}]

Q: ``Mumbai + Sydney"

\noindent A: The answer is 224.

\vspace{2mm}

\noindent Q: ``New York + Seoul"

\noindent A: The answer is 53.

\vspace{2mm}

\noindent Q: ``Paris + Beijing"

\noindent A: Concisely explain your steps and write your answer as an integer in the last sentence starting with ``The answer is''. 

\end{mdframed}

\begin{mdframed}[frametitle={Example prompts for Informative CoT}]

Q: ``Mumbai + Sydney"

\noindent A: \textcolor{blue}{Using the longitudes of cities, the equation ``Mumbai + Sydney" translates as ``Mumbai" = 73, ``Sydney" = 151. Here the longitudes of the western hemisphere are negative numbers. And we round the coordinates to the nearest integer. This gives the result.} The answer is 224.

\vspace{2mm}

\noindent Q: ``New York + Seoul"

\noindent A: \textcolor{blue}{Using the longitudes of cities, the equation ``New York + Seoul" translates as ``New York" = -74, ``Seoul" = 127. Here the longitudes of the western hemisphere are negative numbers. And we round the coordinates to the nearest integer. This gives the result.} The answer is 53.

\vspace{2mm}

\noindent Q: ``Paris + Beijing"

\noindent A: Concisely explain your steps and write your answer as an integer in the last sentence starting with ``The answer is'/. 

\end{mdframed}

\newpage

\begin{mdframed}[frametitle={Example prompts for Partially informative CoT-(a)}]

Q: ``Mumbai + Sydney"

\noindent A: \textcolor{blue}{Mumbai has longitude: 73.} The answer is 224.

\vspace{2mm}

\noindent Q: ``New York + Seoul"

\noindent A: \textcolor{blue}{New York has longitude: -74.} The answer is 53.

\vspace{2mm}

\noindent Q: Paris + Beijing

\noindent A: Concisely explain your steps and write your answer as an integer in the last sentence starting with ``The answer is''. 

\end{mdframed}

\begin{mdframed}[frametitle={Example prompts for Partially informative CoT-(b)}]

Q: ``Mumbai + Sydney"

\noindent A: \textcolor{red}{The greater Mumbai area has around 20 million residents, predominantly of South Asian ethnicity. The median age is 31 years. Sydney has a population of around 5.3 million people, with a breakdown of 58$\%$ White, 34.2$\%$  Asian, and 2.6$\%$  Aboriginal/Torres Strait Islander. The remaining percentages include others. The median age is 36 years.} The answer is 224.

\vspace{2mm}

\noindent Q: ``New York + Seoul"

\noindent A: \textcolor{red}{With a population of approximately 8.4 million people, New York is diverse. Around 32.1$\%$  identify as White, 29.1$\%$  as Hispanic, 24.3$\%$  as Black, 14.1$\%$  as Asian, with the rest being a mixture of Native American, Pacific Islander, and other ethnicities. The median age is 36 years. Seoul has a population of around 9.7 million, predominantly Korean, with a median age of 41 years.} The answer is 53.

\vspace{2mm}
\noindent Q: ``Paris + Beijing''

\noindent A: Concisely explain your steps and write your answer as an integer in the last sentence starting with 'The answer is'. 

\end{mdframed}
\newpage
\begin{mdframed}[frametitle={Example prompts for Partially informative CoT-(c)}]

Q: ``Mumbai + Sydney"

\noindent A: \textcolor{red}{The greater Mumbai area has around 20 million residents, predominantly of South Asian ethnicity. The median age is 31 years. Sydney has a population of around 5.3 million people, with a breakdown of 58$\%$ White, 34.2$\%$  Asian, and 2.6$\%$  Aboriginal/Torres Strait Islander. The remaining percentages include others. The median age is 36 years.} \textcolor{blue}{Mumbai has longitude: 73.} The answer is 224.

\vspace{2mm}

\noindent Q: ``New York + Seoul"

\noindent A: \textcolor{red}{With a population of approximately 8.4 million people, New York is diverse. Around 32.1$\%$  identify as White, 29.1$\%$  as Hispanic, 24.3$\%$  as Black, 14.1$\%$  as Asian, with the rest being a mixture of Native American, Pacific Islander, and other ethnicities. The median age is 36 years. Seoul has a population of around 9.7 million, predominantly Korean, with a median age of 41 years.}\textcolor{blue}{New York has longitude: -74.}  The answer is 53.

\vspace{2mm}
\noindent Q: ``Paris + Beijing''

\noindent A: Concisely explain your steps and write your answer as an integer in the last sentence starting with 'The answer is'. 

\end{mdframed}

\begin{mdframed}[frametitle={Example prompts for Partially informative CoT-(d)}]

Q: ``Mumbai + Sydney"

\noindent A: \textcolor{red}{The greater Mumbai area has around 20 million residents, predominantly of South Asian ethnicity. The median age is 31 years. Sydney has a population of around 5.3 million people, with a breakdown of 58$\%$ White, 34.2$\%$  Asian, and 2.6$\%$  Aboriginal/Torres Strait Islander. The remaining percentages include others. The median age is 36 years.} \textcolor{blue}{Mumbai has longitude: 73. Sydney has longitude: 151.} The answer is 224.

\vspace{2mm}

\noindent Q: ``New York + Seoul"

\noindent A: \textcolor{red}{With a population of approximately 8.4 million people, New York is diverse. Around 32.1$\%$  identify as White, 29.1$\%$  as Hispanic, 24.3$\%$  as Black, 14.1$\%$  as Asian, with the rest being a mixture of Native American, Pacific Islander, and other ethnicities. The median age is 36 years. Seoul has a population of around 9.7 million, predominantly Korean, with a median age of 41 years.}  \textcolor{blue}{New York has longitude: -74. Seoul has longitude: 127.} The answer is 53.

\vspace{2mm}
\noindent Q: ``Paris + Beijing''

\noindent A: Concisely explain your steps and write your answer as an integer in the last sentence starting with ``The answer is". 

\end{mdframed}

\newpage

\subsubsection*{Failure cases}
We present examples of failure cases for vanilla \ac{icl}, informative \ac{cot}, Partially informative CoT-(a), Partially informative CoT-(c), and uninformative \ac{cot} methods. The ground truth is written in \textcolor{blue}{blue}, while the incorrect answer produced by GPT-4 is in \textcolor{red}{red}.

\begin{mdframed}[frametitle={Failed example for Vanilla ICL}]

Testing query: Q: ``Sao Paulo - Cairo",

\noindent Ground truth: \textcolor{blue}{-78},

\vspace{2mm}
 
\noindent LLM output: ``The questions seem to be asking for the difference in time zones between two cities. Sao Paulo, Brazil is in the Brasilia Time Zone, which is GMT -3. Cairo, Egypt is in the Eastern European Time Zone, which is GMT +2. To find the difference, subtract the time zone of Sao Paulo from that of Cairo: -3 - (+2) = -5. The answer is \textcolor{red}{-5}."

\end{mdframed}

\begin{mdframed}[frametitle={Failed example for Informative CoT}]

Testing query: Q: ``Moscow + New York",

\noindent Ground truth:  \textcolor{blue}{-36},

\vspace{2mm}
 
\noindent LLM output: ``Using the longitudes of cities, the equation ``Moscow + New York" translates as ``Moscow" = 37, ``New York" = -74. Here the longitudes of the western hemisphere are negative numbers. And we round the coordinates to the nearest integer. This gives the result. The answer is  \textcolor{red}{-37}."

\end{mdframed}

\begin{mdframed}[frametitle={Failed example for Partially informative CoT-(a)}]

Testing query: Q: ``Sao Paulo - New York",

\noindent Ground truth: \textcolor{blue}{27},

\vspace{2mm}
 
\noindent LLM output: ``Let's consider the geographic coordinates that specify their east-west or north-south positions on the Earth's surface. Sao Paulo is located at approximately 46 degrees west, and New York is located at approximately 74 degrees west. If we subtract the longitude of Sao Paulo from that of New York, we get -28. The answer is \textcolor{red}{-28}."

\end{mdframed}

\newpage

\begin{mdframed}[frametitle={Failed example for Partially informative CoT-(b)}]

Testing query: Q: ``Sydney - New York",

\noindent Ground truth: \textcolor{blue}{225},

\vspace{2mm}
 
\noindent LLM output: ``Sydney has a population of around 5.3 million people, with a breakdown of 58$\%$ White, 34.2$\%$ Asian, and 2.6$\%$ Aboriginal/Torres Strait Islander. The remaining percentages include others. The median age is 36 years. New York City has a population of approximately 8.4 million people, with a diverse population that includes White (42.7$\%$), Hispanic (29.1$\%$), Black (24.3$\%$), and Asian (14$\%$) individuals. The median age is 36 years. The answer is \textcolor{red}{-2}."

\end{mdframed}

\begin{mdframed}[frametitle={Failed example for Partially informative CoT-(c)}]

Testing query: Q: ``Mexico City + Beijing",

\noindent Ground truth:  \textcolor{blue}{17},

\vspace{2mm}
 
\noindent LLM output: ``Mexico City has a population of approximately 21.7 million in its metropolitan area. The majority of residents are of Mestizo (mixed Indigenous and European) descent, with a significant number of Indigenous peoples, and a smaller number of European and other ethnicities. The median age is 27 years. Beijing, with a population of around 21.5 million, is predominantly Han Chinese, with a median age of 37.6 years. Mexico City has longitude: -99. The answer is \textcolor{red}{-77}."

\end{mdframed}

\begin{mdframed}[frametitle={Failed example for Partially informative CoT-(d)}]

Testing query: Q: ``Moscow - Tokyo",

\noindent Ground truth: \textcolor{blue}{-102},

\vspace{2mm}
 
\noindent LLM output: ``Moscow, the capital of Russia, has a population of approximately 12.5 million people, predominantly of Russian ethnicity. The median age is 40 years. Tokyo, the capital of Japan, has a population of around 14 million in its 23 special wards, predominantly of Japanese ethnicity. The median age in Tokyo is 45.5 years. Moscow has longitude: 37. Tokyo has longitude: 140. Subtracting Moscow's longitude from Tokyo's gives us 140 - 37. The answer is \textcolor{red}{103}."

\end{mdframed}

Upon closer inspection of GPT-4 outputs, incorrect results from informative \ac{cot} are solely due to rounding. 
That is, when rounding the coordinates to the nearest integer, GPT-4 makes an error. 
In particular, in the failed example shown above, the longitude of Moscow is $37.6$, which should be rounded to $38$ instead of $37$. 
However, such a rounding error is the only error source. Therefore, with informative CoT, GPT-4 in fact understands that extracting the longitudes is the key to solving the CityEquation task. 

Vanilla \ac{icl} prompts produce incorrect reasoning steps like using time zones, indicating a propensity to misinterpret prompts without clear guidance. Thus, it is challenging for vanilla ICL  to realize that longitudes are the key to solving the CityEquations talk.

Furthermore, the errors incurred by Partially Informative CoT-(a) typically involve rounding and sign issues, particularly in reasoning steps related to longitudes. For example, in the failed case mentioned above, the longitude of New York is given as $-74$ instead of $74$. This suggests that while the prompts enable GPT-4 to associate the problem with the cities' longitudes, it sometimes struggles to handle the signs correctly.
Additionally, the failure of Partially Informative CoT-(b) often results from the use of irrelevant information about the cities, such as demographic data, in the computation. In the failure example of Partially Informative CoT-(c) mentioned earlier, GPT-4 lists demographic data for both cities and the longitude of Mexico City, using both to compute the answer. Partially Informative CoT-(d) includes demographic data and longitudes for both cities in the intermediate reasoning steps. However, the final answer is based solely on the computation involving the longitudes. In the failure case above, the arithmetic formula incorrectly switches the minuend and subtrahend and rounds the longitude of Moscow to 37 inaccurately.
This suggests that while GPT-4 correctly identifies that the longitudes are the only useful information for solving the task, it struggles with using this information correctly. By comparing Partially Informative CoT-(b) through CoT-(d), we observe that adding more relevant information to the CoT prompts improves GPT-4's performance on the CityEquation task.

\end{document}